\documentclass{article}
\PassOptionsToPackage{square,numbers,sort&compress}{natbib}
 \usepackage[preprint]{neurips_2026}

\usepackage[utf8]{inputenc} 
\usepackage[T1]{fontenc}    
\usepackage{hyperref}       
\usepackage{url}            
\usepackage{booktabs}       
\usepackage{amsfonts}       
\usepackage{nicefrac}       
\usepackage{microtype}      
\usepackage{xcolor}         

\usepackage{algorithm, algorithmic}
\usepackage{amsmath}
\usepackage{amssymb}
\usepackage{mathtools}
\usepackage{amsthm}
\usepackage{makecell}
\usepackage{placeins}

\theoremstyle{plain}
\newtheorem{theorem}{Theorem}[section]
\newtheorem{proposition}[theorem]{Proposition}

\theoremstyle{definition}

\newtheorem{example}{Example}
\theoremstyle{remark}

\usepackage{bbm}

\title{Audited Conformal Prediction for Classification under Unknown Distribution Shift}

\author{%
  Yanfei Zhou \\
  Department of Data Sciences and Operations\\
  University of Southern California\\
  Los Angeles, CA 90089 \\
  \texttt{yanfeizh@usc.edu} \\
  \And
  Rizal Fathony\\
  AI Foundations, Capital One \\
  New York, NY 10017 \\
  \texttt{rizal.fathony@capitalone.com}\\
  \And
  Nam H. Nguyen\\
  AI Foundations, Capital One \\
  New York, NY 10017 \\
  \texttt{nam.nguyen@capitalone.com}\\
  \And
  Matteo Sesia \\
  Department of Data Sciences and Operations\\
  Thomas Lord Department of Computer Science\\
  University of Southern California\\
  Los Angeles, CA 90089 \\
  \texttt{sesia@marshall.usc.edu} \\
}

\begin{document}

\maketitle

\begin{abstract}
We consider the problem of uncertainty quantification for a pretrained classification model deployed under unknown distribution shift. We propose Audited Conformal Prediction (ACP), a method that leverages a small labeled dataset from the target population to train an auxiliary \emph{audit} model identifying inputs where the legacy model is likely to fail. By integrating the audit model's outputs into the conformal prediction framework, ACP produces prediction sets that guarantee marginal coverage while achieving substantially higher conditional coverage in practice than existing approaches. We develop and analyze two complementary integration strategies---one targeting marginal coverage with improved conditional performance, the other providing explicit group-conditional coverage guarantees---and establish theoretical guarantees for both. Experiments on synthetic and real-world datasets validate the method and illustrate trade-offs between prediction set size and conditional coverage.
\end{abstract}

\section{Introduction}\label{sec:introduction}
\subsection{Background and Motivation}

In real-world machine learning applications, considerable resources are often dedicated 
to training complex models on large historical datasets. However, the performance of 
these ``legacy'' models can degrade over time when deployed in dynamic environments 
where data distributions shift unpredictably \citep{Nikolaidis2022, Finlayson2021, 
park2021reliable}. A clear example is fraud detection, where fraudulent tactics 
continuously evolve to circumvent detection, causing the underlying data distribution 
to change over time in a nearly adversarial way.
Such distribution shifts challenge not only predictive accuracy 
\citep{candela2009datashift, KARIMIAN2023119946}, but also the reliability of 
uncertainty estimates \citep{ovadia2019trust, Silva2025conformalshift}. While 
regularly retraining on fresh data can mitigate this, labeling new data tends to be 
slow and resource-intensive. In practice, newly labeled data is often scarce relative 
to the large historical datasets used to train the legacy model. Efficiently adapting 
models to distribution shift with limited new labeled samples is therefore an important 
challenge.

While significant prior work has focused on preserving predictive accuracy after 
distribution shifts \citep{TUDORAN2024114527, JACKSON2023103360, masashi2007covariateshift, 
ganin2016dann, long2018conditional}, this paper focuses on maintaining reliable 
uncertainty estimates for classification models via \emph{conformal prediction}.
Conformal prediction \citep{vovk2005algorithmic, shafer2008tutorial, lei2014distribution} 
is a statistical framework for quantifying predictive uncertainty through prediction 
sets with finite-sample coverage guarantees. Set size reflects input-dependent 
uncertainty: smaller sets indicate higher confidence, larger sets greater uncertainty. 
The variant most relevant here is \emph{split} conformal prediction 
\citep{papadopoulos2002inductive}, which calibrates prediction sets using a fixed, 
pretrained model evaluated on an independent dataset from the target population.
Although standard conformal methods guarantee exact marginal coverage regardless of 
model quality, the per-sample quality of prediction sets can degrade when the training 
and calibration distributions differ. Without mitigation, some test instances may 
receive overly conservative prediction sets while others are systematically 
undercovered---reducing the practical usefulness of these uncertainty estimates.

\subsection{Preview of Contributions}

While retraining a legacy model on target-population data may seem intuitive, 
it can be suboptimal when the model is complex and new data are limited. 
We propose an alternative that can leverage limited new labeled data more effectively: 
rather than retraining the model, we introduce an auxiliary \emph{audit} 
model---a binary classifier designed to predict the reliability of the legacy model on 
new inputs. This task is simpler than full retraining and, as we 
demonstrate, sufficient for our purposes.

The audit model offers several key advantages. First, it can leverage information 
already encoded by the legacy model (e.g., logits, confidence patterns, or intermediate 
features). Second, even an imperfect or overconfident audit model can provide useful 
directional signals about legacy model reliability. Crucially, the audit model directly 
targets the core issue---predicting when the legacy model is likely to fail---without 
requiring any assumption about the distribution shift. Unlike methods 
focused solely on detecting distribution shift (which may not always affect model 
performance equally), our focus is on identifying failures that may stem from, but are 
not limited to, such shifts. We illustrate these points with examples and 
empirical evidence in Section~\ref{sec:experiments} and 
Appendix~\ref{app:add_motivating_examples}.

We refer to our approach as \emph{Audited Conformal Prediction} (ACP); a schematic 
workflow is provided in Figure~\ref{fig:AuditedCP_pipeline}. While prior work has 
explored training effective audit models, our key contribution is integrating audit 
model predictions within the conformal prediction framework to balance reliability and 
efficiency---improving conditional coverage while maintaining small, informative 
prediction sets. We validate ACP through extensive numerical experiments in 
Section~\ref{sec:experiments}.

\begin{figure}[!htb]
    \centering
    \includegraphics[width=0.9\linewidth]{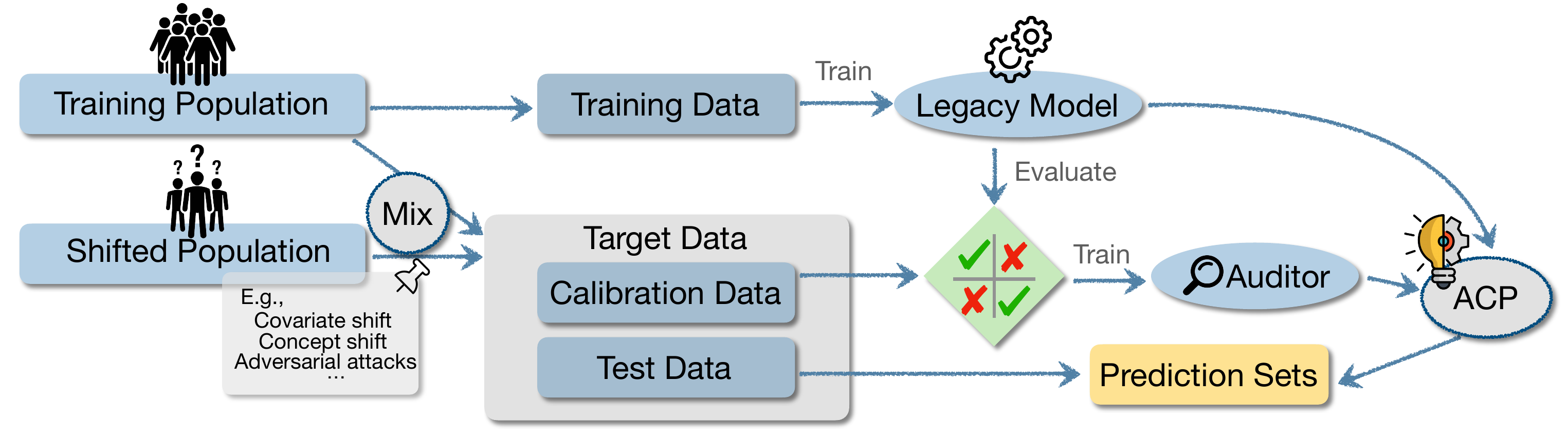}
\caption{Schematic overview of the ACP method for uncertainty-aware classification in the presence of distributional shifts. The method integrates the predictions of a pre-trained legacy model with those of an audit model designed to predict the reliability of the former on new instances.}
    \label{fig:AuditedCP_pipeline}
\end{figure}

\subsection{Related Work}

Conformal prediction is a broad area concerned with uncertainty quantification for 
black-box models, with applications in classification \citep{lei2013distribution, 
sadinle2019least, romano2020classification}, regression \citep{lei2014distribution, 
romano2019conformalized}, and outlier detection \citep{bates2021testing}; see 
\cite{sesia2026elements} for a review. Standard formulations require labeled 
calibration data exchangeable with test data, though recent work relaxes this 
assumption \citep{tibshirani2019conformal, barber2023beyond, gibbs2021adaptive, 
xu2021conformal, yang2022doublyrobust}. While conformal methods guarantee exact 
marginal coverage regardless of model quality, prediction sets can become uninformative 
or unreliable when the underlying model is inaccurate or overconfident 
\citep{kaur2025conformalpredictionsetsimproved, NEURIPS2024_c1c49aba}.

This limitation has motivated work along two main directions: integrating conformal 
ideas into model training to mitigate overconfidence \citep{colombo2020training, 
bellotti2021optimized, stutz2022learning, Einbinder2022, liu2025cadapter, 
bian2023training, xie2024boosted}, and providing conditional guarantees beyond 
marginal coverage for a pretrained model \citep{vovk2003mondrian, 
romano2019malice, Barber2019TheLO, jung2023batch, zhou2024conformal, 
gibbs2025conditional, jung2025speedcpfastkernelbasedconditional}, including 
group-conditional coverage informed by model uncertainty 
\citep{kaur2025conformalpredictionsetsimproved}. In contrast, our method trains an 
audit model on new data from the target distribution, providing a direct 
reliability measure without requiring access to the original training data.
This approach also distinguishes our work from conformal methods that aggregate predictions from 
multiple models or ensembles \citep{carlsson2014aggregated, linusson2020efficient, 
yang2021finite, liang2023ces}: our two models play complementary roles, with one 
predicting labels of interest and the other explicitly designed to anticipate failures of the 
former.

Beyond conformal prediction, our work connects to several broader lines of research. 
One direction addresses performance degradation via transfer learning and fine-tuning 
\citep{Jason2014transfer, kumar2022finetuning}, or by augmenting models with 
predictive abstention \citep{geifman2017selective, geifman19a}. Related approaches train auxiliary predictors 
to identify likely misclassifications from intermediate representations 
\citep{corbiere2019addressing, granese2021doctor}, or use nonparametric confidence 
measures such as trust scores \citep{jiang2018trust}. Another line studies learning 
under structured distribution shifts \citep{masashi2007covariateshift, 
ben-david2006analysis, ganin2016dann, long2018conditional, redko2019optimal}, with 
recent analyses highlighting potential misalignment between standard calibration 
objectives and failure prediction under shift \citep{zhu2023rethinking}. Finally, 
post-hoc calibration methods such as Platt scaling and temperature scaling 
\citep{platt1999probabilistic, guo2017calibration} serve as natural baselines in our 
experiments.

\section{Technical Preliminaries}\label{sec:background}

\subsection{Marginal vs. Conditional Coverage in Conformal Classification}

We consider a dataset $\mathcal{D} = \{(X_i, Y_i)\}_{i=1}^{n+1}$ containing $n+1$ exchangeable (e.g., i.i.d.) samples from an arbitrary \emph{target} distribution $P_T$. 
For each individual $i \in [n+1] := \{1,\ldots,n+1\}$, $X_i \in \mathcal{X}$ is a (possibly high-dimensional) feature vector and $Y_i \in [K] := \{1, \dots, K\}$ a categorical label. The first $n$ samples are fully observed, comprising the calibration dataset $\mathcal{D}_{\text{cal}} = \{(X_i, Y_i)\}_{i=1}^{n}$, while the test label $Y_{n+1}$ is unobserved. Our goal is to predict $Y_{n+1}$ as a function of $X_{n+1}$.

Let $f_0$ denote a classification model (e.g., a logistic regression, random forest, 
or deep neural network) pretrained on a dataset from a source distribution $P_S$ that 
may differ arbitrarily from $P_T$. We assume only that $f_0$ is fixed and outputs an 
estimated conditional distribution $\hat{p}(\cdot \mid X) := f_0(X)$ over labels.
Split conformal prediction quantifies the predictive uncertainty of $f_0$ for 
$Y_{n+1}$ by constructing a prediction set $\hat{C}(X_{n+1}) \subseteq [K]$ as 
follows. For each calibration point $(X_i, Y_i)$, one computes a \emph{nonconformity 
score} $E_i = E(X_i, Y_i, f_0(X_i))$, using a suitable function $E$, so that higher values suggest 
 $Y_i$ is more atypical given $X_i$. For a desired 
miscoverage level $\alpha \in [0,1]$, one computes the threshold
$\hat{q} = \lceil(1-\alpha)(n+1)\rceil$-th smallest value in $\{E_1, \ldots, E_n, +\infty\}$. The prediction set for a new test point $X_{n+1}$ then includes all labels 
whose nonconformity score falls below this threshold:
\begin{equation}\label{eq:predset}
    \hat{C}(X_{n+1}) = \big\{ y \in [K] : E(X_{n+1}, y, f_0(X_{n+1})) \leq \hat{q} \big\}.
\end{equation}
Although several nonconformity score functions $E$ are possible, our experiments adopt the adaptive 
score of \cite{romano2020classification}, reviewed in Appendix~\ref{app:aps_review}. 
By exchangeability of $\mathcal{D}$, this construction guarantees \emph{marginal 
coverage} at level $1-\alpha$:
\begin{equation}\label{eq:marginal_coverage}
    \mathbb{P}\big[ Y_{n+1} \in \hat{C}(X_{n+1}) \big] \geq 1-\alpha,
\end{equation}
where the probability above is taken over the randomness of all data in $\mathcal{D}$. 

Marginal coverage guarantees that the prediction set contains the true label with 
probability at least $1-\alpha$ on average over the target population, but imposes no 
constraints on how coverage is distributed across individual samples. As an 
illustration, consider a population where 90\% of samples were seen during training 
and the model has memorized them without generalizing. A nominal 90\% 
marginal coverage rate is then achieved trivially---by covering all seen samples while 
systematically failing on unseen ones---even though prediction sets are uninformative 
precisely where reliability matters most.

This limitation motivates stronger notions of coverage. An ideal target is 
\emph{feature-conditional coverage},
\begin{equation} \label{eq:feat-cond-coverage}
\mathbb{P}\big[ Y_{n+1} \in \hat{C}(X_{n+1}) \mid X_{n+1} = x \big] \geq 1-\alpha, 
\qquad \forall x \in \mathcal{X},
\end{equation}
which ensures consistent coverage for every possible test input. However, this is 
generally impossible to achieve exactly in finite samples without strong assumptions 
on the data-generating process \citep{barber2019limits} or severe restrictions on 
$\mathcal{X}$ \citep{lee2021distribution}, short of trivially large prediction sets.

A practical compromise is \emph{group-conditional coverage},
\begin{equation}\label{eq:group-cond-coverage}
    \mathbb{P}\big[ Y_{n+1} \in \hat{C}(X_{n+1}) \,\big|\, X_{n+1}\in G \big] \geq 
    1-\alpha, \qquad \forall G \in \mathcal{G},
\end{equation}
for a pre-specified family of groups $\mathcal{G} \subseteq 2^{\mathcal{X}}$. 
\citet{romano2019malice} partition $\mathcal{X}$ into disjoint groups based on 
protected categories, motivated by fairness considerations; this is extended by 
\cite{zhou2024conformal, jin2025selection} to data-driven groups and by 
\cite{gibbs2025conditional} to overlapping groups. Most relevant to our work, 
\cite{kaur2025conformalpredictionsetsimproved} target \eqref{eq:group-cond-coverage} 
for groups defined by uncertainty estimates from the legacy model itself, but assume 
access to the original training data and that the legacy model remains reliable under 
$P_T$. Our approach relaxes both assumptions: we train an audit model on recent 
labeled data from the target distribution to identify inputs where the legacy model is 
likely to fail, without requiring access to the original training data. Moreover, 
rather than enforcing group-conditional coverage for pre-defined groups, we maintain 
marginal coverage guarantees while targeting strong feature-conditional performance in 
practice.

\subsection{Audited Conditional Coverage}\label{sec:approx_cond_cov}

Our goal is to construct conformal prediction sets that achieve high conditional 
coverage across subpopulations with varying levels of legacy model accuracy under 
$P_T$: prediction sets that adaptively expand for test instances where $f_0$ is less 
reliable, while guaranteeing marginal coverage at level $1-\alpha$ and remaining 
maximally informative on average.

To formalize this, we introduce an \emph{oracle audit model} $r^*: \mathcal{X} 
\to [0,1]$, which computes the true conditional probability that the point prediction 
$f_0^{\text{pt}}(X)$ is correct for a new draw $(X,Y) \sim P_T$:
\begin{equation}\label{eq:oracle_audit}
    r^*(x) \;:=\; \mathbb{P}\left[ f_0^{\text{pt}}(X)=Y \mid X = x \right].
\end{equation}
Larger $r^*(x)$ indicates greater reliability of $f_0$ at $X=x$ under $P_T$.
Let $\mathcal{R}$ be a partition of $[0,1]$ into a finite number of (possibly 
overlapping) bins. We define \emph{audited conditional coverage} as:
\begin{equation}\label{eq:audited-cond-coverage}
    \mathbb{P}\big[ Y_{n+1} \in \hat{C}(X_{n+1}) \mid r^*(X_{n+1}) \in R \big] \geq 
    1-\alpha, \qquad \forall R \in \mathcal{R}.
\end{equation}
As shown in Proposition~\ref{prop:implication} (Appendix~\ref{app:proofs}), this is weaker than feature-conditional coverage~\eqref{eq:feat-cond-coverage}.

Conditioning on $r^*(X)$ provides a meaningful dimension reduction: it 
maps the potentially high-dimensional space $\mathcal{X}$ to a scalar that directly 
reflects the predictive reliability of $f_0$, simultaneously capturing intrinsic data 
noise, model limitations, and distribution shift. The granularity of $\mathcal{R}$ 
induces a natural trade-off: finer partitions yield stronger guarantees closer to 
feature-conditional coverage but reduce per-bin sample sizes, leading to more 
conservative prediction sets; coarser partitions tend to produce smaller, 
more efficient sets, but with guarantees closer to marginal coverage.

In practice, the oracle audit model is unavailable since $P_T$ is unknown; 
accordingly, we estimate it from observed data, as described in the next section.

\section{Methodology}\label{sec:methodology}

The proposed ACP method approximates the oracle audit model using a practical audit 
model trained on a subset of $\mathcal{D}_{\text{cal}}$, whose output is then 
integrated with that of the legacy model to construct conformal prediction sets that 
guarantee marginal coverage while approximately satisfying~\eqref{eq:audited-cond-coverage}.

\subsection{Training an Audit Model}\label{sec:train_correctness}

We randomly split $\mathcal{D}_{\text{cal}}$ into two disjoint subsets 
$\mathcal{D}_{\text{cal}}^1$ and $\mathcal{D}_{\text{cal}}^2$, and approximate 
$r^*(x)$ by training $\hat{r} : \mathcal{X} \to [0,1]$ on $\mathcal{D}_{\text{cal}}^1$, 
possibly augmented with features extracted from $f_0$. Specifically, for each 
$(X_i, Y_i) \in \mathcal{D}_{\text{cal}}^1$ we form the binary label
$R_i = \mathbbm{1}[ f_0^{\text{pt}}(X_i) = Y_i]$,
and fit $\hat{r}$ to predict $R_i$ via standard supervised learning (e.g., minimizing 
cross-entropy). Training an auxiliary model to predict base classifier correctness is 
well-established in confidence estimation and selective classification 
\citep{GeifmanElYaniv2017, CattelanSilva2024FixConfidence, ZhuEtAl2022FailurePred}.

The input features for $\hat{r}$ are flexible; we follow common practice by using 
signals correlated with prediction reliability, including: (i) the original covariates 
$X_i$; (ii) distribution proximity measures quantifying how far $X_i$ deviates from 
the training distribution, such as distance-based metrics or trust scores 
\citep{JiangEtAl2018TrustScore}, when a reference set is available; (iii) legacy 
model predictive signals such as maximum softmax probability, entropy, or top-two 
probability margin \citep{HendrycksGimpel2016}; and (iv) model-specific signals such 
as ensemble disagreement \citep{LakshminarayananEtAl2016DeepEnsembles}, OOD scores 
from logits and temperature scaling \citep{LiangEtAl2018ODIN}, or penultimate-layer 
representations.

The architecture of $\hat{r}$ may similarly be chosen based on available features and 
computational budget, ranging from logistic regression and tree-based models (random 
forests, gradient boosting) to neural networks \citep{GeifmanElYaniv2017}. 
Importantly, training $\hat{r}$ requires no modification to $f_0$ and no assumptions 
on the structure of the distribution shift. While a strong audit model improves 
practical effectiveness, our primary contribution is not the design of a new learning 
procedure for $\hat{r}$, but rather its integration into the conformal classification 
framework, described next.

\subsection{Integrating an Audit Model into Conformal Predictions}\label{sec:integration}

\subsubsection{Calibrating an Audited Probabilistic Classifier for Marginal Coverage}\label{sec:ACP-MC}

If the conditional label distribution under $P_T$ were known, existing conformal 
methods would already yield the smallest possible prediction sets with 
feature-conditional coverage \citep{romano2020classification}. This motivates 
leveraging $\hat{r}$ to refine the legacy model's estimate of that distribution, 
which is then calibrated via standard conformal prediction to guarantee marginal 
coverage.

We focus here on binary classification ($K=2$, $Y \in \{1,2\}$), deferring the 
multi-class extension to Appendix~\ref{app:APA_mutlclass}. For any $x \in \mathcal{X}$, define
\begin{equation}\label{eq:binary_comb_rule}
    m(x; \hat{r}, f_0) \;:=\;
    \begin{cases}
        \big(\hat{r}(x),\; 1-\hat{r}(x)\big), & \text{if } f_0^{\text{pt}}(x)=1,\\[4pt]
        \big(1-\hat{r}(x),\; \hat{r}(x)\big), & \text{if } f_0^{\text{pt}}(x)=2.
    \end{cases}
\end{equation}
This gives an estimate of $P_T(Y \mid X=x)$ that is consistent whenever $\hat{r}$ 
consistently estimates $r^*$.
\begin{proposition}\label{prop:consistency_ACP-MC_binary}
In the binary setting, let $m(x;\hat{r}_N,f_0)$ be as in~\eqref{eq:binary_comb_rule}, 
and suppose $\hat{r}_N$, trained on $N$ i.i.d.\ samples from $P_T$, satisfies 
$\hat{r}_N(x)\xrightarrow{p} r^*(x)$ for all $x \in \mathcal{X}$ as $N \to \infty$. Then, for any 
fixed $f_0$,
\[
    m(x;\hat{r}_N,f_0)\xrightarrow{p} P_T(Y\mid X=x),
\]
where $P_T(Y\mid X=x) = \left( \mathbb{P}[Y=1\mid X=x], \mathbb{P}[Y=2\mid X=x] \right)$ under $P_T$.
\end{proposition}
\noindent A proof is given in Appendix~\ref{app:proofs}.

Using $m(x;\hat{r},f_0)$ in place of $f_0$ as an estimate of $P_T(Y \mid X=x)$, prediction 
sets with guaranteed marginal coverage are constructed by applying the standard 
conformal approach (Appendix~\ref{app:aps_review}) to the remaining calibration data 
$\mathcal{D}_{\text{cal}}^2$. Algorithm~\ref{alg:ACP_ACP-MC_binary} summarizes the 
full procedure.

Relative to using $f_0$ directly, this approach lets $\hat{r}$ mitigate 
overconfidence or underconfidence and partially correct for distribution shift. 
Although audited conditional coverage~\eqref{eq:audited-cond-coverage} is not 
explicitly guaranteed, empirical results in Section~\ref{sec:experiments} show this method
substantially improves conditional coverage for the more challenging test points while 
maintaining informative prediction sets.

\begin{algorithm}[!htb]
    \caption{ACP for Binary Classification with Audited Probabilistic Classifiers}
    \label{alg:ACP_ACP-MC_binary}
    \begin{algorithmic}[1]
        \STATE \textbf{Input}: calibration data $\mathcal{D}_{\text{cal}}$; test 
        features $X_{n+1}$; pretrained legacy model $f_0$; score function $E$; 
        level $\alpha \in (0,1)$.
        \STATE Randomly split $\mathcal{D}_{\text{cal}}$ into $\mathcal{D}_{\text{cal}}^1$ 
        and $\mathcal{D}_{\text{cal}}^2$.
        \STATE Compute $R_i = \mathbbm{1}[f_0^{\text{pt}}(X_i)=Y_i]$ for 
        $i \in \mathcal{D}_{\text{cal}}^1$.
        \STATE Train $\hat{r}$ on $\{(\tilde{X}_i, R_i)\}_{i \in \mathcal{D}_{\text{cal}}^1}$, 
        where $\tilde{X}_i$ aggregates $X_i$ and features derived from $f_0$ 
        (Sec.~\ref{sec:train_correctness}).
        \FOR{$i \in \mathcal{D}_{\text{cal}}^2$}
            \STATE Compute $m(X_i;\hat{r},f_0)$ via~\eqref{eq:binary_comb_rule} and 
            nonconformity score $E_i = E(X_i, Y_i, m(X_i;\hat{r},f_0))$.
        \ENDFOR
        \STATE Set $\hat{q}$ to the $\lceil(1-\alpha)(1+|\mathcal{D}_{\text{cal}}^2|)\rceil$-th 
        smallest value of $\{E_1,\ldots,E_{|\mathcal{D}_{\text{cal}}^2|},\infty\}$.
        \STATE \textbf{Output}: $\hat{C}(X_{n+1}) = \{y : E(X_{n+1}, y, 
        m(X_{n+1};\hat{r},f_0)) \leq \hat{q}\}$.
    \end{algorithmic}
\end{algorithm}

\subsubsection{Conformal Prediction with Audited Conditional Coverage}\label{sec:ACP-ACC}

We now present an alternative approach that provides guarantees 
beyond marginal coverage. The key idea is to partition the feature space by binning 
the values of the audit model $\hat{r}$, then apply conformal 
calibration on $\mathcal{D}_{\text{cal}}^2$ to achieve audited conditional 
coverage~\eqref{eq:audited-cond-coverage} with $r^*$ replaced by $\hat{r}$.

Given a finite threshold set $\mathcal{T} = \{t_1 < \cdots < t_m\} \subset (0,1)$, 
define the binary group indicator
\begin{equation}\label{eq:group_indicator}
    \hat{G}_j(x) = \mathbbm{1}\{\hat{r}(x) \le t_j\},
\end{equation}
where $\hat{G}_j(x)=1$ and $\hat{G}_j(x)=0$ correspond to low- and high-confidence 
regions of $f_0$ at threshold $t_j$. Our goal is to construct $\hat{C}(X_{n+1})$ 
satisfying
\begin{equation}\label{eq:audit-eq-cover-2}
    \mathbb{P}\big[Y_{n+1}\in\hat{C}(X_{n+1})\mid\hat{G}_j(X_{n+1})=g\big]\ge 1-\alpha,
    \qquad \forall j\in[m],\; g\in\{0,1\}.
\end{equation}
For example, $\mathcal{T}=\{0.5,0.7,0.9\}$ yields group-conditional coverage over six 
partly overlapping groups \cite{gibbs2025conditional}. When $\mathcal{T}$ is a singleton, the problem reduces to 
equalized coverage over two disjoint groups \cite{vovk2003mondrian}, though prediction sets may be sensitive 
to the choice of threshold. In the limit of an infinitely fine grid, 
\eqref{eq:audit-eq-cover-2} becomes equivalent to coverage conditional on the exact 
value of $\hat{r}(X_{n+1})$---a one-dimensional reduction of full 
feature-conditional coverage~\eqref{eq:feat-cond-coverage}.

We construct prediction sets satisfying~\eqref{eq:audit-eq-cover-2} using the 
conditional conformal calibration strategy of \cite{gibbs2025conditional}, which 
fits an augmented quantile regression over the audit-induced group indicators. Define
\begin{equation}
    g_\beta(x) = \beta_0 + \sum_{j=1}^m \beta_j \hat{G}_j(x).
\end{equation}
For each candidate label $y$, let $E_i = E(X_i,Y_i,f_0(X_i))$ for $i\in 
\mathcal{D}_{\text{cal}}^2$ and $E_{n+1}(y)=E(X_{n+1},y,f_0(X_{n+1}))$. We solve
\begin{equation}\label{eq:condconf_opt}
    \hat\beta(y) = \arg\min_{\beta}\;\frac{1}{|\mathcal{D}_{\text{cal}}^2|+1}
    \left[\sum_{i\in\mathcal{D}_{\text{cal}}^2}
    \ell_\alpha\!\left(g_\beta(X_i),E_i\right)
    +\ell_\alpha\!\left(g_\beta(X_{n+1}),E_{n+1}(y)\right)\right],
\end{equation}
where $\ell_\alpha(g,e)=(1-\alpha)\max\{e-g,0\}+\alpha\max\{g-e,0\}$ is the pinball 
loss, and then output
\begin{equation}\label{eq:condconf_Pset}
    \hat{C}(X_{n+1}) = \big\{y : E_{n+1}(y)\le g_{\hat\beta(y)}(X_{n+1})\big\}.
\end{equation}
Algorithm~\ref{alg:ACP_condconf} summarizes the full procedure.

\begin{algorithm}[!htb]
    \caption{ACP with Conditional Conformal Calibration}
    \label{alg:ACP_condconf}
    \begin{algorithmic}[1]
        \STATE \textbf{Input}: calibration data $\mathcal{D}_{\text{cal}}$; test features $X_{n+1}$; 
        pretrained legacy model $f_0$; threshold set $\mathcal{T}$; score function $E$; 
        level $\alpha\in(0,1)$.
        \STATE Randomly split $\mathcal{D}_{\text{cal}}$ into $\mathcal{D}_{\text{cal}}^1$ and 
        $\mathcal{D}_{\text{cal}}^2$.
        \STATE Compute $R_i=\mathbbm{1}[f_0^{\text{pt}}(X_i)=Y_i]$ for 
        $i\in\mathcal{D}_{\text{cal}}^1$; train $\hat{r}$ on 
        $\{(\tilde{X}_i,R_i)\}$ (Section~\ref{sec:train_correctness}).
        \STATE Compute $E_i=E(X_i,Y_i,f_0(X_i))$ and 
        $\hat{G}_j(X_i)$~\eqref{eq:group_indicator} for $i\in\mathcal{D}_{\text{cal}}^2
        \cup\{n+1\}$, each $t_j\in\mathcal{T}$.
        \FOR{$y\in[K]$}
            \STATE Impute $E_{n+1}(y)=E(X_{n+1},y,f_0(X_{n+1}))$; solve 
            \eqref{eq:condconf_opt} for $\hat\beta(y)$.
        \ENDFOR
        \STATE \textbf{Output}: prediction set $\hat{C}(X_{n+1})$ via~\eqref{eq:condconf_Pset}.
    \end{algorithmic}
\end{algorithm}

If $\hat{r}$ estimates $r^*$ consistently, this method asymptotically achieves 
audited conditional coverage~\eqref{eq:audited-cond-coverage} for the partition 
$\mathcal{R}=\{[0,t_1],(t_1,1],\dots,[0,t_m],(t_m,1]\}$. The following theorem, 
proved in Appendix~\ref{app:proofs}, gives a finite-sample guarantee as a function 
of the discrepancy between $\hat{r}$ and $r^*$.

\begin{theorem}\label{thm:condconf_oracle}
Let $\hat{r}$ be a fixed audit model, $\hat{G}_j(x)=\mathbbm{1}[\hat{r}(x)\le t_j]$, 
and $G_j^*(x)=\mathbbm{1}[r^*(x)\le t_j]$ the corresponding oracle groups, for 
$\mathcal{T}=\{t_1,\dots,t_m\}$. Let $\hat{C}(X_{n+1})$ be the prediction set from 
Algorithm~\ref{alg:ACP_condconf}, and assume $\mathbb{P}[G_j^*(X_{n+1})=g]>0$ for 
all $j\in[m]$, $g\in\{0,1\}$. Then
\[
    \mathbb{P}\!\left[Y_{n+1}\in\hat{C}(X_{n+1})\mid G_j^*(X_{n+1})=g\right]
    \ge 1-\alpha
    -\frac{(1+\alpha)\,\mathbb{P}[\hat{G}_j(X_{n+1})\neq G_j^*(X_{n+1})]}
    {\mathbb{P}[G_j^*(X_{n+1})=g]}.
\]
Moreover, if $\hat{r}(x)\xrightarrow{p}r^*(x)$ for all $x$ and $\mathbb{P}[r^*(X_{n+1})
=t_j]=0$ for all $t_j\in\mathcal{T}$, then $\mathbb{P}[\hat{G}_j(X_{n+1})\neq 
G_j^*(X_{n+1})]\to 0$.
\end{theorem}

As shown in Section~\ref{sec:experiments}, this method often achieves substantially 
higher conditional coverage than alternatives, at the cost of solving a convex 
optimization problem per candidate label. The special case $\mathcal{T}=\{t\}$ admits 
an efficient solution via Mondrian conformal prediction \citep{vovk2003mondrian}: 
calibrate using $\{i\in\mathcal{D}_{\text{cal}}^2:\hat{r}(X_i)\le t\}$ if 
$\hat{r}(X_{n+1})\le t$, and using $\{i\in\mathcal{D}_{\text{cal}}^2:\hat{r}(X_i)>t\}$ 
otherwise. This special case is summarized in 
Algorithm~\ref{alg:ACP_AEqualized} (Appendix~\ref{app:add_algorithms}).

\subsection{Adaptive Audited Conformal Prediction}\label{sec:adaptive_ACP}

In practice, practitioners may face competing objectives: a retrained model may 
achieve higher conditional coverage at the cost of larger prediction sets, which may 
be preferable when reliability is prioritized over efficiency. We therefore introduce 
\emph{AACP}, a data-driven procedure for adaptively selecting between ACP and a competing 
baseline (e.g., split conformal prediction based on a retrained model), guided by a 
user-specified criterion such as conditional coverage, average prediction set size, or 
a weighted trade-off between the two.

We illustrate AACP by selecting between retrained-model conformal prediction and ACP, using higher conditional coverage on unreliable groups as the selection criterion. 
To preserve 
exchangeability---which is necessary for marginal coverage---we follow 
\citep{yang2021finite, liang2023ces} and split $\mathcal{D}_{\text{cal}}$ into three 
disjoint subsets $\mathcal{D}_{\text{cal}}^1$, $\mathcal{D}_{\text{cal-select}}^2$, 
and $\mathcal{D}_{\text{cal-calib}}^2$, using $\mathcal{D}_{\text{cal}}^1$ both to train 
$\hat{r}$ and retrain the legacy model.

For a threshold $\epsilon\in[0,1]$, define the ``unreliable'' group
\begin{equation}\label{eq:aacp_groupdef}
    \mathcal{H} = \big\{i\in\mathcal{D}_{\text{cal-select}}^2 : \hat{r}(X_i)\le\epsilon\big\}.
\end{equation}
Smaller $\epsilon$ restricts $\mathcal{H}$ to the most unreliable samples; larger 
$\epsilon$ broadens it to moderately difficult ones. In practice, $\epsilon$ is chosen 
to reflect the user's tolerance for unreliability while ensuring $\mathcal{H}$ is 
large enough for a well-powered subsequent test.

For each candidate method, we estimate its conditional coverage on $\mathcal{H}$:
\begin{equation}\label{eq:aacp_estcondcov}
    \hat\delta_{\text{Retrain}} = \frac{1}{|\mathcal{H}|}\sum_{i\in\mathcal{H}}
    \mathbbm{1}[Y_i\in\hat{C}_{\text{Retrain}}(X_i)], \qquad
    \hat\delta_{\text{ACP}} = \frac{1}{|\mathcal{H}|}\sum_{i\in\mathcal{H}}
    \mathbbm{1}[Y_i\in\hat{C}_{\text{ACP}}(X_i)],
\end{equation}
where $\hat{C}_{\text{Retrain}}$ and $\hat{C}_{\text{ACP}}$ are calibrated on 
nonconformity scores over $\mathcal{D}_{\text{cal-select}}^2$ using the retrained 
model and ACP, respectively. We then apply a one-sided $t$-test for:
\begin{equation}\label{eq:aacp_test}
    H_0: \hat\delta_{\text{ACP}} \le \hat\delta_{\text{Retrain}}
    \quad\text{vs.}\quad
    H_1: \hat\delta_{\text{ACP}} > \hat\delta_{\text{Retrain}}.
\end{equation}
If $H_0$ is rejected at a pre-specified level (e.g., 10\%), we select ACP; otherwise 
we default to retraining. The final prediction set is then constructed by calibrating 
on $\mathcal{D}_{\text{cal-calib}}^2$ with the selected method. Algorithm~\ref{alg:AACP} 
summarizes the procedure; numerical results in 
Figures~\ref{fig:supp_exp_sim_nnew_adaptive_marg}--Table~\ref{tab:supp_exp_sim_nnew_adaptive} 
(Appendix~\ref{app:add_experiments}) confirm that AACP reliably selects the method leading to higher conditional coverage.

\section{Numerical Experiments}\label{sec:experiments}

\subsection{Setup and Benchmarks}

We evaluate three versions of our method: ACP with Marginal Coverage (ACP-MC; 
Section~\ref{sec:ACP-MC}), ACP with Audited Conditional Coverage (ACP-ACC; 
Section~\ref{sec:ACP-ACC}) using threshold set $\mathcal{T}=\{0.5,0.6,0.7,0.8,0.9\}$, 
and a single-threshold ($t=0.7$) special case of ACP-ACC targeting Audited Equalized 
Coverage (ACP-AEC).

We compare against three benchmarks based on standard conformal prediction: 
(i) using the legacy model directly; (ii) using the legacy model recalibrated via 
Adaptive Temperature Scaling (AdaTS; \cite{xie2024calibrating}), which rescales 
predicted probabilities as $\hat\pi_k(x)\mapsto\hat\pi_k(x)^{1/T}/\sum_{j}\hat\pi_j(x)^{1/T}$ 
with $T>0$ learned by minimizing negative log-likelihood on held-out data; and 
(iii) using a retrained model.

For all methods we target miscoverage level $\alpha=0.1$. The legacy model is fixed 
and pretrained as detailed in Appendix~\ref{app:add_experiments}. All 
methods except standard CP use $\mathcal{D}_{\text{cal}}^1$ to train their additional 
components---the audit model for ACP, temperature $T$ for AdaTS, and the retrained 
model for the retraining baseline---and $\mathcal{D}_{\text{cal}}^2$ to compute 
nonconformity scores and calibrate prediction sets. Further details, 
including audit feature construction and data-splitting proportions, are in 
Appendix~\ref{app:add_experiments}.

\subsection{Synthetic Data}

We generate synthetic data in which a proportion $\beta=0.1$ of target samples 
undergoes distribution shift introducing greater intrinsic label uncertainty; the 
data-generating process is described in Example~\ref{exp:concept-shift} 
(Appendix~\ref{app:add_motivating_examples}). We use $100$ features, $K=5$ classes, 
and train the legacy model on $10{,}000$ historical samples. All results are averaged 
over $1{,}000$ test points and $50$ independent runs.

Figure~\ref{fig:main_exp_sim_nnew_overall} reports marginal coverage, average 
prediction set size, and conditional coverage as functions of $|\mathcal{D}_{\text{cal}}|$, 
ranging from $200$ to $5{,}000$. Since the data-generating process is known, we 
stratify results by the oracle score $r^*$ into two bins: $r^*\le 0.5$ (hard, unreliable 
samples) and $r^*>0.5$ (easy samples).

\begin{figure*}[!htb]
    \centering
    \includegraphics[width=\linewidth]{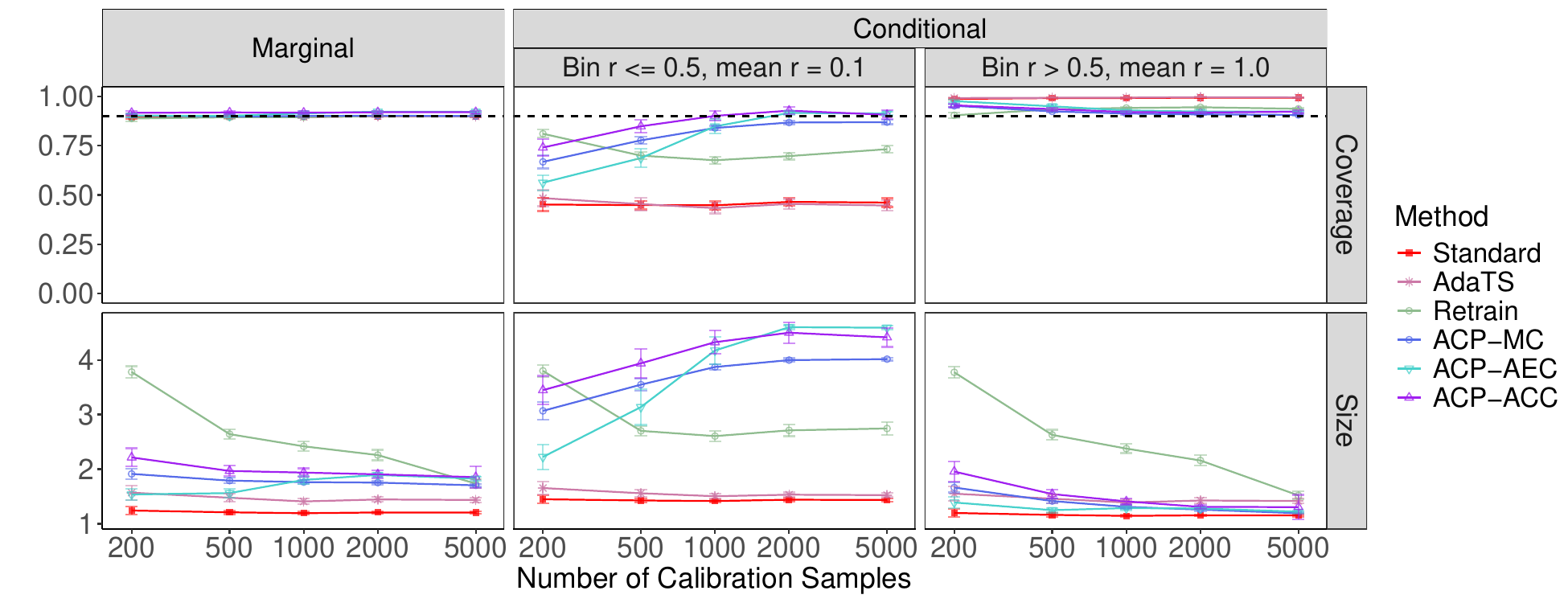}
    \caption{Performance of conformal prediction sets for $5$-class synthetic data 
      as a function of calibration sample size 
    $|\mathcal{D}_{\text{cal}}|$, at miscoverage level $\alpha=0.1$. Left:
    marginal coverage and average prediction set size. All methods achieve the nominal 
    $90\%$ marginal coverage; ACP variants (ACP-MC, ACP-AEC, ACP-ACC) maintain 
    compact set sizes. Right: coverage conditional on oracle reliability bin 
    ($r^*\le 0.5$: hard samples; $r^*>0.5$: easy samples). Standard CP and AdaTS 
    undercover hard samples, while ACP improves conditional coverage on hard 
    samples without inflating set sizes on easy ones. Error bars denote two standard 
    errors. See Table~\ref{tab:main_exp_sim_ndata} for numerical details.}
    \label{fig:main_exp_sim_nnew_overall}
\end{figure*}

Standard CP and AdaTS yield the smallest average set sizes but exhibit substantial 
undercoverage on the hard bin. Retraining improves conditional coverage at the cost 
of overly conservative prediction sets, particularly at small calibration sizes; 
moreover, it produces similar set sizes for both bins at small-to-moderate sample 
sizes, indicating it cannot effectively distinguish unreliable samples until 
sufficient labeled target data are available.

By contrast, all three ACP variants substantially improve coverage on the hard group 
while maintaining compact sets on average, assigning larger sets to hard samples and 
smaller sets to easy ones---reflecting a more accurate characterization of 
input-dependent uncertainty. Among the three, ACP-ACC is most conservative, as it 
enforces coverage over multiple overlapping subgroups. All ACP methods improve with 
calibration size, as more labeled data yield more accurate audit models. Overall, ACP 
achieves a more favorable reliability-efficiency trade-off than existing baselines.

Additional synthetic experiments are in Appendix~\ref{app:add_experiments}:
Figures~\ref{fig:supp_exp_sim_beta_marg}--\ref{fig:supp_exp_sim_beta_cond} and 
Table~\ref{tab:main_exp_sim_beta} vary $\beta$ at fixed calibration size $2{,}000$;
Figures~\ref{fig:supp_exp_sim_K_cond}--\ref{fig:supp_exp_sim_K_marg} and 
Table~\ref{tab:main_exp_sim_nnew_K} vary $K$ from $2$ to $20$ (ACP-MC uses the 
binary rule of Section~\ref{sec:methodology} when $K=2$ and the multiclass rule of 
Appendix~\ref{app:APA_mutlclass} otherwise);
Figures~\ref{fig:supp_exp_sim_n_new_equalized_marg}--\ref{fig:supp_exp_sim_n_new_etacompare_cond} 
and Tables~\ref{tab:main_exp_sim_equalized}--\ref{tab:main_exp_sim_etacompare} 
examine sensitivity to hyperparameters such as the ACP-AEC threshold, revealing a 
trade-off between conditional coverage and set size as $t$ increases;
Figures~\ref{fig:supp_exp_sim_n_new_full_marg}--\ref{fig:supp_exp_sim_n_new_full_cond} 
and Table~\ref{tab:main_exp_sim_nnew_full} compare against additional benchmarks, 
including the approach of \cite{kaur2025conformalpredictionsetsimproved}, temperature-scaling 
variants, and an oracle ACP-MC using $r^*$ directly;
Figures~\ref{fig:supp_exp_sim_nnew_adaptive_marg}--\ref{fig:supp_exp_sim_nnew_adaptive_cond} 
and Table~\ref{tab:supp_exp_sim_nnew_adaptive} cover AACP (Section~\ref{sec:adaptive_ACP});
and Figures~\ref{fig:supp_exp_sim_n_new_covshift_marg}--\ref{fig:supp_exp_sim_covshift_n_new_full_cond} 
and Tables~\ref{tab:supp_exp_sim_nnew_covshift}--\ref{tab:supp_exp_sim_covshift_nnew_full} 
present results under alternative distribution shifts, including covariate shift 
(Appendix~\ref{app:add_motivating_examples}).

\subsection{Real Data}
We evaluate ACP and all benchmarks on two real-world datasets: Camelyon17 \citep{bandi2018detection} from the WILDS benchmark \citep{wilds2021}, and CIFAR-10/CIFAR-10-C \citep{hendrycks2019robustness, krizhevsky2009learning}.
In both settings, we consider a target distribution formed by mixing in-domain and out-of-domain samples to simulate distribution shift. 
Since the oracle audit score $r^*$ is unavailable, we estimate it for each test point 
as the fraction of correct legacy-model predictions among its $50$ nearest neighbors 
(by cosine distance in embedding space) from an independent hold-out set of $5{,}000$ 
target samples, and bin the resulting scores into three equal-width bins.
Results are averaged over $1{,}000$ test points and $30$ runs using independent splits of the data. Full experimental setup details are in 
Appendix~\ref{app:add_experiments}.

We first consider Camelyon17, a binary tumor classification task on $96\times96$ histopathological images. 
Distribution shift arises naturally here, as 
models trained at one hospital are deployed across institutions with differences in 
patient population, staining, and image acquisition. We train the legacy model on 
images from hospital 0 and construct a target distribution comprising $90\%$ 
in-domain samples (hospital 0) and $10\%$ out-of-domain samples (hospital 2), 
following the architecture and hyperparameters of \citet{wilds2021} for both the 
legacy model and the retraining benchmark. 

Figure~\ref{fig:main_exp_real_nnew_overall} and Table~\ref{tab:main_exp_real_ndata} 
report performance as $|\mathcal{D}_{\text{cal}}|$ varies from $200$ to $2{,}000$. 
All methods achieve marginal coverage. Standard CP and AdaTS exhibit substantial 
conditional undercoverage on unreliable samples; retraining produces overly 
conservative sets that also fail to differentiate samples across reliability levels. 
All three ACP variants improve conditional coverage across reliability levels while 
maintaining compact, informative prediction sets comparable in size to standard CP and 
AdaTS. Figure~\ref{fig:supp_exp_real_beta_overall} and 
Table~\ref{tab:supp_exp_real_beta} further vary the proportion of out-of-domain 
samples.

\begin{figure*}[!htb]
    \centering
    \includegraphics[width=\linewidth]{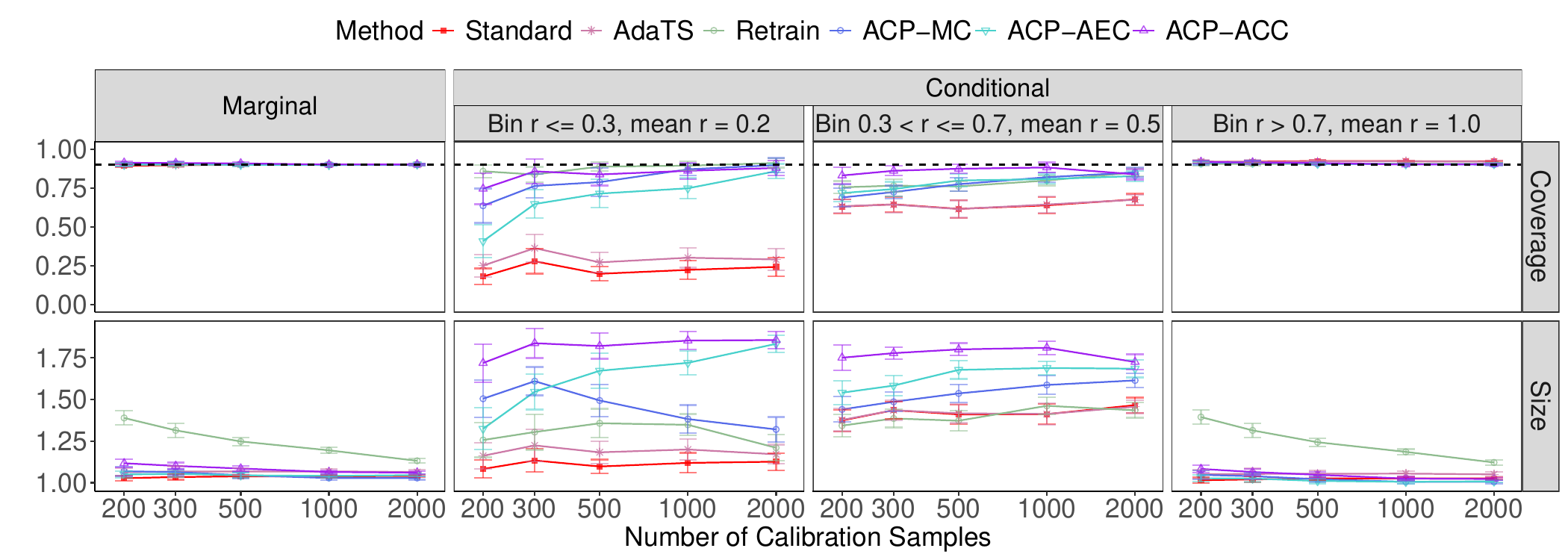}
    \caption{Performance of conformal prediction sets on the Camelyon17 binary tumor 
    classification task \citep{bandi2018detection} as a function of calibration sample 
    size $|\mathcal{D}_{\text{cal}}|$, at miscoverage level $\alpha=0.1$. The target 
    distribution contains $90\%$ in-domain (hospital 0) and $10\%$ out-of-domain 
    (hospital 2) samples. Left: marginal coverage and average prediction set 
    size; all methods achieve nominal $90\%$ marginal coverage, and ACP variants 
    (ACP-MC, ACP-AEC, ACP-ACC) maintain compact set sizes. Right: coverage 
    conditional on estimated reliability bin (fraction of correct legacy-model 
    predictions among $50$ nearest neighbors in embedding space). Standard CP and 
    AdaTS undercover unreliable samples, while ACP variants achieve higher conditional 
    coverage without inflating set sizes. Error bars denote two standard errors. See 
    Table~\ref{tab:main_exp_real_ndata} for numerical details.}
    \label{fig:main_exp_real_nnew_overall}
\end{figure*}

We further evaluate on CIFAR-10, a dataset of $32\times32$ color images from $10$ classes, with distribution shift introduced using CIFAR-10-C, which applies common corruptions to CIFAR-10 test images at varying severity levels. We train the legacy model on clean CIFAR-10 and construct a target distribution comprising $90\%$ in-domain (clean) samples and $10\%$ out-of-domain samples drawn from CIFAR-10-C with contrast corruption at severity level $5$. Both the legacy and retraining models use a ResNet-18 architecture \citep{he2016image} with hyperparameters described in Appendix~\ref{app:add_experiments}.

Figures~\ref{fig:supp_exp_real_nnew_overall_cifar10}--\ref{fig:supp_exp_real_beta_overall_cifar10} and Tables~\ref{tab:supp_exp_real_ndata_cifar10}--\ref{tab:supp_exp_real_beta_cifar10} in Appendix~\ref{app:add_experiments} report performance as $|\mathcal{D}_{\text{cal}}|$ varies from $200$ to $2{,}000$ and as the proportion of out-of-domain samples in the target distribution varies from $0.1$ to $0.4$. The results again show that ACP balances reliability and efficiency, improving conditional coverage while maintaining reasonably small and informative prediction sets.

\section{Discussion}~\label{sec:discussion}

ACP provides a practical and statistically principled approach to uncertainty-aware 
classification under arbitrary distribution shift, improving conditional coverage 
relative to existing conformal prediction methods while maintaining informative 
prediction sets. Its main limitation is dependence on audit model quality, which can 
be difficult to attain in very low-data regimes, compounded by the need to split the labeled 
data from the target distribution to preserve valid inference. Promising directions for future work include 
more data-efficient extensions (e.g., via leave-one-out or cross-validation conformal 
methods \citep{barber2019predictive}), alternative strategies for integrating audit 
models into conformal prediction, and extensions to online or streaming settings 
where distribution shifts evolve in real time.

Software implementing the algorithms and data experiments are available online at \url{https://github.com/FionaZ3696/Audited-Conformal-Prediction}.

\subsection*{Acknowledgements}

The authors thank the Center for Advanced Research Computing at the University of Southern California for providing computational resources and Tianmin Xie at the University of Southern California for helpful comments with manuscript editing.
M.~S.~and Y.~Z.~were partly supported the USC - Capital One CREDIF award.

\newpage


\bibliographystyle{unsrtnat}
\bibliography{ref}

\clearpage
\appendix

\renewcommand{\thesection}{A\arabic{section}}
\renewcommand{\theequation}{A\arabic{equation}}
\renewcommand{\thetheorem}{A\arabic{theorem}}
\renewcommand{\thecorollary}{A\arabic{corollary}}
\renewcommand{\theproposition}{A\arabic{proposition}}
\renewcommand{\thelemma}{A\arabic{lemma}}
\renewcommand{\thetable}{A\arabic{table}}
\renewcommand{\thefigure}{A\arabic{figure}}
\renewcommand{\thealgorithm}{A\arabic{algorithm}}
\setcounter{figure}{0}
\setcounter{table}{0}
\setcounter{theorem}{0}
\setcounter{algorithm}{0}

\clearpage
\section{Review of Adaptive Nonconformity Score (APS) for Classification}\label{app:aps_review}

This section reviews the adaptive nonconformity score (APS) function for constructing conformal classification sets proposed by \citet{romano2020classification}, with slight modifications to the notation to make it consistent with the rest of this paper.

\paragraph{With the oracle classifier.}
Suppose one has access to the oracle conditional class probabilities under the target distribution, namely $
\pi_k(x) = \mathbb{P}_T(Y = k \mid X = x), k \in [K], \; x \in \mathcal{X}.
$
For each fixed \(x\), let $
\pi_{(1)}(x) \ge \pi_{(2)}(x) \ge \cdots \ge \pi_{(K)}(x)
$
denote the ordered statistics for $\pi_k(x)$. For any $\tau \in [0,1]$, define as in~\cite{romano2020classification} the \emph{generalized conditional quantile} function:
\begin{align} \label{eq:oracle-threshold}
  L(x; \pi, \tau) & = \min \{ k \in \{1,\ldots,K\} \ : \ \pi_{(1)}(x) + \pi_{(2)}(x) + \ldots + \pi_{(k)}(x) \geq \tau \}.
  \end{align}
Furthermore, let $U \sim \text{Unif}(0,1)$ and $u$ be a realization of $U$, define a function $\mathcal{S}$ that takes inputs $x. \pi, \tau \in [0,1]$ and $u \in (0,1)$ as:
\begin{align} \label{eq:define-S}
    \mathcal{S}(x, u ; \pi, \tau) & =
    \begin{cases}
    \text{ `$y$' indices of the $L(x ; \pi,\tau)-1$ largest $\pi_{y}(x)$},
    & \text{ if } u \leq V(x ; \pi,\tau) , \\
    \text{ `$y$' indices of the $L(x ; \pi,\tau)$ largest $\pi_{y}(x)$},
    & \text{ otherwise},
    \end{cases}
\end{align}
\begin{align*}
    V(x; \pi, \tau) & =  \frac{1}{\pi_{(L(x ; \pi, \tau))}(x)} \left[\sum_{k=1}^{L(x ; \pi, \tau)} \pi_{(k)}(x) - \tau \right].
\end{align*}
The oracle conformal prediction set, defined as 
\begin{align}
    C_\alpha^{\text{oracle}}(x) = \mathcal{S}(x, U; \pi, 1-\alpha),
\end{align}
is the smallest randomized prediction sets with feature conditional coverage at level $1-\alpha$. 

\paragraph{Estimating the oracle classifier.} In practice, $\pi_k(x)$ is generally unavailable and must be estimated from the data. Let \(\hat{\pi}(x) = (\hat{\pi}_1(x),\ldots,\hat{\pi}_K(x))\) denote the output of a black-box classifier, with $\hat{\pi}_k(x)$ estimates the true unknown class probabilities $\pi_k(x)$ for each $k \in [K]$. 
Define as in \citep{romano2020classification} a \emph{generalized inverse quantile} nonconformity score function $E$ with input $x, y, u, \hat{\pi}$: 
\begin{equation}\label{eq:aps_scores}
    E_i = E(X_i, Y_i, U_i, \hat{\pi}) = \min\big\{ \tau \in [0,1] : Y_i \in \mathcal{S}(X_i, U_i; \hat{\pi}, \tau) \big\}.
\end{equation}
Intuitively, $E_i$ is the smallest threshold $\tau$ such that the prediction set constructed using $\hat{\pi}$ contains the true label $Y_i$.
In the split conformal prediction regime where a hold-out calibration data $\mathcal{D}_{\text{calib}}$ is available. The empirical threshold $\hat{\tau}$ is then chosen as the $\lceil (1-\alpha)(1 + |\mathcal{D}_{\text{cal}}|) \rceil$-th largest element of the set $\{E_i\}_{i \in \mathcal{D}_{\text{cal}}}$, and is subsequently used to construct the prediction set for a test point with feature $X_{n+1}$ as
\begin{align}\label{eq:aps_set}
C_\alpha^{\text{APS}}(x) = \mathcal{S}(X_{n+1}, U_{n+1}; \hat{\pi}, \hat{\tau}). 
\end{align}

\FloatBarrier

\section{Motivating Examples}\label{app:add_motivating_examples}

This section presents motivating examples illustrating settings in which training an audit model may provide more robust and informative uncertainty signals than retraining under various types of distribution shift in the target distribution. Example~\ref{exp:concept-shift} is implemented in our numerical experiments to produce Figures~\ref{fig:main_exp_sim_nnew_overall}--\ref{fig:supp_exp_sim_n_new_full_cond}. Example~\ref{exp:covariate-shift} is implemented to produce Figures~\ref{fig:supp_exp_sim_n_new_covshift_marg}--\ref{fig:supp_exp_sim_n_new_covshift_etacompare_cond}.

\begin{example}[Concept shift; the target distribution contains random samples with greater intrinsic label uncertainty.]\label{exp:concept-shift}
Fix $K \ge 2$ classes and $p$ features, and let $X \in [0, 1]^p$ with $P_S(X) = P_T(X) = \mathrm{Unif}([0, 1]^p)$. Under the historical distribution $P_S$, labels are deterministic functions of the first coordinate: setting $a_k = k/K$ for $k=0,\dots,K$ and $g(x)=k$ if $a_{k-1} < x_1 \le a_k$, we have $P_S(Y=g(x) \mid X=x) = 1$ almost surely. Assume $f_0^{\text{pt}}$ learns this rule exactly, so $f_0^{\text{pt}}(x) = g(x)$.
At deployment, a concept shift introduces intrinsic uncertainty governed by the second coordinate. For $\beta \in (0, 1)$, the target conditional distirbution is defined as:
$$
P_T\big[Y=\cdot \mid X=x\big]=
\begin{cases}
\mathrm{Unif}\{1,\dots,K\}, & x_2 < \beta, \\
\delta_{g(x)}, & x_2 \ge \beta,
\end{cases}$$
where $\delta_{g(x)}$ is a point mass at $g(x)$. Intuitively, in the region $\{x : x_2 < \beta\}$, labels are intrinsically ambiguous and no predictor can outperform random guessing; outside it, the original deterministic rule is preserved. Although $f_0^{\text{pt}}$ is optimal under $P_S$, its correctness probability under $P_T$ satisfies:
$$\mathbb{P}_T\big[r=1 \mid X=x \big]=
\begin{cases}
1/K, & x_2 < \beta, \\
1, & x_2 \ge \beta.
\end{cases}$$
\end{example}
As demonstrated in Section~\ref{sec:experiments}, an audit model can be advantageous when predicting $r$ is easier than relearning the full $K$-class decision boundaries. This may occur when the audit model identifies $r$ as a one-dimensional threshold problem in $x_2$: the legacy model’s predictions are unreliable when $x_2 < \beta$ and reliable otherwise. By contrast, retraining a full $K$-class classifier from scratch—especially with a flexible architecture and limited labeled data—may lead to overfitting on labels from the uncertain region ${x: x_2 < \beta}$, resulting in spuriously confident predictions that could reduce the efficiency and reliability of conformal prediction sets.

\begin{example}[Covariate shift; the target distribution contains a higher proportion of intrinsically uncertain samples]\label{exp:covariate-shift}
Fix $K \ge 2$ classes and $p$ features, and let $X=(X_1,\dots,X_p)\in[0,1]^p$. Under the historical distribution, features are i.i.d.\ uniform:
$$P_S(X)=\mathrm{Unif}([0,1]^p).$$
Let $a_k=k/K$ for $k=0,\dots,K$, and define $g(x)=k$ if $a_{k-1}<x_1\le a_k$. For a fixed $\beta\in(0,1)$, assume the conditional label law is the same under both historical and target distributions:
$$P_S(Y=\cdot\mid X=x)=P_T(Y=\cdot\mid X=x)=
\begin{cases}
\mathrm{Unif}\{1,\dots,K\}, & x_2<\beta,\\
\delta_{g(x)}, & x_2\ge \beta,
\end{cases}$$
where $\delta_{g(x)}$ is a point mass at $g(x)$. Thus, labels are intrinsically ambiguous in the region $\{x:x_2<\beta\}$ and deterministic otherwise. Assume $f_0^{\text{pt}}$ predicts $f_0^{\text{pt}}(x)=g(x)$, which is optimal outside the ambiguous region and cannot be improved upon inside it.

At deployment, we introduce a covariate shift in the second coordinate: for some $a\in(0,1)$,
$$X_2\sim \mathrm{Unif}(0,1)\quad \text{under }P_S,\qquad X_2\sim \mathrm{Unif}(0,a)\quad \text{under }P_T,$$
while all other coordinates remain $\mathrm{Unif}(0,1)$ under both distributions. Hence $P_S(X)\neq P_T(X)$, but $P_S(Y\mid X)=P_T(Y\mid X)$. Notably, the correctness probability of $f_0^{\text{pt}}(x)$ conditional on $X=x$ is unchanged:
$$\mathbb P(r=1\mid X=x)=
\begin{cases}
1/K, & x_2<\beta,\\
1, & x_2\ge \beta.
\end{cases}$$
However, the target distribution places more mass on the intrinsically uncertain region. Under $P_S$,
$$P_S(X_2<\beta)=\beta,$$
whereas under $P_T$,
$$P_T(X_2<\beta)=
\begin{cases}
\beta/a, & \beta<a,\\
1, & \beta\ge a.
\end{cases}$$
Thus, although the pointwise difficulty of prediction does not change, the target population contains a larger proportion of hard samples.

An audit model can be advantageous when it can identify \(r\) as a one-dimensional threshold problem in \(x_2\): the legacy model’s predictions are unreliable when \(x_2 < \beta\) and reliable otherwise. The standard conformal procedure calibrates only for marginal coverage and thus may produce systematically lower coverage rates for samples within the hard region. Retraining a full \(K\)-class predictor does not address this core issue, because the conditional label law is unchanged and the ambiguous region remains intrinsically noisy; the challenge is not to relearn the class boundary, but to recognize that the target population now contains many more samples on which no predictor can be confident.
\end{example}

\begin{example}[Concept shift; the target distribution contains random samples with continuously varying intrinsic uncertainty]\label{exp:continuous-concept-shift}
Fix $K \ge 2$ classes and $p$ features, and let $X \in [0, 1]^p$ with $P_S(X) = P_T(X) = \mathrm{Unif}([0, 1]^p)$. Under $P_S$, labels are nearly deterministic functions of the first coordinate: setting $a_k = k/K$ for $k=0,\dots,K$ and $g(x)=k$ if $a_{k-1} < x_1 \le a_k$, we let
$P_S(Y=g(x) \mid X=x)= 1$ almost surely. Assume $f_0^{\text{pt}}$ learns the historical rule accurately, so $f_0^{\text{pt}}(x) = g(x)$.

At deployment, a concept shift introduces intrinsic uncertainty that varies smoothly with the second coordinate. For some $q_{\max}\in(0,(K-1)/K]$, define
$$
q(x)=q_{\max}x_2,
$$
and let
$$
P_T(Y=\cdot \mid X=x)=
\begin{cases}
1-q(x), & \text{on } g(x),\\
q(x)/(K-1), & \text{on each } j\neq g(x).
\end{cases}
$$
Thus, the target labels remain concentrated on $g(x)$ when $x_2$ is small, but become progressively more ambiguous as $x_2$ increases. In particular, uncertainty is weakest near $x_2=0$ and strongest near $x_2=1$; when $q_{\max}=(K-1)/K$, the hardest points are exactly uniform over the $K$ classes.

Although $f_0^{\text{pt}}$ is optimal under $P_S$, its correctness probability under the target distribution now varies continuously with $x$:
$$
\mathbb{P}_T(r=1 \mid X=x)=1-q_{\max}x_2.
$$
Hence $r^*(x)$ is no longer binary-valued, but a continuous function ranging from $1$ down to $1-q_{\max}$. For example, when $K=5$ and $q_{\max}=0.8$, we have $r^*(x)\in[0.2,1]$.

An audit model can be advantageous here again when it can identify \(r\) as a one-dimensional threshold problem in \(x_2\): using the original features or information from the legacy model, an informative audit model should learn that reliability decreases smoothly as $x_2$ grows. In contrast, a flexible multiclass model re-trained on limited target samples may again overfit local label noise and output spuriously confident predictions.
\end{example}

\begin{example}[Label-noise shift; target labels are randomly corrupted independent of $X$]\label{exp:label-noise-shift}
Fix $K \ge 2$ classes and $p$ features, and let $X \in [0,1]^p$ with $P_S(X)=P_T(X)=\mathrm{Unif}([0,1]^p)$. Under $P_S$, labels are deterministic functions of the first coordinate: setting $a_k = k/K$ for $k=0,\dots,K$ and $g(x)=k$ if $a_{k-1} < x_1 \le a_k$, we have $P_S(Y=g(x)\mid X=x)=1$ almost surely. Assume $f_0^{\text{pt}}$ learns this rule accurately, so $f_0^{\text{pt}}(x) = g(x)$.

At deployment, the feature distribution is unchanged, but labels are subject to symmetric random corruption independent of $X$. Let $Y^\star=g(X)$ denote the latent clean label, and fix a noise rate $\beta\in[0,1]$. Under $P_T$, define
$$
Y=
\begin{cases}
Y^\star, & \text{w.p. }1-\beta,\\
\mathrm{Unif}\big(\{1,\dots,K\}\setminus\{Y^\star\}\big), & \text{w.p. }\beta,
\end{cases}
\qquad \text{independently of }X.
$$
Equivalently, for every $x$,
$$
P_T(Y=\cdot\mid X=x)=
\begin{cases}
1-\beta, & \text{on } g(x),\\
\beta/(K-1), & \text{on each } j\neq g(x).
\end{cases}
$$
Thus $P_T(Y\mid X)\neq P_S(Y\mid X)$, but the shift is entirely due to irreducible label randomness rather than any new structure in the feature space.

Now consider the legacy correctness indicator $r=\mathbbm{1}\{Y=f_0^{\text{pt}}(X)\}$. If $f_0^{\text{pt}}(x)=g(x)$, then under $P_T$,
$$
\mathbb{P}_T(r=1\mid X=x)=\mathbb{P}_T(Y=g(x)\mid X=x)=1-\beta,
$$
which is constant in $x$. In other words, the oracle reliability surface is flat: all target points are equally reliable, and the only source of uncertainty is the homogeneous label noise.

Since the true reliability is constant, an audit model can be advantageous when it successfully identifies this constant by effectively estimating a single Bernoulli rate, which reflects the noise-imposed ceiling on predictive accuracy. By contrast, retraining a flexible multiclass predictor from scratch, especially with limited target labels, may overfit these random perturbations and produce spurious feature dependent confidence.
\end{example}

\begin{example}[No shift] 
$P_T = P_S$ and the legacy model $f_0^{\text{pt}}$ is already a good approximation to $P_T(Y \mid X)$. In 
this case, retraining a flexible model on limited target data might introduce unnecessary estimation variance and risk of overconfidence, making its uncertainty estimates unstable even though the environment did not change. An audit model instead only needs to estimate the legacy model's correctness, which reduce the chance of perturbing a well-functioning legacy predictor.
\end{example}

\FloatBarrier

\section{Extension of Algorithm~\ref{alg:ACP_ACP-MC_binary} to Multiclass Classification}\label{app:APA_mutlclass}

\subsection{Multiclass combination rule}

Extending the binary adjustment rule to a multi-class setting is non-trivial because the event of a correct prediction ($f_0^{\text{pt}}(X) = Y$) no longer uniquely identifies the true outcome $Y$. To address this, we introduce a second oracle quantity, $\eta^*(X)$, which defines how the probability mass should be redistributed when the legacy model is incorrect. Specifically, for any $x \in \mathcal{X}$, let $\eta^*(x)$ denote the oracle conditional distribution of the label $Y$ given that the legacy model’s top prediction $f_0^{\text{pt}}(x)$ is incorrect:
\begin{equation*}
    \eta^*(x)
	\; :=\;
	\mathbb{P}\!\left[Y \,\big|\, X=x,\; Y \neq f_0^{\text{pt}}(x)\right].
\end{equation*}
Intuitively, $\eta^*(x)$ answers the question: \emph{``If the legacy model’s top-1 prediction is wrong, where should the probability mass go?''}. By definition, $\eta^*_k(x) = 0$ whenever $k = f_0^{\text{pt}}(x)$.

For each class label $k \in [K]$, $m_k$ is given by a weighted mixture of the audit score, $r(x)$, and the reallocation score, $\eta(x)$:
\begin{equation}\label{eq:multiclass-comb-rule}
	m_k(x; r, \eta, f_0)
	\;=\;
	r(x)\,\mathbbm{1}\{ k = f_0^{\text{pt}}(x)\}
	\;+\;
	\big(1 - r(x)\big)\,\eta_k(x).
\end{equation}
The first term assigns probability mass $r(x)$ to the predicted class $f_0^{\text{pt}}(x)$, while the second term allocates the remaining mass $1 - r(x)$ across the remaining classes according to $\eta_k(x)$. 
In words, $m(x; \hat{r}, \hat{\eta}, f_0)$ gives an estimate of the conditional distribution of $Y \mid X=x$ under $P_T$ that is consistent if the fitted audit model $\hat{r}$ and the reallocation estimate $\hat{\eta}$ provides a consistent estimate of the oracle audit model $r^*$ and the oracle reallocation model $\eta^*$, for any fixed legacy model $f_0$ and any target distribution $P_T$. This result formally established in Proposition~\ref{prop:consistency_apa_multiclass} and is proved in Appendix~\ref{app:proofs}.

Using $m(x; \hat r, \hat \eta, f_0)$ as an estimate of $P_T(Y \mid X=x)$, prediction sets with guaranteed marginal coverage can be constructed by applying the standard conformal prediction approach, as reviewed in Appendix~\ref{app:aps_review}, using the remaining calibration data in $\mathcal{D}_{\text{cal}}^2$.
Algorithm~\ref{alg:ACP_APA_multiclass} summarizes this procedure.

In practice, the oracle reallocation model $\eta^*(x)$ is unknown and needs to be estimated. We propose several tractable strategies for estimating $\hat{\eta}(x)$, trading off computational complexity with estimation accuracy.  

\begin{enumerate}
    \item \textbf{Residual Renormalization ($\hat{\eta}_{\text{Renorm}}$)}.  
    This is a fast and intuitive approach that relies on the relative confidence of the legacy model $f_0$ across the non-top classes. It assumes that while the top-1 prediction is wrong, the relative rankings of the remaining classes are reliable. Let \(\hat{\pi}_k(x)\) denote the predicted probability assigned by the legacy model to class \(k\). The reallocation model renormalizes the probability mass over the set of alternative classes $k \neq f_0^{\text{pt}}(x)$:
\[
	\hat{\eta}_k (x) 
	\;=\;
	\begin{cases}
	 	0, & \text{if } k = f_0^{\text{pt}}(x),\\[4pt]
	 	\frac{\hat{\pi}_k(x)}{1-\hat{\pi}_{f_0^{\text{pt}}(x)}}, & \text{if } k \neq f_0^{\text{pt}}(x).
	\end{cases}
\] 
This approach is computationally efficient and favored when the legacy model $f_0$ is generally well-behaved, despite its overconfidence on the top-predicted label.
    \item \textbf{Confusion Matrix-based Estimation ($\hat{\eta}_{\text{Counts}}$)}.  
This method reallocates the probability mass by utilizing an empirical estimate of the conditional error distribution. We construct a confusion matrix based on the observed data in $\mathcal{D}_{\text{cal}}^1$ to estimate the distribution of true labels $Y$ given a specific incorrect prediction $f_0^{\text{pt}}$. Specifically, for each predicted class $j$, we estimate the probability $P(Y=k \mid f_0^{\text{pt}}=j, Y \neq j)$ using the empirical counts from samples in $\mathcal{D}_{\text{cal}}^1$. This approach is robust to severe misrankings caused by the legacy model $f_0$ but might prone to biases when calibration sample size is small with many classes.

    \item \textbf{Learned Wrong-Label Model ($\hat{\eta}_{\text{Learned}}$)}. The most flexible approach is to train a dedicated multi-class classifier to predict $\eta^*(x)$. This model, $\hat{\eta}_{\text{Learned}}$, is trained on the subset of $\mathcal{D}_{\text{cal}}^1$ where the legacy prediction was incorrect ($f_0^{\text{pt}}(X_i) \neq Y_i$). Features can include the original input $X$, as well as summary information such as entropy, predicted probability margins, and other indicators used for the audit model $\hat{r}(x)$.
\end{enumerate}

\FloatBarrier

\section{Algorithms}\label{app:add_algorithms}

\begin{algorithm}[!htb]
    \caption{ACP for Multiclass Classification with Audited Probabilistic Classifiers}
    \label{alg:ACP_APA_multiclass}
    \begin{algorithmic}[1]
        \STATE \textbf{Input}: calibration data $\mathcal{D}_{\text{cal}}$; 
        test point with features $X_{n+1}$; 
        pre-trained legacy model $f_0$; \\
        \STATE \textcolor{white}{\textbf{Input}:} fixed rule computing nonconformity score $E$; level $\alpha \in (0,1)$; 
        \STATE \textcolor{white}{\textbf{Input}:} selected $\eta$ estimation method from $(\hat{\eta}_{\text{Renorm}}$, $\hat{\eta}_{\text{Counts}}$, or $\hat{\eta}_{\text{Learned}})$
        \STATE Randomly split $\mathcal{D}_{\text{cal}}$ into disjoint subsets $\mathcal{D}_{\text{cal}}^1$ and $\mathcal{D}_{\text{cal}}^2$.
        \STATE Compute the audit outcome labels $R_i = \mathbbm{1}\{f_0^{\text{pt}}(X_i) = Y_i\}$ for $i \in \mathcal{D}_{\text{cal}}^1$. 
        \STATE Train audit model $\hat{r}$ on $\{(\tilde{X}_i, R_i)\}$ for $i \in \mathcal{D}_{\text{cal}}^1$, where $\tilde{X}_i$ aggregates the original features $X_i$ and additional features derived from $f_0$ as described in Section~\ref{sec:train_correctness}.
        \STATE Estimate the reallocation distribution $\hat{\eta}$ on $\mathcal{D}_{\text{cal}}^1$ using the selected method. 
        \FOR{$i \in \mathcal{D}_{\text{cal}}^2$}
            \STATE Compute $m(X_i; \hat{r}, \hat{\eta}, f_0)$ using~\eqref{eq:multiclass-comb-rule}.
            \STATE Compute nonconformity score $E_i = E(X_i$, $Y_i, m(X_i; \hat{r}, \hat{\eta}, f_0))$.
        \ENDFOR
        \STATE Compute the empirical threshold $\hat{q}$ as the $\lceil(1-\alpha)(1 + |\mathcal{D}_{\text{cal}}^2|)\rceil$-th largest value of $\{E_i\}$.
        \STATE Construct $\hat{C}(X_{n+1})$ as $\hat C(X_{n+1}) = \left\{ y: E(X_{n+1},y, m(X_{n+1}; \hat{r}, \hat{\eta}, f_0)) \le \hat{q}\right\}$. 
        \STATE \textbf{Output}: prediction set $\hat{C}(X_{n+1})$.
    \end{algorithmic}
\end{algorithm}

\begin{algorithm}[!htb]
    \caption{ACP with Audited Equalized Coverage}
    \label{alg:ACP_AEqualized}
    \begin{algorithmic}[1]
        \STATE \textbf{Input}: calibration data $\mathcal{D}_{\text{cal}}$; 
        test point with features $X_{n+1}$; 
        pre-trained legacy model $f_0$; \\
        \STATE \textcolor{white}{\textbf{Input}:} a threshold $t$; fixed nonconformity score function $E$; level $\alpha \in (0,1)$.
        \STATE Randomly split $\mathcal{D}_{\text{cal}}$ into disjoint subsets $\mathcal{D}_{\text{cal}}^1$ and $\mathcal{D}_{\text{cal}}^2$.
        \STATE Compute the audit outcome labels $R_i = \mathbbm{1}\{f_0^{\text{pt}}(X_i) = Y_i\}$ for $i \in \mathcal{D}_{\text{cal}}^1$. 
        \STATE Train audit model $\hat{r}$ on $\{(\tilde{X}_i, R_i)\}$ for $i \in \mathcal{D}_{\text{cal}}^1$, where $\tilde{X}_i$ aggregates the original features $X_i$ and additional features derived from $f_0$ as described in Section~\ref{sec:train_correctness}.
        \STATE Compute nonconformity scores $E_i = E(X_i, Y_i, f_0(X_i))$ for $i \in \mathcal{D}_{\text{cal}}^2$. 
        \IF{$\hat r(X_{n+1})\le t$}
            \STATE Define $\mathcal{D}_{\text{cal}}^{2,\leq t} = \{(X_i,Y_i)\in \mathcal{D}_{\mathrm{cal}}^2:\hat r(X_i)\le t\}$. 
            \STATE Compute empirical threshold $\hat{q}$ as the $\lceil(1-\alpha)(1 + |\mathcal{D}_{\text{cal}}^{2,\leq t}|)\rceil$-th largest value of $\{E_i\}_{i \in \mathcal{D}_{\text{cal}}^{2,\leq t}}$.
        \ELSE 
            \STATE Define $\mathcal{D}_{\text{cal}}^{2, > t} = \{(X_i,Y_i)\in \mathcal{D}_{\mathrm{cal}}^2:\hat r(X_i) > t\}$. 
            \STATE Compute empirical threshold $\hat{q}$ as the $\lceil(1-\alpha)(1 + |\mathcal{D}_{\text{cal}}^{2,>t}|)\rceil$-th largest value of $\{E_i\}_{i \in \mathcal{D}_{\text{cal}}^{2,> t}}$.
        \ENDIF
        \STATE Construct $\hat{C}(X_{n+1})$ as $\hat C(X_{n+1}) = \left\{ y: E(X_{n+1},y, f_0(X_{n+1})) \le \hat{q}\right\}$.
        \STATE \textbf{Output}: prediction set $\hat{C}(X_{n+1})$.
    \end{algorithmic}
\end{algorithm}

\begin{algorithm}[!htb]
    \caption{Adaptive Audited Conformal Prediction (AACP)}
    \label{alg:AACP}
    \begin{algorithmic}[1]
        \STATE \textbf{Input}: calibration data $\mathcal{D}_{\text{cal}}$; 
        test point with features $X_{n+1}$; 
        pre-trained legacy model $f_0$; \\
        \STATE \textcolor{white}{\textbf{Input}:} fixed nonconformity score function $E$; level $\alpha \in (0,1)$; 
        chosen ACP integration method; 
        \STATE \textcolor{white}{\textbf{Input}:} chosen baseline method (e.g., retrained model based conformal prediction);
        \STATE \textcolor{white}{\textbf{Input}:} 
        selection criterion (e.g., higher conditional coverage on unreliable samples);
        \STATE \textcolor{white}{\textbf{Input}:} 
        threshold $\epsilon\in[0,1]$; test significance level $\gamma$. 
        
        \STATE \textbf{(Model training)}
        \STATE Randomly split $\mathcal{D}_{\text{cal}}$ into disjoint subsets $\mathcal{D}_{\text{cal}}^1$ and $\mathcal{D}_{\text{cal}}^2$.
        \STATE Compute the audit outcome labels $R_i = \mathbbm{1}\{f_0^{\text{pt}}(X_i) = Y_i\}$ for $i \in \mathcal{D}_{\text{cal}}^1$. 
        \STATE Train audit model $\hat{r}$ on $\{(\tilde{X}_i, R_i)\}$ for $i \in \mathcal{D}_{\text{cal}}^1$, where $\tilde{X}_i$ aggregates the original features $X_i$ and additional features derived from $f_0$ as described in Section~\ref{sec:train_correctness}.
        \STATE Retrain the legacy model on $\{(X_i, Y_i)\}$ for $i \in \mathcal{D}_{\text{cal}}^1$.
        
        \STATE \textbf{(Data-driven model selection)}
        \STATE Randomly split $\mathcal{D}_{\text{cal}}^2$ into disjoint subsets $\mathcal{D}_{\text{cal-select}}^2$ and $\mathcal{D}_{\text{cal-calib}}^2$.
        \STATE Compute the unreliable group $\mathcal{H}$ according to \eqref{eq:aacp_groupdef} on $\mathcal{D}_{\text{cal-select}}^2$.
        \STATE Construct the candidate prediction sets $\hat{C}_{\text{ACP}}(X_i)$ and $\hat{C}_{\text{Retrain}}(X_i)$ for $i\in\mathcal H$ using $\mathcal{D}_{\text{cal-select}}^2$.
        \STATE Estimate the conditional coverages on $\mathcal{H}$ for ACP and Retrain according to \eqref{eq:aacp_estcondcov}. 
        \STATE Perform a one-sided $t$-test of $H_0$ against $H_1$, as defined in \eqref{eq:aacp_test}, at significance level $\gamma$.
        \IF{$H_0$ is rejected}
            \STATE Set $\mathcal{M}_{\text{sel}}=\text{ACP}$. 
        \ELSE
            \STATE Set $\mathcal{M}_{\text{sel}}=\text{Retrain}$.
        \ENDIF
        
        \STATE \textbf{(Conformal calibration)}
        \STATE Using $\mathcal{D}_{\text{cal-calib}}^2$, compute the nonconformity scores and empirical threshold using the selected method $\mathcal{M}_{\text{sel}}$.
        \STATE Construct the final prediction set $\hat{C}_{\text{AACP}}(X_{n+1})$ for the test point using the selected method $\mathcal{M}_{\text{sel}}$ and derived empirical threshold.
        \STATE \textbf{Output}: prediction set $\hat{C}_{\text{AACP}}(X_{n+1})$.
    \end{algorithmic}
\end{algorithm}

\FloatBarrier

\section{Additional Theorems and Proofs}\label{app:proofs}
\begin{proof}[Proof of Proposition~\ref{prop:consistency_ACP-MC_binary}]
Fix $x \in \mathcal{X}$. By definition, $r^*(x) = \mathbb{P}_T(Y = f_0^{\text{pt}}(x) \mid X = x)$, which implies $\mathbb{P}_T(Y = 1 - f_0^{\text{pt}}(x) \mid X = x) = 1 - r^*(x)$. Thus, evaluating equation \eqref{eq:binary_comb_rule} at $x$ with the oracle audit score $r^*(x)$ yields:
\[
m(x; r^*, f_0) = \big(\mathbb{P}_T(Y=1 \mid X=x), \mathbb{P}_T(Y=2 \mid X=x)\big).
\]

Define the function $g_x : [0,1] \to [0,1]^2$ mapping a scalar score $t$ to the combined probability vector:
\[
g_x(t) := 
\begin{cases} 
    (1-t, t), & \text{if } f_0^{\text{pt}}(x) = 2, \\ 
    (t, 1-t), & \text{if } f_0^{\text{pt}}(x) = 1. 
\end{cases}
\]

Let $\hat r_N$ denote the audit model trained on $N$ i.i.d.\ samples from $P_T$. By definition, $m(x; \hat{r}_N, f_0) = g_x(\hat{r}_N(x))$. Because the legacy model's point prediction $f_0^{\text{pt}}(x)$ are fixed, $g_x$ is affine in $t$ and therefore continuous.

By the assumption that $\hat{r}_N(x) \xrightarrow{p} r^*(x)$, the continuous mapping theorem gives:
\[
m(x; \hat{r}_N, f_0) = g_x(\hat{r}_N(x)) \xrightarrow{p} g_x(r^*(x)) = m(x; r^*, f_0).
\]

Combining this with our initial evaluation yields:
\[
m(x; \hat{r}_N, f_0) \xrightarrow{p} \big(\mathbb{P}_T(Y=0 \mid X=x), \mathbb{P}_T(Y=1 \mid X=x)\big) = P_T(Y \mid X=x).
\]
\end{proof}

\begin{proposition}
\label{prop:consistency_apa_multiclass}
In the multiclass classification setting, define $m(x; \hat r_n, \eta_n, f_0)$ as in \eqref{eq:multiclass-comb-rule}. Assume $\hat r_N$ and $\eta_N$, trained on $N$ i.i.d.\ samples from $P_T$, converges pointwise in probability to $r^*(x)$ and $\eta^*$ in the large-$N$ limit; i.e., $\hat r_N(x)\xrightarrow{p} r^*(x)$ and $\hat \eta_N(x)\xrightarrow{p} \eta^*(x)$
for any $x \in \mathcal X$. Then, for any fixed legacy model $f_0$,
\[
m(x; \hat r_N,f_0)\xrightarrow{p} P_T(Y\mid X=x),
\]
where $P_T(Y\mid X=x) = \left( \mathbb{P}[Y=1\mid X=x], \dots, \mathbb{P}[Y=K\mid X=x] \right)$ under $P_T$.
\end{proposition}

\begin{proof}[Proof of Proposition~\ref{prop:consistency_apa_multiclass}]
Fix $x \in \mathcal{X}$. By definition, the oracle audit score is $r^*(x) = \mathbb{P}_T(Y = f_0^{\text{pt}}(x) \mid X = x)$, which implies $\mathbb{P}_T(Y \neq f_0^{\text{pt}}(x) \mid X = x) = 1 - r^*(x)$. 

Evaluating equation \eqref{eq:multiclass-comb-rule} at $x$ with the oracle quantities $r^*(x)$ and $\eta^*(x)$ for any class $k \in [K]$ yields two cases:

If $k = f_0^{\text{pt}}(x)$, by definition $\eta^*_k(x) = 0$, so:
\[
m_k(x; r^*, \eta^*, f_0) = r^*(x) \cdot 1 + \big(1 - r^*(x)\big) \cdot 0 = \mathbb{P}_T(Y = k \mid X = x).
\]

If $k \neq f_0^{\text{pt}}(x)$, we note that the event $\{Y = k\}$ is a subset of $\{Y \neq f_0^{\text{pt}}(x)\}$. Thus, by the definition of conditional probability:
\[
\mathbb{P}_T(Y = k \mid X = x) = \mathbb{P}_T(Y \neq f_0^{\text{pt}}(x) \mid X = x) \, \mathbb{P}_T(Y = k \mid X = x, \, Y \neq f_0^{\text{pt}}(x)).
\]
Substituting $1 - r^*(x)$ and $\eta^*_k(x)$ into the right side yields:
\[
\mathbb{P}_T(Y = k \mid X = x) = \big(1 - r^*(x)\big) \eta^*_k(x).
\]
Therefore, the multiclass rule gives:
\[
m_k(x; r^*, \eta^*, f_0) = r^*(x) \cdot 0 + \big(1 - r^*(x)\big) \eta^*_k(x) = \mathbb{P}_T(Y = k \mid X = x).
\]
Since this holds for all $k \in [K]$, the combined vector recovers the true conditional distribution: $m(x; r^*, \eta^*, f_0) = \left( P_T( Y =1\mid X=x) ,\dots,P_T( Y =K\mid X=x) \right) = P_T(Y \mid X = x)$.

Next, define the function $g_x : [0,1] \times [0,1]^K \to [0,1]^K$, where for a scalar $t$ and a vector $e \in [0,1]^K$, the $k$-th component of $g_x$ is:
\[
\big[g_x(t, e)\big]_k := t \mathbbm{1}\{k = f_0^{\text{pt}}(x)\} + (1-t)e_k.
\]
Let $\hat r_N$ and $\hat{\eta}$ denote the audit model and the reallocation rule trained on $N$ i.i.d.\ samples from $P_T$. By definition, $m(x; \hat{r}_N, \hat{\eta}_N, f_0) = g_x(\hat{r}_N(x), \hat{\eta}_N(x))$. Because $f_0^{\text{pt}}(x)$ is fixed for this pointwise evaluation, $g_x(t, e)$ is a fixed bilinear function of its inputs $t$ and $e$, meaning it is continuous.

By the assumption that $\hat{r}_N(x) \xrightarrow{p} r^*(x)$ and $\hat{\eta}_N(x) \xrightarrow{p} \eta^*(x)$, the random vector $(\hat{r}_N(x), \hat{\eta}_N(x))$ converges in probability to $(r^*(x), \eta^*(x))$. The continuous mapping theorem then gives:
\[
m(x; \hat{r}_N, \hat{\eta}_N, f_0) = g_x(\hat{r}_N(x), \hat{\eta}_N(x)) \xrightarrow{p} g_x(r^*(x), \eta^*(x)) = m(x; r^*, \eta^*, f_0).
\]
This yields:
\[
m(x; \hat{r}_N, \hat{\eta}_N, f_0) \xrightarrow{p} P_T(Y \mid X = x).
\]
\end{proof}

\begin{proposition}\label{prop:implication}
If a prediction set $\hat{C}(X)$ satisfies feature-conditional coverage at level $1-\alpha$ (i.e., $\mathbb{P}[ Y \in \hat{C}(X) \mid X] \geq 1-\alpha$), then $\hat{C}(X)$ necessarily satisfies audited conditional coverage defined in equation \ref{eq:audited-cond-coverage}.
\end{proposition}

\begin{proof}[Proof of Proposition~\ref{prop:implication}]
Fix any bin \(R\in\mathcal R\) such that \(\mathbb P(r^*(X)\in R)>0\). The tower property gives
\[
\mathbb P\big[Y\in \hat C(X)\mid r^*(X)\in R\big]
=
\mathbb E\!\left[
\mathbb P\big[Y\in \hat C(X)\mid X\big]
\,\middle|\,
r^*(X)\in R
\right].
\]
By feature-conditional coverage,
\[
\mathbb P\big[Y\in \hat C(X)\mid X\big]\ge 1-\alpha
\qquad \text{almost surely}
\]
By monotonicity of conditional expectation,
\[
\mathbb E\!\left[
\mathbb P\big[Y\in \hat C(X)\mid X\big]
\,\middle|\,
r^*(X)\in R
\right]
\ge
\mathbb E\!\left[1-\alpha \mid r^*(X)\in R\right]
=
1-\alpha.
\]
Therefore,
\[
\mathbb P\big(Y\in \hat C(X)\mid r^*(X)\in R\big)\ge 1-\alpha.
\]
Since \(R\in\mathcal R\) was arbitrary, \(\hat C(X)\) satisfies audited conditional coverage.
\end{proof}

\begin{proof}[Proof of Theorem~\ref{thm:condconf_oracle}]
Fix $j\in[m]$ and $g\in\{0,1\}$. Define the events
\[
F:=\{Y_{n+1}\notin \hat C(X_{n+1})\},\qquad
B_g:=\{G_j^*(X_{n+1})=g\},\qquad
C_g:=\{\hat G_j(X_{n+1})=g\}.
\]
Let $p_{j,g}:=\mathbb P(B_g)$, which is positive by assumption. Then
\[
\mathbb P(F\cap B_g)
=
\mathbb P(F\cap B_g\cap C_g)
+
\mathbb P(F\cap B_g\cap C_g^c)
\le
\mathbb P(F\cap C_g)
+
\mathbb P(B_g\cap C_g^c).
\]
Dividing by $p_{j,g}$ gives
\[
\mathbb P(F\mid B_g)
\le
\frac{\mathbb P(F\cap C_g)}{p_{j,g}}
+
\frac{\mathbb P(B_g\cap C_g^c)}{p_{j,g}}.
\]

By Corollary~1 of \citet{gibbs2025conditional}, applied to the groups induced by the fixed fitted audit model $\hat r$,
\[
\mathbb P(F\mid C_g)\le \alpha.
\]
Equivalently,
\[
\mathbb P(F\cap C_g)\le \alpha \mathbb P(C_g).
\]
Therefore,
\[
\mathbb P(F\mid B_g)
\le
\frac{\alpha \mathbb P(C_g)}{p_{j,g}}
+
\frac{\mathbb P(B_g\cap C_g^c)}{p_{j,g}}.
\]

Let
\[
\delta_j := \mathbb P\!\left[\hat G_j(X_{n+1})\neq G_j^*(X_{n+1})\right].
\]
Since $C_g\setminus B_g\subseteq \{\hat G_j(X_{n+1})\neq G_j^*(X_{n+1})\}$ and
$B_g\cap C_g^c\subseteq \{\hat G_j(X_{n+1})\neq G_j^*(X_{n+1})\}$, we have
\[
\mathbb P(C_g)
\le
\mathbb P(B_g)+\delta_j
=
p_{j,g}+\delta_j,
\qquad
\mathbb P(B_g\cap C_g^c)\le \delta_j.
\]
Hence,
\[
\mathbb P(F\mid B_g)
\le
\alpha+\frac{(1+\alpha)\delta_j}{p_{j,g}}.
\]
Equivalently,
\[
\mathbb P\!\left[
Y_{n+1}\in \hat C(X_{n+1})
\,\middle|\,
G_j^*(X_{n+1})=g
\right]
\ge
1-\alpha
-
\frac{(1+\alpha)
\mathbb P[\hat G_j(X_{n+1})\neq G_j^*(X_{n+1})]}
{\mathbb P[G_j^*(X_{n+1})=g]}.
\]

It remains to show that the disagreement probability vanishes under the stated consistency condition. Let $\hat r_N$ denote the audit model trained on $N$ i.i.d.\ samples from $P_T$, and define
\[
\hat G_{j,N}(X_{n+1})
=
\mathbbm{1}\{\hat r_N(X_{n+1})\le t_j\}.
\]
If $\hat G_{j,N}(X_{n+1})\neq G_j^*(X_{n+1})$, then
\[
\bigl|\hat r_N(X_{n+1})-r^*(X_{n+1})\bigr|
\ge
\bigl|r^*(X_{n+1})-t_j\bigr|.
\]
Therefore, for any $\varepsilon>0$,
\[
\left\{\hat G_{j,N}(X_{n+1})\neq G_j^*(X_{n+1})\right\}
\subseteq
\left\{\bigl|\hat r_N(X_{n+1})-r^*(X_{n+1})\bigr|>\varepsilon\right\}
\cup
\left\{\bigl|r^*(X_{n+1})-t_j\bigr|\le \varepsilon\right\}.
\]
Taking probabilities gives
\[
\mathbb P\!\left[
\hat G_{j,N}(X_{n+1})\neq G_j^*(X_{n+1})
\right]
\le
\mathbb P\!\left(
\bigl|\hat r_N(X_{n+1})-r^*(X_{n+1})\bigr|>\varepsilon
\right)
+
\mathbb P\!\left(
\bigl|r^*(X_{n+1})-t_j\bigr|\le \varepsilon
\right).
\]

For the first term, define
\[
q_N(x)
:=
\mathbb P\!\left(
|\hat r_N(x)-r^*(x)|>\varepsilon
\right),
\]
where the probability is over the randomness in the training data used to fit $\hat r_N$. By the assumed pointwise convergence in probability, $q_N(x)\to 0$ for every $x\in\mathcal X$. Since $0\le q_N(x)\le 1$, dominated convergence implies
\[
\mathbb P\!\left(
|\hat r_N(X_{n+1})-r^*(X_{n+1})|>\varepsilon
\right)
=
\mathbb E[q_N(X_{n+1})]
\to 0.
\]
The second term does not depend on $N$, and the no-boundary-mass assumption gives
\[
\lim_{\varepsilon\downarrow 0}
\mathbb P\!\left(
|r^*(X_{n+1})-t_j|\le \varepsilon
\right)
=0.
\]
Thus, for any $\xi>0$, we may first choose $\varepsilon>0$ so that the second term is at most $\xi/2$, and then choose $N$ large enough so that the first term is at most $\xi/2$. Therefore,
\[
\mathbb P\!\left[
\hat G_{j,N}(X_{n+1})\neq G_j^*(X_{n+1})
\right]\to 0.
\]
\end{proof}

\FloatBarrier

\section{Implementation Details and Additional Numerical Experiments}\label{app:add_experiments}

\subsection{Implementation Details}
\paragraph{Model architecture and training details for the legacy model.}

For synthetic experiments, the legacy model is a five-layer fully connected network with ReLU activations and a softmax output layer, trained using the Adam optimizer with learning rate $10^{-3}$, batch size $10$, and weight decay $10^{-4}$, using \texttt{BCEWithLogitsLoss} for binary classification and cross-entropy loss for multiclass classification. Training converges within $30$ epochs in all cases.

For the camelyon17 real data experiments, the legacy model follows the configurations and hyperparameters of \citet{wilds2021}, using the DenseNet-121 architecture, trained using SGD with momentum $0.9$, learning rate $10^{-3}$, weight decay $10^{-2}$, batch size $32$, and cross-entropy loss. 

For the CIFAR-10 and CIFAR-10-C experiments, we follow the standard training protocol \citep{he2016image, zagoruyko2017wideresidualnetworks}: a ResNet-18 trained for 100 epochs with SGD (momentum 0.9, weight decay 5e-4), batch size 128, learning rate 0.1 with cosine annealing \citep{loshchilov2018decoupled}, and augmentations with RandomCrop with 4-pixel padding and RandomHorizontalFlip.

The retraining baseline uses the same architecture and hyperparameters as the legacy model for both synthetic and real experiments. Unless otherwise stated, the target miscoverage level is $\alpha = 0.1$.

\paragraph{Data splitting details.}
Several methods---including adaptive temperature scaling (AdaTS), retraining, and our ACP---require additional splits of the calibration set. For all of these methods, we randomly split \(\mathcal{D}_{\text{cal}}\) into two equal parts. The first half, \(\mathcal{D}_{\text{cal}}^1\), is used to fit the method-specific learning component. In particular, we fit the temperature parameter \(T\) for AdaTS, retrain a model using the same neural network architecture as the original legacy model for the retraining baseline, and train an audit model for ACP. The second half, \(\mathcal{D}_{\text{cal}}^2\), is then used to compute nonconformity scores and estimate the empirical quantile threshold.

For AACP, we first split \(\mathcal{D}_{\text{cal}}\) into two equal subsets, \(\mathcal{D}_{\text{cal}}^1\) and \(\mathcal{D}_{\text{cal}}^2\), as above. The first subset is used to train the audit models \(\hat{r}\) (and \(\hat{\eta}\) for multiclass ACP-MC), as well as to retrain the legacy model from scratch. The second subset, \(\mathcal{D}_{\text{cal}}^2\), is further partitioned into \(\mathcal{D}_{\text{cal-select}}^2\) and \(\mathcal{D}_{\text{cal-calib}}^2\) in a \(25\%/75\%\) split: the smaller portion is used for model selection, while the remaining portion is used to compute nonconformity scores and calibrate the final conformal threshold. For the adaptive selection procedure, we set the threshold \(\epsilon = 0.7\) when partitioning the data used to evaluate the selection criterion. This relatively conservative choice is used uniformly across experiments to prioritize recall, with the goal of capturing as many genuinely hard-to-predict samples as possible.

\paragraph{Model architecture for the audit model.}

The audit model $\hat{r}$ is a random forest with $300$ trees and maximum depth $12$, calibrated via isotonic regression, and trained on binary labels indicating whether the legacy model correctly predicts each sample; for multiclass ACP-MC, $\hat{\eta}$ is additionally fit via multinomial logistic regression with $\ell_2$ regularization, as described in Appendix~\ref{app:APA_mutlclass}. For both synthetic and real data experiments, we use the following feature groups together with the input covariates (raw features for tabular data such as in the synthetic experiments, and embeddings for image data such as for the camelyon17 expeirments):

\begin{itemize}
    \item \textbf{Basic predictive signals}: maximum softmax probability, predictive entropy, top-two probability margin, and second-highest class probability.
    \item \textbf{Distributional shift signals}: Isolation Forest anomaly score, its negation, and a binary in/out-of-distribution indicator, all computed relative to the historical training data.
    \item \textbf{Embedding-space distance features}: distance to the overall historical centroid, minimum distance to a historical sample, and average distance to the $5$ nearest historical neighbors, all measured in the legacy model's penultimate-layer embedding space.
    \item \textbf{Class-conditional embedding features}: distance to the predicted class centroid, distance to the nearest alternative class centroid, their ratio (a trust-score proxy), KNN distances to the predicted and nearest alternative classes, and distances to all $K$ class centroids.
    \item \textbf{Uncertainty estimates}: for models supporting MC Dropout, predictive variance and mean entropy across stochastic forward passes; otherwise, tree-level disagreement variance, agreement rate, and prediction entropy across ensemble members.
    \item \textbf{Embedding-space density features}: local density and average local density estimated from the $k$-nearest-neighbor distances in the legacy model's embedding space.
\end{itemize}

We emphasize that this feature set, while effective, reflects one particular instantiation of a general and flexible toolbox. In particular, although we leverage the historical training data here to compute distributional shift signals and class-conditional statistics, access to historical data is \emph{not} required to apply our method: practitioners can construct a useful audit model using only legacy-model-derived signals (e.g., softmax probabilities, entropy, and confidence margin), which are always available at inference time.

\paragraph{Hyperparameter of ACP.} Unless otherwise specified, for the multiclass ACP-MC method we use $\hat{\eta}_{\text{Counts}}$ (described in Appendix~\ref{app:APA_mutlclass}), and for the ACP-AEC method we set the single threshold to $0.7$. For the ACP-ACC method we set the threshold set to be $\{ 0.5, 0.6, 0.7, 0.8, 0.9\}$. 

\paragraph{Compute resources.} The numerical experiments described in this paper were carried out on a computing cluster. Real-data experiments (Camelyon17 and CIFAR-10/CIFAR-10-C), each involving up to $2{,}000$ calibration samples and $1{,}000$ test samples, averaged about $30$ minutes per run on a single NVIDIA P100 GPU with $16$ GB of memory. Across all configurations and $30$ random seeds, the real-data experiments consumed approximately $400$ GPU-hours.

Synthetic-data experiments, each involving up to $5{,}000$ calibration samples and $1{,}000$ test samples, averaged about $10$ minutes per run on a single CPU node ($4$ cores, no GPU) with $16$ GB of memory. Across all configurations and $100$ random seeds, the synthetic experiments consumed approximately $150$ CPU-hours.

These totals do not include preliminary or failed experiments.

\subsection{Additional Numerical Results}
\subsubsection{Synthetic data experiments}

\begin{table}[!htb]
\centering
    \caption{Performance of conformal prediction sets for $5$-class synthetic data generated from Example~\ref{exp:concept-shift}, as a function of the total calibration sample size. All methods achieve the target marginal coverage of \(90\%\), while our methods (ACP-MC, ACP-AEC, and ACP-ACC) retain practically efficient set sizes and are more uncertainty-aware, assigning larger sets to more unreliable samples and smaller sets to more reliable ones. Red numbers indicate the three smallest average set sizes and the three highest conditional coverages. See the corresponding plots in Figure~\ref{fig:main_exp_sim_nnew_overall}.}
  \label{tab:main_exp_sim_ndata}
\centering
\fontsize{6}{6}\selectfont
\begin{tabular}[t]{rlllllll}
\toprule
\multicolumn{2}{c}{ } & \multicolumn{2}{c}{\makecell{Marginal}} & \multicolumn{2}{c}{\makecell{Hard Bin \\ ($r^* \leq 0.5$)}} & \multicolumn{2}{c}{\makecell{Easy Bin \\ ($r^* > 0.5$)}} \\
\cmidrule(l{3pt}r{3pt}){3-4} \cmidrule(l{3pt}r{3pt}){5-6} \cmidrule(l{3pt}r{3pt}){7-8}
\makecell{Number \\ of \\ Calibration \\ Samples} & Method & Coverage & Size & Coverage & Size & Coverage & Size \\
\midrule
200 & ACP-ACC & \makecell{0.917 \\ (0.005)} & \makecell{2.215 \\ (0.084)} & \textcolor{red}{\makecell{0.741 \\ (0.021)}} & \makecell{3.447 \\ (0.131)} & \makecell{0.955 \\ (0.005)} & \makecell{1.956 \\ (0.090)} \\
 & ACP-AEC & \makecell{0.902 \\ (0.005)} & \textcolor{red}{\makecell{1.537 \\ (0.054)}} & \makecell{0.562 \\ (0.019)} & \makecell{2.224 \\ (0.115)} & \makecell{0.977 \\ (0.004)} & \makecell{1.390 \\ (0.050)} \\
 & ACP-MC & \makecell{0.903 \\ (0.005)} & \makecell{1.914 \\ (0.049)} & \textcolor{red}{\makecell{0.668 \\ (0.017)}} & \makecell{3.068 \\ (0.080)} & \makecell{0.953 \\ (0.004)} & \makecell{1.664 \\ (0.045)} \\
 & AdaTS & \makecell{0.903 \\ (0.005)} & \textcolor{red}{\makecell{1.570 \\ (0.065)}} & \makecell{0.484 \\ (0.021)} & \makecell{1.655 \\ (0.060)} & \makecell{0.992 \\ (0.002)} & \makecell{1.551 \\ (0.066)} \\
 & Retrain & \makecell{0.888 \\ (0.007)} & \makecell{3.779 \\ (0.053)} & \textcolor{red}{\makecell{0.811 \\ (0.010)}} & \makecell{3.800 \\ (0.055)} & \makecell{0.904 \\ (0.007)} & \makecell{3.774 \\ (0.053)} \\
 & Standard & \makecell{0.894 \\ (0.004)} & \textcolor{red}{\makecell{1.243 \\ (0.038)}} & \makecell{0.451 \\ (0.017)} & \makecell{1.449 \\ (0.036)} & \makecell{0.988 \\ (0.002)} & \makecell{1.198 \\ (0.039)} \\
\midrule
500 & ACP-ACC & \makecell{0.919 \\ (0.004)} & \makecell{1.968 \\ (0.047)} & \textcolor{red}{\makecell{0.849 \\ (0.017)}} & \makecell{3.938 \\ (0.133)} & \makecell{0.936 \\ (0.004)} & \makecell{1.545 \\ (0.040)} \\
 & ACP-AEC & \makecell{0.902 \\ (0.004)} & \textcolor{red}{\makecell{1.563 \\ (0.040)}} & \makecell{0.687 \\ (0.023)} & \makecell{3.137 \\ (0.164)} & \makecell{0.950 \\ (0.004)} & \makecell{1.250 \\ (0.021)} \\
 & ACP-MC & \makecell{0.898 \\ (0.003)} & \makecell{1.791 \\ (0.026)} & \textcolor{red}{\makecell{0.777 \\ (0.010)}} & \makecell{3.547 \\ (0.053)} & \makecell{0.924 \\ (0.003)} & \makecell{1.417 \\ (0.021)} \\
 & AdaTS & \makecell{0.899 \\ (0.003)} & \textcolor{red}{\makecell{1.477 \\ (0.039)}} & \makecell{0.454 \\ (0.016)} & \makecell{1.559 \\ (0.034)} & \makecell{0.994 \\ (0.001)} & \makecell{1.460 \\ (0.039)} \\
 & Retrain & \makecell{0.894 \\ (0.003)} & \makecell{2.641 \\ (0.047)} & \textcolor{red}{\makecell{0.700 \\ (0.009)}} & \makecell{2.701 \\ (0.046)} & \makecell{0.935 \\ (0.003)} & \makecell{2.628 \\ (0.047)} \\
 & Standard & \makecell{0.897 \\ (0.003)} & \textcolor{red}{\makecell{1.210 \\ (0.017)}} & \makecell{0.449 \\ (0.011)} & \makecell{1.427 \\ (0.019)} & \makecell{0.992 \\ (0.001)} & \makecell{1.163 \\ (0.016)} \\
\midrule
1000 & ACP-ACC & \makecell{0.918 \\ (0.003)} & \makecell{1.937 \\ (0.039)} & \textcolor{red}{\makecell{0.902 \\ (0.013)}} & \makecell{4.324 \\ (0.108)} & \makecell{0.921 \\ (0.004)} & \makecell{1.410 \\ (0.028)} \\
 & ACP-AEC & \makecell{0.914 \\ (0.004)} & \makecell{1.804 \\ (0.036)} & \textcolor{red}{\makecell{0.848 \\ (0.018)}} & \makecell{4.173 \\ (0.124)} & \makecell{0.929 \\ (0.003)} & \makecell{1.287 \\ (0.018)} \\
 & ACP-MC & \makecell{0.901 \\ (0.002)} & \textcolor{red}{\makecell{1.764 \\ (0.019)}} & \textcolor{red}{\makecell{0.840 \\ (0.005)}} & \makecell{3.871 \\ (0.025)} & \makecell{0.913 \\ (0.002)} & \makecell{1.311 \\ (0.015)} \\
 & AdaTS & \makecell{0.896 \\ (0.003)} & \textcolor{red}{\makecell{1.410 \\ (0.026)}} & \makecell{0.433 \\ (0.014)} & \makecell{1.505 \\ (0.022)} & \makecell{0.994 \\ (0.001)} & \makecell{1.390 \\ (0.026)} \\
 & Retrain & \makecell{0.895 \\ (0.002)} & \makecell{2.418 \\ (0.043)} & \makecell{0.676 \\ (0.009)} & \makecell{2.605 \\ (0.047)} & \makecell{0.942 \\ (0.003)} & \makecell{2.377 \\ (0.042)} \\
 & Standard & \makecell{0.897 \\ (0.002)} & \textcolor{red}{\makecell{1.194 \\ (0.013)}} & \makecell{0.448 \\ (0.011)} & \makecell{1.417 \\ (0.015)} & \makecell{0.992 \\ (0.001)} & \makecell{1.146 \\ (0.011)} \\
\midrule
2000 & ACP-ACC & \makecell{0.920 \\ (0.003)} & \makecell{1.906 \\ (0.038)} & \textcolor{red}{\makecell{0.928 \\ (0.009)}} & \makecell{4.498 \\ (0.095)} & \makecell{0.918 \\ (0.002)} & \makecell{1.309 \\ (0.020)} \\
 & ACP-AEC & \makecell{0.923 \\ (0.002)} & \makecell{1.895 \\ (0.020)} & \textcolor{red}{\makecell{0.917 \\ (0.004)}} & \makecell{4.598 \\ (0.018)} & \makecell{0.924 \\ (0.002)} & \makecell{1.280 \\ (0.012)} \\
 & ACP-MC & \makecell{0.902 \\ (0.002)} & \textcolor{red}{\makecell{1.754 \\ (0.017)}} & \textcolor{red}{\makecell{0.868 \\ (0.006)}} & \makecell{3.996 \\ (0.018)} & \makecell{0.909 \\ (0.002)} & \makecell{1.260 \\ (0.011)} \\
 & AdaTS & \makecell{0.900 \\ (0.002)} & \textcolor{red}{\makecell{1.445 \\ (0.023)}} & \makecell{0.455 \\ (0.013)} & \makecell{1.533 \\ (0.020)} & \makecell{0.997 \\ (0.001)} & \makecell{1.427 \\ (0.023)} \\
 & Retrain & \makecell{0.901 \\ (0.002)} & \makecell{2.258 \\ (0.048)} & \makecell{0.697 \\ (0.009)} & \makecell{2.710 \\ (0.053)} & \makecell{0.945 \\ (0.002)} & \makecell{2.159 \\ (0.049)} \\
 & Standard & \makecell{0.900 \\ (0.002)} & \textcolor{red}{\makecell{1.208 \\ (0.011)}} & \makecell{0.465 \\ (0.010)} & \makecell{1.441 \\ (0.014)} & \makecell{0.994 \\ (0.001)} & \makecell{1.155 \\ (0.009)} \\
\midrule
5000 & ACP-ACC & \makecell{0.920 \\ (0.003)} & \makecell{1.854 \\ (0.098)} & \textcolor{red}{\makecell{0.907 \\ (0.011)}} & \makecell{4.413 \\ (0.087)} & \makecell{0.923 \\ (0.003)} & \makecell{1.306 \\ (0.111)} \\
 & ACP-AEC & \makecell{0.922 \\ (0.001)} & \makecell{1.827 \\ (0.017)} & \textcolor{red}{\makecell{0.913 \\ (0.004)}} & \makecell{4.590 \\ (0.019)} & \makecell{0.923 \\ (0.002)} & \makecell{1.222 \\ (0.008)} \\
 & ACP-MC & \makecell{0.900 \\ (0.002)} & \textcolor{red}{\makecell{1.705 \\ (0.014)}} & \textcolor{red}{\makecell{0.870 \\ (0.005)}} & \makecell{4.014 \\ (0.013)} & \makecell{0.906 \\ (0.002)} & \makecell{1.194 \\ (0.007)} \\
 & AdaTS & \makecell{0.898 \\ (0.002)} & \textcolor{red}{\makecell{1.437 \\ (0.023)}} & \makecell{0.446 \\ (0.013)} & \makecell{1.526 \\ (0.021)} & \makecell{0.995 \\ (0.001)} & \makecell{1.418 \\ (0.023)} \\
 & Retrain & \makecell{0.901 \\ (0.001)} & \makecell{1.736 \\ (0.039)} & \makecell{0.732 \\ (0.009)} & \makecell{2.745 \\ (0.058)} & \makecell{0.938 \\ (0.002)} & \makecell{1.520 \\ (0.038)} \\
 & Standard & \makecell{0.899 \\ (0.002)} & \textcolor{red}{\makecell{1.206 \\ (0.013)}} & \makecell{0.461 \\ (0.011)} & \makecell{1.437 \\ (0.015)} & \makecell{0.994 \\ (0.001)} & \makecell{1.154 \\ (0.011)} \\
\bottomrule
\end{tabular}

\end{table}

\begin{figure*}[tbh]
    \centering
    \includegraphics[width=0.9\linewidth]{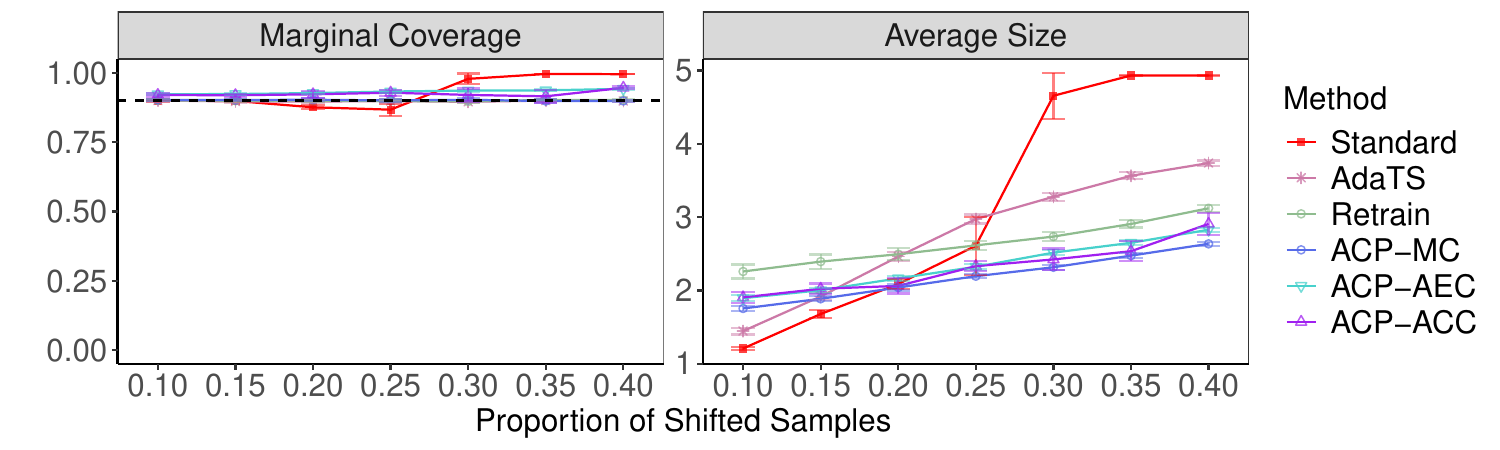}
    \caption{Performance of prediction sets for $5$-class synthetic data generated from Example~\ref{exp:concept-shift} as a function of the proportion of shifted samples with intrinsic uncertainty. All methods achieve the target marginal coverage of $90\%$, while our methods retain practically efficient sets sizes. Error bars denote two standard errors. Corresponding conditional performance are provided in Figure~\ref{fig:supp_exp_sim_beta_cond}. }
    \label{fig:supp_exp_sim_beta_marg}
\end{figure*}

\begin{figure*}[tbh]
    \centering
    \includegraphics[width=0.9\linewidth]{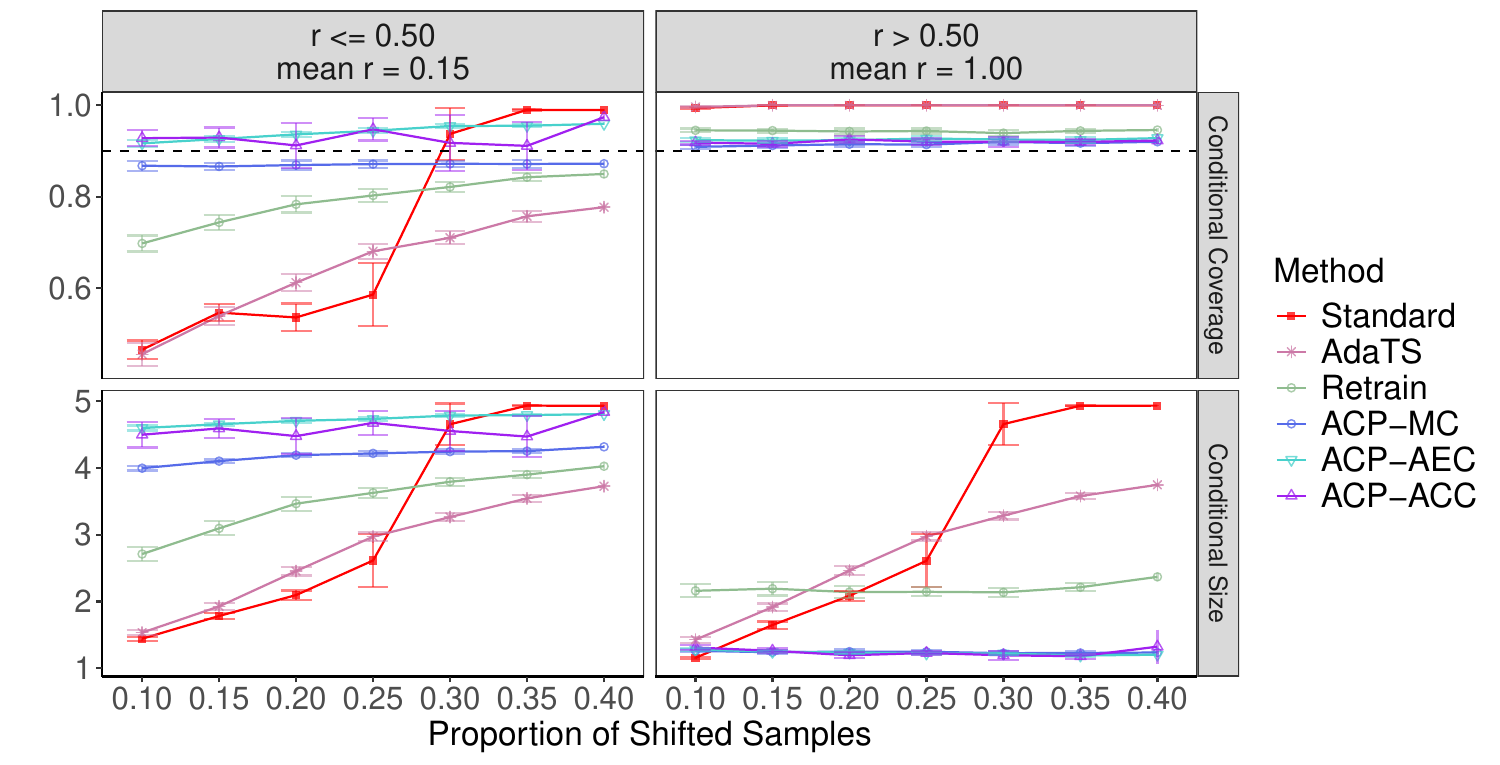}
    \caption{Our methods are more uncertainty-aware: they attain higher conditional coverage on hard samples while assigning larger prediction sets to hard samples and smaller ones to easy samples. Corresponding marginal performance are provided in Figure~\ref{fig:supp_exp_sim_beta_marg} and numerical details are provided in Table~\ref{tab:main_exp_sim_beta}.}
    \label{fig:supp_exp_sim_beta_cond}
\end{figure*}

\begin{table}[!htb]
\centering
    \caption{Performance of conformal prediction sets for $5$-class synthetic data generated from Example~\ref{exp:concept-shift}, as a function the proportion of shifted samples. See the corresponding plots in Figure~\ref{fig:supp_exp_sim_beta_marg}--\ref{fig:supp_exp_sim_beta_cond}.}
  \label{tab:main_exp_sim_beta}
\centering
\fontsize{6}{6}\selectfont
\begin{tabular}[t]{rlllllll}
\toprule
\multicolumn{2}{c}{ } & \multicolumn{2}{c}{\makecell{Marginal}} & \multicolumn{2}{c}{\makecell{Hard Bin \\ ($r^* \leq 0.5$)}} & \multicolumn{2}{c}{\makecell{Easy Bin \\ ($r^* > 0.5$)}} \\
\cmidrule(l{3pt}r{3pt}){3-4} \cmidrule(l{3pt}r{3pt}){5-6} \cmidrule(l{3pt}r{3pt}){7-8}
\makecell{Proportion \\ of \\ Shifted \\ Samples} & Method & Coverage & Size & Coverage & Size & Coverage & Size \\
\midrule
0.1 & ACP-ACC & \makecell{0.920 \\ (0.003)} & \makecell{1.906 \\ (0.038)} & \textcolor{red}{\makecell{0.928 \\ (0.009)}} & \makecell{4.498 \\ (0.095)} & \makecell{0.918 \\ (0.002)} & \makecell{1.309 \\ (0.020)} \\
 & ACP-AEC & \makecell{0.923 \\ (0.002)} & \makecell{1.895 \\ (0.020)} & \textcolor{red}{\makecell{0.917 \\ (0.004)}} & \makecell{4.598 \\ (0.018)} & \makecell{0.924 \\ (0.002)} & \makecell{1.280 \\ (0.012)} \\
 & ACP-MC & \makecell{0.902 \\ (0.002)} & \textcolor{red}{\makecell{1.754 \\ (0.017)}} & \textcolor{red}{\makecell{0.868 \\ (0.006)}} & \makecell{3.996 \\ (0.018)} & \makecell{0.909 \\ (0.002)} & \makecell{1.260 \\ (0.011)} \\
 & AdaTS & \makecell{0.900 \\ (0.002)} & \textcolor{red}{\makecell{1.445 \\ (0.023)}} & \makecell{0.455 \\ (0.013)} & \makecell{1.533 \\ (0.020)} & \makecell{0.997 \\ (0.001)} & \makecell{1.427 \\ (0.023)} \\
 & Retrain & \makecell{0.901 \\ (0.002)} & \makecell{2.258 \\ (0.048)} & \makecell{0.697 \\ (0.009)} & \makecell{2.710 \\ (0.053)} & \makecell{0.945 \\ (0.002)} & \makecell{2.159 \\ (0.049)} \\
 & Standard & \makecell{0.900 \\ (0.002)} & \textcolor{red}{\makecell{1.208 \\ (0.011)}} & \makecell{0.465 \\ (0.010)} & \makecell{1.441 \\ (0.014)} & \makecell{0.994 \\ (0.001)} & \makecell{1.155 \\ (0.009)} \\
\midrule
0.15 & ACP-ACC & \makecell{0.919 \\ (0.003)} & \makecell{2.023 \\ (0.035)} & \textcolor{red}{\makecell{0.930 \\ (0.011)}} & \makecell{4.593 \\ (0.071)} & \makecell{0.916 \\ (0.003)} & \makecell{1.261 \\ (0.017)} \\
 & ACP-AEC & \makecell{0.924 \\ (0.002)} & \makecell{2.003 \\ (0.017)} & \textcolor{red}{\makecell{0.926 \\ (0.003)}} & \makecell{4.655 \\ (0.012)} & \makecell{0.923 \\ (0.002)} & \makecell{1.246 \\ (0.011)} \\
 & ACP-MC & \makecell{0.902 \\ (0.002)} & \textcolor{red}{\makecell{1.886 \\ (0.016)}} & \textcolor{red}{\makecell{0.866 \\ (0.004)}} & \makecell{4.102 \\ (0.013)} & \makecell{0.912 \\ (0.002)} & \makecell{1.239 \\ (0.013)} \\
 & AdaTS & \makecell{0.897 \\ (0.002)} & \textcolor{red}{\makecell{1.918 \\ (0.026)}} & \makecell{0.539 \\ (0.010)} & \makecell{1.926 \\ (0.026)} & \makecell{1.000 \\ (0.000)} & \makecell{1.915 \\ (0.026)} \\
 & Retrain & \makecell{0.899 \\ (0.002)} & \makecell{2.393 \\ (0.049)} & \makecell{0.744 \\ (0.008)} & \makecell{3.095 \\ (0.052)} & \makecell{0.945 \\ (0.002)} & \makecell{2.192 \\ (0.051)} \\
 & Standard & \makecell{0.899 \\ (0.002)} & \textcolor{red}{\makecell{1.678 \\ (0.026)}} & \makecell{0.546 \\ (0.009)} & \makecell{1.782 \\ (0.022)} & \makecell{1.000 \\ (0.000)} & \makecell{1.647 \\ (0.027)} \\
\midrule
0.2 & ACP-ACC & \makecell{0.923 \\ (0.006)} & \textcolor{red}{\makecell{2.062 \\ (0.052)}} & \textcolor{red}{\makecell{0.912 \\ (0.025)}} & \makecell{4.477 \\ (0.137)} & \makecell{0.925 \\ (0.003)} & \makecell{1.193 \\ (0.023)} \\
 & ACP-AEC & \makecell{0.927 \\ (0.002)} & \makecell{2.164 \\ (0.019)} & \textcolor{red}{\makecell{0.936 \\ (0.003)}} & \makecell{4.706 \\ (0.014)} & \makecell{0.924 \\ (0.002)} & \makecell{1.234 \\ (0.012)} \\
 & ACP-MC & \makecell{0.903 \\ (0.003)} & \textcolor{red}{\makecell{2.040 \\ (0.020)}} & \textcolor{red}{\makecell{0.869 \\ (0.005)}} & \makecell{4.192 \\ (0.015)} & \makecell{0.915 \\ (0.002)} & \makecell{1.249 \\ (0.015)} \\
 & AdaTS & \makecell{0.896 \\ (0.002)} & \makecell{2.461 \\ (0.032)} & \makecell{0.612 \\ (0.009)} & \makecell{2.454 \\ (0.033)} & \makecell{1.000 \\ (0.000)} & \makecell{2.464 \\ (0.032)} \\
 & Retrain & \makecell{0.900 \\ (0.002)} & \makecell{2.494 \\ (0.042)} & \makecell{0.783 \\ (0.009)} & \makecell{3.465 \\ (0.052)} & \makecell{0.943 \\ (0.003)} & \makecell{2.140 \\ (0.046)} \\
 & Standard & \makecell{0.876 \\ (0.004)} & \textcolor{red}{\makecell{2.088 \\ (0.036)}} & \makecell{0.535 \\ (0.015)} & \makecell{2.095 \\ (0.035)} & \makecell{1.000 \\ (0.000)} & \makecell{2.085 \\ (0.037)} \\
\midrule
0.25 & ACP-ACC & \makecell{0.928 \\ (0.005)} & \textcolor{red}{\makecell{2.332 \\ (0.033)}} & \textcolor{red}{\makecell{0.947 \\ (0.012)}} & \makecell{4.675 \\ (0.090)} & \makecell{0.920 \\ (0.003)} & \makecell{1.229 \\ (0.016)} \\
 & ACP-AEC & \makecell{0.933 \\ (0.002)} & \textcolor{red}{\makecell{2.323 \\ (0.019)}} & \textcolor{red}{\makecell{0.944 \\ (0.003)}} & \makecell{4.734 \\ (0.013)} & \makecell{0.927 \\ (0.002)} & \makecell{1.221 \\ (0.012)} \\
 & ACP-MC & \makecell{0.901 \\ (0.003)} & \textcolor{red}{\makecell{2.194 \\ (0.015)}} & \textcolor{red}{\makecell{0.871 \\ (0.004)}} & \makecell{4.218 \\ (0.016)} & \makecell{0.914 \\ (0.003)} & \makecell{1.246 \\ (0.014)} \\
 & AdaTS & \makecell{0.898 \\ (0.003)} & \makecell{2.974 \\ (0.031)} & \makecell{0.680 \\ (0.008)} & \makecell{2.971 \\ (0.031)} & \makecell{1.000 \\ (0.000)} & \makecell{2.976 \\ (0.031)} \\
 & Retrain & \makecell{0.899 \\ (0.002)} & \makecell{2.615 \\ (0.030)} & \makecell{0.803 \\ (0.007)} & \makecell{3.627 \\ (0.038)} & \makecell{0.944 \\ (0.003)} & \makecell{2.144 \\ (0.034)} \\
 & Standard & \makecell{0.867 \\ (0.011)} & \makecell{2.610 \\ (0.199)} & \makecell{0.585 \\ (0.035)} & \makecell{2.614 \\ (0.199)} & \makecell{1.000 \\ (0.000)} & \makecell{2.609 \\ (0.199)} \\
\midrule
0.3 & ACP-ACC & \makecell{0.920 \\ (0.012)} & \textcolor{red}{\makecell{2.423 \\ (0.074)}} & \textcolor{red}{\makecell{0.918 \\ (0.031)}} & \makecell{4.554 \\ (0.149)} & \makecell{0.920 \\ (0.005)} & \makecell{1.193 \\ (0.033)} \\
 & ACP-AEC & \makecell{0.936 \\ (0.002)} & \textcolor{red}{\makecell{2.516 \\ (0.017)}} & \textcolor{red}{\makecell{0.955 \\ (0.002)}} & \makecell{4.786 \\ (0.007)} & \makecell{0.925 \\ (0.003)} & \makecell{1.220 \\ (0.011)} \\
 & ACP-MC & \makecell{0.903 \\ (0.002)} & \textcolor{red}{\makecell{2.316 \\ (0.017)}} & \makecell{0.872 \\ (0.003)} & \makecell{4.244 \\ (0.017)} & \makecell{0.920 \\ (0.003)} & \makecell{1.225 \\ (0.016)} \\
 & AdaTS & \makecell{0.895 \\ (0.002)} & \makecell{3.278 \\ (0.028)} & \makecell{0.710 \\ (0.007)} & \makecell{3.264 \\ (0.028)} & \makecell{1.000 \\ (0.000)} & \makecell{3.285 \\ (0.029)} \\
 & Retrain & \makecell{0.897 \\ (0.002)} & \makecell{2.735 \\ (0.029)} & \makecell{0.821 \\ (0.006)} & \makecell{3.794 \\ (0.030)} & \makecell{0.939 \\ (0.003)} & \makecell{2.136 \\ (0.034)} \\
 & Standard & \makecell{0.978 \\ (0.010)} & \makecell{4.656 \\ (0.156)} & \textcolor{red}{\makecell{0.937 \\ (0.029)}} & \makecell{4.655 \\ (0.156)} & \makecell{1.000 \\ (0.000)} & \makecell{4.657 \\ (0.157)} \\
\midrule
0.35 & ACP-ACC & \makecell{0.915 \\ (0.011)} & \textcolor{red}{\makecell{2.535 \\ (0.069)}} & \textcolor{red}{\makecell{0.911 \\ (0.026)}} & \makecell{4.470 \\ (0.154)} & \makecell{0.920 \\ (0.003)} & \makecell{1.180 \\ (0.019)} \\
 & ACP-AEC & \makecell{0.937 \\ (0.002)} & \textcolor{red}{\makecell{2.650 \\ (0.015)}} & \textcolor{red}{\makecell{0.955 \\ (0.002)}} & \makecell{4.792 \\ (0.007)} & \makecell{0.924 \\ (0.003)} & \makecell{1.190 \\ (0.011)} \\
 & ACP-MC & \makecell{0.898 \\ (0.003)} & \textcolor{red}{\makecell{2.475 \\ (0.019)}} & \makecell{0.872 \\ (0.004)} & \makecell{4.252 \\ (0.016)} & \makecell{0.917 \\ (0.003)} & \makecell{1.228 \\ (0.016)} \\
 & AdaTS & \makecell{0.900 \\ (0.002)} & \makecell{3.564 \\ (0.024)} & \makecell{0.757 \\ (0.006)} & \makecell{3.544 \\ (0.025)} & \makecell{1.000 \\ (0.000)} & \makecell{3.579 \\ (0.024)} \\
 & Retrain & \makecell{0.902 \\ (0.001)} & \makecell{2.907 \\ (0.025)} & \makecell{0.843 \\ (0.004)} & \makecell{3.900 \\ (0.024)} & \makecell{0.944 \\ (0.002)} & \makecell{2.213 \\ (0.031)} \\
 & Standard & \makecell{0.996 \\ (0.000)} & \makecell{4.933 \\ (0.002)} & \textcolor{red}{\makecell{0.990 \\ (0.001)}} & \makecell{4.934 \\ (0.003)} & \makecell{1.000 \\ (0.000)} & \makecell{4.932 \\ (0.002)} \\
\midrule
0.4 & ACP-ACC & \makecell{0.946 \\ (0.003)} & \textcolor{red}{\makecell{2.907 \\ (0.076)}} & \textcolor{red}{\makecell{0.974 \\ (0.003)}} & \makecell{4.838 \\ (0.018)} & \makecell{0.923 \\ (0.003)} & \makecell{1.323 \\ (0.127)} \\
 & ACP-AEC & \makecell{0.942 \\ (0.001)} & \textcolor{red}{\makecell{2.828 \\ (0.014)}} & \textcolor{red}{\makecell{0.960 \\ (0.002)}} & \makecell{4.808 \\ (0.009)} & \makecell{0.928 \\ (0.003)} & \makecell{1.202 \\ (0.012)} \\
 & ACP-MC & \makecell{0.899 \\ (0.002)} & \textcolor{red}{\makecell{2.636 \\ (0.015)}} & \makecell{0.872 \\ (0.003)} & \makecell{4.316 \\ (0.011)} & \makecell{0.920 \\ (0.002)} & \makecell{1.235 \\ (0.012)} \\
 & AdaTS & \makecell{0.899 \\ (0.002)} & \makecell{3.738 \\ (0.018)} & \makecell{0.777 \\ (0.004)} & \makecell{3.727 \\ (0.018)} & \makecell{1.000 \\ (0.000)} & \makecell{3.747 \\ (0.017)} \\
 & Retrain & \makecell{0.902 \\ (0.002)} & \makecell{3.121 \\ (0.023)} & \makecell{0.850 \\ (0.004)} & \makecell{4.026 \\ (0.021)} & \makecell{0.946 \\ (0.003)} & \makecell{2.368 \\ (0.032)} \\
 & Standard & \makecell{0.995 \\ (0.000)} & \makecell{4.933 \\ (0.002)} & \textcolor{red}{\makecell{0.990 \\ (0.001)}} & \makecell{4.931 \\ (0.002)} & \makecell{1.000 \\ (0.000)} & \makecell{4.934 \\ (0.003)} \\
\bottomrule
\end{tabular}

\end{table}

\begin{figure*}[tbh]
    \centering
    \includegraphics[width=0.9\linewidth]{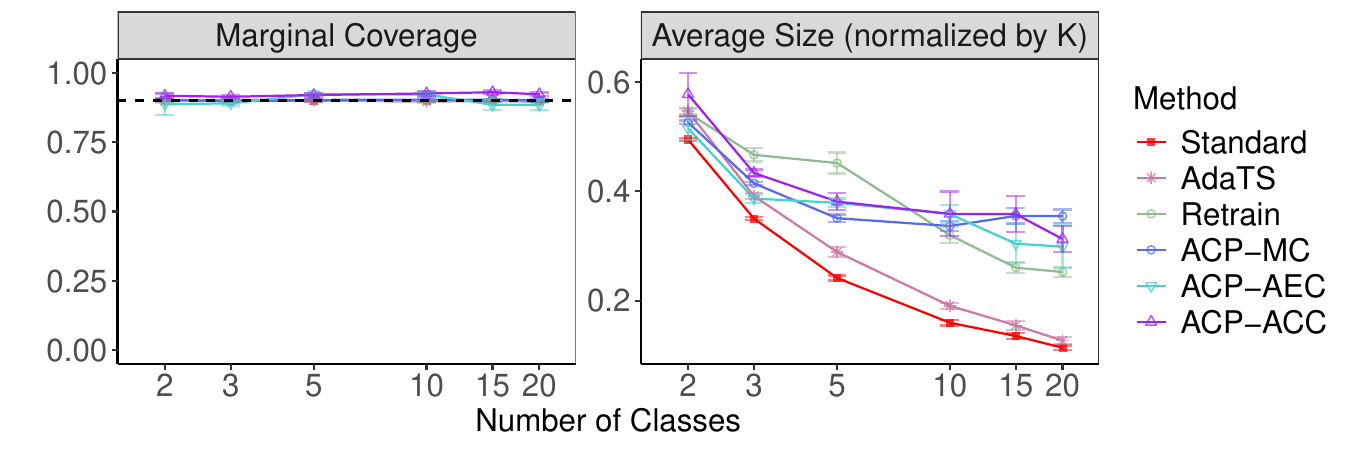}
    \caption{Performance of prediction sets for $5$-class synthetic data generated from Example~\ref{exp:concept-shift} as a function of number of classes. All methods achieve the target marginal coverage of $90\%$, while our methods retain practically efficient sets sizes (relative to the total number of classes). Error bars denote two standard errors. Corresponding conditional performance are provided in Figure~\ref{fig:supp_exp_sim_K_cond}. }
    \label{fig:supp_exp_sim_K_marg}
\end{figure*}

\begin{figure*}[tbh]
    \centering
    \includegraphics[width=0.9\linewidth]{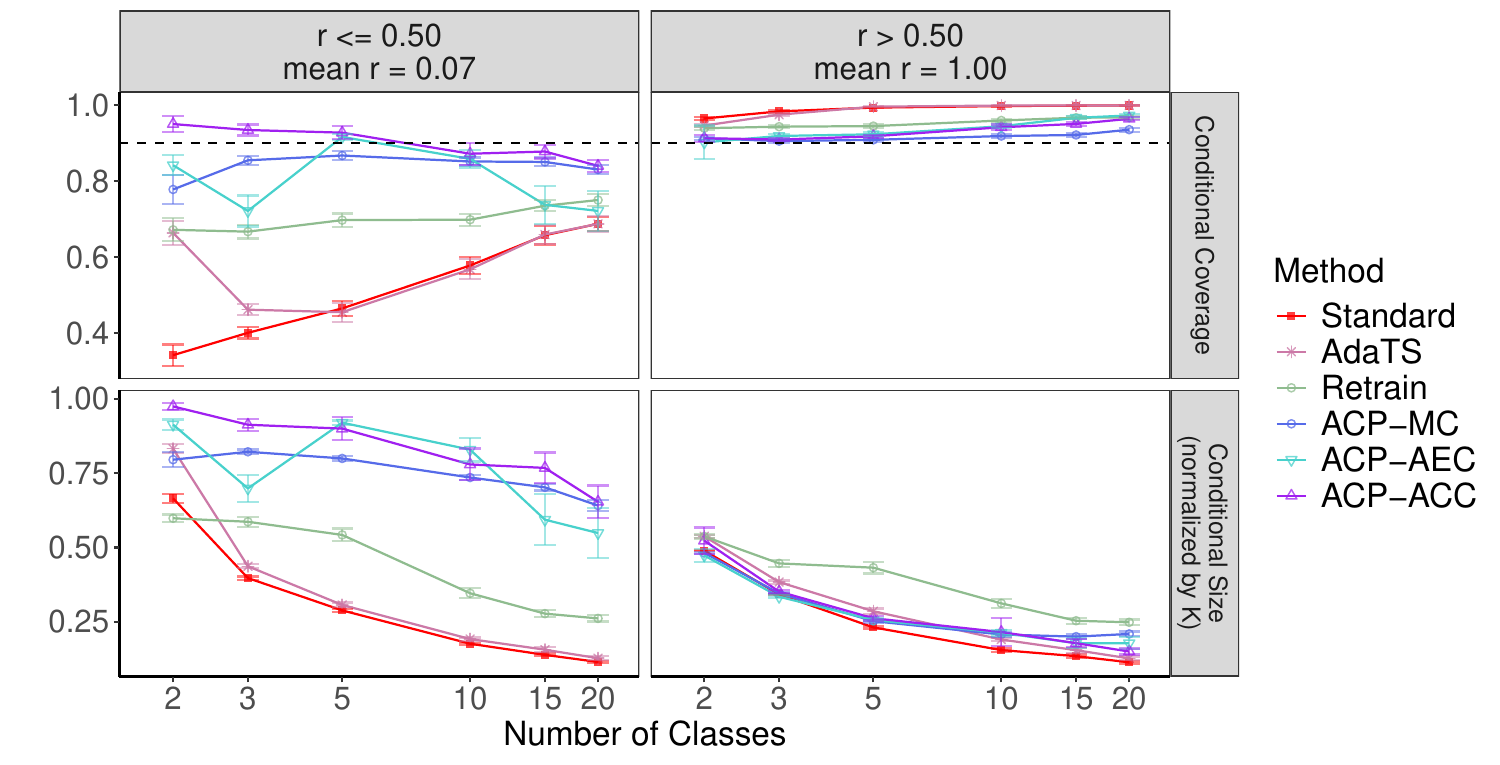}
    \caption{Our methods are more uncertainty-aware: they attain higher conditional coverage on hard samples while assigning larger prediction sets to hard samples and smaller ones to easy samples. Corresponding marginal performance are provided in Figure~\ref{fig:supp_exp_sim_K_marg} and numerical details are provided in Table~\ref{tab:main_exp_sim_nnew_K}. }
    \label{fig:supp_exp_sim_K_cond}
\end{figure*}

\begin{table}[!htb]
\centering
    \caption{Performance of prediction sets for $5$-class synthetic data generated from Example~\ref{exp:concept-shift} as a function of number of classes. See the corresponding plots in Figure~\ref{fig:supp_exp_sim_K_marg}--\ref{fig:supp_exp_sim_K_cond}.}
  \label{tab:main_exp_sim_nnew_K}
\centering
\fontsize{6}{6}\selectfont
\begin{tabular}[t]{rlllllll}
\toprule
\multicolumn{2}{c}{ } & \multicolumn{2}{c}{\makecell{Marginal}} & \multicolumn{2}{c}{\makecell{Hard Bin \\ ($r^* \leq 0.5$)}} & \multicolumn{2}{c}{\makecell{Easy Bin \\ ($r^* > 0.5$)}} \\
\cmidrule(l{3pt}r{3pt}){3-4} \cmidrule(l{3pt}r{3pt}){5-6} \cmidrule(l{3pt}r{3pt}){7-8}
\makecell{Number \\ of \\ Classes} & Method & Coverage & Size & Coverage & Size & Coverage & Size \\
\midrule
2 & ACP-ACC & \makecell{0.917 \\ (0.004)} & \makecell{1.155 \\ (0.039)} & \textcolor{red}{\makecell{0.951 \\ (0.010)}} & \makecell{1.949 \\ (0.011)} & \makecell{0.913 \\ (0.005)} & \makecell{1.045 \\ (0.044)} \\
 & ACP-AEC & \makecell{0.887 \\ (0.020)} & \textcolor{red}{\makecell{1.030 \\ (0.021)}} & \textcolor{red}{\makecell{0.842 \\ (0.013)}} & \makecell{1.826 \\ (0.018)} & \makecell{0.902 \\ (0.022)} & \makecell{0.946 \\ (0.022)} \\
 & ACP-MC & \makecell{0.903 \\ (0.002)} & \textcolor{red}{\makecell{1.052 \\ (0.003)}} & \textcolor{red}{\makecell{0.778 \\ (0.019)}} & \makecell{1.590 \\ (0.025)} & \makecell{0.913 \\ (0.002)} & \makecell{0.964 \\ (0.003)} \\
 & AdaTS & \makecell{0.899 \\ (0.002)} & \makecell{1.093 \\ (0.006)} & \makecell{0.664 \\ (0.015)} & \makecell{1.664 \\ (0.015)} & \makecell{0.947 \\ (0.002)} & \makecell{1.075 \\ (0.006)} \\
 & Retrain & \makecell{0.897 \\ (0.002)} & \makecell{1.087 \\ (0.007)} & \makecell{0.672 \\ (0.015)} & \makecell{1.196 \\ (0.013)} & \makecell{0.940 \\ (0.002)} & \makecell{1.074 \\ (0.007)} \\
 & Standard & \makecell{0.901 \\ (0.002)} & \textcolor{red}{\makecell{0.989 \\ (0.003)}} & \makecell{0.341 \\ (0.014)} & \makecell{1.330 \\ (0.015)} & \makecell{0.965 \\ (0.002)} & \makecell{0.978 \\ (0.002)} \\
\midrule
3 & ACP-ACC & \makecell{0.914 \\ (0.002)} & \makecell{1.300 \\ (0.009)} & \textcolor{red}{\makecell{0.935 \\ (0.007)}} & \makecell{2.738 \\ (0.031)} & \makecell{0.910 \\ (0.002)} & \makecell{1.055 \\ (0.008)} \\
 & ACP-AEC & \makecell{0.890 \\ (0.004)} & \textcolor{red}{\makecell{1.160 \\ (0.013)}} & \textcolor{red}{\makecell{0.722 \\ (0.021)}} & \makecell{2.096 \\ (0.069)} & \makecell{0.919 \\ (0.002)} & \makecell{1.002 \\ (0.006)} \\
 & ACP-MC & \makecell{0.899 \\ (0.002)} & \makecell{1.244 \\ (0.005)} & \textcolor{red}{\makecell{0.855 \\ (0.006)}} & \makecell{2.465 \\ (0.012)} & \makecell{0.906 \\ (0.002)} & \makecell{1.036 \\ (0.005)} \\
 & AdaTS & \makecell{0.901 \\ (0.002)} & \textcolor{red}{\makecell{1.174 \\ (0.008)}} & \makecell{0.461 \\ (0.007)} & \makecell{1.310 \\ (0.010)} & \makecell{0.975 \\ (0.001)} & \makecell{1.150 \\ (0.008)} \\
 & Retrain & \makecell{0.904 \\ (0.002)} & \makecell{1.400 \\ (0.018)} & \makecell{0.667 \\ (0.009)} & \makecell{1.759 \\ (0.025)} & \makecell{0.944 \\ (0.002)} & \makecell{1.339 \\ (0.018)} \\
 & Standard & \makecell{0.899 \\ (0.002)} & \textcolor{red}{\makecell{1.050 \\ (0.004)}} & \makecell{0.400 \\ (0.008)} & \makecell{1.191 \\ (0.008)} & \makecell{0.984 \\ (0.001)} & \makecell{1.026 \\ (0.003)} \\
\midrule
5 & ACP-ACC & \makecell{0.920 \\ (0.003)} & \makecell{1.906 \\ (0.038)} & \textcolor{red}{\makecell{0.928 \\ (0.009)}} & \makecell{4.498 \\ (0.095)} & \makecell{0.918 \\ (0.002)} & \makecell{1.309 \\ (0.020)} \\
 & ACP-AEC & \makecell{0.923 \\ (0.002)} & \makecell{1.895 \\ (0.020)} & \textcolor{red}{\makecell{0.917 \\ (0.004)}} & \makecell{4.598 \\ (0.018)} & \makecell{0.924 \\ (0.002)} & \makecell{1.280 \\ (0.012)} \\
 & ACP-MC & \makecell{0.902 \\ (0.002)} & \textcolor{red}{\makecell{1.754 \\ (0.017)}} & \textcolor{red}{\makecell{0.868 \\ (0.006)}} & \makecell{3.996 \\ (0.018)} & \makecell{0.909 \\ (0.002)} & \makecell{1.260 \\ (0.011)} \\
 & AdaTS & \makecell{0.900 \\ (0.002)} & \textcolor{red}{\makecell{1.445 \\ (0.023)}} & \makecell{0.455 \\ (0.013)} & \makecell{1.533 \\ (0.020)} & \makecell{0.997 \\ (0.001)} & \makecell{1.427 \\ (0.023)} \\
 & Retrain & \makecell{0.901 \\ (0.002)} & \makecell{2.258 \\ (0.048)} & \makecell{0.697 \\ (0.009)} & \makecell{2.710 \\ (0.053)} & \makecell{0.945 \\ (0.002)} & \makecell{2.159 \\ (0.049)} \\
 & Standard & \makecell{0.900 \\ (0.002)} & \textcolor{red}{\makecell{1.208 \\ (0.011)}} & \makecell{0.465 \\ (0.010)} & \makecell{1.441 \\ (0.014)} & \makecell{0.994 \\ (0.001)} & \makecell{1.155 \\ (0.009)} \\
\midrule
10 & ACP-ACC & \makecell{0.926 \\ (0.004)} & \makecell{3.588 \\ (0.204)} & \textcolor{red}{\makecell{0.872 \\ (0.015)}} & \makecell{7.791 \\ (0.268)} & \makecell{0.942 \\ (0.004)} & \makecell{2.153 \\ (0.232)} \\
 & ACP-AEC & \makecell{0.922 \\ (0.004)} & \makecell{3.585 \\ (0.080)} & \textcolor{red}{\makecell{0.859 \\ (0.011)}} & \makecell{8.288 \\ (0.199)} & \makecell{0.945 \\ (0.002)} & \makecell{2.130 \\ (0.047)} \\
 & ACP-MC & \makecell{0.903 \\ (0.002)} & \makecell{3.365 \\ (0.047)} & \textcolor{red}{\makecell{0.852 \\ (0.005)}} & \makecell{7.352 \\ (0.044)} & \makecell{0.919 \\ (0.002)} & \makecell{2.072 \\ (0.037)} \\
 & AdaTS & \makecell{0.897 \\ (0.002)} & \textcolor{red}{\makecell{1.906 \\ (0.029)}} & \makecell{0.568 \\ (0.013)} & \makecell{1.918 \\ (0.027)} & \makecell{1.000 \\ (0.000)} & \makecell{1.903 \\ (0.029)} \\
 & Retrain & \makecell{0.898 \\ (0.002)} & \textcolor{red}{\makecell{3.199 \\ (0.074)}} & \makecell{0.699 \\ (0.008)} & \makecell{3.459 \\ (0.083)} & \makecell{0.960 \\ (0.002)} & \makecell{3.119 \\ (0.073)} \\
 & Standard & \makecell{0.898 \\ (0.002)} & \textcolor{red}{\makecell{1.598 \\ (0.025)}} & \makecell{0.578 \\ (0.011)} & \makecell{1.762 \\ (0.016)} & \makecell{0.998 \\ (0.000)} & \makecell{1.548 \\ (0.027)} \\
\midrule
15 & ACP-ACC & \makecell{0.930 \\ (0.003)} & \makecell{5.374 \\ (0.245)} & \textcolor{red}{\makecell{0.878 \\ (0.009)}} & \makecell{11.512 \\ (0.386)} & \makecell{0.951 \\ (0.003)} & \makecell{2.669 \\ (0.105)} \\
 & ACP-AEC & \makecell{0.885 \\ (0.010)} & \makecell{4.563 \\ (0.267)} & \textcolor{red}{\makecell{0.738 \\ (0.025)}} & \makecell{8.900 \\ (0.644)} & \makecell{0.967 \\ (0.002)} & \makecell{2.666 \\ (0.127)} \\
 & ACP-MC & \makecell{0.903 \\ (0.002)} & \makecell{5.324 \\ (0.104)} & \textcolor{red}{\makecell{0.851 \\ (0.006)}} & \makecell{10.530 \\ (0.080)} & \makecell{0.922 \\ (0.003)} & \makecell{3.000 \\ (0.056)} \\
 & AdaTS & \makecell{0.900 \\ (0.002)} & \textcolor{red}{\makecell{2.324 \\ (0.062)}} & \makecell{0.659 \\ (0.013)} & \makecell{2.340 \\ (0.063)} & \makecell{1.000 \\ (0.000)} & \makecell{2.316 \\ (0.062)} \\
 & Retrain & \makecell{0.899 \\ (0.002)} & \textcolor{red}{\makecell{3.906 \\ (0.078)}} & \makecell{0.735 \\ (0.007)} & \makecell{4.159 \\ (0.086)} & \makecell{0.967 \\ (0.002)} & \makecell{3.801 \\ (0.077)} \\
 & Standard & \makecell{0.899 \\ (0.002)} & \textcolor{red}{\makecell{2.033 \\ (0.040)}} & \makecell{0.658 \\ (0.012)} & \makecell{2.081 \\ (0.036)} & \makecell{0.999 \\ (0.000)} & \makecell{2.014 \\ (0.041)} \\
\midrule
20 & ACP-ACC & \makecell{0.923 \\ (0.003)} & \makecell{6.254 \\ (0.236)} & \textcolor{red}{\makecell{0.841 \\ (0.008)}} & \makecell{13.069 \\ (0.545)} & \makecell{0.965 \\ (0.002)} & \makecell{2.997 \\ (0.103)} \\
 & ACP-AEC & \makecell{0.885 \\ (0.009)} & \makecell{5.976 \\ (0.396)} & \makecell{0.722 \\ (0.027)} & \makecell{10.970 \\ (0.848)} & \makecell{0.973 \\ (0.002)} & \makecell{3.558 \\ (0.198)} \\
 & ACP-MC & \makecell{0.902 \\ (0.002)} & \makecell{7.093 \\ (0.119)} & \textcolor{red}{\makecell{0.831 \\ (0.006)}} & \makecell{12.818 \\ (0.180)} & \makecell{0.936 \\ (0.003)} & \makecell{4.182 \\ (0.082)} \\
 & AdaTS & \makecell{0.898 \\ (0.002)} & \textcolor{red}{\makecell{2.540 \\ (0.065)}} & \makecell{0.688 \\ (0.011)} & \makecell{2.548 \\ (0.066)} & \makecell{1.000 \\ (0.000)} & \makecell{2.534 \\ (0.063)} \\
 & Retrain & \makecell{0.897 \\ (0.002)} & \textcolor{red}{\makecell{5.048 \\ (0.097)}} & \textcolor{red}{\makecell{0.750 \\ (0.008)}} & \makecell{5.224 \\ (0.104)} & \makecell{0.968 \\ (0.002)} & \makecell{4.964 \\ (0.097)} \\
 & Standard & \makecell{0.897 \\ (0.002)} & \textcolor{red}{\makecell{2.281 \\ (0.043)}} & \makecell{0.688 \\ (0.009)} & \makecell{2.294 \\ (0.043)} & \makecell{1.000 \\ (0.000)} & \makecell{2.273 \\ (0.043)} \\
\bottomrule
\end{tabular}

\end{table}

\begin{figure*}[tbh]
    \centering
    \includegraphics[width=0.9\linewidth]{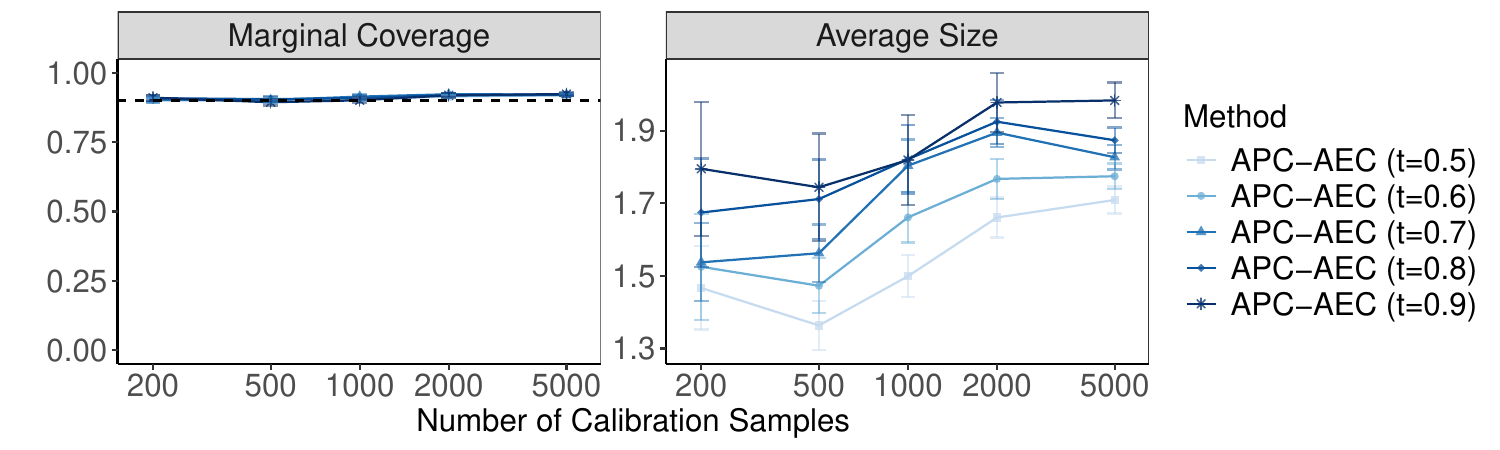}
    \caption{Performance of prediction sets for our AEqualized method with different thresholds for $5$-class synthetic data generated from Example~\ref{exp:concept-shift} as a function calibration sample sizes. All methods achieve the target marginal coverage of $90\%$, smaller thresholds lead to smaller average set sizes. Error bars denote two standard errors. Corresponding conditional performance are provided in Figure~\ref{fig:supp_exp_sim_n_new_equalized_cond}. }
\label{fig:supp_exp_sim_n_new_equalized_marg}
\end{figure*}

\begin{figure*}[tbh]
    \centering
    \includegraphics[width=0.9\linewidth]{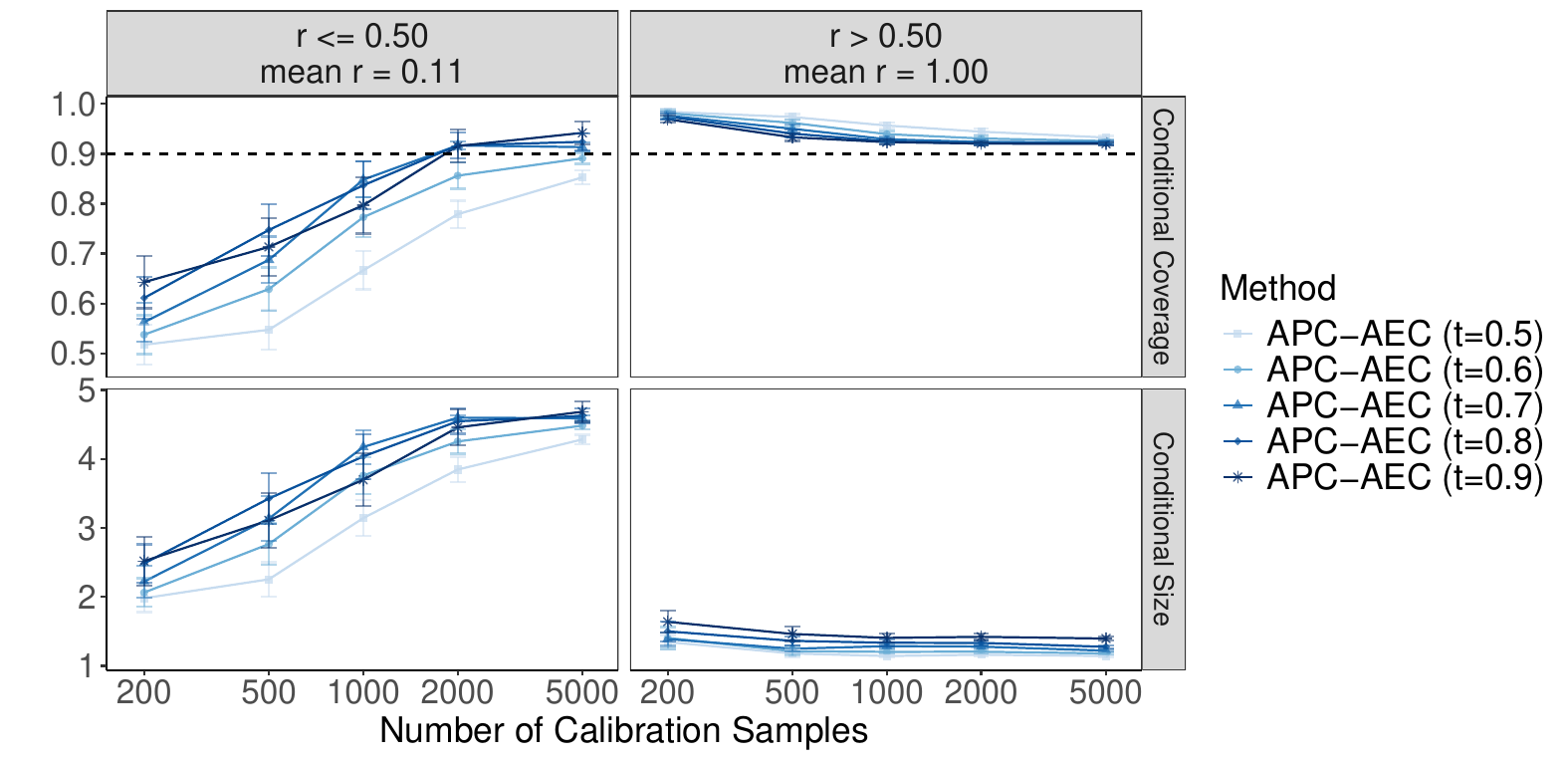}
    \caption{AEqualized uncertainty-aware: they attain higher conditional coverage on hard samples while assigning larger prediction sets to hard samples and smaller ones to easy samples. Higher thresholds lead to higher conditional coverage on the hard samples. Corresponding marginal performance are provided in Figure~\ref{fig:supp_exp_sim_n_new_equalized_marg} and numerical details are provided in Table~\ref{tab:main_exp_sim_equalized}. }
    \label{fig:supp_exp_sim_n_new_equalized_cond}
\end{figure*}

\begin{table}[!htb]
\centering
    \caption{Performance of conformal prediction sets for our AEqualized method with different thresholds for $5$-class synthetic data generated from Example~\ref{exp:concept-shift}, as a function the of the total calibration sample size. See the corresponding plots in Figure~\ref{fig:supp_exp_sim_n_new_equalized_marg}--\ref{fig:supp_exp_sim_n_new_equalized_cond}.}
  \label{tab:main_exp_sim_equalized}
\centering
\fontsize{6}{6}\selectfont
\begin{tabular}[t]{rlllllll}
\toprule
\multicolumn{2}{c}{ } & \multicolumn{2}{c}{\makecell{Marginal}} & \multicolumn{2}{c}{\makecell{Hard Bin \\ ($r^* \leq 0.5$)}} & \multicolumn{2}{c}{\makecell{Easy Bin \\ ($r^* > 0.5$)}} \\
\cmidrule(l{3pt}r{3pt}){3-4} \cmidrule(l{3pt}r{3pt}){5-6} \cmidrule(l{3pt}r{3pt}){7-8}
\makecell{Number \\ of \\ Calibration \\ Samples} & Method & Coverage & Size & Coverage & Size & Coverage & Size \\
\midrule
200 & APC-AEC (t=0.5) & \makecell{0.901 \\ (0.004)} & \textcolor{red}{\makecell{1.467 \\ (0.057)}} & \makecell{0.517 \\ (0.020)} & \makecell{1.982 \\ (0.100)} & \makecell{0.984 \\ (0.003)} & \makecell{1.346 \\ (0.052)} \\
 & APC-AEC (t=0.6) & \makecell{0.902 \\ (0.004)} & \textcolor{red}{\makecell{1.525 \\ (0.073)}} & \makecell{0.537 \\ (0.020)} & \makecell{2.064 \\ (0.104)} & \makecell{0.981 \\ (0.003)} & \makecell{1.406 \\ (0.079)} \\
 & APC-AEC (t=0.7) & \makecell{0.902 \\ (0.005)} & \textcolor{red}{\makecell{1.537 \\ (0.054)}} & \textcolor{red}{\makecell{0.562 \\ (0.019)}} & \makecell{2.224 \\ (0.115)} & \makecell{0.977 \\ (0.004)} & \makecell{1.390 \\ (0.050)} \\
 & APC-AEC (t=0.8) & \makecell{0.909 \\ (0.005)} & \makecell{1.675 \\ (0.075)} & \textcolor{red}{\makecell{0.611 \\ (0.021)}} & \makecell{2.486 \\ (0.138)} & \makecell{0.975 \\ (0.003)} & \makecell{1.502 \\ (0.073)} \\
 & APC-AEC (t=0.9) & \makecell{0.910 \\ (0.005)} & \makecell{1.795 \\ (0.092)} & \textcolor{red}{\makecell{0.643 \\ (0.026)}} & \makecell{2.519 \\ (0.177)} & \makecell{0.969 \\ (0.004)} & \makecell{1.640 \\ (0.080)} \\
\midrule
500 & APC-AEC (t=0.5) & \makecell{0.899 \\ (0.004)} & \textcolor{red}{\makecell{1.364 \\ (0.034)}} & \makecell{0.547 \\ (0.020)} & \makecell{2.256 \\ (0.123)} & \makecell{0.974 \\ (0.003)} & \makecell{1.181 \\ (0.026)} \\
 & APC-AEC (t=0.6) & \makecell{0.903 \\ (0.004)} & \textcolor{red}{\makecell{1.473 \\ (0.038)}} & \makecell{0.629 \\ (0.022)} & \makecell{2.767 \\ (0.150)} & \makecell{0.962 \\ (0.004)} & \makecell{1.213 \\ (0.027)} \\
 & APC-AEC (t=0.7) & \makecell{0.902 \\ (0.004)} & \textcolor{red}{\makecell{1.563 \\ (0.040)}} & \textcolor{red}{\makecell{0.687 \\ (0.023)}} & \makecell{3.137 \\ (0.164)} & \makecell{0.950 \\ (0.004)} & \makecell{1.250 \\ (0.021)} \\
 & APC-AEC (t=0.8) & \makecell{0.905 \\ (0.005)} & \makecell{1.712 \\ (0.055)} & \textcolor{red}{\makecell{0.747 \\ (0.026)}} & \makecell{3.427 \\ (0.184)} & \makecell{0.940 \\ (0.004)} & \makecell{1.364 \\ (0.032)} \\
 & APC-AEC (t=0.9) & \makecell{0.894 \\ (0.006)} & \makecell{1.744 \\ (0.074)} & \textcolor{red}{\makecell{0.713 \\ (0.029)}} & \makecell{3.111 \\ (0.199)} & \makecell{0.933 \\ (0.004)} & \makecell{1.464 \\ (0.051)} \\
\midrule
1000 & APC-AEC (t=0.5) & \makecell{0.904 \\ (0.003)} & \textcolor{red}{\makecell{1.499 \\ (0.029)}} & \makecell{0.666 \\ (0.019)} & \makecell{3.143 \\ (0.130)} & \makecell{0.957 \\ (0.003)} & \makecell{1.144 \\ (0.009)} \\
 & APC-AEC (t=0.6) & \makecell{0.909 \\ (0.004)} & \textcolor{red}{\makecell{1.661 \\ (0.035)}} & \makecell{0.773 \\ (0.020)} & \makecell{3.759 \\ (0.136)} & \makecell{0.939 \\ (0.003)} & \makecell{1.204 \\ (0.014)} \\
 & APC-AEC (t=0.7) & \makecell{0.914 \\ (0.004)} & \textcolor{red}{\makecell{1.804 \\ (0.036)}} & \textcolor{red}{\makecell{0.848 \\ (0.018)}} & \makecell{4.173 \\ (0.124)} & \makecell{0.929 \\ (0.003)} & \makecell{1.287 \\ (0.018)} \\
 & APC-AEC (t=0.8) & \makecell{0.909 \\ (0.004)} & \makecell{1.821 \\ (0.048)} & \textcolor{red}{\makecell{0.837 \\ (0.024)}} & \makecell{4.036 \\ (0.162)} & \makecell{0.925 \\ (0.003)} & \makecell{1.336 \\ (0.023)} \\
 & APC-AEC (t=0.9) & \makecell{0.901 \\ (0.005)} & \makecell{1.819 \\ (0.062)} & \textcolor{red}{\makecell{0.797 \\ (0.028)}} & \makecell{3.698 \\ (0.192)} & \makecell{0.923 \\ (0.003)} & \makecell{1.407 \\ (0.033)} \\
\midrule
2000 & APC-AEC (t=0.5) & \makecell{0.914 \\ (0.002)} & \textcolor{red}{\makecell{1.661 \\ (0.028)}} & \makecell{0.779 \\ (0.014)} & \makecell{3.846 \\ (0.093)} & \makecell{0.944 \\ (0.003)} & \makecell{1.163 \\ (0.009)} \\
 & APC-AEC (t=0.6) & \makecell{0.917 \\ (0.003)} & \textcolor{red}{\makecell{1.767 \\ (0.028)}} & \makecell{0.856 \\ (0.013)} & \makecell{4.253 \\ (0.088)} & \makecell{0.931 \\ (0.003)} & \makecell{1.207 \\ (0.012)} \\
 & APC-AEC (t=0.7) & \makecell{0.923 \\ (0.002)} & \textcolor{red}{\makecell{1.895 \\ (0.020)}} & \textcolor{red}{\makecell{0.917 \\ (0.004)}} & \makecell{4.598 \\ (0.018)} & \makecell{0.924 \\ (0.002)} & \makecell{1.280 \\ (0.012)} \\
 & APC-AEC (t=0.8) & \makecell{0.919 \\ (0.003)} & \makecell{1.925 \\ (0.030)} & \textcolor{red}{\makecell{0.917 \\ (0.013)}} & \makecell{4.546 \\ (0.091)} & \makecell{0.920 \\ (0.002)} & \makecell{1.334 \\ (0.017)} \\
 & APC-AEC (t=0.9) & \makecell{0.919 \\ (0.004)} & \makecell{1.978 \\ (0.041)} & \textcolor{red}{\makecell{0.915 \\ (0.016)}} & \makecell{4.457 \\ (0.129)} & \makecell{0.921 \\ (0.002)} & \makecell{1.423 \\ (0.023)} \\
\midrule
5000 & APC-AEC (t=0.5) & \makecell{0.919 \\ (0.001)} & \textcolor{red}{\makecell{1.709 \\ (0.019)}} & \makecell{0.853 \\ (0.007)} & \makecell{4.283 \\ (0.033)} & \makecell{0.933 \\ (0.002)} & \makecell{1.144 \\ (0.007)} \\
 & APC-AEC (t=0.6) & \makecell{0.920 \\ (0.001)} & \textcolor{red}{\makecell{1.774 \\ (0.017)}} & \makecell{0.891 \\ (0.005)} & \makecell{4.483 \\ (0.025)} & \makecell{0.926 \\ (0.001)} & \makecell{1.181 \\ (0.007)} \\
 & APC-AEC (t=0.7) & \makecell{0.922 \\ (0.001)} & \textcolor{red}{\makecell{1.827 \\ (0.017)}} & \textcolor{red}{\makecell{0.913 \\ (0.004)}} & \makecell{4.590 \\ (0.019)} & \makecell{0.923 \\ (0.002)} & \makecell{1.222 \\ (0.008)} \\
 & APC-AEC (t=0.8) & \makecell{0.921 \\ (0.002)} & \makecell{1.874 \\ (0.018)} & \textcolor{red}{\makecell{0.924 \\ (0.008)}} & \makecell{4.623 \\ (0.053)} & \makecell{0.921 \\ (0.002)} & \makecell{1.278 \\ (0.010)} \\
 & APC-AEC (t=0.9) & \makecell{0.924 \\ (0.003)} & \makecell{1.983 \\ (0.025)} & \textcolor{red}{\makecell{0.942 \\ (0.011)}} & \makecell{4.685 \\ (0.074)} & \makecell{0.920 \\ (0.002)} & \makecell{1.398 \\ (0.015)} \\
\bottomrule
\end{tabular}

\end{table}

\begin{figure*}[tbh]
    \centering
    \includegraphics[width=0.9\linewidth]{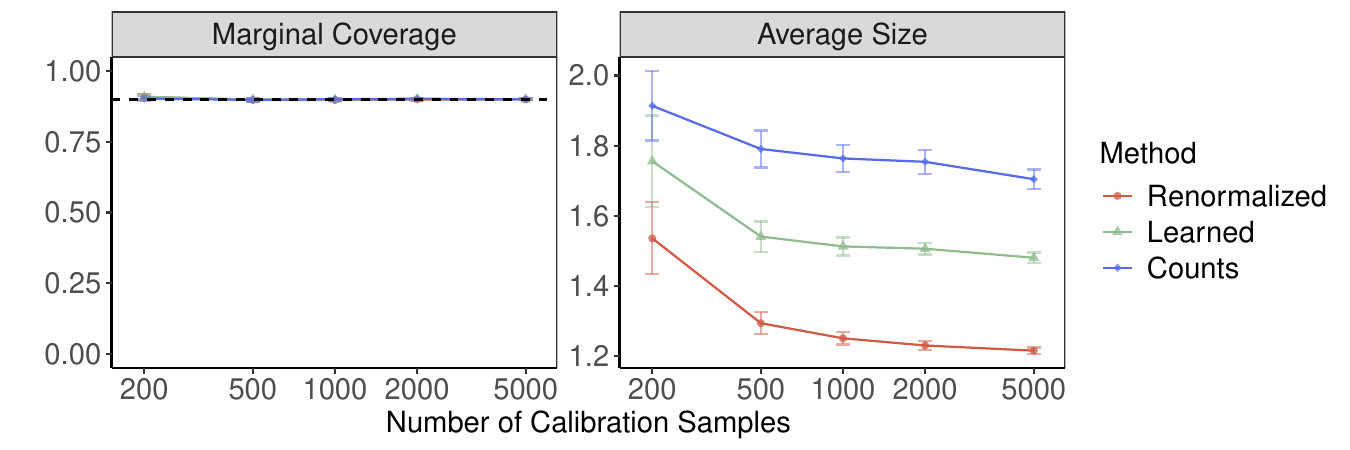}
    \caption{Performance of prediction sets for $5$-class synthetic data generated from Example~\ref{exp:concept-shift} using our ACP-MC method as a function of the total calibration sample size. In particular, we compare the different estimation to $\eta$ as described in Appendix~\ref{app:APA_mutlclass}. All $\eta$ choices achieve the target marginal coverage of $90\%$, while Renormalized achieves smallest prediction set size. Error bars denote two standard errors. Corresponding conditional performance are provided in Figure~\ref{fig:supp_exp_sim_n_new_etacompare_cond}. }
    \label{fig:supp_exp_sim_n_new_etacompare_marg}
\end{figure*}

\begin{figure*}[tbh]
    \centering
    \includegraphics[width=0.9\linewidth]{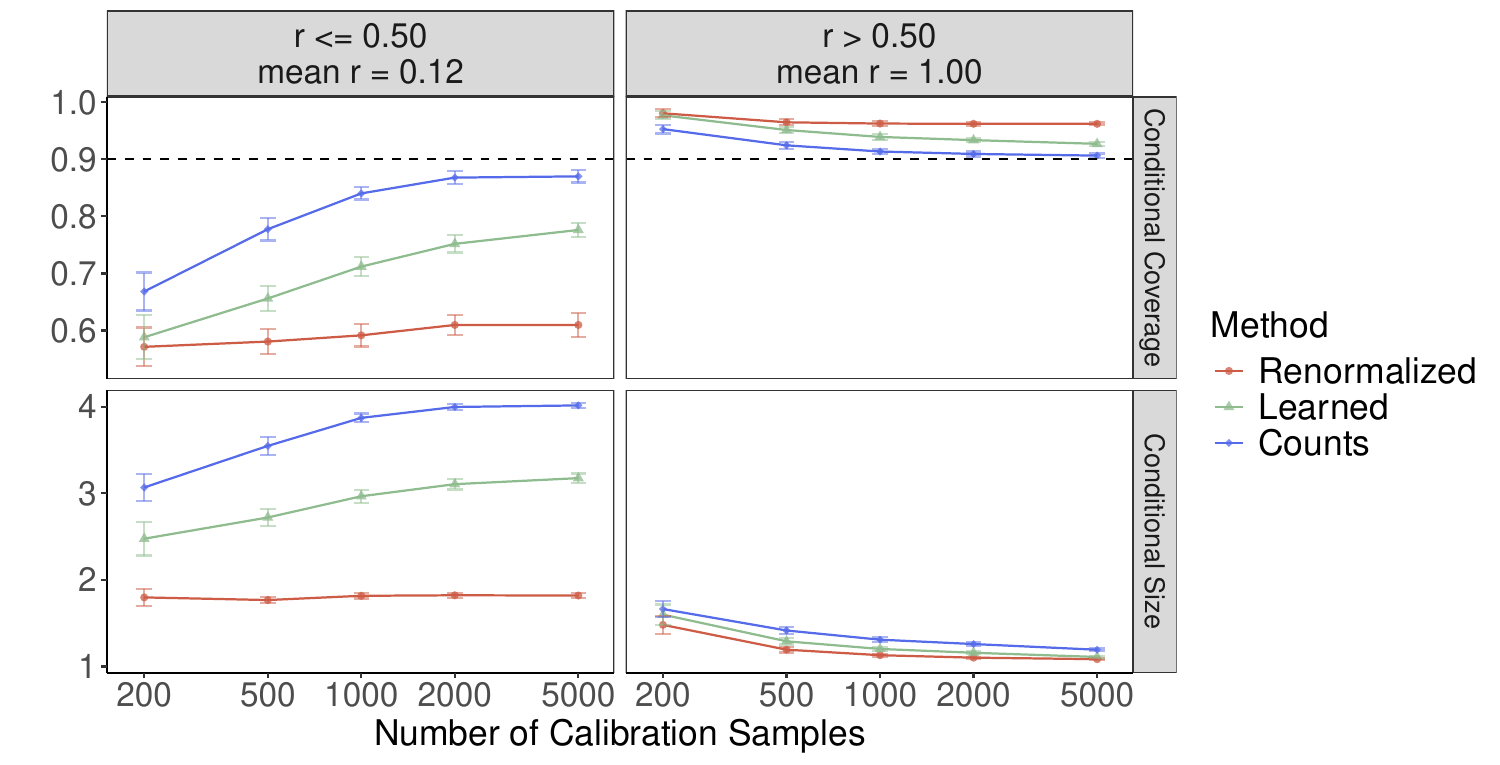}
    \caption{Our ACP-MC with all $\eta$ choices are uncertainty-aware: they attain high conditional coverage on hard samples while assigning larger prediction sets to hard samples and smaller ones to easy samples. Among them, Counts achieves highest conditional coverage on the hard bin. Corresponding marginal performance are provided in Figure~\ref{fig:supp_exp_sim_n_new_etacompare_marg} and numerical details are provided in Table~\ref{tab:main_exp_sim_etacompare}. }
    \label{fig:supp_exp_sim_n_new_etacompare_cond}
\end{figure*}

\begin{table}[!htb]
\centering
    \caption{Performance of conformal prediction sets for $5$-class synthetic data generated from Example~\ref{exp:concept-shift} using our ACP-MC method, as a function the of the total calibration sample size. See the corresponding plots in Figure~\ref{fig:supp_exp_sim_n_new_etacompare_marg}--\ref{fig:supp_exp_sim_n_new_etacompare_cond}.}
  \label{tab:main_exp_sim_etacompare}
\centering
\fontsize{6}{6}\selectfont
\begin{tabular}[t]{rlllllll}
\toprule
\multicolumn{2}{c}{ } & \multicolumn{2}{c}{\makecell{Marginal}} & \multicolumn{2}{c}{\makecell{$r^* \leq 0.50$ \\ mean $r^* = 0.12$}} & \multicolumn{2}{c}{\makecell{$r^* > 0.50$ \\ mean $r^* = 1.00$}} \\
\cmidrule(l{3pt}r{3pt}){3-4} \cmidrule(l{3pt}r{3pt}){5-6} \cmidrule(l{3pt}r{3pt}){7-8}
\makecell{Number \\ of \\ Calibration \\ Samples} & Method & Coverage & Size & Coverage & Size & Coverage & Size \\
\midrule
200 & Counts & \makecell{0.903 \\ (0.005)} & \textcolor{red}{\makecell{1.914 \\ (0.049)}} & \textcolor{red}{\makecell{0.668 \\ (0.017)}} & \makecell{3.068 \\ (0.080)} & \textcolor{red}{\makecell{0.953 \\ (0.004)}} & \makecell{1.664 \\ (0.045)} \\
 & Learned & \makecell{0.909 \\ (0.005)} & \textcolor{red}{\makecell{1.756 \\ (0.065)}} & \textcolor{red}{\makecell{0.588 \\ (0.019)}} & \makecell{2.476 \\ (0.097)} & \textcolor{red}{\makecell{0.977 \\ (0.003)}} & \makecell{1.600 \\ (0.059)} \\
 & Renormalized & \makecell{0.909 \\ (0.005)} & \textcolor{red}{\makecell{1.536 \\ (0.051)}} & \textcolor{red}{\makecell{0.571 \\ (0.017)}} & \makecell{1.798 \\ (0.048)} & \textcolor{red}{\makecell{0.980 \\ (0.003)}} & \makecell{1.481 \\ (0.053)} \\
\midrule
500 & Counts & \makecell{0.898 \\ (0.003)} & \textcolor{red}{\makecell{1.791 \\ (0.026)}} & \textcolor{red}{\makecell{0.777 \\ (0.010)}} & \makecell{3.547 \\ (0.053)} & \textcolor{red}{\makecell{0.924 \\ (0.003)}} & \makecell{1.417 \\ (0.021)} \\
 & Learned & \makecell{0.900 \\ (0.003)} & \textcolor{red}{\makecell{1.541 \\ (0.022)}} & \textcolor{red}{\makecell{0.656 \\ (0.011)}} & \makecell{2.721 \\ (0.049)} & \textcolor{red}{\makecell{0.951 \\ (0.003)}} & \makecell{1.292 \\ (0.018)} \\
 & Renormalized & \makecell{0.898 \\ (0.003)} & \textcolor{red}{\makecell{1.294 \\ (0.016)}} & \textcolor{red}{\makecell{0.581 \\ (0.011)}} & \makecell{1.768 \\ (0.019)} & \textcolor{red}{\makecell{0.965 \\ (0.003)}} & \makecell{1.195 \\ (0.016)} \\
\midrule
1000 & Counts & \makecell{0.901 \\ (0.002)} & \textcolor{red}{\makecell{1.764 \\ (0.019)}} & \textcolor{red}{\makecell{0.840 \\ (0.005)}} & \makecell{3.871 \\ (0.025)} & \textcolor{red}{\makecell{0.913 \\ (0.002)}} & \makecell{1.311 \\ (0.015)} \\
 & Learned & \makecell{0.899 \\ (0.003)} & \textcolor{red}{\makecell{1.513 \\ (0.013)}} & \textcolor{red}{\makecell{0.712 \\ (0.008)}} & \makecell{2.966 \\ (0.038)} & \textcolor{red}{\makecell{0.939 \\ (0.002)}} & \makecell{1.204 \\ (0.011)} \\
 & Renormalized & \makecell{0.898 \\ (0.002)} & \textcolor{red}{\makecell{1.251 \\ (0.009)}} & \textcolor{red}{\makecell{0.592 \\ (0.010)}} & \makecell{1.817 \\ (0.016)} & \textcolor{red}{\makecell{0.962 \\ (0.002)}} & \makecell{1.131 \\ (0.008)} \\
\midrule
2000 & Counts & \makecell{0.902 \\ (0.002)} & \textcolor{red}{\makecell{1.754 \\ (0.017)}} & \textcolor{red}{\makecell{0.868 \\ (0.006)}} & \makecell{3.996 \\ (0.018)} & \textcolor{red}{\makecell{0.909 \\ (0.002)}} & \makecell{1.260 \\ (0.011)} \\
 & Learned & \makecell{0.901 \\ (0.002)} & \textcolor{red}{\makecell{1.506 \\ (0.008)}} & \textcolor{red}{\makecell{0.752 \\ (0.008)}} & \makecell{3.106 \\ (0.030)} & \textcolor{red}{\makecell{0.933 \\ (0.002)}} & \makecell{1.159 \\ (0.008)} \\
 & Renormalized & \makecell{0.900 \\ (0.002)} & \textcolor{red}{\makecell{1.231 \\ (0.006)}} & \textcolor{red}{\makecell{0.610 \\ (0.009)}} & \makecell{1.823 \\ (0.014)} & \textcolor{red}{\makecell{0.962 \\ (0.002)}} & \makecell{1.103 \\ (0.006)} \\
\midrule
5000 & Counts & \makecell{0.900 \\ (0.002)} & \textcolor{red}{\makecell{1.705 \\ (0.014)}} & \textcolor{red}{\makecell{0.870 \\ (0.005)}} & \makecell{4.014 \\ (0.013)} & \textcolor{red}{\makecell{0.906 \\ (0.002)}} & \makecell{1.194 \\ (0.007)} \\
 & Learned & \makecell{0.900 \\ (0.002)} & \textcolor{red}{\makecell{1.481 \\ (0.008)}} & \textcolor{red}{\makecell{0.776 \\ (0.006)}} & \makecell{3.175 \\ (0.027)} & \textcolor{red}{\makecell{0.927 \\ (0.002)}} & \makecell{1.111 \\ (0.006)} \\
 & Renormalized & \makecell{0.900 \\ (0.002)} & \textcolor{red}{\makecell{1.216 \\ (0.004)}} & \textcolor{red}{\makecell{0.610 \\ (0.010)}} & \makecell{1.820 \\ (0.015)} & \textcolor{red}{\makecell{0.962 \\ (0.001)}} & \makecell{1.086 \\ (0.004)} \\
\bottomrule
\end{tabular}

\end{table}

\begin{figure*}[tbh]
    \centering
    \includegraphics[width=0.9\linewidth]{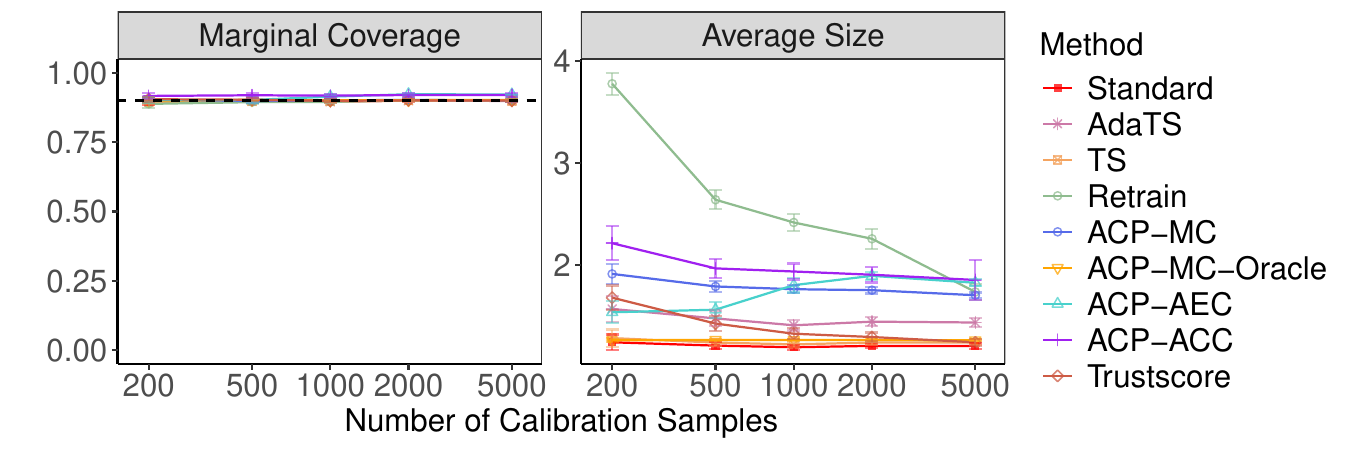}
    \caption{Performance of prediction sets constructed with additional benchmarks for $5$-class synthetic data generated from Example~\ref{exp:concept-shift} as a function of the total calibration sample size. All methods achieve the target marginal coverage of $90\%$, while our methods retain practically efficient sets sizes. Error bars denote two standard errors. Corresponding conditional performance are provided in Figure~\ref{fig:supp_exp_sim_n_new_full_cond}. }
    \label{fig:supp_exp_sim_n_new_full_marg}
\end{figure*}

\begin{figure*}[tbh]
    \centering
    \includegraphics[width=0.9\linewidth]{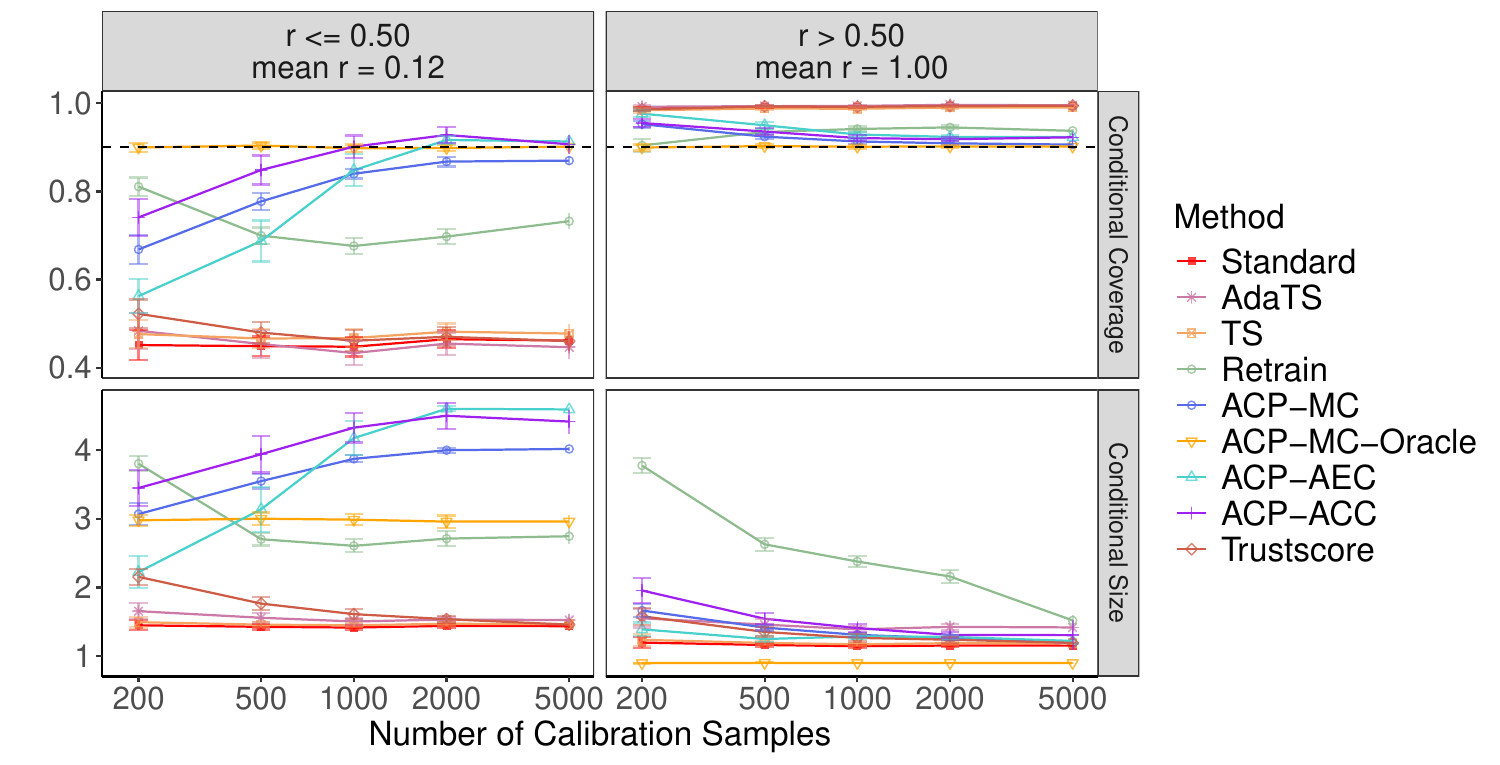}
    \caption{Our methods are more uncertainty-aware: they attain higher conditional coverage on hard samples while assigning larger prediction sets to hard samples and smaller ones to easy samples. Corresponding marginal performance are provided in Figure~\ref{fig:supp_exp_sim_n_new_full_marg} and numerical details are provided in Table~\ref{tab:main_exp_sim_nnew_full}. }
    \label{fig:supp_exp_sim_n_new_full_cond}
\end{figure*}

\begin{table}[!htb]
\centering
    \caption{Performance of conformal prediction sets constructed with additional benchmarks for $5$-class synthetic data generated from Example~\ref{exp:concept-shift}, as a function the of the total calibration sample size. See the corresponding plots in Figure~\ref{fig:supp_exp_sim_n_new_full_marg}--\ref{fig:supp_exp_sim_n_new_full_cond}.}
  \label{tab:main_exp_sim_nnew_full}
\centering
\fontsize{6}{6}\selectfont
\begin{tabular}[t]{rlllllll}
\toprule
\multicolumn{2}{c}{ } & \multicolumn{2}{c}{\makecell{Marginal}} & \multicolumn{2}{c}{\makecell{Hard Bin \\ ($r^* \leq 0.5$)}} & \multicolumn{2}{c}{\makecell{Easy Bin \\ ($r^* > 0.5$)}} \\
\cmidrule(l{3pt}r{3pt}){3-4} \cmidrule(l{3pt}r{3pt}){5-6} \cmidrule(l{3pt}r{3pt}){7-8}
\makecell{Number \\ of \\ Calibration \\ Samples} & Method & Coverage & Size & Coverage & Size & Coverage & Size \\
\midrule
200 & ACP-ACC & \makecell{0.917 \\ (0.005)} & \makecell{2.215 \\ (0.084)} & \textcolor{red}{\makecell{0.741 \\ (0.021)}} & \makecell{3.447 \\ (0.131)} & \makecell{0.955 \\ (0.005)} & \makecell{1.956 \\ (0.090)} \\
 & ACP-AEC & \makecell{0.902 \\ (0.005)} & \makecell{1.537 \\ (0.054)} & \makecell{0.562 \\ (0.019)} & \makecell{2.224 \\ (0.115)} & \makecell{0.977 \\ (0.004)} & \makecell{1.390 \\ (0.050)} \\
 & ACP-MC & \makecell{0.903 \\ (0.005)} & \makecell{1.914 \\ (0.049)} & \makecell{0.668 \\ (0.017)} & \makecell{3.068 \\ (0.080)} & \makecell{0.953 \\ (0.004)} & \makecell{1.664 \\ (0.045)} \\
 & ACP-MC-Oracle & \makecell{0.899 \\ (0.003)} & \textcolor{red}{\makecell{1.262 \\ (0.006)}} & \textcolor{red}{\makecell{0.900 \\ (0.005)}} & \makecell{2.977 \\ (0.039)} & \makecell{0.900 \\ (0.003)} & \makecell{0.900 \\ (0.003)} \\
 & AdaTS & \makecell{0.903 \\ (0.005)} & \makecell{1.570 \\ (0.065)} & \makecell{0.484 \\ (0.021)} & \makecell{1.655 \\ (0.060)} & \makecell{0.992 \\ (0.002)} & \makecell{1.551 \\ (0.066)} \\
 & Retrain & \makecell{0.888 \\ (0.007)} & \makecell{3.779 \\ (0.053)} & \textcolor{red}{\makecell{0.811 \\ (0.010)}} & \makecell{3.800 \\ (0.055)} & \makecell{0.904 \\ (0.007)} & \makecell{3.774 \\ (0.053)} \\
 & Standard & \makecell{0.894 \\ (0.004)} & \textcolor{red}{\makecell{1.243 \\ (0.038)}} & \makecell{0.451 \\ (0.017)} & \makecell{1.449 \\ (0.036)} & \makecell{0.988 \\ (0.002)} & \makecell{1.198 \\ (0.039)} \\
 & TS & \makecell{0.895 \\ (0.004)} & \textcolor{red}{\makecell{1.284 \\ (0.042)}} & \makecell{0.476 \\ (0.016)} & \makecell{1.495 \\ (0.038)} & \makecell{0.984 \\ (0.003)} & \makecell{1.238 \\ (0.043)} \\
 & Trustscore & \makecell{0.905 \\ (0.004)} & \makecell{1.681 \\ (0.060)} & \makecell{0.522 \\ (0.016)} & \makecell{2.154 \\ (0.057)} & \makecell{0.986 \\ (0.002)} & \makecell{1.581 \\ (0.062)} \\
\midrule
500 & ACP-ACC & \makecell{0.919 \\ (0.004)} & \makecell{1.968 \\ (0.047)} & \textcolor{red}{\makecell{0.849 \\ (0.017)}} & \makecell{3.938 \\ (0.133)} & \makecell{0.936 \\ (0.004)} & \makecell{1.545 \\ (0.040)} \\
 & ACP-AEC & \makecell{0.902 \\ (0.004)} & \makecell{1.563 \\ (0.040)} & \makecell{0.687 \\ (0.023)} & \makecell{3.137 \\ (0.164)} & \makecell{0.950 \\ (0.004)} & \makecell{1.250 \\ (0.021)} \\
 & ACP-MC & \makecell{0.898 \\ (0.003)} & \makecell{1.791 \\ (0.026)} & \textcolor{red}{\makecell{0.777 \\ (0.010)}} & \makecell{3.547 \\ (0.053)} & \makecell{0.924 \\ (0.003)} & \makecell{1.417 \\ (0.021)} \\
 & ACP-MC-Oracle & \makecell{0.903 \\ (0.002)} & \textcolor{red}{\makecell{1.267 \\ (0.005)}} & \textcolor{red}{\makecell{0.904 \\ (0.003)}} & \makecell{3.000 \\ (0.045)} & \makecell{0.903 \\ (0.002)} & \makecell{0.903 \\ (0.002)} \\
 & AdaTS & \makecell{0.899 \\ (0.003)} & \makecell{1.477 \\ (0.039)} & \makecell{0.454 \\ (0.016)} & \makecell{1.559 \\ (0.034)} & \makecell{0.994 \\ (0.001)} & \makecell{1.460 \\ (0.039)} \\
 & Retrain & \makecell{0.894 \\ (0.003)} & \makecell{2.641 \\ (0.047)} & \makecell{0.700 \\ (0.009)} & \makecell{2.701 \\ (0.046)} & \makecell{0.935 \\ (0.003)} & \makecell{2.628 \\ (0.047)} \\
 & Standard & \makecell{0.897 \\ (0.003)} & \textcolor{red}{\makecell{1.210 \\ (0.017)}} & \makecell{0.449 \\ (0.011)} & \makecell{1.427 \\ (0.019)} & \makecell{0.992 \\ (0.001)} & \makecell{1.163 \\ (0.016)} \\
 & TS & \makecell{0.897 \\ (0.002)} & \textcolor{red}{\makecell{1.240 \\ (0.017)}} & \makecell{0.466 \\ (0.010)} & \makecell{1.458 \\ (0.018)} & \makecell{0.988 \\ (0.002)} & \makecell{1.193 \\ (0.016)} \\
 & Trustscore & \makecell{0.903 \\ (0.003)} & \makecell{1.426 \\ (0.035)} & \makecell{0.480 \\ (0.012)} & \makecell{1.765 \\ (0.049)} & \makecell{0.993 \\ (0.001)} & \makecell{1.353 \\ (0.033)} \\
\midrule
1000 & ACP-ACC & \makecell{0.918 \\ (0.003)} & \makecell{1.937 \\ (0.039)} & \textcolor{red}{\makecell{0.902 \\ (0.013)}} & \makecell{4.324 \\ (0.108)} & \makecell{0.921 \\ (0.004)} & \makecell{1.410 \\ (0.028)} \\
 & ACP-AEC & \makecell{0.914 \\ (0.004)} & \makecell{1.804 \\ (0.036)} & \textcolor{red}{\makecell{0.848 \\ (0.018)}} & \makecell{4.173 \\ (0.124)} & \makecell{0.929 \\ (0.003)} & \makecell{1.287 \\ (0.018)} \\
 & ACP-MC & \makecell{0.901 \\ (0.002)} & \makecell{1.764 \\ (0.019)} & \makecell{0.840 \\ (0.005)} & \makecell{3.871 \\ (0.025)} & \makecell{0.913 \\ (0.002)} & \makecell{1.311 \\ (0.015)} \\
 & ACP-MC-Oracle & \makecell{0.901 \\ (0.002)} & \textcolor{red}{\makecell{1.265 \\ (0.005)}} & \textcolor{red}{\makecell{0.898 \\ (0.004)}} & \makecell{2.986 \\ (0.041)} & \makecell{0.901 \\ (0.002)} & \makecell{0.901 \\ (0.002)} \\
 & AdaTS & \makecell{0.896 \\ (0.003)} & \makecell{1.410 \\ (0.026)} & \makecell{0.433 \\ (0.014)} & \makecell{1.505 \\ (0.022)} & \makecell{0.994 \\ (0.001)} & \makecell{1.390 \\ (0.026)} \\
 & Retrain & \makecell{0.895 \\ (0.002)} & \makecell{2.418 \\ (0.043)} & \makecell{0.676 \\ (0.009)} & \makecell{2.605 \\ (0.047)} & \makecell{0.942 \\ (0.003)} & \makecell{2.377 \\ (0.042)} \\
 & Standard & \makecell{0.897 \\ (0.002)} & \textcolor{red}{\makecell{1.194 \\ (0.013)}} & \makecell{0.448 \\ (0.011)} & \makecell{1.417 \\ (0.015)} & \makecell{0.992 \\ (0.001)} & \makecell{1.146 \\ (0.011)} \\
 & TS & \makecell{0.896 \\ (0.002)} & \textcolor{red}{\makecell{1.222 \\ (0.012)}} & \makecell{0.467 \\ (0.010)} & \makecell{1.449 \\ (0.014)} & \makecell{0.988 \\ (0.001)} & \makecell{1.173 \\ (0.011)} \\
 & Trustscore & \makecell{0.899 \\ (0.002)} & \makecell{1.327 \\ (0.027)} & \makecell{0.461 \\ (0.012)} & \makecell{1.612 \\ (0.036)} & \makecell{0.992 \\ (0.001)} & \makecell{1.266 \\ (0.026)} \\
\midrule
2000 & ACP-ACC & \makecell{0.920 \\ (0.003)} & \makecell{1.906 \\ (0.038)} & \textcolor{red}{\makecell{0.928 \\ (0.009)}} & \makecell{4.498 \\ (0.095)} & \makecell{0.918 \\ (0.002)} & \makecell{1.309 \\ (0.020)} \\
 & ACP-AEC & \makecell{0.923 \\ (0.002)} & \makecell{1.895 \\ (0.020)} & \textcolor{red}{\makecell{0.917 \\ (0.004)}} & \makecell{4.598 \\ (0.018)} & \makecell{0.924 \\ (0.002)} & \makecell{1.280 \\ (0.012)} \\
 & ACP-MC & \makecell{0.902 \\ (0.002)} & \makecell{1.754 \\ (0.017)} & \makecell{0.868 \\ (0.006)} & \makecell{3.996 \\ (0.018)} & \makecell{0.909 \\ (0.002)} & \makecell{1.260 \\ (0.011)} \\
 & ACP-MC-Oracle & \makecell{0.902 \\ (0.002)} & \textcolor{red}{\makecell{1.264 \\ (0.005)}} & \textcolor{red}{\makecell{0.898 \\ (0.004)}} & \makecell{2.958 \\ (0.045)} & \makecell{0.902 \\ (0.002)} & \makecell{0.902 \\ (0.002)} \\
 & AdaTS & \makecell{0.900 \\ (0.002)} & \makecell{1.445 \\ (0.023)} & \makecell{0.455 \\ (0.013)} & \makecell{1.533 \\ (0.020)} & \makecell{0.997 \\ (0.001)} & \makecell{1.427 \\ (0.023)} \\
 & Retrain & \makecell{0.901 \\ (0.002)} & \makecell{2.258 \\ (0.048)} & \makecell{0.697 \\ (0.009)} & \makecell{2.710 \\ (0.053)} & \makecell{0.945 \\ (0.002)} & \makecell{2.159 \\ (0.049)} \\
 & Standard & \makecell{0.900 \\ (0.002)} & \textcolor{red}{\makecell{1.208 \\ (0.011)}} & \makecell{0.465 \\ (0.010)} & \makecell{1.441 \\ (0.014)} & \makecell{0.994 \\ (0.001)} & \makecell{1.155 \\ (0.009)} \\
 & TS & \makecell{0.900 \\ (0.002)} & \textcolor{red}{\makecell{1.241 \\ (0.010)}} & \makecell{0.482 \\ (0.009)} & \makecell{1.471 \\ (0.013)} & \makecell{0.990 \\ (0.001)} & \makecell{1.189 \\ (0.009)} \\
 & Trustscore & \makecell{0.901 \\ (0.002)} & \makecell{1.295 \\ (0.023)} & \makecell{0.470 \\ (0.011)} & \makecell{1.539 \\ (0.025)} & \makecell{0.994 \\ (0.001)} & \makecell{1.240 \\ (0.022)} \\
\midrule
5000 & ACP-ACC & \makecell{0.920 \\ (0.003)} & \makecell{1.854 \\ (0.098)} & \textcolor{red}{\makecell{0.907 \\ (0.011)}} & \makecell{4.413 \\ (0.087)} & \makecell{0.923 \\ (0.003)} & \makecell{1.306 \\ (0.111)} \\
 & ACP-AEC & \makecell{0.922 \\ (0.001)} & \makecell{1.827 \\ (0.017)} & \textcolor{red}{\makecell{0.913 \\ (0.004)}} & \makecell{4.590 \\ (0.019)} & \makecell{0.923 \\ (0.002)} & \makecell{1.222 \\ (0.008)} \\
 & ACP-MC & \makecell{0.900 \\ (0.002)} & \makecell{1.705 \\ (0.014)} & \makecell{0.870 \\ (0.005)} & \makecell{4.014 \\ (0.013)} & \makecell{0.906 \\ (0.002)} & \makecell{1.194 \\ (0.007)} \\
 & ACP-MC-Oracle & \makecell{0.902 \\ (0.002)} & \makecell{1.264 \\ (0.005)} & \textcolor{red}{\makecell{0.902 \\ (0.004)}} & \makecell{2.959 \\ (0.047)} & \makecell{0.902 \\ (0.002)} & \makecell{0.902 \\ (0.002)} \\
 & AdaTS & \makecell{0.898 \\ (0.002)} & \makecell{1.437 \\ (0.023)} & \makecell{0.446 \\ (0.013)} & \makecell{1.526 \\ (0.021)} & \makecell{0.995 \\ (0.001)} & \makecell{1.418 \\ (0.023)} \\
 & Retrain & \makecell{0.901 \\ (0.001)} & \makecell{1.736 \\ (0.039)} & \makecell{0.732 \\ (0.009)} & \makecell{2.745 \\ (0.058)} & \makecell{0.938 \\ (0.002)} & \makecell{1.520 \\ (0.038)} \\
 & Standard & \makecell{0.899 \\ (0.002)} & \textcolor{red}{\makecell{1.206 \\ (0.013)}} & \makecell{0.461 \\ (0.011)} & \makecell{1.437 \\ (0.015)} & \makecell{0.994 \\ (0.001)} & \makecell{1.154 \\ (0.011)} \\
 & TS & \makecell{0.899 \\ (0.002)} & \textcolor{red}{\makecell{1.240 \\ (0.013)}} & \makecell{0.477 \\ (0.011)} & \makecell{1.466 \\ (0.014)} & \makecell{0.990 \\ (0.001)} & \makecell{1.189 \\ (0.011)} \\
 & Trustscore & \makecell{0.900 \\ (0.001)} & \textcolor{red}{\makecell{1.243 \\ (0.017)}} & \makecell{0.460 \\ (0.011)} & \makecell{1.459 \\ (0.017)} & \makecell{0.995 \\ (0.001)} & \makecell{1.194 \\ (0.016)} \\
\bottomrule
\end{tabular}

\end{table}


\begin{figure*}[tbh]
    \centering
    \includegraphics[width=0.9\linewidth]{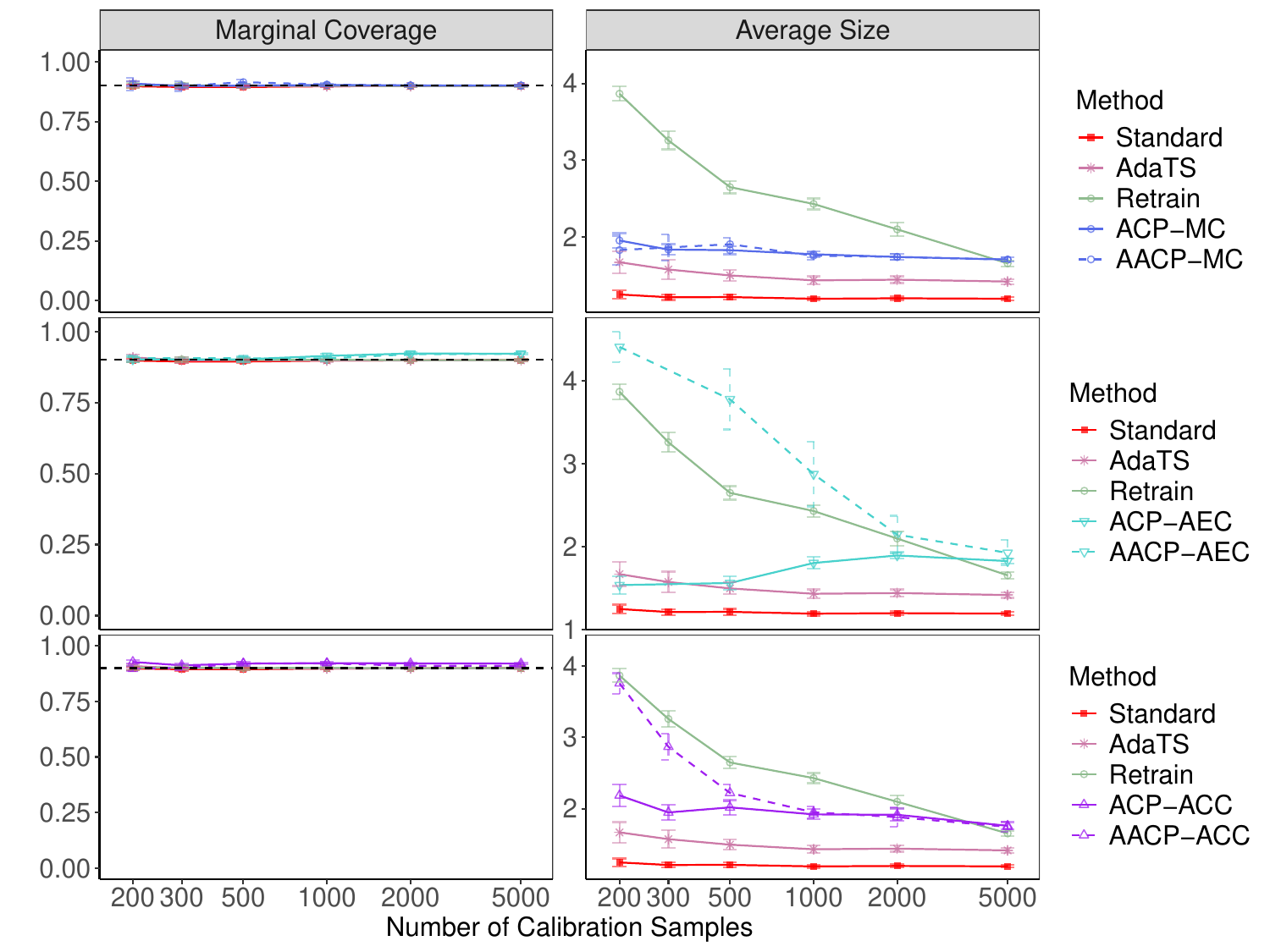}
    \caption{Performance of prediction sets constructed with different methods (including the Adaptive ACP (AACP) variants) for $5$-class synthetic data generated from Example~\ref{exp:concept-shift} as a function of the total calibration sample size. The AACP variants (AACP-MC, AACP-AEC, and AACP-ACC), compare retrain with the ACP methods and select the method that achieves higher conditional coverage. In this experiments, we set the selection threshold $\epsilon = 0.7$. Error bars denote two standard errors. Corresponding conditional performance are provided in Figure~\ref{fig:supp_exp_sim_nnew_adaptive_cond}. }
    \label{fig:supp_exp_sim_nnew_adaptive_marg}
\end{figure*}

\begin{figure*}[tbh]
    \centering
    \includegraphics[width=0.9\linewidth]{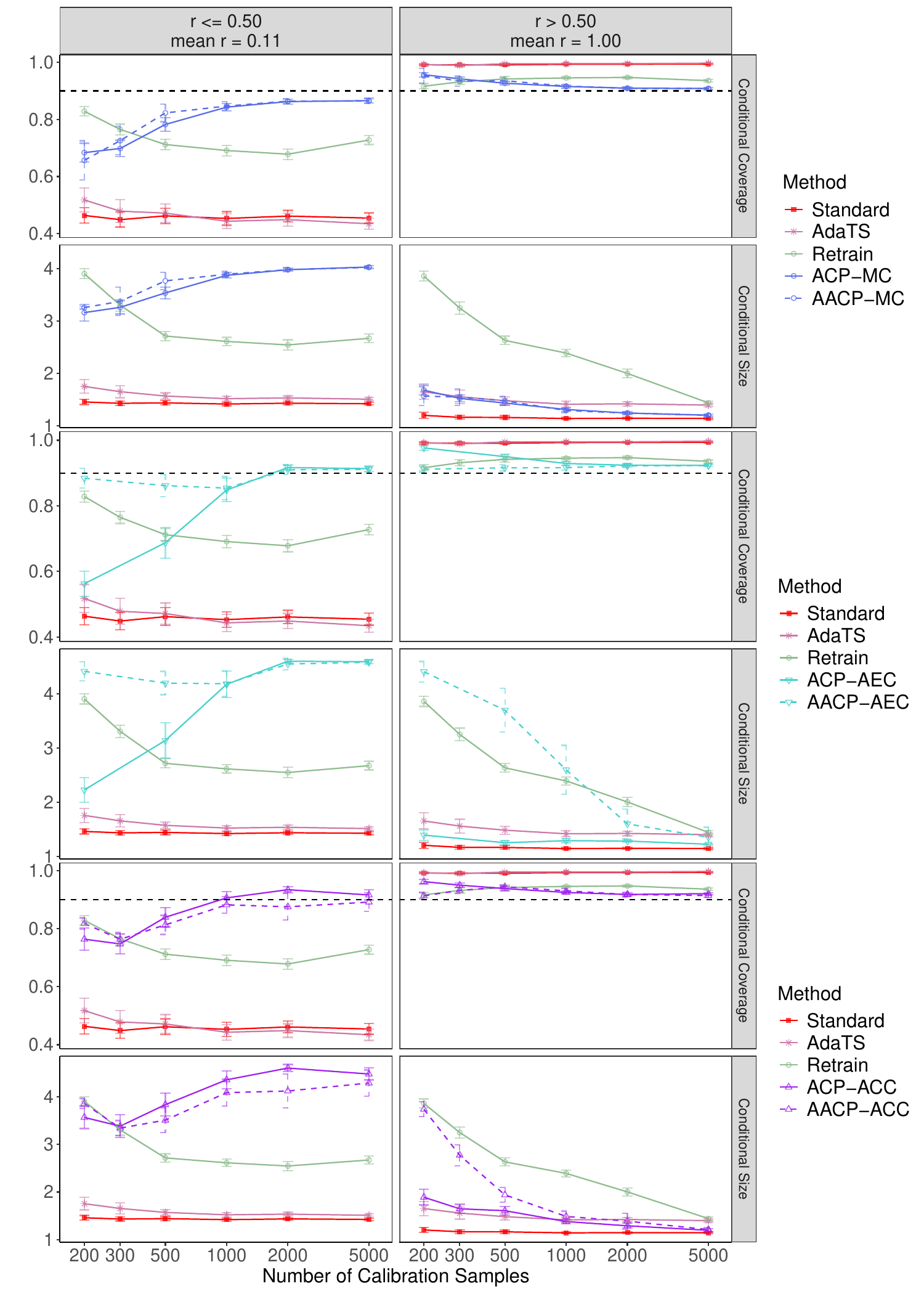}
    \caption{Conditional performance of the benchmark methods and the Adaptive ACP (AACP) variants. Corresponding marginal performance are provided in Figure~\ref{fig:supp_exp_sim_nnew_adaptive_marg} and numerical details are provided in Table~\ref{tab:supp_exp_sim_nnew_adaptive}. }
    \label{fig:supp_exp_sim_nnew_adaptive_cond}
\end{figure*}

\begin{table}[!htb]
\centering
    \caption{Performance of conformal prediction sets constructed with different methods including adaptive ACP (AACP) for $5$-class synthetic data generated from Example~\ref{exp:concept-shift}, as a function the of the total calibration sample size. See the corresponding plots in Figure~\ref{fig:supp_exp_sim_nnew_adaptive_marg}--\ref{fig:supp_exp_sim_nnew_adaptive_cond}.}
  \label{tab:supp_exp_sim_nnew_adaptive}
\centering
\fontsize{5}{5}\selectfont
\begin{tabular}[t]{rlllllll}
\toprule
\multicolumn{2}{c}{ } & \multicolumn{2}{c}{\makecell{Marginal}} & \multicolumn{2}{c}{\makecell{Hard Bin \\ ($r^* \leq 0.5$)}} & \multicolumn{2}{c}{\makecell{Easy Bin \\ ($r^* > 0.5$)}} \\
\cmidrule(l{3pt}r{3pt}){3-4} \cmidrule(l{3pt}r{3pt}){5-6} \cmidrule(l{3pt}r{3pt}){7-8}
\makecell{Number of Calibration Samples} & Method & Coverage & Size & Coverage & Size & Coverage & Size \\
\midrule
200 & AACP-ACC & \makecell{0.897 \\ (0.006)} & \makecell{3.756 \\ (0.072)} & \textcolor{red}{\makecell{0.818 \\ (0.010)}} & \makecell{3.857 \\ (0.053)} & \makecell{0.913 \\ (0.005)} & \makecell{3.737 \\ (0.078)} \\
 & AACP-AEC & \makecell{0.906 \\ (0.005)} & \makecell{4.407 \\ (0.093)} & \textcolor{red}{\makecell{0.884 \\ (0.015)}} & \makecell{4.412 \\ (0.088)} & \makecell{0.912 \\ (0.005)} & \makecell{4.405 \\ (0.094)} \\
 & AACP-MC & \makecell{0.907 \\ (0.013)} & \makecell{1.828 \\ (0.094)} & \makecell{0.657 \\ (0.034)} & \makecell{3.259 \\ (0.014)} & \makecell{0.952 \\ (0.013)} & \makecell{1.568 \\ (0.087)} \\
 & ACP-ACC & \makecell{0.927 \\ (0.004)} & \makecell{2.186 \\ (0.077)} & \makecell{0.764 \\ (0.019)} & \makecell{3.565 \\ (0.116)} & \makecell{0.962 \\ (0.004)} & \makecell{1.890 \\ (0.082)} \\
 & ACP-AEC & \makecell{0.902 \\ (0.005)} & \textcolor{red}{\makecell{1.537 \\ (0.054)}} & \makecell{0.562 \\ (0.019)} & \makecell{2.224 \\ (0.115)} & \makecell{0.977 \\ (0.004)} & \makecell{1.390 \\ (0.050)} \\
 & ACP-MC & \makecell{0.909 \\ (0.005)} & \makecell{1.953 \\ (0.048)} & \makecell{0.684 \\ (0.016)} & \makecell{3.162 \\ (0.078)} & \makecell{0.957 \\ (0.004)} & \makecell{1.683 \\ (0.045)} \\
 & AdaTS & \makecell{0.909 \\ (0.005)} & \textcolor{red}{\makecell{1.670 \\ (0.072)}} & \makecell{0.518 \\ (0.021)} & \makecell{1.754 \\ (0.065)} & \makecell{0.992 \\ (0.002)} & \makecell{1.652 \\ (0.074)} \\
 & Retrain & \makecell{0.901 \\ (0.005)} & \makecell{3.866 \\ (0.047)} & \textcolor{red}{\makecell{0.828 \\ (0.008)}} & \makecell{3.903 \\ (0.046)} & \makecell{0.916 \\ (0.005)} & \makecell{3.858 \\ (0.047)} \\
 & Standard & \makecell{0.898 \\ (0.003)} & \textcolor{red}{\makecell{1.249 \\ (0.027)}} & \makecell{0.463 \\ (0.013)} & \makecell{1.458 \\ (0.024)} & \makecell{0.991 \\ (0.001)} & \makecell{1.203 \\ (0.028)} \\
\midrule
300 & AACP-ACC & \makecell{0.901 \\ (0.006)} & \makecell{2.869 \\ (0.093)} & \textcolor{red}{\makecell{0.763 \\ (0.011)}} & \makecell{3.340 \\ (0.080)} & \makecell{0.931 \\ (0.006)} & \makecell{2.771 \\ (0.110)} \\
 & AACP-MC & \makecell{0.898 \\ (0.010)} & \makecell{1.862 \\ (0.084)} & \makecell{0.726 \\ (0.025)} & \makecell{3.376 \\ (0.136)} & \makecell{0.934 \\ (0.009)} & \makecell{1.553 \\ (0.080)} \\
 & ACP-ACC & \makecell{0.912 \\ (0.005)} & \makecell{1.949 \\ (0.052)} & \textcolor{red}{\makecell{0.747 \\ (0.017)}} & \makecell{3.381 \\ (0.119)} & \makecell{0.950 \\ (0.004)} & \makecell{1.647 \\ (0.047)} \\
 & ACP-MC & \makecell{0.899 \\ (0.004)} & \textcolor{red}{\makecell{1.834 \\ (0.035)}} & \makecell{0.698 \\ (0.014)} & \makecell{3.263 \\ (0.062)} & \makecell{0.942 \\ (0.003)} & \makecell{1.525 \\ (0.033)} \\
 & AdaTS & \makecell{0.899 \\ (0.005)} & \textcolor{red}{\makecell{1.575 \\ (0.063)}} & \makecell{0.479 \\ (0.020)} & \makecell{1.654 \\ (0.057)} & \makecell{0.989 \\ (0.003)} & \makecell{1.557 \\ (0.064)} \\
 & Retrain & \makecell{0.902 \\ (0.004)} & \makecell{3.258 \\ (0.058)} & \textcolor{red}{\makecell{0.765 \\ (0.009)}} & \makecell{3.304 \\ (0.057)} & \makecell{0.931 \\ (0.005)} & \makecell{3.248 \\ (0.059)} \\
 & Standard & \makecell{0.895 \\ (0.003)} & \textcolor{red}{\makecell{1.213 \\ (0.018)}} & \makecell{0.449 \\ (0.013)} & \makecell{1.432 \\ (0.021)} & \makecell{0.991 \\ (0.002)} & \makecell{1.164 \\ (0.018)} \\
\midrule
500 & AACP-ACC & \makecell{0.920 \\ (0.003)} & \makecell{2.222 \\ (0.059)} & \makecell{0.813 \\ (0.016)} & \makecell{3.508 \\ (0.132)} & \makecell{0.944 \\ (0.003)} & \makecell{1.939 \\ (0.078)} \\
 & AACP-AEC & \makecell{0.906 \\ (0.004)} & \makecell{3.776 \\ (0.182)} & \textcolor{red}{\makecell{0.862 \\ (0.017)}} & \makecell{4.196 \\ (0.108)} & \makecell{0.916 \\ (0.005)} & \makecell{3.695 \\ (0.201)} \\
 & AACP-MC & \makecell{0.915 \\ (0.006)} & \makecell{1.905 \\ (0.041)} & \textcolor{red}{\makecell{0.823 \\ (0.015)}} & \makecell{3.768 \\ (0.081)} & \makecell{0.936 \\ (0.005)} & \makecell{1.480 \\ (0.041)} \\
 & ACP-ACC & \makecell{0.920 \\ (0.005)} & \makecell{2.020 \\ (0.051)} & \textcolor{red}{\makecell{0.839 \\ (0.017)}} & \makecell{3.833 \\ (0.119)} & \makecell{0.938 \\ (0.004)} & \makecell{1.607 \\ (0.044)} \\
 & ACP-AEC & \makecell{0.902 \\ (0.004)} & \textcolor{red}{\makecell{1.563 \\ (0.040)}} & \makecell{0.687 \\ (0.023)} & \makecell{3.137 \\ (0.164)} & \makecell{0.950 \\ (0.004)} & \makecell{1.250 \\ (0.021)} \\
 & ACP-MC & \makecell{0.901 \\ (0.003)} & \makecell{1.827 \\ (0.027)} & \makecell{0.782 \\ (0.011)} & \makecell{3.538 \\ (0.057)} & \makecell{0.927 \\ (0.003)} & \makecell{1.437 \\ (0.022)} \\
 & AdaTS & \makecell{0.900 \\ (0.003)} & \textcolor{red}{\makecell{1.497 \\ (0.036)}} & \makecell{0.472 \\ (0.016)} & \makecell{1.571 \\ (0.031)} & \makecell{0.995 \\ (0.001)} & \makecell{1.482 \\ (0.037)} \\
 & Retrain & \makecell{0.900 \\ (0.004)} & \makecell{2.649 \\ (0.041)} & \makecell{0.712 \\ (0.009)} & \makecell{2.715 \\ (0.043)} & \makecell{0.942 \\ (0.004)} & \makecell{2.632 \\ (0.041)} \\
 & Standard & \makecell{0.894 \\ (0.003)} & \textcolor{red}{\makecell{1.215 \\ (0.018)}} & \makecell{0.462 \\ (0.014)} & \makecell{1.440 \\ (0.021)} & \makecell{0.991 \\ (0.002)} & \makecell{1.163 \\ (0.017)} \\
\midrule
1000 & AACP-ACC & \makecell{0.921 \\ (0.004)} & \makecell{1.950 \\ (0.042)} & \textcolor{red}{\makecell{0.882 \\ (0.015)}} & \makecell{4.085 \\ (0.140)} & \makecell{0.930 \\ (0.004)} & \makecell{1.486 \\ (0.055)} \\
 & AACP-AEC & \makecell{0.905 \\ (0.004)} & \makecell{2.874 \\ (0.195)} & \textcolor{red}{\makecell{0.854 \\ (0.017)}} & \makecell{4.182 \\ (0.117)} & \makecell{0.916 \\ (0.004)} & \makecell{2.594 \\ (0.227)} \\
 & AACP-MC & \makecell{0.905 \\ (0.003)} & \textcolor{red}{\makecell{1.755 \\ (0.025)}} & \makecell{0.847 \\ (0.008)} & \makecell{3.894 \\ (0.028)} & \makecell{0.917 \\ (0.003)} & \makecell{1.294 \\ (0.019)} \\
 & ACP-ACC & \makecell{0.922 \\ (0.003)} & \makecell{1.922 \\ (0.036)} & \textcolor{red}{\makecell{0.906 \\ (0.011)}} & \makecell{4.355 \\ (0.093)} & \makecell{0.925 \\ (0.002)} & \makecell{1.379 \\ (0.023)} \\
 & ACP-AEC & \makecell{0.914 \\ (0.004)} & \makecell{1.804 \\ (0.036)} & \makecell{0.848 \\ (0.018)} & \makecell{4.173 \\ (0.124)} & \makecell{0.929 \\ (0.003)} & \makecell{1.287 \\ (0.018)} \\
 & ACP-MC & \makecell{0.903 \\ (0.002)} & \makecell{1.774 \\ (0.020)} & \makecell{0.843 \\ (0.006)} & \makecell{3.873 \\ (0.025)} & \makecell{0.915 \\ (0.003)} & \makecell{1.314 \\ (0.016)} \\
 & AdaTS & \makecell{0.898 \\ (0.002)} & \textcolor{red}{\makecell{1.433 \\ (0.027)}} & \makecell{0.443 \\ (0.013)} & \makecell{1.520 \\ (0.023)} & \makecell{0.995 \\ (0.001)} & \makecell{1.415 \\ (0.027)} \\
 & Retrain & \makecell{0.901 \\ (0.002)} & \makecell{2.430 \\ (0.036)} & \makecell{0.691 \\ (0.009)} & \makecell{2.612 \\ (0.038)} & \makecell{0.945 \\ (0.002)} & \makecell{2.390 \\ (0.036)} \\
 & Standard & \makecell{0.898 \\ (0.002)} & \textcolor{red}{\makecell{1.193 \\ (0.012)}} & \makecell{0.453 \\ (0.012)} & \makecell{1.419 \\ (0.016)} & \makecell{0.993 \\ (0.001)} & \makecell{1.142 \\ (0.009)} \\
\midrule
2000 & AACP-ACC & \makecell{0.910 \\ (0.004)} & \makecell{1.883 \\ (0.067)} & \makecell{0.875 \\ (0.023)} & \makecell{4.120 \\ (0.177)} & \makecell{0.917 \\ (0.004)} & \makecell{1.387 \\ (0.083)} \\
 & AACP-AEC & \makecell{0.920 \\ (0.002)} & \makecell{2.145 \\ (0.114)} & \textcolor{red}{\makecell{0.910 \\ (0.007)}} & \makecell{4.547 \\ (0.053)} & \makecell{0.922 \\ (0.003)} & \makecell{1.597 \\ (0.137)} \\
 & AACP-MC & \makecell{0.901 \\ (0.002)} & \makecell{1.738 \\ (0.019)} & \makecell{0.864 \\ (0.005)} & \makecell{3.981 \\ (0.018)} & \makecell{0.909 \\ (0.003)} & \makecell{1.246 \\ (0.013)} \\
 & ACP-ACC & \makecell{0.921 \\ (0.002)} & \makecell{1.916 \\ (0.043)} & \textcolor{red}{\makecell{0.934 \\ (0.005)}} & \makecell{4.601 \\ (0.038)} & \makecell{0.918 \\ (0.003)} & \makecell{1.290 \\ (0.026)} \\
 & ACP-AEC & \makecell{0.923 \\ (0.002)} & \makecell{1.895 \\ (0.020)} & \textcolor{red}{\makecell{0.917 \\ (0.004)}} & \makecell{4.598 \\ (0.018)} & \makecell{0.924 \\ (0.002)} & \makecell{1.280 \\ (0.012)} \\
 & ACP-MC & \makecell{0.901 \\ (0.002)} & \textcolor{red}{\makecell{1.738 \\ (0.018)}} & \makecell{0.862 \\ (0.004)} & \makecell{3.980 \\ (0.014)} & \makecell{0.909 \\ (0.002)} & \makecell{1.246 \\ (0.012)} \\
 & AdaTS & \makecell{0.898 \\ (0.002)} & \textcolor{red}{\makecell{1.442 \\ (0.022)}} & \makecell{0.449 \\ (0.012)} & \makecell{1.535 \\ (0.019)} & \makecell{0.995 \\ (0.001)} & \makecell{1.422 \\ (0.022)} \\
 & Retrain & \makecell{0.899 \\ (0.002)} & \makecell{2.099 \\ (0.043)} & \makecell{0.678 \\ (0.009)} & \makecell{2.546 \\ (0.049)} & \makecell{0.947 \\ (0.002)} & \makecell{2.001 \\ (0.042)} \\
 & Standard & \makecell{0.899 \\ (0.002)} & \textcolor{red}{\makecell{1.199 \\ (0.011)}} & \makecell{0.461 \\ (0.010)} & \makecell{1.434 \\ (0.014)} & \makecell{0.993 \\ (0.001)} & \makecell{1.147 \\ (0.009)} \\
\midrule
5000 & AACP-ACC & \makecell{0.910 \\ (0.004)} & \makecell{1.748 \\ (0.037)} & \makecell{0.891 \\ (0.016)} & \makecell{4.286 \\ (0.137)} & \makecell{0.914 \\ (0.003)} & \makecell{1.211 \\ (0.024)} \\
 & AACP-AEC & \makecell{0.921 \\ (0.002)} & \makecell{1.926 \\ (0.076)} & \textcolor{red}{\makecell{0.912 \\ (0.004)}} & \makecell{4.580 \\ (0.019)} & \makecell{0.923 \\ (0.002)} & \makecell{1.351 \\ (0.092)} \\
 & AACP-MC & \makecell{0.901 \\ (0.002)} & \makecell{1.703 \\ (0.013)} & \makecell{0.864 \\ (0.005)} & \makecell{4.024 \\ (0.013)} & \makecell{0.908 \\ (0.002)} & \makecell{1.204 \\ (0.010)} \\
 & ACP-ACC & \makecell{0.920 \\ (0.002)} & \makecell{1.762 \\ (0.026)} & \textcolor{red}{\makecell{0.916 \\ (0.009)}} & \makecell{4.476 \\ (0.062)} & \makecell{0.920 \\ (0.002)} & \makecell{1.194 \\ (0.015)} \\
 & ACP-AEC & \makecell{0.922 \\ (0.001)} & \makecell{1.827 \\ (0.017)} & \textcolor{red}{\makecell{0.913 \\ (0.004)}} & \makecell{4.590 \\ (0.019)} & \makecell{0.923 \\ (0.002)} & \makecell{1.222 \\ (0.008)} \\
 & ACP-MC & \makecell{0.901 \\ (0.002)} & \makecell{1.705 \\ (0.013)} & \makecell{0.866 \\ (0.004)} & \makecell{4.032 \\ (0.012)} & \makecell{0.908 \\ (0.002)} & \makecell{1.205 \\ (0.010)} \\
 & AdaTS & \makecell{0.899 \\ (0.002)} & \textcolor{red}{\makecell{1.418 \\ (0.017)}} & \makecell{0.434 \\ (0.010)} & \makecell{1.510 \\ (0.015)} & \makecell{0.997 \\ (0.001)} & \makecell{1.398 \\ (0.017)} \\
 & Retrain & \makecell{0.899 \\ (0.002)} & \textcolor{red}{\makecell{1.655 \\ (0.019)}} & \makecell{0.728 \\ (0.008)} & \makecell{2.671 \\ (0.040)} & \makecell{0.936 \\ (0.002)} & \makecell{1.441 \\ (0.018)} \\
 & Standard & \makecell{0.899 \\ (0.002)} & \textcolor{red}{\makecell{1.194 \\ (0.009)}} & \makecell{0.454 \\ (0.009)} & \makecell{1.424 \\ (0.013)} & \makecell{0.994 \\ (0.001)} & \makecell{1.144 \\ (0.008)} \\
\bottomrule
\end{tabular}

\end{table}


\begin{figure*}[tbh]
    \centering
    \includegraphics[width=0.9\linewidth]{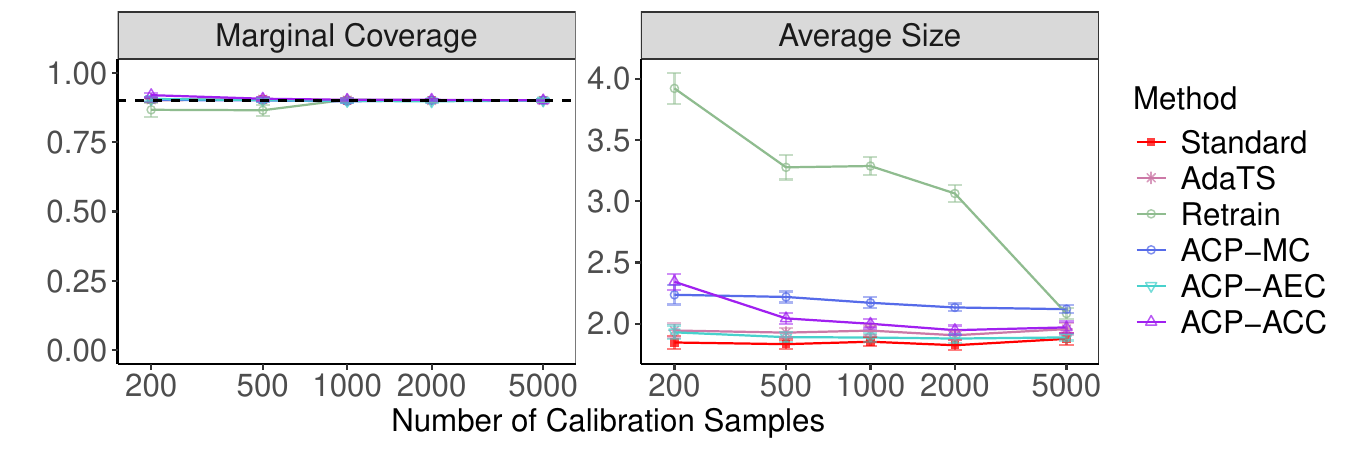}
    \caption{Performance of prediction sets for $5$-class synthetic data generated from Example~\ref{exp:covariate-shift} with covariate shift at $a = 0.5$, as a function of the total calibration sample size. All methods achieve the target marginal coverage of $90\%$, while our methods retain practically efficient sets sizes. Error bars denote two standard errors. Corresponding conditional performance are provided in Figure~\ref{fig:supp_exp_sim_n_new_covshift_cond}. }
\label{fig:supp_exp_sim_n_new_covshift_marg}
\end{figure*}

\begin{figure*}[tbh]
    \centering
    \includegraphics[width=0.9\linewidth]{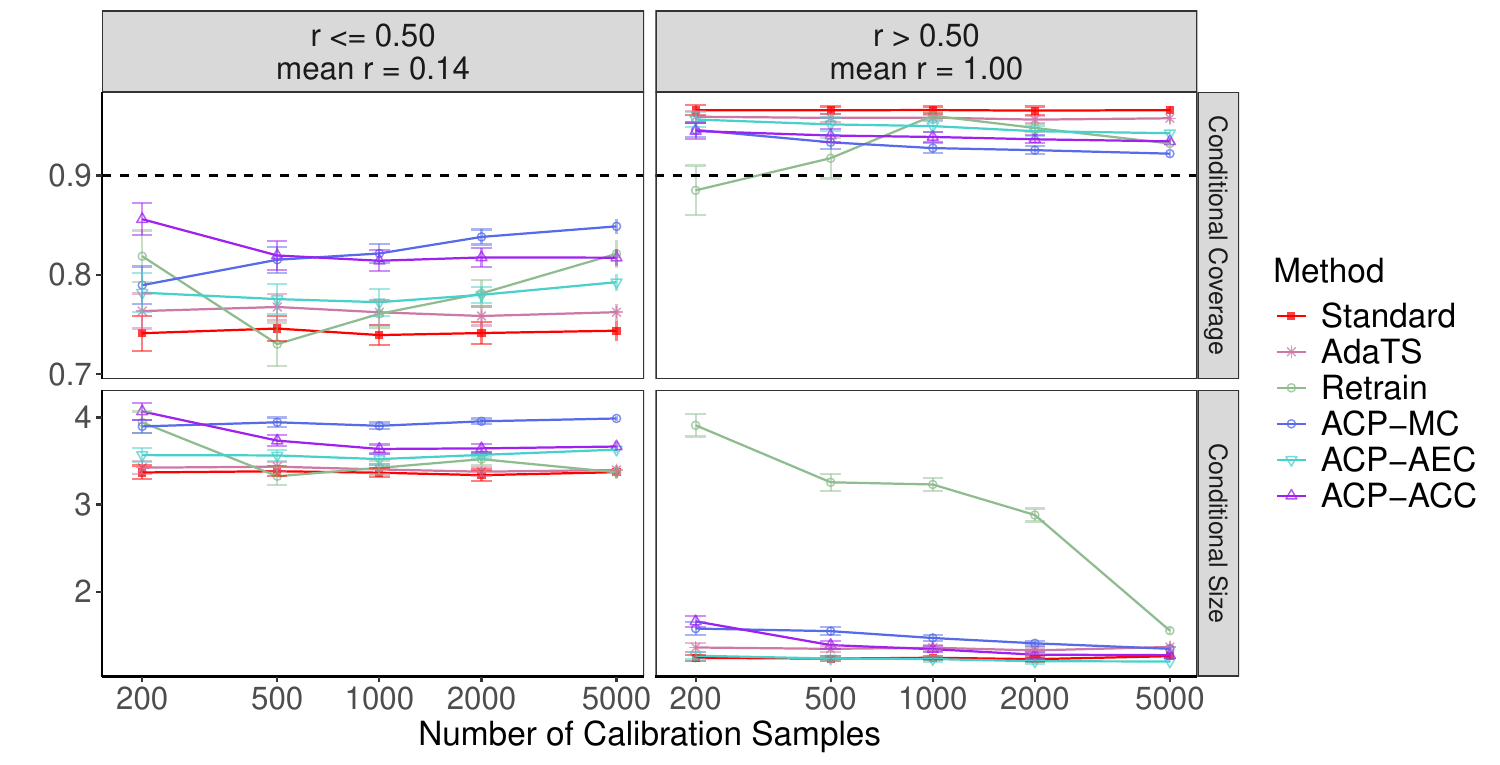}
    \caption{Our methods are more uncertainty-aware: they attain higher conditional coverage on hard samples while assigning larger prediction sets to hard samples and smaller ones to easy samples. Corresponding marginal performance are provided in Figure~\ref{fig:supp_exp_sim_n_new_covshift_marg} and numerical details are provided in Table~\ref{tab:supp_exp_sim_nnew_covshift}. }
    \label{fig:supp_exp_sim_n_new_covshift_cond}
\end{figure*}

\begin{table}[!htb]
\centering
    \caption{Performance of prediction sets for $5$-class synthetic data generated from Example~\ref{exp:covariate-shift} with covariate shift at $a = 0.5$, as a function the of the total calibration sample size. See the corresponding plots in Figure~\ref{fig:supp_exp_sim_n_new_covshift_marg}--\ref{fig:supp_exp_sim_n_new_covshift_cond}.}
  \label{tab:supp_exp_sim_nnew_covshift}
\centering
\fontsize{6}{6}\selectfont
\begin{tabular}[t]{rlllllll}
\toprule
\multicolumn{2}{c}{ } & \multicolumn{2}{c}{\makecell{Marginal}} & \multicolumn{2}{c}{\makecell{Hard Bin \\ ($r^* \leq 0.5$)}} & \multicolumn{2}{c}{\makecell{Easy Bin \\ ($r^* > 0.5$)}} \\
\cmidrule(l{3pt}r{3pt}){3-4} \cmidrule(l{3pt}r{3pt}){5-6} \cmidrule(l{3pt}r{3pt}){7-8}
\makecell{Number \\ of \\ Calibration \\ Samples} & Method & Coverage & Size & Coverage & Size & Coverage & Size \\
\midrule
200 & ACP-ACC & \makecell{0.920 \\ (0.004)} & \makecell{2.341 \\ (0.033)} & \textcolor{red}{\makecell{0.856 \\ (0.008)}} & \makecell{4.070 \\ (0.047)} & \makecell{0.945 \\ (0.004)} & \makecell{1.659 \\ (0.033)} \\
 & ACP-AEC & \makecell{0.906 \\ (0.004)} & \textcolor{red}{\makecell{1.929 \\ (0.028)}} & \makecell{0.782 \\ (0.010)} & \makecell{3.569 \\ (0.041)} & \makecell{0.957 \\ (0.004)} & \makecell{1.264 \\ (0.022)} \\
 & ACP-MC & \makecell{0.902 \\ (0.004)} & \makecell{2.236 \\ (0.040)} & \textcolor{red}{\makecell{0.790 \\ (0.009)}} & \makecell{3.898 \\ (0.039)} & \makecell{0.946 \\ (0.004)} & \makecell{1.576 \\ (0.038)} \\
 & AdaTS & \makecell{0.903 \\ (0.004)} & \textcolor{red}{\makecell{1.944 \\ (0.028)}} & \makecell{0.763 \\ (0.009)} & \makecell{3.424 \\ (0.037)} & \makecell{0.959 \\ (0.003)} & \makecell{1.359 \\ (0.026)} \\
 & Retrain & \makecell{0.867 \\ (0.012)} & \makecell{3.921 \\ (0.063)} & \textcolor{red}{\makecell{0.819 \\ (0.013)}} & \makecell{3.944 \\ (0.063)} & \makecell{0.885 \\ (0.013)} & \makecell{3.911 \\ (0.064)} \\
 & Standard & \makecell{0.902 \\ (0.003)} & \textcolor{red}{\makecell{1.846 \\ (0.026)}} & \makecell{0.741 \\ (0.009)} & \makecell{3.369 \\ (0.040)} & \makecell{0.966 \\ (0.002)} & \makecell{1.240 \\ (0.018)} \\
\midrule
500 & ACP-ACC & \makecell{0.907 \\ (0.003)} & \makecell{2.043 \\ (0.023)} & \textcolor{red}{\makecell{0.819 \\ (0.007)}} & \makecell{3.735 \\ (0.031)} & \makecell{0.940 \\ (0.003)} & \makecell{1.386 \\ (0.021)} \\
 & ACP-AEC & \makecell{0.902 \\ (0.003)} & \textcolor{red}{\makecell{1.890 \\ (0.021)}} & \textcolor{red}{\makecell{0.775 \\ (0.008)}} & \makecell{3.565 \\ (0.033)} & \makecell{0.952 \\ (0.003)} & \makecell{1.232 \\ (0.015)} \\
 & ACP-MC & \makecell{0.901 \\ (0.003)} & \makecell{2.218 \\ (0.023)} & \textcolor{red}{\makecell{0.815 \\ (0.007)}} & \makecell{3.945 \\ (0.027)} & \makecell{0.934 \\ (0.004)} & \makecell{1.547 \\ (0.024)} \\
 & AdaTS & \makecell{0.904 \\ (0.003)} & \textcolor{red}{\makecell{1.927 \\ (0.019)}} & \makecell{0.767 \\ (0.007)} & \makecell{3.435 \\ (0.027)} & \makecell{0.958 \\ (0.002)} & \makecell{1.343 \\ (0.018)} \\
 & Retrain & \makecell{0.865 \\ (0.010)} & \makecell{3.276 \\ (0.050)} & \makecell{0.730 \\ (0.011)} & \makecell{3.325 \\ (0.052)} & \makecell{0.918 \\ (0.010)} & \makecell{3.256 \\ (0.050)} \\
 & Standard & \makecell{0.904 \\ (0.002)} & \textcolor{red}{\makecell{1.833 \\ (0.019)}} & \makecell{0.746 \\ (0.006)} & \makecell{3.383 \\ (0.029)} & \makecell{0.966 \\ (0.002)} & \makecell{1.231 \\ (0.014)} \\
\midrule
1000 & ACP-ACC & \makecell{0.903 \\ (0.002)} & \makecell{1.999 \\ (0.019)} & \textcolor{red}{\makecell{0.814 \\ (0.005)}} & \makecell{3.639 \\ (0.026)} & \makecell{0.939 \\ (0.003)} & \makecell{1.339 \\ (0.016)} \\
 & ACP-AEC & \makecell{0.899 \\ (0.003)} & \textcolor{red}{\makecell{1.887 \\ (0.018)}} & \textcolor{red}{\makecell{0.772 \\ (0.007)}} & \makecell{3.523 \\ (0.031)} & \makecell{0.950 \\ (0.003)} & \makecell{1.226 \\ (0.014)} \\
 & ACP-MC & \makecell{0.897 \\ (0.002)} & \makecell{2.170 \\ (0.022)} & \textcolor{red}{\makecell{0.821 \\ (0.005)}} & \makecell{3.905 \\ (0.019)} & \makecell{0.928 \\ (0.003)} & \makecell{1.468 \\ (0.019)} \\
 & AdaTS & \makecell{0.902 \\ (0.002)} & \textcolor{red}{\makecell{1.943 \\ (0.020)}} & \makecell{0.762 \\ (0.006)} & \makecell{3.403 \\ (0.028)} & \makecell{0.958 \\ (0.002)} & \makecell{1.356 \\ (0.016)} \\
 & Retrain & \makecell{0.903 \\ (0.002)} & \makecell{3.287 \\ (0.036)} & \makecell{0.761 \\ (0.007)} & \makecell{3.422 \\ (0.037)} & \makecell{0.960 \\ (0.002)} & \makecell{3.231 \\ (0.036)} \\
 & Standard & \makecell{0.901 \\ (0.002)} & \textcolor{red}{\makecell{1.852 \\ (0.019)}} & \makecell{0.739 \\ (0.005)} & \makecell{3.366 \\ (0.027)} & \makecell{0.966 \\ (0.002)} & \makecell{1.241 \\ (0.012)} \\
\midrule
2000 & ACP-ACC & \makecell{0.903 \\ (0.002)} & \makecell{1.947 \\ (0.019)} & \textcolor{red}{\makecell{0.817 \\ (0.005)}} & \makecell{3.646 \\ (0.023)} & \makecell{0.937 \\ (0.002)} & \makecell{1.276 \\ (0.017)} \\
 & ACP-AEC & \makecell{0.897 \\ (0.002)} & \textcolor{red}{\makecell{1.878 \\ (0.017)}} & \makecell{0.780 \\ (0.004)} & \makecell{3.571 \\ (0.015)} & \makecell{0.945 \\ (0.003)} & \makecell{1.200 \\ (0.013)} \\
 & ACP-MC & \makecell{0.901 \\ (0.002)} & \makecell{2.132 \\ (0.016)} & \textcolor{red}{\makecell{0.838 \\ (0.004)}} & \makecell{3.957 \\ (0.016)} & \makecell{0.926 \\ (0.002)} & \makecell{1.407 \\ (0.012)} \\
 & AdaTS & \makecell{0.900 \\ (0.002)} & \textcolor{red}{\makecell{1.905 \\ (0.018)}} & \makecell{0.758 \\ (0.005)} & \makecell{3.378 \\ (0.027)} & \makecell{0.956 \\ (0.002)} & \makecell{1.325 \\ (0.017)} \\
 & Retrain & \makecell{0.901 \\ (0.002)} & \makecell{3.063 \\ (0.036)} & \textcolor{red}{\makecell{0.781 \\ (0.007)}} & \makecell{3.522 \\ (0.035)} & \makecell{0.948 \\ (0.002)} & \makecell{2.881 \\ (0.037)} \\
 & Standard & \makecell{0.901 \\ (0.001)} & \textcolor{red}{\makecell{1.824 \\ (0.020)}} & \makecell{0.741 \\ (0.006)} & \makecell{3.338 \\ (0.032)} & \makecell{0.965 \\ (0.002)} & \makecell{1.225 \\ (0.015)} \\
\midrule
5000 & ACP-ACC & \makecell{0.900 \\ (0.001)} & \makecell{1.969 \\ (0.024)} & \textcolor{red}{\makecell{0.817 \\ (0.004)}} & \makecell{3.667 \\ (0.024)} & \makecell{0.934 \\ (0.001)} & \makecell{1.271 \\ (0.016)} \\
 & ACP-AEC & \makecell{0.900 \\ (0.002)} & \textcolor{red}{\makecell{1.888 \\ (0.013)}} & \makecell{0.792 \\ (0.004)} & \makecell{3.631 \\ (0.017)} & \makecell{0.943 \\ (0.002)} & \makecell{1.199 \\ (0.010)} \\
 & ACP-MC & \makecell{0.901 \\ (0.001)} & \makecell{2.118 \\ (0.016)} & \textcolor{red}{\makecell{0.849 \\ (0.004)}} & \makecell{3.989 \\ (0.014)} & \makecell{0.922 \\ (0.002)} & \makecell{1.344 \\ (0.011)} \\
 & AdaTS & \makecell{0.900 \\ (0.002)} & \textcolor{red}{\makecell{1.956 \\ (0.024)}} & \makecell{0.762 \\ (0.004)} & \makecell{3.397 \\ (0.030)} & \makecell{0.958 \\ (0.002)} & \makecell{1.365 \\ (0.021)} \\
 & Retrain & \makecell{0.900 \\ (0.002)} & \makecell{2.082 \\ (0.024)} & \textcolor{red}{\makecell{0.821 \\ (0.007)}} & \makecell{3.371 \\ (0.037)} & \makecell{0.932 \\ (0.003)} & \makecell{1.552 \\ (0.021)} \\
 & Standard & \makecell{0.901 \\ (0.002)} & \textcolor{red}{\makecell{1.876 \\ (0.025)}} & \makecell{0.744 \\ (0.005)} & \makecell{3.372 \\ (0.034)} & \makecell{0.966 \\ (0.002)} & \makecell{1.259 \\ (0.017)} \\
\bottomrule
\end{tabular}

\end{table}


\begin{figure*}[tbh]
    \centering
    \includegraphics[width=0.9\linewidth]{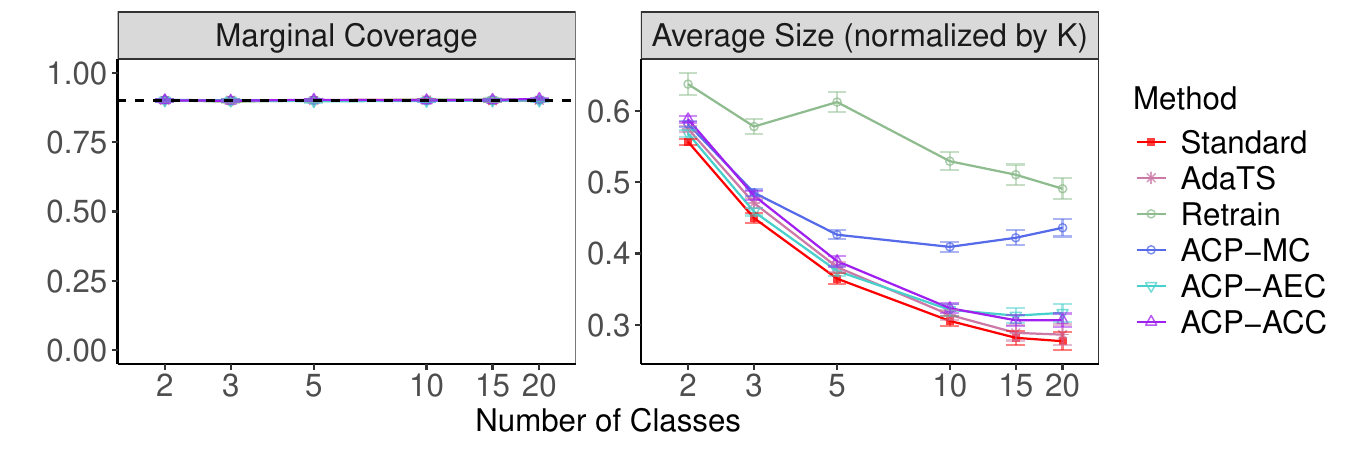}
    \caption{Performance of prediction sets for $5$-class synthetic data generated from Example~\ref{exp:covariate-shift} with covariate shift at $a = 0.5$, as a function of number of classes. All methods achieve the target marginal coverage of $90\%$, while our methods retain practically efficient sets sizes. Error bars denote two standard errors. Corresponding conditional performance are provided in Figure~\ref{fig:supp_exp_sim_K_covshift_cond}. }
\label{fig:supp_exp_sim_K_covshift_marg}
\end{figure*}

\begin{figure*}[tbh]
    \centering
    \includegraphics[width=0.9\linewidth]{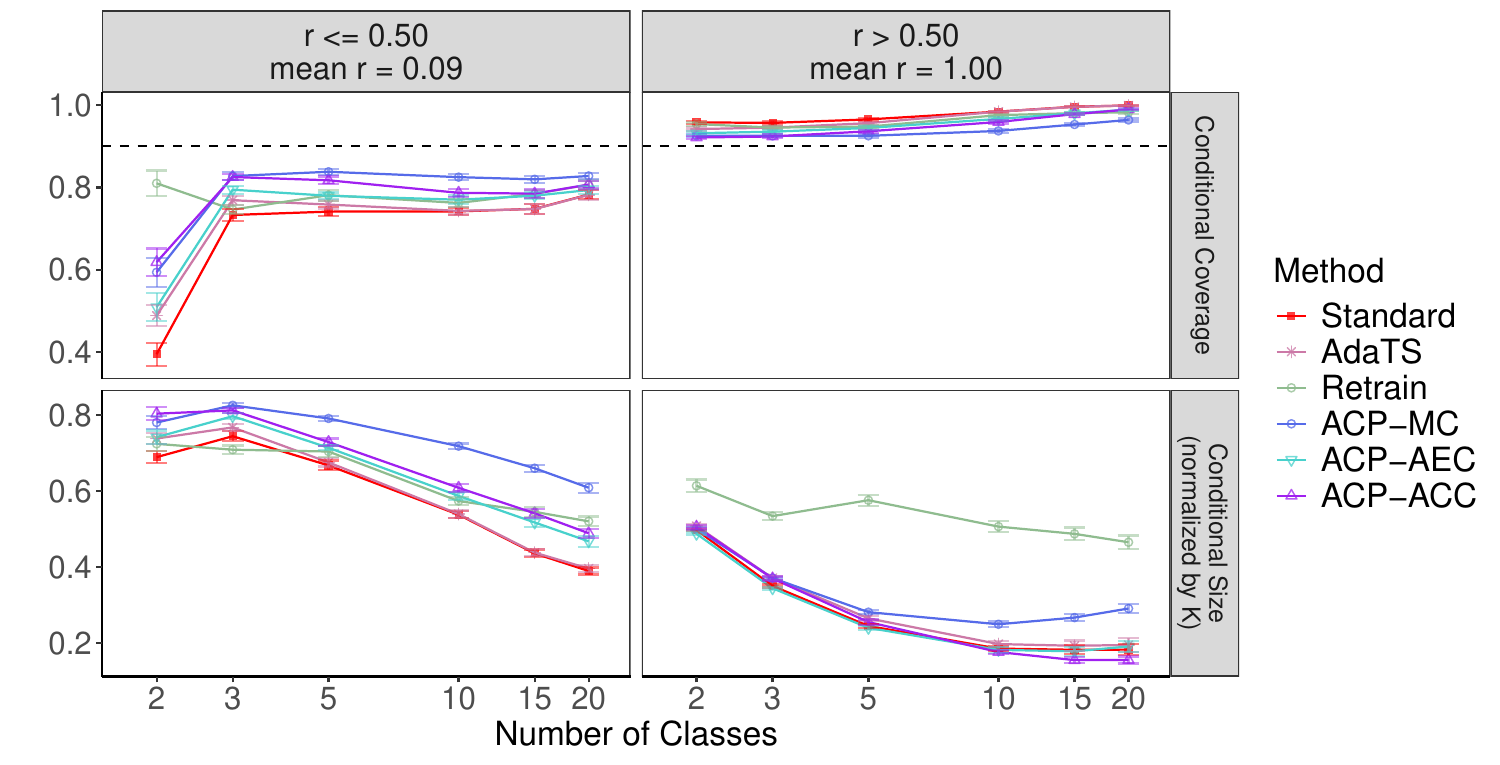}
    \caption{Our methods are more uncertainty-aware: they attain higher conditional coverage on hard samples while assigning larger prediction sets to hard samples and smaller ones to easy samples. Corresponding marginal performance are provided in Figure~\ref{fig:supp_exp_sim_K_covshift_marg} and numerical details are provided in Table~\ref{tab:supp_exp_sim_K_covshift}. }
    \label{fig:supp_exp_sim_K_covshift_cond}
\end{figure*}

\begin{table}[!htb]
\centering
    \caption{Performance of prediction sets for $5$-class synthetic data generated from Example~\ref{exp:covariate-shift} with covariate shift at $a = 0.5$, as a function of number of classes. See the corresponding plots in Figure~\ref{fig:supp_exp_sim_K_covshift_marg}--\ref{fig:supp_exp_sim_K_covshift_cond}.}
  \label{tab:supp_exp_sim_K_covshift}
\centering
\fontsize{6}{6}\selectfont
\begin{tabular}[t]{rlllllll}
\toprule
\multicolumn{2}{c}{ } & \multicolumn{2}{c}{\makecell{Marginal}} & \multicolumn{2}{c}{\makecell{Hard Bin \\ ($r^* \leq 0.5$)}} & \multicolumn{2}{c}{\makecell{Easy Bin \\ ($r^* > 0.5$)}} \\
\cmidrule(l{3pt}r{3pt}){3-4} \cmidrule(l{3pt}r{3pt}){5-6} \cmidrule(l{3pt}r{3pt}){7-8}
\makecell{Number \\ of \\ Classes} & Method & Coverage & Size & Coverage & Size & Coverage & Size \\
\midrule
2 & ACP-ACC & \makecell{0.901 \\ (0.002)} & \makecell{1.177 \\ (0.005)} & \textcolor{red}{\makecell{0.618 \\ (0.017)}} & \makecell{1.608 \\ (0.017)} & \makecell{0.921 \\ (0.002)} & \makecell{1.012 \\ (0.005)} \\
 & ACP-AEC & \makecell{0.900 \\ (0.002)} & \textcolor{red}{\makecell{1.137 \\ (0.005)}} & \makecell{0.509 \\ (0.017)} & \makecell{1.484 \\ (0.018)} & \makecell{0.932 \\ (0.002)} & \makecell{0.975 \\ (0.003)} \\
 & ACP-MC & \makecell{0.902 \\ (0.002)} & \makecell{1.163 \\ (0.004)} & \textcolor{red}{\makecell{0.594 \\ (0.018)}} & \makecell{1.561 \\ (0.016)} & \makecell{0.926 \\ (0.002)} & \makecell{1.001 \\ (0.004)} \\
 & AdaTS & \makecell{0.901 \\ (0.002)} & \textcolor{red}{\makecell{1.150 \\ (0.004)}} & \makecell{0.488 \\ (0.013)} & \makecell{1.477 \\ (0.014)} & \makecell{0.942 \\ (0.002)} & \makecell{1.016 \\ (0.004)} \\
 & Retrain & \makecell{0.902 \\ (0.003)} & \makecell{1.276 \\ (0.016)} & \textcolor{red}{\makecell{0.810 \\ (0.015)}} & \makecell{1.449 \\ (0.019)} & \makecell{0.954 \\ (0.002)} & \makecell{1.228 \\ (0.017)} \\
 & Standard & \makecell{0.901 \\ (0.002)} & \textcolor{red}{\makecell{1.114 \\ (0.004)}} & \makecell{0.394 \\ (0.014)} & \makecell{1.379 \\ (0.015)} & \makecell{0.958 \\ (0.002)} & \makecell{0.995 \\ (0.003)} \\
\midrule
3 & ACP-ACC & \makecell{0.899 \\ (0.002)} & \makecell{1.444 \\ (0.009)} & \textcolor{red}{\makecell{0.825 \\ (0.004)}} & \makecell{2.438 \\ (0.010)} & \makecell{0.924 \\ (0.002)} & \makecell{1.110 \\ (0.008)} \\
 & ACP-AEC & \makecell{0.900 \\ (0.002)} & \textcolor{red}{\makecell{1.376 \\ (0.008)}} & \textcolor{red}{\makecell{0.795 \\ (0.005)}} & \makecell{2.392 \\ (0.012)} & \makecell{0.936 \\ (0.002)} & \makecell{1.033 \\ (0.005)} \\
 & ACP-MC & \makecell{0.901 \\ (0.002)} & \makecell{1.456 \\ (0.007)} & \textcolor{red}{\makecell{0.828 \\ (0.005)}} & \makecell{2.478 \\ (0.008)} & \makecell{0.925 \\ (0.002)} & \makecell{1.113 \\ (0.006)} \\
 & AdaTS & \makecell{0.901 \\ (0.002)} & \textcolor{red}{\makecell{1.413 \\ (0.010)}} & \makecell{0.769 \\ (0.005)} & \makecell{2.303 \\ (0.013)} & \makecell{0.945 \\ (0.002)} & \makecell{1.113 \\ (0.008)} \\
 & Retrain & \makecell{0.895 \\ (0.002)} & \makecell{1.734 \\ (0.015)} & \makecell{0.746 \\ (0.006)} & \makecell{2.127 \\ (0.016)} & \makecell{0.946 \\ (0.002)} & \makecell{1.603 \\ (0.017)} \\
 & Standard & \makecell{0.901 \\ (0.002)} & \textcolor{red}{\makecell{1.349 \\ (0.011)}} & \makecell{0.733 \\ (0.007)} & \makecell{2.234 \\ (0.019)} & \makecell{0.957 \\ (0.002)} & \makecell{1.051 \\ (0.006)} \\
\midrule
5 & ACP-ACC & \makecell{0.903 \\ (0.002)} & \makecell{1.947 \\ (0.019)} & \textcolor{red}{\makecell{0.817 \\ (0.005)}} & \makecell{3.646 \\ (0.023)} & \makecell{0.937 \\ (0.002)} & \makecell{1.276 \\ (0.017)} \\
 & ACP-AEC & \makecell{0.897 \\ (0.002)} & \textcolor{red}{\makecell{1.878 \\ (0.017)}} & \makecell{0.780 \\ (0.004)} & \makecell{3.571 \\ (0.015)} & \makecell{0.945 \\ (0.003)} & \makecell{1.200 \\ (0.013)} \\
 & ACP-MC & \makecell{0.901 \\ (0.002)} & \makecell{2.132 \\ (0.016)} & \textcolor{red}{\makecell{0.838 \\ (0.004)}} & \makecell{3.957 \\ (0.016)} & \makecell{0.926 \\ (0.002)} & \makecell{1.407 \\ (0.012)} \\
 & AdaTS & \makecell{0.900 \\ (0.002)} & \textcolor{red}{\makecell{1.905 \\ (0.018)}} & \makecell{0.758 \\ (0.005)} & \makecell{3.378 \\ (0.027)} & \makecell{0.956 \\ (0.002)} & \makecell{1.325 \\ (0.017)} \\
 & Retrain & \makecell{0.901 \\ (0.002)} & \makecell{3.063 \\ (0.036)} & \textcolor{red}{\makecell{0.781 \\ (0.007)}} & \makecell{3.522 \\ (0.035)} & \makecell{0.948 \\ (0.002)} & \makecell{2.881 \\ (0.037)} \\
 & Standard & \makecell{0.901 \\ (0.001)} & \textcolor{red}{\makecell{1.824 \\ (0.020)}} & \makecell{0.741 \\ (0.006)} & \makecell{3.338 \\ (0.032)} & \makecell{0.965 \\ (0.002)} & \makecell{1.225 \\ (0.015)} \\
\midrule
10 & ACP-ACC & \makecell{0.901 \\ (0.002)} & \makecell{3.234 \\ (0.033)} & \textcolor{red}{\makecell{0.787 \\ (0.005)}} & \makecell{6.088 \\ (0.052)} & \makecell{0.959 \\ (0.002)} & \makecell{1.763 \\ (0.029)} \\
 & ACP-AEC & \makecell{0.899 \\ (0.002)} & \textcolor{red}{\makecell{3.211 \\ (0.041)}} & \textcolor{red}{\makecell{0.770 \\ (0.005)}} & \makecell{5.873 \\ (0.054)} & \makecell{0.967 \\ (0.002)} & \makecell{1.821 \\ (0.041)} \\
 & ACP-MC & \makecell{0.899 \\ (0.002)} & \makecell{4.094 \\ (0.034)} & \textcolor{red}{\makecell{0.825 \\ (0.004)}} & \makecell{7.185 \\ (0.034)} & \makecell{0.937 \\ (0.002)} & \makecell{2.495 \\ (0.037)} \\
 & AdaTS & \makecell{0.902 \\ (0.002)} & \textcolor{red}{\makecell{3.139 \\ (0.036)}} & \makecell{0.743 \\ (0.004)} & \makecell{5.397 \\ (0.048)} & \makecell{0.984 \\ (0.001)} & \makecell{1.974 \\ (0.037)} \\
 & Retrain & \makecell{0.904 \\ (0.002)} & \makecell{5.295 \\ (0.064)} & \makecell{0.762 \\ (0.005)} & \makecell{5.738 \\ (0.054)} & \makecell{0.976 \\ (0.002)} & \makecell{5.066 \\ (0.071)} \\
 & Standard & \makecell{0.902 \\ (0.001)} & \textcolor{red}{\makecell{3.055 \\ (0.033)}} & \makecell{0.742 \\ (0.004)} & \makecell{5.379 \\ (0.046)} & \makecell{0.984 \\ (0.001)} & \makecell{1.854 \\ (0.029)} \\
\midrule
15 & ACP-ACC & \makecell{0.902 \\ (0.002)} & \textcolor{red}{\makecell{4.596 \\ (0.056)}} & \textcolor{red}{\makecell{0.785 \\ (0.005)}} & \makecell{8.124 \\ (0.087)} & \makecell{0.978 \\ (0.001)} & \makecell{2.327 \\ (0.057)} \\
 & ACP-AEC & \makecell{0.900 \\ (0.002)} & \makecell{4.694 \\ (0.078)} & \makecell{0.780 \\ (0.004)} & \makecell{7.763 \\ (0.093)} & \makecell{0.980 \\ (0.001)} & \makecell{2.673 \\ (0.089)} \\
 & ACP-MC & \makecell{0.901 \\ (0.002)} & \makecell{6.333 \\ (0.080)} & \textcolor{red}{\makecell{0.820 \\ (0.004)}} & \makecell{9.904 \\ (0.068)} & \makecell{0.953 \\ (0.002)} & \makecell{4.006 \\ (0.073)} \\
 & AdaTS & \makecell{0.899 \\ (0.002)} & \textcolor{red}{\makecell{4.332 \\ (0.081)}} & \makecell{0.748 \\ (0.006)} & \makecell{6.571 \\ (0.068)} & \makecell{0.996 \\ (0.001)} & \makecell{2.893 \\ (0.102)} \\
 & Retrain & \makecell{0.904 \\ (0.002)} & \makecell{7.658 \\ (0.108)} & \textcolor{red}{\makecell{0.785 \\ (0.005)}} & \makecell{8.182 \\ (0.099)} & \makecell{0.981 \\ (0.002)} & \makecell{7.313 \\ (0.124)} \\
 & Standard & \makecell{0.899 \\ (0.002)} & \textcolor{red}{\makecell{4.228 \\ (0.074)}} & \makecell{0.748 \\ (0.006)} & \makecell{6.541 \\ (0.077)} & \makecell{0.996 \\ (0.001)} & \makecell{2.734 \\ (0.089)} \\
\midrule
20 & ACP-ACC & \makecell{0.906 \\ (0.002)} & \textcolor{red}{\makecell{6.132 \\ (0.093)}} & \textcolor{red}{\makecell{0.807 \\ (0.005)}} & \makecell{9.783 \\ (0.121)} & \makecell{0.990 \\ (0.001)} & \makecell{3.106 \\ (0.093)} \\
 & ACP-AEC & \makecell{0.900 \\ (0.002)} & \makecell{6.338 \\ (0.124)} & \makecell{0.794 \\ (0.005)} & \makecell{9.358 \\ (0.145)} & \makecell{0.989 \\ (0.001)} & \makecell{3.821 \\ (0.138)} \\
 & ACP-MC & \makecell{0.902 \\ (0.002)} & \makecell{8.725 \\ (0.123)} & \textcolor{red}{\makecell{0.828 \\ (0.004)}} & \makecell{12.179 \\ (0.129)} & \makecell{0.964 \\ (0.002)} & \makecell{5.818 \\ (0.110)} \\
 & AdaTS & \makecell{0.901 \\ (0.002)} & \textcolor{red}{\makecell{5.726 \\ (0.142)}} & \makecell{0.782 \\ (0.006)} & \makecell{7.924 \\ (0.091)} & \makecell{0.999 \\ (0.000)} & \makecell{3.899 \\ (0.176)} \\
 & Retrain & \makecell{0.903 \\ (0.002)} & \makecell{9.819 \\ (0.149)} & \textcolor{red}{\makecell{0.807 \\ (0.005)}} & \makecell{10.409 \\ (0.128)} & \makecell{0.983 \\ (0.002)} & \makecell{9.308 \\ (0.178)} \\
 & Standard & \makecell{0.901 \\ (0.001)} & \textcolor{red}{\makecell{5.542 \\ (0.125)}} & \makecell{0.783 \\ (0.006)} & \makecell{7.794 \\ (0.092)} & \makecell{0.999 \\ (0.000)} & \makecell{3.655 \\ (0.151)} \\
\bottomrule
\end{tabular}

\end{table}


\begin{figure*}[tbh]
    \centering
    \includegraphics[width=0.9\linewidth]{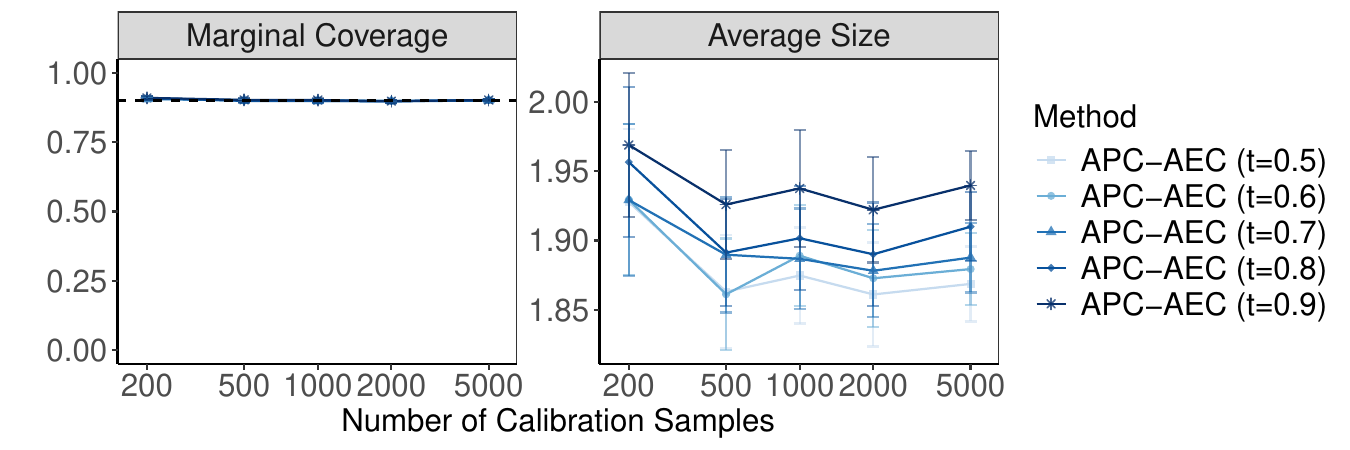}
    \caption{Performance of prediction sets for our AEqualized method with different thresholds for $5$-class synthetic data generated from Example~\ref{exp:covariate-shift} with covariate shift at $a = 0.5$, as a function calibration sample sizes. All methods achieve the target marginal coverage of $90\%$, smaller thresholds lead to smaller set sizes \emph{on average}. Error bars denote two standard errors. Corresponding conditional performance are provided in Figure~\ref{fig:supp_exp_sim_n_new_covshift_equalized_cond}. }
    \label{fig:supp_exp_sim_n_new_covshift_equalized_marg}
\end{figure*}

\begin{figure*}[tbh]
    \centering
    \includegraphics[width=0.9\linewidth]{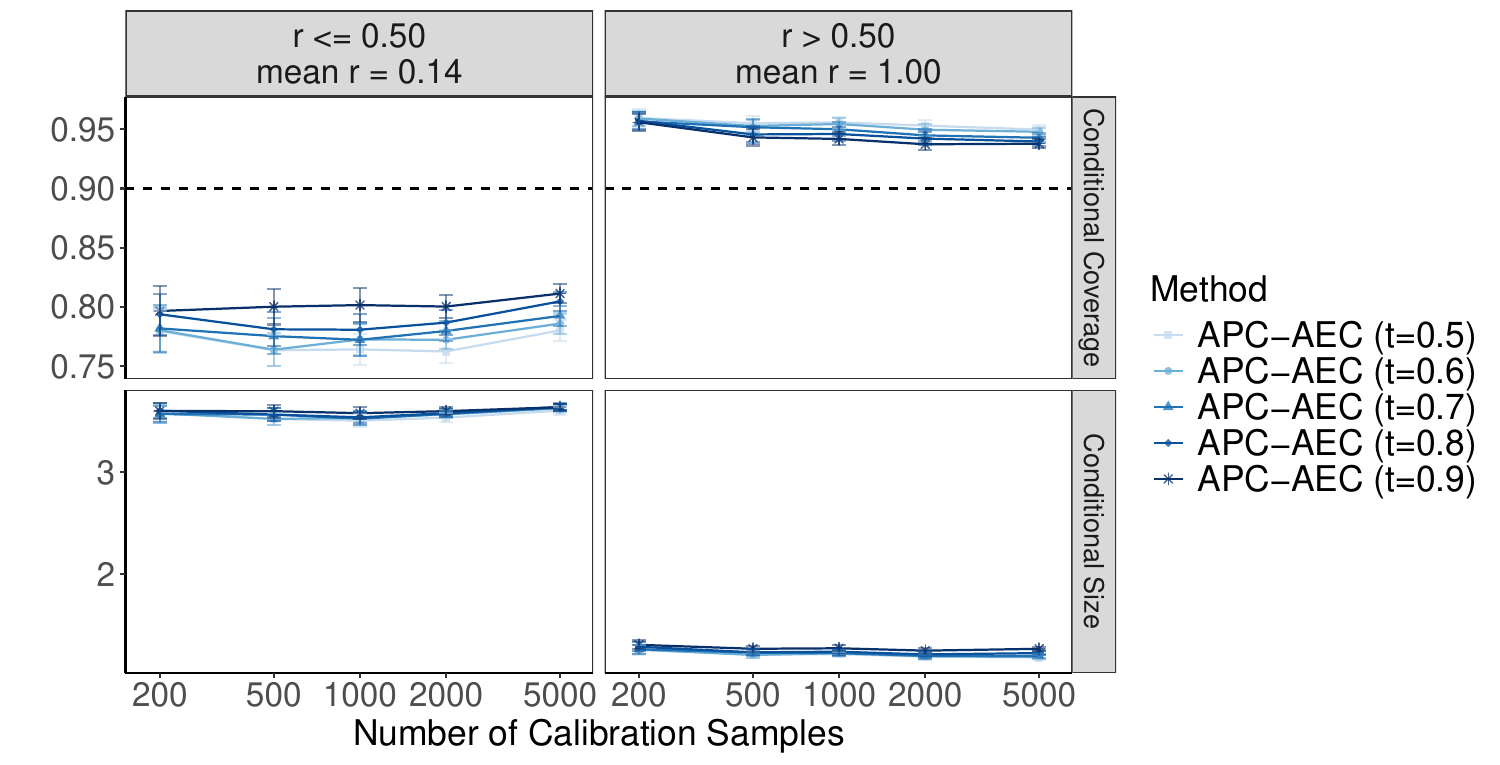}
    \caption{AEqualized uncertainty-aware: they attain higher conditional coverage on hard samples while assigning larger prediction sets to hard samples and smaller ones to easy samples. Higher thresholds lead to higher conditional coverage on the hard samples \emph{on average}. Corresponding marginal performance are provided in Figure~\ref{fig:supp_exp_sim_n_new_covshift_equalized_marg} and numerical details are provided in Table~\ref{tab:supp_exp_sim_covshift_n_new_equalized}. }
    \label{fig:supp_exp_sim_n_new_covshift_equalized_cond}
\end{figure*}

\begin{table}[!htb]
\centering
    \caption{Performance of prediction sets for our AEqualized method with different thresholds for $5$-class synthetic data generated from Example~\ref{exp:covariate-shift} with covariate shift at $a = 0.5$, as a function calibration sample sizes. See the corresponding plots in Figure~\ref{fig:supp_exp_sim_n_new_covshift_equalized_marg}--\ref{fig:supp_exp_sim_n_new_covshift_equalized_cond}.}
  \label{tab:supp_exp_sim_covshift_n_new_equalized}
\centering
\fontsize{6}{6}\selectfont
\begin{tabular}[t]{rlllllll}
\toprule
\multicolumn{2}{c}{ } & \multicolumn{2}{c}{\makecell{Marginal}} & \multicolumn{2}{c}{\makecell{Hard Bin \\ ($r^* \leq 0.5$)}} & \multicolumn{2}{c}{\makecell{Easy Bin \\ ($r^* > 0.5$)}} \\
\cmidrule(l{3pt}r{3pt}){3-4} \cmidrule(l{3pt}r{3pt}){5-6} \cmidrule(l{3pt}r{3pt}){7-8}
\makecell{Number \\ of \\ Calibration \\ Samples} & Method & Coverage & Size & Coverage & Size & Coverage & Size \\
\midrule
200 & APC-AEC (t=0.5) & \makecell{0.907 \\ (0.004)} & \textcolor{red}{\makecell{1.928 \\ (0.026)}} & \makecell{0.780 \\ (0.009)} & \makecell{3.583 \\ (0.040)} & \makecell{0.959 \\ (0.004)} & \makecell{1.258 \\ (0.023)} \\
 & APC-AEC (t=0.6) & \makecell{0.907 \\ (0.004)} & \textcolor{red}{\makecell{1.929 \\ (0.027)}} & \makecell{0.780 \\ (0.010)} & \makecell{3.575 \\ (0.043)} & \makecell{0.959 \\ (0.003)} & \makecell{1.262 \\ (0.021)} \\
 & APC-AEC (t=0.7) & \makecell{0.906 \\ (0.004)} & \textcolor{red}{\makecell{1.929 \\ (0.028)}} & \textcolor{red}{\makecell{0.782 \\ (0.010)}} & \makecell{3.569 \\ (0.041)} & \makecell{0.957 \\ (0.004)} & \makecell{1.264 \\ (0.022)} \\
 & APC-AEC (t=0.8) & \makecell{0.910 \\ (0.004)} & \makecell{1.957 \\ (0.027)} & \textcolor{red}{\makecell{0.794 \\ (0.009)}} & \makecell{3.602 \\ (0.037)} & \makecell{0.957 \\ (0.004)} & \makecell{1.289 \\ (0.022)} \\
 & APC-AEC (t=0.9) & \makecell{0.910 \\ (0.004)} & \makecell{1.969 \\ (0.026)} & \textcolor{red}{\makecell{0.797 \\ (0.010)}} & \makecell{3.598 \\ (0.038)} & \makecell{0.956 \\ (0.004)} & \makecell{1.308 \\ (0.021)} \\
\midrule
500 & APC-AEC (t=0.5) & \makecell{0.901 \\ (0.003)} & \textcolor{red}{\makecell{1.863 \\ (0.020)}} & \makecell{0.764 \\ (0.007)} & \makecell{3.520 \\ (0.031)} & \makecell{0.955 \\ (0.003)} & \makecell{1.212 \\ (0.016)} \\
 & APC-AEC (t=0.6) & \makecell{0.899 \\ (0.003)} & \textcolor{red}{\makecell{1.861 \\ (0.020)}} & \makecell{0.764 \\ (0.007)} & \makecell{3.522 \\ (0.030)} & \makecell{0.953 \\ (0.003)} & \makecell{1.209 \\ (0.014)} \\
 & APC-AEC (t=0.7) & \makecell{0.902 \\ (0.003)} & \textcolor{red}{\makecell{1.890 \\ (0.021)}} & \textcolor{red}{\makecell{0.775 \\ (0.008)}} & \makecell{3.565 \\ (0.033)} & \makecell{0.952 \\ (0.003)} & \makecell{1.232 \\ (0.015)} \\
 & APC-AEC (t=0.8) & \makecell{0.899 \\ (0.003)} & \makecell{1.891 \\ (0.019)} & \textcolor{red}{\makecell{0.781 \\ (0.007)}} & \makecell{3.564 \\ (0.031)} & \makecell{0.946 \\ (0.003)} & \makecell{1.235 \\ (0.014)} \\
 & APC-AEC (t=0.9) & \makecell{0.902 \\ (0.003)} & \makecell{1.926 \\ (0.020)} & \textcolor{red}{\makecell{0.800 \\ (0.008)}} & \makecell{3.597 \\ (0.031)} & \makecell{0.943 \\ (0.003)} & \makecell{1.269 \\ (0.015)} \\
\midrule
1000 & APC-AEC (t=0.5) & \makecell{0.900 \\ (0.002)} & \textcolor{red}{\makecell{1.875 \\ (0.017)}} & \makecell{0.764 \\ (0.006)} & \makecell{3.504 \\ (0.032)} & \makecell{0.956 \\ (0.002)} & \makecell{1.217 \\ (0.013)} \\
 & APC-AEC (t=0.6) & \makecell{0.902 \\ (0.003)} & \textcolor{red}{\makecell{1.889 \\ (0.018)}} & \textcolor{red}{\makecell{0.773 \\ (0.007)}} & \makecell{3.531 \\ (0.033)} & \makecell{0.954 \\ (0.003)} & \makecell{1.226 \\ (0.014)} \\
 & APC-AEC (t=0.7) & \makecell{0.899 \\ (0.003)} & \textcolor{red}{\makecell{1.887 \\ (0.018)}} & \makecell{0.772 \\ (0.007)} & \makecell{3.523 \\ (0.031)} & \makecell{0.950 \\ (0.003)} & \makecell{1.226 \\ (0.014)} \\
 & APC-AEC (t=0.8) & \makecell{0.899 \\ (0.003)} & \makecell{1.902 \\ (0.019)} & \textcolor{red}{\makecell{0.781 \\ (0.007)}} & \makecell{3.537 \\ (0.029)} & \makecell{0.946 \\ (0.003)} & \makecell{1.241 \\ (0.014)} \\
 & APC-AEC (t=0.9) & \makecell{0.902 \\ (0.003)} & \makecell{1.938 \\ (0.021)} & \textcolor{red}{\makecell{0.802 \\ (0.007)}} & \makecell{3.577 \\ (0.030)} & \makecell{0.942 \\ (0.003)} & \makecell{1.274 \\ (0.015)} \\
\midrule
2000 & APC-AEC (t=0.5) & \makecell{0.898 \\ (0.002)} & \textcolor{red}{\makecell{1.861 \\ (0.019)}} & \makecell{0.763 \\ (0.005)} & \makecell{3.534 \\ (0.022)} & \makecell{0.953 \\ (0.002)} & \makecell{1.192 \\ (0.015)} \\
 & APC-AEC (t=0.6) & \makecell{0.899 \\ (0.002)} & \textcolor{red}{\makecell{1.873 \\ (0.017)}} & \makecell{0.772 \\ (0.004)} & \makecell{3.565 \\ (0.016)} & \makecell{0.950 \\ (0.003)} & \makecell{1.195 \\ (0.014)} \\
 & APC-AEC (t=0.7) & \makecell{0.897 \\ (0.002)} & \textcolor{red}{\makecell{1.878 \\ (0.017)}} & \textcolor{red}{\makecell{0.780 \\ (0.004)}} & \makecell{3.571 \\ (0.015)} & \makecell{0.945 \\ (0.003)} & \makecell{1.200 \\ (0.013)} \\
 & APC-AEC (t=0.8) & \makecell{0.898 \\ (0.003)} & \makecell{1.890 \\ (0.019)} & \textcolor{red}{\makecell{0.787 \\ (0.005)}} & \makecell{3.577 \\ (0.020)} & \makecell{0.942 \\ (0.003)} & \makecell{1.214 \\ (0.014)} \\
 & APC-AEC (t=0.9) & \makecell{0.898 \\ (0.002)} & \makecell{1.922 \\ (0.019)} & \textcolor{red}{\makecell{0.800 \\ (0.005)}} & \makecell{3.597 \\ (0.018)} & \makecell{0.937 \\ (0.002)} & \makecell{1.251 \\ (0.015)} \\
\midrule
5000 & APC-AEC (t=0.5) & \makecell{0.902 \\ (0.002)} & \textcolor{red}{\makecell{1.869 \\ (0.014)}} & \makecell{0.780 \\ (0.005)} & \makecell{3.603 \\ (0.019)} & \makecell{0.950 \\ (0.002)} & \makecell{1.184 \\ (0.012)} \\
 & APC-AEC (t=0.6) & \makecell{0.902 \\ (0.002)} & \textcolor{red}{\makecell{1.879 \\ (0.013)}} & \makecell{0.786 \\ (0.004)} & \makecell{3.626 \\ (0.018)} & \makecell{0.948 \\ (0.002)} & \makecell{1.190 \\ (0.011)} \\
 & APC-AEC (t=0.7) & \makecell{0.900 \\ (0.002)} & \textcolor{red}{\makecell{1.888 \\ (0.013)}} & \textcolor{red}{\makecell{0.792 \\ (0.004)}} & \makecell{3.631 \\ (0.017)} & \makecell{0.943 \\ (0.002)} & \makecell{1.199 \\ (0.010)} \\
 & APC-AEC (t=0.8) & \makecell{0.901 \\ (0.002)} & \makecell{1.910 \\ (0.012)} & \textcolor{red}{\makecell{0.805 \\ (0.004)}} & \makecell{3.638 \\ (0.017)} & \makecell{0.940 \\ (0.002)} & \makecell{1.227 \\ (0.010)} \\
 & APC-AEC (t=0.9) & \makecell{0.902 \\ (0.002)} & \makecell{1.940 \\ (0.012)} & \textcolor{red}{\makecell{0.811 \\ (0.004)}} & \makecell{3.636 \\ (0.016)} & \makecell{0.938 \\ (0.002)} & \makecell{1.269 \\ (0.010)} \\
\bottomrule
\end{tabular}

\end{table} 


\begin{figure*}[tbh]
    \centering
    \includegraphics[width=0.9\linewidth]{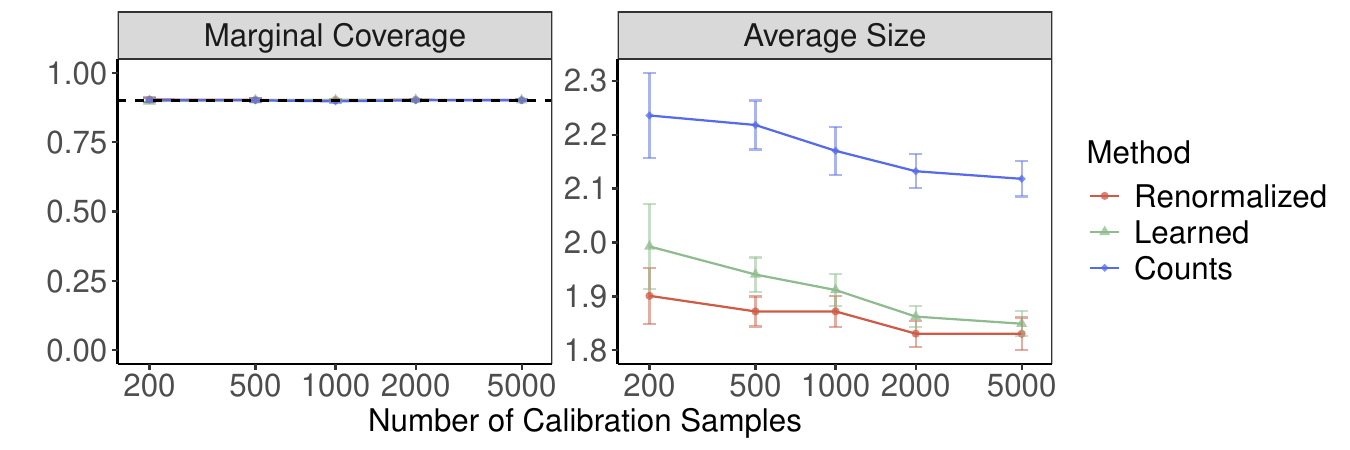}
    \caption{Performance of prediction sets for $5$-class synthetic data generated from Example~\ref{exp:covariate-shift} using our ACP-MC method as a function of the total calibration sample size. In particular, we compare the different estimation to $\eta$ as described in Appendix~\ref{app:APA_mutlclass}. All $\eta$ choices achieve the target marginal coverage of $90\%$, while Renormalized achieves smallest prediction set size. Error bars denote two standard errors. Corresponding conditional performance are provided in Figure~\ref{fig:supp_exp_sim_n_new_covshift_etacompare_cond}. }
    \label{fig:supp_exp_sim_n_new_covshift_etacompare_marg}
\end{figure*}

\begin{figure*}[tbh]
    \centering
    \includegraphics[width=0.9\linewidth]{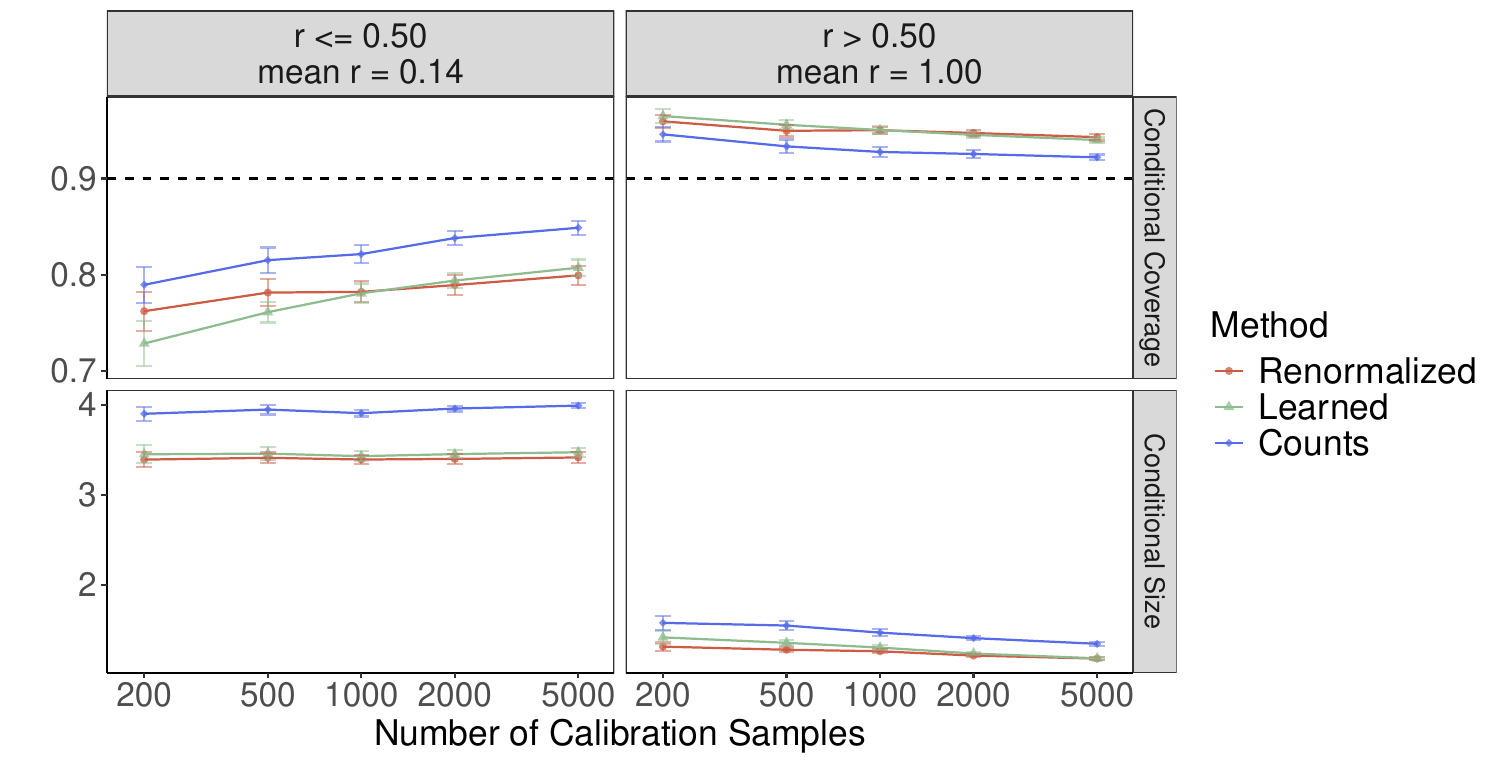}
    \caption{Our ACP-MC with all $\eta$ choices are uncertainty-aware: they attain high conditional coverage on hard samples while assigning larger prediction sets to hard samples and smaller ones to easy samples. Among them, Counts achieves highest conditional coverage on the hard bin. Corresponding marginal performance are provided in Figure~\ref{fig:supp_exp_sim_n_new_covshift_etacompare_marg} and numerical details are provided in Table~\ref{tab:supp_exp_sim_covshift_n_new_etacompare_apa}. }
    \label{fig:supp_exp_sim_n_new_covshift_etacompare_cond}
\end{figure*}

\begin{table}[!htb]
\centering
    \caption{Performance of prediction sets for $5$-class synthetic data generated from Example~\ref{exp:covariate-shift} using our ACP-MC method as a function of the total calibration sample size. See the corresponding plots in Figure~\ref{fig:supp_exp_sim_n_new_covshift_etacompare_marg}--\ref{fig:supp_exp_sim_n_new_covshift_etacompare_cond}.}
  \label{tab:supp_exp_sim_covshift_n_new_etacompare_apa}
\centering
\fontsize{6}{6}\selectfont
\begin{tabular}[t]{rlllllll}
\toprule
\multicolumn{2}{c}{ } & \multicolumn{2}{c}{\makecell{Marginal}} & \multicolumn{2}{c}{\makecell{Hard Bin \\ ($r^* \leq 0.5$)}} & \multicolumn{2}{c}{\makecell{Easy Bin \\ ($r^* > 0.5$)}} \\
\cmidrule(l{3pt}r{3pt}){3-4} \cmidrule(l{3pt}r{3pt}){5-6} \cmidrule(l{3pt}r{3pt}){7-8}
\makecell{Number \\ of \\ Calibration \\ Samples} & Method & Coverage & Size & Coverage & Size & Coverage & Size \\
\midrule
200 & Counts & \makecell{0.902 \\ (0.004)} & \textcolor{red}{\makecell{2.236 \\ (0.040)}} & \textcolor{red}{\makecell{0.790 \\ (0.009)}} & \makecell{3.898 \\ (0.039)} & \makecell{0.946 \\ (0.004)} & \makecell{1.576 \\ (0.038)} \\
 & Learned & \makecell{0.898 \\ (0.005)} & \textcolor{red}{\makecell{1.993 \\ (0.040)}} & \textcolor{red}{\makecell{0.728 \\ (0.012)}} & \makecell{3.449 \\ (0.050)} & \makecell{0.965 \\ (0.003)} & \makecell{1.415 \\ (0.035)} \\
 & Renormalized & \makecell{0.904 \\ (0.004)} & \textcolor{red}{\makecell{1.901 \\ (0.026)}} & \textcolor{red}{\makecell{0.762 \\ (0.010)}} & \makecell{3.389 \\ (0.041)} & \makecell{0.960 \\ (0.003)} & \makecell{1.312 \\ (0.023)} \\
\midrule
500 & Counts & \makecell{0.901 \\ (0.003)} & \textcolor{red}{\makecell{2.218 \\ (0.023)}} & \textcolor{red}{\makecell{0.815 \\ (0.007)}} & \makecell{3.945 \\ (0.027)} & \makecell{0.934 \\ (0.004)} & \makecell{1.547 \\ (0.024)} \\
 & Learned & \makecell{0.902 \\ (0.003)} & \textcolor{red}{\makecell{1.941 \\ (0.016)}} & \textcolor{red}{\makecell{0.761 \\ (0.005)}} & \makecell{3.456 \\ (0.035)} & \makecell{0.956 \\ (0.002)} & \makecell{1.355 \\ (0.016)} \\
 & Renormalized & \makecell{0.903 \\ (0.003)} & \textcolor{red}{\makecell{1.872 \\ (0.014)}} & \textcolor{red}{\makecell{0.781 \\ (0.007)}} & \makecell{3.409 \\ (0.030)} & \makecell{0.950 \\ (0.003)} & \makecell{1.278 \\ (0.013)} \\
\midrule
1000 & Counts & \makecell{0.897 \\ (0.002)} & \textcolor{red}{\makecell{2.170 \\ (0.022)}} & \textcolor{red}{\makecell{0.821 \\ (0.005)}} & \makecell{3.905 \\ (0.019)} & \makecell{0.928 \\ (0.003)} & \makecell{1.468 \\ (0.019)} \\
 & Learned & \makecell{0.902 \\ (0.002)} & \textcolor{red}{\makecell{1.912 \\ (0.015)}} & \textcolor{red}{\makecell{0.781 \\ (0.005)}} & \makecell{3.428 \\ (0.027)} & \makecell{0.951 \\ (0.002)} & \makecell{1.302 \\ (0.012)} \\
 & Renormalized & \makecell{0.902 \\ (0.002)} & \textcolor{red}{\makecell{1.872 \\ (0.015)}} & \textcolor{red}{\makecell{0.782 \\ (0.005)}} & \makecell{3.390 \\ (0.026)} & \makecell{0.950 \\ (0.002)} & \makecell{1.261 \\ (0.011)} \\
\midrule
2000 & Counts & \makecell{0.901 \\ (0.002)} & \textcolor{red}{\makecell{2.132 \\ (0.016)}} & \textcolor{red}{\makecell{0.838 \\ (0.004)}} & \makecell{3.957 \\ (0.016)} & \makecell{0.926 \\ (0.002)} & \makecell{1.407 \\ (0.012)} \\
 & Learned & \makecell{0.903 \\ (0.002)} & \textcolor{red}{\makecell{1.863 \\ (0.010)}} & \textcolor{red}{\makecell{0.794 \\ (0.004)}} & \makecell{3.450 \\ (0.024)} & \makecell{0.945 \\ (0.002)} & \makecell{1.236 \\ (0.009)} \\
 & Renormalized & \makecell{0.903 \\ (0.002)} & \textcolor{red}{\makecell{1.831 \\ (0.012)}} & \textcolor{red}{\makecell{0.789 \\ (0.005)}} & \makecell{3.397 \\ (0.029)} & \makecell{0.947 \\ (0.002)} & \makecell{1.213 \\ (0.009)} \\
\midrule
5000 & Counts & \makecell{0.901 \\ (0.001)} & \textcolor{red}{\makecell{2.118 \\ (0.016)}} & \textcolor{red}{\makecell{0.849 \\ (0.004)}} & \makecell{3.989 \\ (0.014)} & \makecell{0.922 \\ (0.002)} & \makecell{1.344 \\ (0.011)} \\
 & Learned & \makecell{0.902 \\ (0.002)} & \textcolor{red}{\makecell{1.849 \\ (0.011)}} & \textcolor{red}{\makecell{0.807 \\ (0.004)}} & \makecell{3.471 \\ (0.025)} & \makecell{0.940 \\ (0.001)} & \makecell{1.182 \\ (0.007)} \\
 & Renormalized & \makecell{0.902 \\ (0.002)} & \textcolor{red}{\makecell{1.831 \\ (0.015)}} & \textcolor{red}{\makecell{0.799 \\ (0.005)}} & \makecell{3.413 \\ (0.032)} & \makecell{0.943 \\ (0.002)} & \makecell{1.180 \\ (0.007)} \\
\bottomrule
\end{tabular}

\end{table}

\begin{figure*}[tbh]
    \centering
    \includegraphics[width=0.9\linewidth]{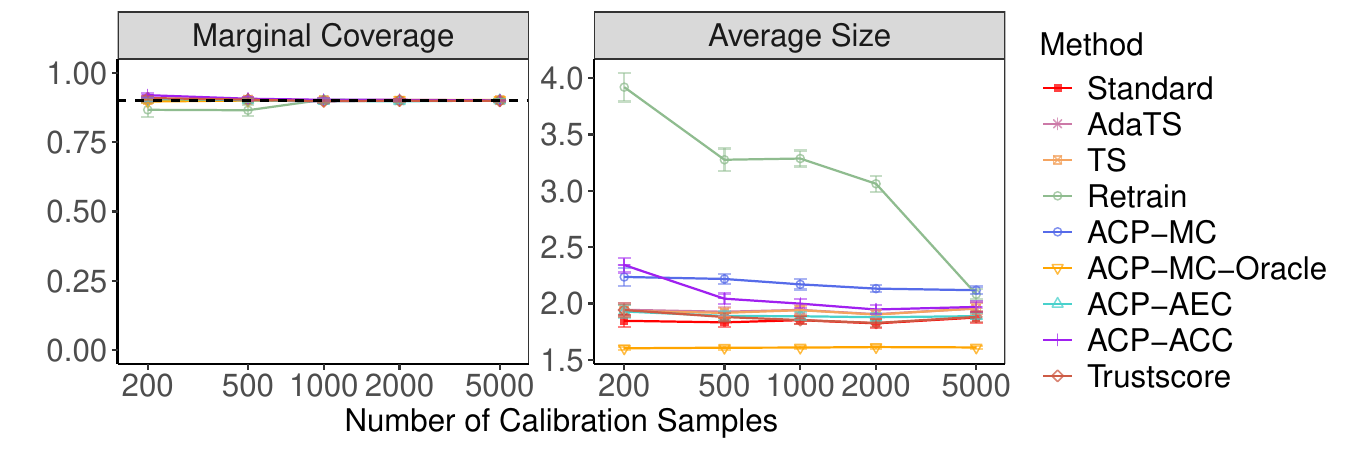}
    \caption{Performance of prediction sets constructed with additional benchmarks for $5$-class synthetic data generated from Example~\ref{exp:covariate-shift} with covariate shift at $a = 0.5$, as a function of the total calibration sample size. All methods achieve the target marginal coverage of $90\%$, while our methods retain practically efficient sets sizes. Error bars denote two standard errors. Corresponding conditional performance are provided in Figure~\ref{fig:supp_exp_sim_covshift_n_new_full_cond}. }
    \label{fig:supp_exp_sim_covshift_n_new_full_marg}
\end{figure*}

\begin{figure*}[tbh]
    \centering
    \includegraphics[width=0.9\linewidth]{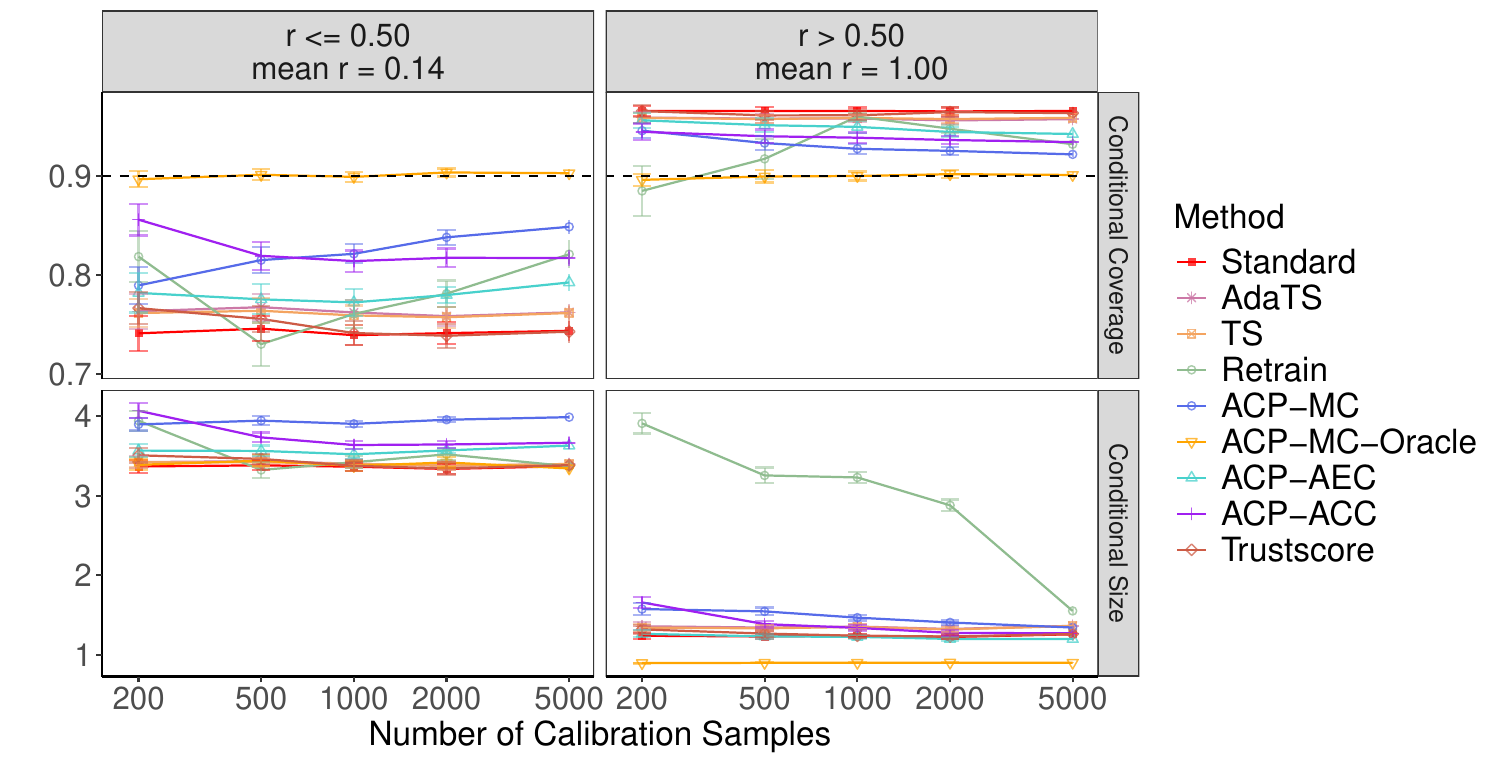}
    \caption{Our methods are more uncertainty-aware: they attain higher conditional coverage on hard samples while assigning larger prediction sets to hard samples and smaller ones to easy samples. Corresponding marginal performance are provided in Figure~\ref{fig:supp_exp_sim_covshift_n_new_full_marg} and numerical details are provided in Table~\ref{tab:supp_exp_sim_covshift_nnew_full}. }
    \label{fig:supp_exp_sim_covshift_n_new_full_cond}
\end{figure*}

\begin{table}[!htb]
\centering
    \caption{Performance of conformal prediction sets constructed with additional benchmarks for $5$-class synthetic data generated from Example~\ref{exp:covariate-shift} with covariate shift at $a = 0.5$, as a function the of the total calibration sample size. See the corresponding plots in Figure~\ref{fig:supp_exp_sim_n_new_full_marg}--\ref{fig:supp_exp_sim_n_new_full_cond}.}
  \label{tab:supp_exp_sim_covshift_nnew_full}
\centering
\fontsize{6}{6}\selectfont
\begin{tabular}[t]{rlllllll}
\toprule
\multicolumn{2}{c}{ } & \multicolumn{2}{c}{\makecell{Marginal}} & \multicolumn{2}{c}{\makecell{Hard Bin \\ ($r^* \leq 0.5$)}} & \multicolumn{2}{c}{\makecell{Easy Bin \\ ($r^* > 0.5$)}} \\
\cmidrule(l{3pt}r{3pt}){3-4} \cmidrule(l{3pt}r{3pt}){5-6} \cmidrule(l{3pt}r{3pt}){7-8}
\makecell{Number \\ of \\ Calibration \\ Samples} & Method & Coverage & Size & Coverage & Size & Coverage & Size \\
\midrule
200 & ACP-ACC & \makecell{0.920 \\ (0.004)} & \makecell{2.341 \\ (0.033)} & \textcolor{red}{\makecell{0.856 \\ (0.008)}} & \makecell{4.070 \\ (0.047)} & \makecell{0.945 \\ (0.004)} & \makecell{1.659 \\ (0.033)} \\
 & ACP-AEC & \makecell{0.906 \\ (0.004)} & \textcolor{red}{\makecell{1.929 \\ (0.028)}} & \makecell{0.782 \\ (0.010)} & \makecell{3.569 \\ (0.041)} & \makecell{0.957 \\ (0.004)} & \makecell{1.264 \\ (0.022)} \\
 & ACP-MC & \makecell{0.902 \\ (0.004)} & \makecell{2.236 \\ (0.040)} & \makecell{0.790 \\ (0.009)} & \makecell{3.898 \\ (0.039)} & \makecell{0.946 \\ (0.004)} & \makecell{1.576 \\ (0.038)} \\
 & ACP-MC-Oracle & \makecell{0.897 \\ (0.003)} & \textcolor{red}{\makecell{1.602 \\ (0.008)}} & \textcolor{red}{\makecell{0.897 \\ (0.004)}} & \makecell{3.390 \\ (0.034)} & \makecell{0.896 \\ (0.003)} & \makecell{0.896 \\ (0.003)} \\
 & AdaTS & \makecell{0.903 \\ (0.004)} & \makecell{1.944 \\ (0.028)} & \makecell{0.763 \\ (0.009)} & \makecell{3.424 \\ (0.037)} & \makecell{0.959 \\ (0.003)} & \makecell{1.359 \\ (0.026)} \\
 & Retrain & \makecell{0.867 \\ (0.012)} & \makecell{3.921 \\ (0.063)} & \textcolor{red}{\makecell{0.819 \\ (0.013)}} & \makecell{3.944 \\ (0.063)} & \makecell{0.885 \\ (0.013)} & \makecell{3.911 \\ (0.064)} \\
 & Standard & \makecell{0.902 \\ (0.003)} & \textcolor{red}{\makecell{1.846 \\ (0.026)}} & \makecell{0.741 \\ (0.009)} & \makecell{3.369 \\ (0.040)} & \makecell{0.966 \\ (0.002)} & \makecell{1.240 \\ (0.018)} \\
 & TS & \makecell{0.903 \\ (0.003)} & \makecell{1.933 \\ (0.026)} & \makecell{0.762 \\ (0.007)} & \makecell{3.421 \\ (0.032)} & \makecell{0.959 \\ (0.002)} & \makecell{1.343 \\ (0.022)} \\
 & Trustscore & \makecell{0.909 \\ (0.003)} & \makecell{1.942 \\ (0.029)} & \makecell{0.766 \\ (0.008)} & \makecell{3.509 \\ (0.045)} & \makecell{0.966 \\ (0.003)} & \makecell{1.320 \\ (0.025)} \\
\midrule
500 & ACP-ACC & \makecell{0.907 \\ (0.003)} & \makecell{2.043 \\ (0.023)} & \textcolor{red}{\makecell{0.819 \\ (0.007)}} & \makecell{3.735 \\ (0.031)} & \makecell{0.940 \\ (0.003)} & \makecell{1.386 \\ (0.021)} \\
 & ACP-AEC & \makecell{0.902 \\ (0.003)} & \makecell{1.890 \\ (0.021)} & \makecell{0.775 \\ (0.008)} & \makecell{3.565 \\ (0.033)} & \makecell{0.952 \\ (0.003)} & \makecell{1.232 \\ (0.015)} \\
 & ACP-MC & \makecell{0.901 \\ (0.003)} & \makecell{2.218 \\ (0.023)} & \textcolor{red}{\makecell{0.815 \\ (0.007)}} & \makecell{3.945 \\ (0.027)} & \makecell{0.934 \\ (0.004)} & \makecell{1.547 \\ (0.024)} \\
 & ACP-MC-Oracle & \makecell{0.900 \\ (0.003)} & \textcolor{red}{\makecell{1.608 \\ (0.008)}} & \textcolor{red}{\makecell{0.901 \\ (0.003)}} & \makecell{3.441 \\ (0.036)} & \makecell{0.900 \\ (0.003)} & \makecell{0.900 \\ (0.003)} \\
 & AdaTS & \makecell{0.904 \\ (0.003)} & \makecell{1.927 \\ (0.019)} & \makecell{0.767 \\ (0.007)} & \makecell{3.435 \\ (0.027)} & \makecell{0.958 \\ (0.002)} & \makecell{1.343 \\ (0.018)} \\
 & Retrain & \makecell{0.865 \\ (0.010)} & \makecell{3.276 \\ (0.050)} & \makecell{0.730 \\ (0.011)} & \makecell{3.325 \\ (0.052)} & \makecell{0.918 \\ (0.010)} & \makecell{3.256 \\ (0.050)} \\
 & Standard & \makecell{0.904 \\ (0.002)} & \textcolor{red}{\makecell{1.833 \\ (0.019)}} & \makecell{0.746 \\ (0.006)} & \makecell{3.383 \\ (0.029)} & \makecell{0.966 \\ (0.002)} & \makecell{1.231 \\ (0.014)} \\
 & TS & \makecell{0.903 \\ (0.003)} & \makecell{1.917 \\ (0.021)} & \makecell{0.764 \\ (0.006)} & \makecell{3.422 \\ (0.027)} & \makecell{0.958 \\ (0.002)} & \makecell{1.333 \\ (0.019)} \\
 & Trustscore & \makecell{0.903 \\ (0.003)} & \textcolor{red}{\makecell{1.880 \\ (0.020)}} & \makecell{0.756 \\ (0.007)} & \makecell{3.464 \\ (0.033)} & \makecell{0.962 \\ (0.002)} & \makecell{1.268 \\ (0.020)} \\
\midrule
1000 & ACP-ACC & \makecell{0.903 \\ (0.002)} & \makecell{1.999 \\ (0.019)} & \textcolor{red}{\makecell{0.814 \\ (0.005)}} & \makecell{3.639 \\ (0.026)} & \makecell{0.939 \\ (0.003)} & \makecell{1.339 \\ (0.016)} \\
 & ACP-AEC & \makecell{0.899 \\ (0.003)} & \makecell{1.887 \\ (0.018)} & \makecell{0.772 \\ (0.007)} & \makecell{3.523 \\ (0.031)} & \makecell{0.950 \\ (0.003)} & \makecell{1.226 \\ (0.014)} \\
 & ACP-MC & \makecell{0.897 \\ (0.002)} & \makecell{2.170 \\ (0.022)} & \textcolor{red}{\makecell{0.821 \\ (0.005)}} & \makecell{3.905 \\ (0.019)} & \makecell{0.928 \\ (0.003)} & \makecell{1.468 \\ (0.019)} \\
 & ACP-MC-Oracle & \makecell{0.900 \\ (0.002)} & \textcolor{red}{\makecell{1.609 \\ (0.007)}} & \textcolor{red}{\makecell{0.899 \\ (0.002)}} & \makecell{3.376 \\ (0.033)} & \makecell{0.900 \\ (0.002)} & \makecell{0.900 \\ (0.002)} \\
 & AdaTS & \makecell{0.902 \\ (0.002)} & \makecell{1.943 \\ (0.020)} & \makecell{0.762 \\ (0.006)} & \makecell{3.403 \\ (0.028)} & \makecell{0.958 \\ (0.002)} & \makecell{1.356 \\ (0.016)} \\
 & Retrain & \makecell{0.903 \\ (0.002)} & \makecell{3.287 \\ (0.036)} & \makecell{0.761 \\ (0.007)} & \makecell{3.422 \\ (0.037)} & \makecell{0.960 \\ (0.002)} & \makecell{3.231 \\ (0.036)} \\
 & Standard & \makecell{0.901 \\ (0.002)} & \textcolor{red}{\makecell{1.852 \\ (0.019)}} & \makecell{0.739 \\ (0.005)} & \makecell{3.366 \\ (0.027)} & \makecell{0.966 \\ (0.002)} & \makecell{1.241 \\ (0.012)} \\
 & TS & \makecell{0.901 \\ (0.002)} & \makecell{1.941 \\ (0.020)} & \makecell{0.759 \\ (0.005)} & \makecell{3.403 \\ (0.025)} & \makecell{0.959 \\ (0.002)} & \makecell{1.351 \\ (0.016)} \\
 & Trustscore & \makecell{0.898 \\ (0.002)} & \textcolor{red}{\makecell{1.853 \\ (0.019)}} & \makecell{0.741 \\ (0.006)} & \makecell{3.377 \\ (0.033)} & \makecell{0.962 \\ (0.002)} & \makecell{1.240 \\ (0.013)} \\
\midrule
2000 & ACP-ACC & \makecell{0.903 \\ (0.002)} & \makecell{1.947 \\ (0.019)} & \textcolor{red}{\makecell{0.817 \\ (0.005)}} & \makecell{3.646 \\ (0.023)} & \makecell{0.937 \\ (0.002)} & \makecell{1.276 \\ (0.017)} \\
 & ACP-AEC & \makecell{0.897 \\ (0.002)} & \makecell{1.878 \\ (0.017)} & \makecell{0.780 \\ (0.004)} & \makecell{3.571 \\ (0.015)} & \makecell{0.945 \\ (0.003)} & \makecell{1.200 \\ (0.013)} \\
 & ACP-MC & \makecell{0.901 \\ (0.002)} & \makecell{2.132 \\ (0.016)} & \textcolor{red}{\makecell{0.838 \\ (0.004)}} & \makecell{3.957 \\ (0.016)} & \makecell{0.926 \\ (0.002)} & \makecell{1.407 \\ (0.012)} \\
 & ACP-MC-Oracle & \makecell{0.903 \\ (0.002)} & \textcolor{red}{\makecell{1.613 \\ (0.007)}} & \textcolor{red}{\makecell{0.904 \\ (0.002)}} & \makecell{3.419 \\ (0.032)} & \makecell{0.902 \\ (0.002)} & \makecell{0.902 \\ (0.002)} \\
 & AdaTS & \makecell{0.900 \\ (0.002)} & \makecell{1.905 \\ (0.018)} & \makecell{0.758 \\ (0.005)} & \makecell{3.378 \\ (0.027)} & \makecell{0.956 \\ (0.002)} & \makecell{1.325 \\ (0.017)} \\
 & Retrain & \makecell{0.901 \\ (0.002)} & \makecell{3.063 \\ (0.036)} & \makecell{0.781 \\ (0.007)} & \makecell{3.522 \\ (0.035)} & \makecell{0.948 \\ (0.002)} & \makecell{2.881 \\ (0.037)} \\
 & Standard & \makecell{0.901 \\ (0.001)} & \textcolor{red}{\makecell{1.824 \\ (0.020)}} & \makecell{0.741 \\ (0.006)} & \makecell{3.338 \\ (0.032)} & \makecell{0.965 \\ (0.002)} & \makecell{1.225 \\ (0.015)} \\
 & TS & \makecell{0.900 \\ (0.002)} & \makecell{1.902 \\ (0.020)} & \makecell{0.757 \\ (0.005)} & \makecell{3.378 \\ (0.027)} & \makecell{0.958 \\ (0.002)} & \makecell{1.319 \\ (0.017)} \\
 & Trustscore & \makecell{0.900 \\ (0.002)} & \textcolor{red}{\makecell{1.827 \\ (0.019)}} & \makecell{0.739 \\ (0.006)} & \makecell{3.343 \\ (0.038)} & \makecell{0.965 \\ (0.003)} & \makecell{1.231 \\ (0.016)} \\
\midrule
5000 & ACP-ACC & \makecell{0.900 \\ (0.001)} & \makecell{1.969 \\ (0.024)} & \makecell{0.817 \\ (0.004)} & \makecell{3.667 \\ (0.024)} & \makecell{0.934 \\ (0.001)} & \makecell{1.271 \\ (0.016)} \\
 & ACP-AEC & \makecell{0.900 \\ (0.002)} & \makecell{1.888 \\ (0.013)} & \makecell{0.792 \\ (0.004)} & \makecell{3.631 \\ (0.017)} & \makecell{0.943 \\ (0.002)} & \makecell{1.199 \\ (0.010)} \\
 & ACP-MC & \makecell{0.901 \\ (0.001)} & \makecell{2.118 \\ (0.016)} & \textcolor{red}{\makecell{0.849 \\ (0.004)}} & \makecell{3.989 \\ (0.014)} & \makecell{0.922 \\ (0.002)} & \makecell{1.344 \\ (0.011)} \\
 & ACP-MC-Oracle & \makecell{0.902 \\ (0.001)} & \textcolor{red}{\makecell{1.610 \\ (0.007)}} & \textcolor{red}{\makecell{0.903 \\ (0.002)}} & \makecell{3.345 \\ (0.036)} & \makecell{0.901 \\ (0.002)} & \makecell{0.901 \\ (0.002)} \\
 & AdaTS & \makecell{0.900 \\ (0.002)} & \makecell{1.956 \\ (0.024)} & \makecell{0.762 \\ (0.004)} & \makecell{3.397 \\ (0.030)} & \makecell{0.958 \\ (0.002)} & \makecell{1.365 \\ (0.021)} \\
 & Retrain & \makecell{0.900 \\ (0.002)} & \makecell{2.082 \\ (0.024)} & \textcolor{red}{\makecell{0.821 \\ (0.007)}} & \makecell{3.371 \\ (0.037)} & \makecell{0.932 \\ (0.003)} & \makecell{1.552 \\ (0.021)} \\
 & Standard & \makecell{0.901 \\ (0.002)} & \textcolor{red}{\makecell{1.876 \\ (0.025)}} & \makecell{0.744 \\ (0.005)} & \makecell{3.372 \\ (0.034)} & \makecell{0.966 \\ (0.002)} & \makecell{1.259 \\ (0.017)} \\
 & TS & \makecell{0.901 \\ (0.002)} & \makecell{1.956 \\ (0.026)} & \makecell{0.762 \\ (0.004)} & \makecell{3.401 \\ (0.028)} & \makecell{0.959 \\ (0.002)} & \makecell{1.363 \\ (0.020)} \\
 & Trustscore & \makecell{0.899 \\ (0.001)} & \textcolor{red}{\makecell{1.881 \\ (0.024)}} & \makecell{0.743 \\ (0.006)} & \makecell{3.386 \\ (0.038)} & \makecell{0.964 \\ (0.002)} & \makecell{1.262 \\ (0.016)} \\
\bottomrule
\end{tabular}

\end{table}

\clearpage

\subsubsection{Real data experiments}

\begin{table}[!htb]
\centering
    \caption{Performance of conformal prediction sets on the Camelyon17 dataset as a function of the total calibration sample size. All methods achieve the target marginal coverage of \(90\%\), while our methods (ACP-MC, ACP-AEC, and ACP-ACC) retain practically efficient set sizes and are more uncertainty-aware, assigning larger sets to more unreliable samples and smaller sets to more reliable ones. Red numbers indicate the three smallest average set sizes and the three highest conditional coverages. See the corresponding plots in Figure~\ref{fig:main_exp_sim_nnew_overall}.}
  \label{tab:main_exp_real_ndata}
\centering
\fontsize{6}{6}\selectfont
\begin{tabular}[t]{rlllllllll}
\toprule
\multicolumn{2}{c}{ } & \multicolumn{2}{c}{\makecell{Marginal}} & \multicolumn{2}{c}{\makecell{$r^* \leq 0.33$ \\ mean $r^* = 0.17$}} & \multicolumn{2}{c}{\makecell{$0.33 < r^* \leq 0.67$ \\ mean $r^* = 0.53$}} & \multicolumn{2}{c}{\makecell{$r^* > 0.67$ \\ mean $r^* = 0.97$}} \\
\cmidrule(l{3pt}r{3pt}){3-4} \cmidrule(l{3pt}r{3pt}){5-6} \cmidrule(l{3pt}r{3pt}){7-8} \cmidrule(l{3pt}r{3pt}){9-10}
\makecell{Number \\ of \\ Calibration \\ Samples} & Method & Coverage & Size & Coverage & Size & Coverage & Size & Coverage & Size \\
\midrule
200 & ACP-ACC & \makecell{0.912 \\ (0.006)} & \makecell{1.117 \\ (0.012)} & \textcolor{red}{\makecell{0.745 \\ (0.050)}} & \makecell{1.719 \\ (0.057)} & \textcolor{red}{\makecell{0.831 \\ (0.027)}} & \makecell{1.751 \\ (0.038)} & \textcolor{red}{\makecell{0.918 \\ (0.005)}} & \makecell{1.083 \\ (0.011)} \\
 & ACP-AEC & \makecell{0.900 \\ (0.004)} & \textcolor{red}{\makecell{1.050 \\ (0.010)}} & \makecell{0.407 \\ (0.053)} & \makecell{1.324 \\ (0.063)} & \textcolor{red}{\makecell{0.717 \\ (0.027)}} & \makecell{1.541 \\ (0.035)} & \makecell{0.915 \\ (0.005)} & \makecell{1.027 \\ (0.010)} \\
 & ACP-MC & \makecell{0.900 \\ (0.005)} & \makecell{1.069 \\ (0.014)} & \textcolor{red}{\makecell{0.636 \\ (0.055)}} & \makecell{1.505 \\ (0.056)} & \makecell{0.689 \\ (0.030)} & \makecell{1.441 \\ (0.038)} & \makecell{0.911 \\ (0.005)} & \makecell{1.048 \\ (0.014)} \\
 & AdaTS & \makecell{0.899 \\ (0.005)} & \textcolor{red}{\makecell{1.063 \\ (0.013)}} & \makecell{0.250 \\ (0.036)} & \makecell{1.162 \\ (0.039)} & \makecell{0.633 \\ (0.022)} & \makecell{1.378 \\ (0.034)} & \textcolor{red}{\makecell{0.921 \\ (0.005)}} & \makecell{1.051 \\ (0.013)} \\
 & Retrain & \makecell{0.899 \\ (0.005)} & \makecell{1.389 \\ (0.021)} & \textcolor{red}{\makecell{0.858 \\ (0.021)}} & \makecell{1.257 \\ (0.052)} & \textcolor{red}{\makecell{0.754 \\ (0.020)}} & \makecell{1.343 \\ (0.034)} & \makecell{0.905 \\ (0.005)} & \makecell{1.395 \\ (0.021)} \\
 & Standard & \makecell{0.893 \\ (0.004)} & \textcolor{red}{\makecell{1.028 \\ (0.008)}} & \makecell{0.179 \\ (0.026)} & \makecell{1.083 \\ (0.026)} & \makecell{0.630 \\ (0.022)} & \makecell{1.374 \\ (0.032)} & \textcolor{red}{\makecell{0.917 \\ (0.005)}} & \makecell{1.016 \\ (0.008)} \\
\midrule
300 & ACP-ACC & \makecell{0.911 \\ (0.005)} & \makecell{1.101 \\ (0.010)} & \textcolor{red}{\makecell{0.857 \\ (0.040)}} & \makecell{1.838 \\ (0.044)} & \textcolor{red}{\makecell{0.861 \\ (0.016)}} & \makecell{1.779 \\ (0.019)} & \makecell{0.914 \\ (0.005)} & \makecell{1.065 \\ (0.010)} \\
 & ACP-AEC & \makecell{0.905 \\ (0.006)} & \textcolor{red}{\makecell{1.054 \\ (0.010)}} & \makecell{0.647 \\ (0.045)} & \makecell{1.547 \\ (0.053)} & \textcolor{red}{\makecell{0.743 \\ (0.023)}} & \makecell{1.584 \\ (0.031)} & \makecell{0.914 \\ (0.006)} & \makecell{1.025 \\ (0.010)} \\
 & ACP-MC & \makecell{0.907 \\ (0.005)} & \textcolor{red}{\makecell{1.064 \\ (0.010)}} & \textcolor{red}{\makecell{0.763 \\ (0.039)}} & \makecell{1.610 \\ (0.043)} & \makecell{0.724 \\ (0.021)} & \makecell{1.487 \\ (0.030)} & \textcolor{red}{\makecell{0.916 \\ (0.005)}} & \makecell{1.040 \\ (0.009)} \\
 & AdaTS & \makecell{0.897 \\ (0.006)} & \makecell{1.067 \\ (0.016)} & \makecell{0.362 \\ (0.045)} & \makecell{1.224 \\ (0.048)} & \makecell{0.643 \\ (0.026)} & \makecell{1.436 \\ (0.031)} & \textcolor{red}{\makecell{0.917 \\ (0.006)}} & \makecell{1.053 \\ (0.016)} \\
 & Retrain & \makecell{0.899 \\ (0.005)} & \makecell{1.315 \\ (0.021)} & \textcolor{red}{\makecell{0.837 \\ (0.025)}} & \makecell{1.305 \\ (0.053)} & \textcolor{red}{\makecell{0.765 \\ (0.020)}} & \makecell{1.386 \\ (0.027)} & \makecell{0.905 \\ (0.005)} & \makecell{1.315 \\ (0.021)} \\
 & Standard & \makecell{0.897 \\ (0.003)} & \textcolor{red}{\makecell{1.034 \\ (0.008)}} & \makecell{0.278 \\ (0.040)} & \makecell{1.134 \\ (0.035)} & \makecell{0.645 \\ (0.024)} & \makecell{1.436 \\ (0.026)} & \textcolor{red}{\makecell{0.919 \\ (0.004)}} & \makecell{1.020 \\ (0.008)} \\
\midrule
500 & ACP-ACC & \makecell{0.909 \\ (0.004)} & \makecell{1.086 \\ (0.008)} & \textcolor{red}{\makecell{0.837 \\ (0.036)}} & \makecell{1.822 \\ (0.038)} & \textcolor{red}{\makecell{0.874 \\ (0.015)}} & \makecell{1.801 \\ (0.019)} & \textcolor{red}{\makecell{0.911 \\ (0.004)}} & \makecell{1.048 \\ (0.008)} \\
 & ACP-AEC & \makecell{0.900 \\ (0.004)} & \textcolor{red}{\makecell{1.044 \\ (0.008)}} & \makecell{0.714 \\ (0.046)} & \makecell{1.673 \\ (0.053)} & \textcolor{red}{\makecell{0.798 \\ (0.020)}} & \makecell{1.678 \\ (0.027)} & \makecell{0.906 \\ (0.004)} & \makecell{1.011 \\ (0.008)} \\
 & ACP-MC & \makecell{0.901 \\ (0.003)} & \textcolor{red}{\makecell{1.046 \\ (0.009)}} & \textcolor{red}{\makecell{0.788 \\ (0.046)}} & \makecell{1.494 \\ (0.048)} & \textcolor{red}{\makecell{0.774 \\ (0.022)}} & \makecell{1.537 \\ (0.027)} & \makecell{0.906 \\ (0.003)} & \makecell{1.022 \\ (0.008)} \\
 & AdaTS & \makecell{0.900 \\ (0.005)} & \makecell{1.068 \\ (0.011)} & \makecell{0.271 \\ (0.032)} & \makecell{1.183 \\ (0.033)} & \makecell{0.614 \\ (0.028)} & \makecell{1.417 \\ (0.029)} & \textcolor{red}{\makecell{0.921 \\ (0.005)}} & \makecell{1.054 \\ (0.010)} \\
 & Retrain & \makecell{0.901 \\ (0.003)} & \makecell{1.248 \\ (0.012)} & \textcolor{red}{\makecell{0.888 \\ (0.014)}} & \makecell{1.358 \\ (0.043)} & \makecell{0.760 \\ (0.015)} & \makecell{1.373 \\ (0.030)} & \makecell{0.906 \\ (0.003)} & \makecell{1.244 \\ (0.012)} \\
 & Standard & \makecell{0.901 \\ (0.003)} & \textcolor{red}{\makecell{1.039 \\ (0.005)}} & \makecell{0.197 \\ (0.023)} & \makecell{1.098 \\ (0.021)} & \makecell{0.616 \\ (0.029)} & \makecell{1.411 \\ (0.029)} & \textcolor{red}{\makecell{0.924 \\ (0.003)}} & \makecell{1.026 \\ (0.005)} \\
\midrule
1000 & ACP-ACC & \makecell{0.901 \\ (0.002)} & \makecell{1.065 \\ (0.005)} & \textcolor{red}{\makecell{0.861 \\ (0.027)}} & \makecell{1.854 \\ (0.027)} & \textcolor{red}{\makecell{0.882 \\ (0.016)}} & \makecell{1.811 \\ (0.020)} & \makecell{0.902 \\ (0.002)} & \makecell{1.025 \\ (0.006)} \\
 & ACP-AEC & \makecell{0.899 \\ (0.003)} & \textcolor{red}{\makecell{1.041 \\ (0.005)}} & \makecell{0.748 \\ (0.033)} & \makecell{1.721 \\ (0.036)} & \textcolor{red}{\makecell{0.807 \\ (0.015)}} & \makecell{1.689 \\ (0.020)} & \makecell{0.903 \\ (0.003)} & \makecell{1.006 \\ (0.005)} \\
 & ACP-MC & \makecell{0.900 \\ (0.002)} & \textcolor{red}{\makecell{1.029 \\ (0.006)}} & \textcolor{red}{\makecell{0.870 \\ (0.019)}} & \makecell{1.383 \\ (0.042)} & \textcolor{red}{\makecell{0.819 \\ (0.016)}} & \makecell{1.588 \\ (0.028)} & \makecell{0.903 \\ (0.002)} & \makecell{1.005 \\ (0.005)} \\
 & AdaTS & \makecell{0.901 \\ (0.003)} & \makecell{1.068 \\ (0.007)} & \makecell{0.301 \\ (0.032)} & \makecell{1.200 \\ (0.031)} & \makecell{0.644 \\ (0.026)} & \makecell{1.413 \\ (0.032)} & \textcolor{red}{\makecell{0.922 \\ (0.003)}} & \makecell{1.055 \\ (0.008)} \\
 & Retrain & \makecell{0.902 \\ (0.003)} & \makecell{1.195 \\ (0.008)} & \textcolor{red}{\makecell{0.893 \\ (0.014)}} & \makecell{1.348 \\ (0.032)} & \makecell{0.798 \\ (0.016)} & \makecell{1.462 \\ (0.025)} & \textcolor{red}{\makecell{0.905 \\ (0.003)}} & \makecell{1.186 \\ (0.008)} \\
 & Standard & \makecell{0.900 \\ (0.002)} & \textcolor{red}{\makecell{1.039 \\ (0.005)}} & \makecell{0.222 \\ (0.029)} & \makecell{1.119 \\ (0.030)} & \makecell{0.638 \\ (0.026)} & \makecell{1.412 \\ (0.030)} & \textcolor{red}{\makecell{0.923 \\ (0.003)}} & \makecell{1.026 \\ (0.005)} \\
\midrule
2000 & ACP-ACC & \makecell{0.902 \\ (0.002)} & \makecell{1.061 \\ (0.004)} & \textcolor{red}{\makecell{0.879 \\ (0.023)}} & \makecell{1.857 \\ (0.025)} & \textcolor{red}{\makecell{0.838 \\ (0.017)}} & \makecell{1.726 \\ (0.023)} & \makecell{0.904 \\ (0.002)} & \makecell{1.023 \\ (0.004)} \\
 & ACP-AEC & \makecell{0.902 \\ (0.003)} & \textcolor{red}{\makecell{1.046 \\ (0.004)}} & \makecell{0.858 \\ (0.024)} & \makecell{1.835 \\ (0.026)} & \makecell{0.827 \\ (0.017)} & \makecell{1.685 \\ (0.026)} & \textcolor{red}{\makecell{0.905 \\ (0.003)}} & \makecell{1.009 \\ (0.004)} \\
 & ACP-MC & \makecell{0.899 \\ (0.003)} & \textcolor{red}{\makecell{1.029 \\ (0.005)}} & \textcolor{red}{\makecell{0.897 \\ (0.023)}} & \makecell{1.321 \\ (0.037)} & \textcolor{red}{\makecell{0.851 \\ (0.016)}} & \makecell{1.615 \\ (0.022)} & \makecell{0.901 \\ (0.003)} & \makecell{1.005 \\ (0.005)} \\
 & AdaTS & \makecell{0.900 \\ (0.002)} & \makecell{1.065 \\ (0.006)} & \makecell{0.289 \\ (0.035)} & \makecell{1.171 \\ (0.029)} & \makecell{0.674 \\ (0.016)} & \makecell{1.457 \\ (0.019)} & \textcolor{red}{\makecell{0.920 \\ (0.003)}} & \makecell{1.051 \\ (0.007)} \\
 & Retrain & \makecell{0.902 \\ (0.003)} & \makecell{1.132 \\ (0.007)} & \textcolor{red}{\makecell{0.911 \\ (0.015)}} & \makecell{1.210 \\ (0.039)} & \textcolor{red}{\makecell{0.850 \\ (0.013)}} & \makecell{1.435 \\ (0.023)} & \makecell{0.904 \\ (0.003)} & \makecell{1.122 \\ (0.007)} \\
 & Standard & \makecell{0.901 \\ (0.002)} & \textcolor{red}{\makecell{1.041 \\ (0.004)}} & \makecell{0.241 \\ (0.030)} & \makecell{1.127 \\ (0.026)} & \makecell{0.677 \\ (0.018)} & \makecell{1.466 \\ (0.023)} & \textcolor{red}{\makecell{0.923 \\ (0.003)}} & \makecell{1.026 \\ (0.004)} \\
\bottomrule
\end{tabular}

\end{table}



\begin{figure*}[tbh]
    \centering
    \includegraphics[width=\linewidth]{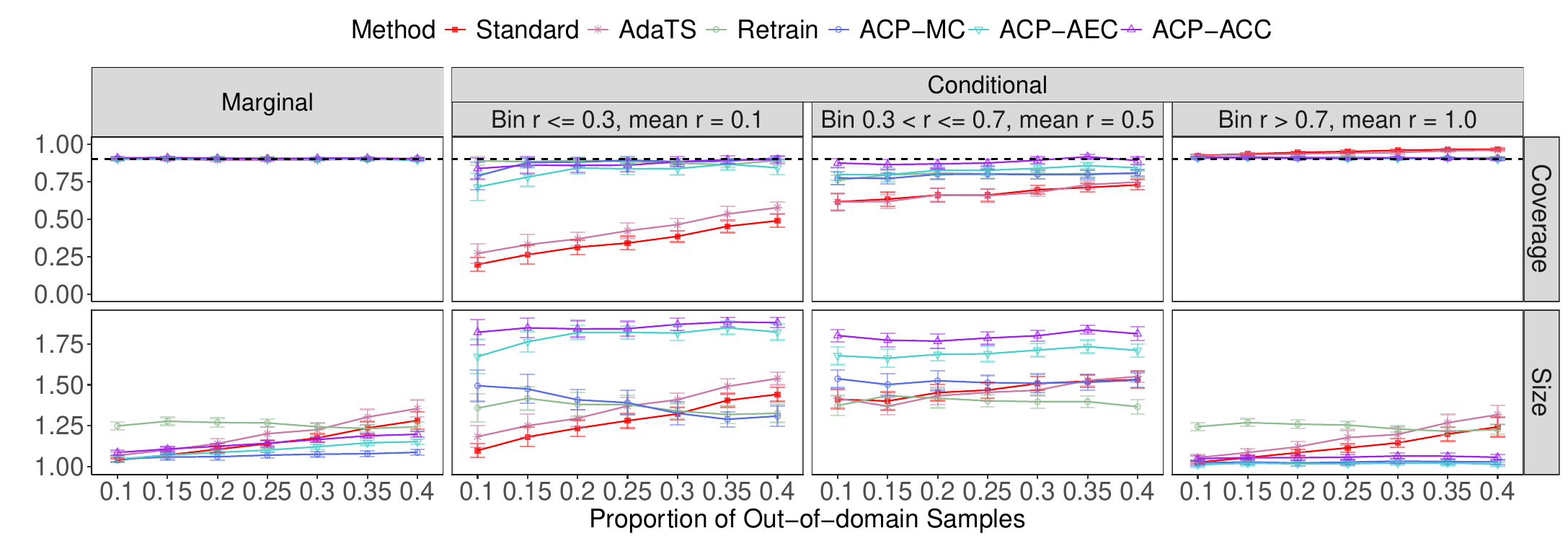}
    \caption{Performance of conformal prediction sets for Camelyon17 dataset as a function of the total calibration sample size. Left Panel: marginal coverage and average prediction set size. All methods achieve the marginal coverage of $90\%$, while our methods (ACP-MC, ACP-AEC, ACP-ACC) retain practically efficient sets sizes. Right Panel: coverage conditional on estimated reliability bin. Our methods achieve higher conditional coverage.
    Error bars denote two standard errors. See Table~\ref{tab:supp_exp_real_beta} for further details.}
    \label{fig:supp_exp_real_beta_overall}
\end{figure*}

\begin{table}[!htb]
\centering
    \caption{Performance of conformal prediction sets on the Camelyon17 dataset as a function of the proportion of out-of-domain samples. All methods achieve the target marginal coverage of \(90\%\), while our methods (ACP-MC, ACP-AEC, and ACP-ACC) retain practically efficient set sizes and are more uncertainty-aware, assigning larger sets to more unreliable samples and smaller sets to more reliable ones. Red numbers indicate the three smallest average set sizes and the three highest conditional coverages. See the corresponding plots in Figure~\ref{fig:supp_exp_real_beta_overall}.}
  \label{tab:supp_exp_real_beta}
\centering
\fontsize{6}{6}\selectfont
\begin{tabular}[t]{rlllllllll}
\toprule
\multicolumn{2}{c}{ } & \multicolumn{2}{c}{\makecell{Marginal}} & \multicolumn{2}{c}{\makecell{$r^* \leq 0.33$ \\ mean $r^* = 0.14$}} & \multicolumn{2}{c}{\makecell{$0.33 < r^* \leq 0.67$ \\ mean $r^* = 0.52$}} & \multicolumn{2}{c}{\makecell{$r^* > 0.67$ \\ mean $r^* = 0.97$}} \\
\cmidrule(l{3pt}r{3pt}){3-4} \cmidrule(l{3pt}r{3pt}){5-6} \cmidrule(l{3pt}r{3pt}){7-8} \cmidrule(l{3pt}r{3pt}){9-10}
\makecell{Proportion \\ of \\ Hard \\ samples} & Method & Coverage & Size & Coverage & Size & Coverage & Size & Coverage & Size \\
\midrule
0.1 & ACP-ACC & \makecell{0.909 \\ (0.004)} & \makecell{1.086 \\ (0.008)} & \textcolor{red}{\makecell{0.837 \\ (0.036)}} & \makecell{1.822 \\ (0.038)} & \textcolor{red}{\makecell{0.874 \\ (0.015)}} & \makecell{1.801 \\ (0.019)} & \textcolor{red}{\makecell{0.911 \\ (0.004)}} & \makecell{1.048 \\ (0.008)} \\
 & ACP-AEC & \makecell{0.900 \\ (0.004)} & \textcolor{red}{\makecell{1.044 \\ (0.008)}} & \makecell{0.714 \\ (0.046)} & \makecell{1.673 \\ (0.053)} & \textcolor{red}{\makecell{0.798 \\ (0.020)}} & \makecell{1.678 \\ (0.027)} & \makecell{0.906 \\ (0.004)} & \makecell{1.011 \\ (0.008)} \\
 & ACP-MC & \makecell{0.901 \\ (0.003)} & \textcolor{red}{\makecell{1.046 \\ (0.009)}} & \textcolor{red}{\makecell{0.788 \\ (0.046)}} & \makecell{1.494 \\ (0.048)} & \textcolor{red}{\makecell{0.774 \\ (0.022)}} & \makecell{1.537 \\ (0.027)} & \makecell{0.906 \\ (0.003)} & \makecell{1.022 \\ (0.008)} \\
 & AdaTS & \makecell{0.900 \\ (0.005)} & \makecell{1.068 \\ (0.011)} & \makecell{0.271 \\ (0.032)} & \makecell{1.183 \\ (0.033)} & \makecell{0.614 \\ (0.028)} & \makecell{1.417 \\ (0.029)} & \textcolor{red}{\makecell{0.921 \\ (0.005)}} & \makecell{1.054 \\ (0.010)} \\
 & Retrain & \makecell{0.901 \\ (0.003)} & \makecell{1.248 \\ (0.012)} & \textcolor{red}{\makecell{0.888 \\ (0.014)}} & \makecell{1.358 \\ (0.043)} & \makecell{0.760 \\ (0.015)} & \makecell{1.373 \\ (0.030)} & \makecell{0.906 \\ (0.003)} & \makecell{1.244 \\ (0.012)} \\
 & Standard & \makecell{0.901 \\ (0.003)} & \textcolor{red}{\makecell{1.039 \\ (0.005)}} & \makecell{0.197 \\ (0.023)} & \makecell{1.098 \\ (0.021)} & \makecell{0.616 \\ (0.029)} & \makecell{1.411 \\ (0.029)} & \textcolor{red}{\makecell{0.924 \\ (0.003)}} & \makecell{1.026 \\ (0.005)} \\
\midrule
0.15 & ACP-ACC & \makecell{0.910 \\ (0.003)} & \makecell{1.107 \\ (0.006)} & \textcolor{red}{\makecell{0.860 \\ (0.028)}} & \makecell{1.849 \\ (0.029)} & \textcolor{red}{\makecell{0.864 \\ (0.016)}} & \makecell{1.773 \\ (0.021)} & \makecell{0.913 \\ (0.004)} & \makecell{1.054 \\ (0.007)} \\
 & ACP-AEC & \makecell{0.905 \\ (0.004)} & \textcolor{red}{\makecell{1.070 \\ (0.008)}} & \makecell{0.780 \\ (0.031)} & \makecell{1.763 \\ (0.031)} & \textcolor{red}{\makecell{0.796 \\ (0.023)}} & \makecell{1.663 \\ (0.028)} & \makecell{0.912 \\ (0.004)} & \makecell{1.020 \\ (0.008)} \\
 & ACP-MC & \makecell{0.902 \\ (0.004)} & \textcolor{red}{\makecell{1.058 \\ (0.009)}} & \textcolor{red}{\makecell{0.877 \\ (0.014)}} & \makecell{1.475 \\ (0.044)} & \makecell{0.771 \\ (0.017)} & \makecell{1.501 \\ (0.033)} & \makecell{0.908 \\ (0.004)} & \makecell{1.027 \\ (0.008)} \\
 & AdaTS & \makecell{0.899 \\ (0.004)} & \makecell{1.101 \\ (0.011)} & \makecell{0.330 \\ (0.034)} & \makecell{1.249 \\ (0.035)} & \makecell{0.617 \\ (0.023)} & \makecell{1.369 \\ (0.027)} & \textcolor{red}{\makecell{0.930 \\ (0.004)}} & \makecell{1.085 \\ (0.011)} \\
 & Retrain & \makecell{0.909 \\ (0.004)} & \makecell{1.277 \\ (0.012)} & \textcolor{red}{\makecell{0.883 \\ (0.011)}} & \makecell{1.417 \\ (0.036)} & \textcolor{red}{\makecell{0.794 \\ (0.015)}} & \makecell{1.433 \\ (0.032)} & \textcolor{red}{\makecell{0.914 \\ (0.003)}} & \makecell{1.268 \\ (0.012)} \\
 & Standard & \makecell{0.901 \\ (0.003)} & \textcolor{red}{\makecell{1.070 \\ (0.008)}} & \makecell{0.263 \\ (0.031)} & \makecell{1.181 \\ (0.030)} & \makecell{0.633 \\ (0.023)} & \makecell{1.401 \\ (0.028)} & \textcolor{red}{\makecell{0.934 \\ (0.003)}} & \makecell{1.054 \\ (0.008)} \\
\midrule
0.2 & ACP-ACC & \makecell{0.906 \\ (0.004)} & \makecell{1.124 \\ (0.008)} & \textcolor{red}{\makecell{0.858 \\ (0.023)}} & \makecell{1.842 \\ (0.025)} & \textcolor{red}{\makecell{0.868 \\ (0.015)}} & \makecell{1.767 \\ (0.021)} & \makecell{0.908 \\ (0.004)} & \makecell{1.054 \\ (0.007)} \\
 & ACP-AEC & \makecell{0.903 \\ (0.004)} & \textcolor{red}{\makecell{1.086 \\ (0.007)}} & \makecell{0.841 \\ (0.020)} & \makecell{1.819 \\ (0.022)} & \textcolor{red}{\makecell{0.826 \\ (0.017)}} & \makecell{1.686 \\ (0.019)} & \makecell{0.908 \\ (0.004)} & \makecell{1.017 \\ (0.006)} \\
 & ACP-MC & \makecell{0.900 \\ (0.004)} & \textcolor{red}{\makecell{1.059 \\ (0.010)}} & \textcolor{red}{\makecell{0.880 \\ (0.015)}} & \makecell{1.408 \\ (0.031)} & \makecell{0.798 \\ (0.017)} & \makecell{1.525 \\ (0.030)} & \makecell{0.904 \\ (0.004)} & \makecell{1.021 \\ (0.009)} \\
 & AdaTS & \makecell{0.894 \\ (0.004)} & \makecell{1.140 \\ (0.015)} & \makecell{0.368 \\ (0.023)} & \makecell{1.295 \\ (0.026)} & \makecell{0.660 \\ (0.022)} & \makecell{1.433 \\ (0.027)} & \textcolor{red}{\makecell{0.934 \\ (0.004)}} & \makecell{1.121 \\ (0.016)} \\
 & Retrain & \makecell{0.908 \\ (0.004)} & \makecell{1.269 \\ (0.014)} & \textcolor{red}{\makecell{0.891 \\ (0.008)}} & \makecell{1.380 \\ (0.036)} & \textcolor{red}{\makecell{0.806 \\ (0.018)}} & \makecell{1.418 \\ (0.030)} & \textcolor{red}{\makecell{0.913 \\ (0.004)}} & \makecell{1.259 \\ (0.014)} \\
 & Standard & \makecell{0.900 \\ (0.003)} & \textcolor{red}{\makecell{1.105 \\ (0.010)}} & \makecell{0.312 \\ (0.024)} & \makecell{1.234 \\ (0.025)} & \makecell{0.660 \\ (0.022)} & \makecell{1.452 \\ (0.025)} & \textcolor{red}{\makecell{0.945 \\ (0.004)}} & \makecell{1.085 \\ (0.010)} \\
\midrule
0.25 & ACP-ACC & \makecell{0.903 \\ (0.004)} & \makecell{1.142 \\ (0.008)} & \textcolor{red}{\makecell{0.859 \\ (0.022)}} & \makecell{1.844 \\ (0.022)} & \textcolor{red}{\makecell{0.874 \\ (0.012)}} & \makecell{1.786 \\ (0.020)} & \makecell{0.905 \\ (0.004)} & \makecell{1.057 \\ (0.008)} \\
 & ACP-AEC & \makecell{0.897 \\ (0.004)} & \textcolor{red}{\makecell{1.100 \\ (0.007)}} & \makecell{0.836 \\ (0.019)} & \makecell{1.820 \\ (0.020)} & \textcolor{red}{\makecell{0.826 \\ (0.016)}} & \makecell{1.690 \\ (0.024)} & \makecell{0.904 \\ (0.004)} & \makecell{1.017 \\ (0.007)} \\
 & ACP-MC & \makecell{0.900 \\ (0.003)} & \textcolor{red}{\makecell{1.069 \\ (0.009)}} & \textcolor{red}{\makecell{0.893 \\ (0.012)}} & \makecell{1.390 \\ (0.037)} & \makecell{0.800 \\ (0.017)} & \makecell{1.513 \\ (0.022)} & \makecell{0.905 \\ (0.004)} & \makecell{1.027 \\ (0.009)} \\
 & AdaTS & \makecell{0.895 \\ (0.004)} & \makecell{1.201 \\ (0.021)} & \makecell{0.422 \\ (0.026)} & \makecell{1.370 \\ (0.026)} & \makecell{0.656 \\ (0.022)} & \makecell{1.454 \\ (0.026)} & \textcolor{red}{\makecell{0.940 \\ (0.006)}} & \makecell{1.177 \\ (0.022)} \\
 & Retrain & \makecell{0.909 \\ (0.003)} & \makecell{1.266 \\ (0.012)} & \textcolor{red}{\makecell{0.884 \\ (0.012)}} & \makecell{1.379 \\ (0.028)} & \textcolor{red}{\makecell{0.803 \\ (0.015)}} & \makecell{1.402 \\ (0.019)} & \textcolor{red}{\makecell{0.915 \\ (0.003)}} & \makecell{1.254 \\ (0.012)} \\
 & Standard & \makecell{0.898 \\ (0.003)} & \textcolor{red}{\makecell{1.139 \\ (0.012)}} & \makecell{0.340 \\ (0.023)} & \makecell{1.280 \\ (0.023)} & \makecell{0.660 \\ (0.021)} & \makecell{1.468 \\ (0.024)} & \textcolor{red}{\makecell{0.951 \\ (0.004)}} & \makecell{1.114 \\ (0.013)} \\
\midrule
0.3 & ACP-ACC & \makecell{0.906 \\ (0.003)} & \textcolor{red}{\makecell{1.165 \\ (0.008)}} & \textcolor{red}{\makecell{0.883 \\ (0.016)}} & \makecell{1.870 \\ (0.019)} & \textcolor{red}{\makecell{0.892 \\ (0.011)}} & \makecell{1.800 \\ (0.017)} & \makecell{0.908 \\ (0.004)} & \makecell{1.064 \\ (0.008)} \\
 & ACP-AEC & \makecell{0.897 \\ (0.004)} & \textcolor{red}{\makecell{1.121 \\ (0.008)}} & \makecell{0.837 \\ (0.021)} & \makecell{1.817 \\ (0.023)} & \textcolor{red}{\makecell{0.839 \\ (0.012)}} & \makecell{1.712 \\ (0.020)} & \makecell{0.903 \\ (0.004)} & \makecell{1.021 \\ (0.007)} \\
 & ACP-MC & \makecell{0.899 \\ (0.003)} & \textcolor{red}{\makecell{1.075 \\ (0.008)}} & \textcolor{red}{\makecell{0.883 \\ (0.010)}} & \makecell{1.326 \\ (0.036)} & \makecell{0.797 \\ (0.015)} & \makecell{1.511 \\ (0.028)} & \makecell{0.905 \\ (0.003)} & \makecell{1.033 \\ (0.007)} \\
 & AdaTS & \makecell{0.894 \\ (0.003)} & \makecell{1.225 \\ (0.022)} & \makecell{0.463 \\ (0.020)} & \makecell{1.409 \\ (0.021)} & \makecell{0.678 \\ (0.013)} & \makecell{1.467 \\ (0.018)} & \textcolor{red}{\makecell{0.944 \\ (0.005)}} & \makecell{1.196 \\ (0.023)} \\
 & Retrain & \makecell{0.902 \\ (0.003)} & \makecell{1.243 \\ (0.009)} & \textcolor{red}{\makecell{0.873 \\ (0.008)}} & \makecell{1.339 \\ (0.023)} & \textcolor{red}{\makecell{0.800 \\ (0.016)}} & \makecell{1.397 \\ (0.021)} & \textcolor{red}{\makecell{0.909 \\ (0.003)}} & \makecell{1.229 \\ (0.010)} \\
 & Standard & \makecell{0.900 \\ (0.003)} & \makecell{1.174 \\ (0.013)} & \makecell{0.385 \\ (0.019)} & \makecell{1.322 \\ (0.019)} & \makecell{0.695 \\ (0.015)} & \makecell{1.508 \\ (0.021)} & \textcolor{red}{\makecell{0.958 \\ (0.004)}} & \makecell{1.144 \\ (0.013)} \\
\midrule
0.35 & ACP-ACC & \makecell{0.907 \\ (0.003)} & \textcolor{red}{\makecell{1.188 \\ (0.007)}} & \textcolor{red}{\makecell{0.897 \\ (0.013)}} & \makecell{1.884 \\ (0.014)} & \textcolor{red}{\makecell{0.916 \\ (0.008)}} & \makecell{1.837 \\ (0.013)} & \makecell{0.906 \\ (0.004)} & \makecell{1.063 \\ (0.010)} \\
 & ACP-AEC & \makecell{0.901 \\ (0.004)} & \textcolor{red}{\makecell{1.145 \\ (0.008)}} & \textcolor{red}{\makecell{0.866 \\ (0.019)}} & \makecell{1.849 \\ (0.020)} & \textcolor{red}{\makecell{0.856 \\ (0.011)}} & \makecell{1.734 \\ (0.020)} & \makecell{0.906 \\ (0.004)} & \makecell{1.022 \\ (0.008)} \\
 & ACP-MC & \makecell{0.899 \\ (0.004)} & \textcolor{red}{\makecell{1.078 \\ (0.009)}} & \textcolor{red}{\makecell{0.886 \\ (0.009)}} & \makecell{1.289 \\ (0.023)} & \makecell{0.798 \\ (0.013)} & \makecell{1.515 \\ (0.024)} & \makecell{0.906 \\ (0.004)} & \makecell{1.030 \\ (0.008)} \\
 & AdaTS & \makecell{0.903 \\ (0.004)} & \makecell{1.303 \\ (0.023)} & \makecell{0.534 \\ (0.025)} & \makecell{1.491 \\ (0.024)} & \makecell{0.730 \\ (0.013)} & \makecell{1.527 \\ (0.018)} & \textcolor{red}{\makecell{0.956 \\ (0.004)}} & \makecell{1.268 \\ (0.025)} \\
 & Retrain & \makecell{0.898 \\ (0.003)} & \makecell{1.232 \\ (0.011)} & \makecell{0.863 \\ (0.009)} & \makecell{1.318 \\ (0.028)} & \textcolor{red}{\makecell{0.803 \\ (0.012)}} & \makecell{1.397 \\ (0.017)} & \textcolor{red}{\makecell{0.907 \\ (0.003)}} & \makecell{1.215 \\ (0.011)} \\
 & Standard & \makecell{0.901 \\ (0.003)} & \makecell{1.235 \\ (0.019)} & \makecell{0.452 \\ (0.021)} & \makecell{1.405 \\ (0.021)} & \makecell{0.711 \\ (0.014)} & \makecell{1.522 \\ (0.019)} & \textcolor{red}{\makecell{0.964 \\ (0.003)}} & \makecell{1.199 \\ (0.021)} \\
\midrule
0.4 & ACP-ACC & \makecell{0.902 \\ (0.004)} & \textcolor{red}{\makecell{1.197 \\ (0.009)}} & \textcolor{red}{\makecell{0.895 \\ (0.013)}} & \makecell{1.881 \\ (0.016)} & \textcolor{red}{\makecell{0.891 \\ (0.012)}} & \makecell{1.812 \\ (0.021)} & \makecell{0.902 \\ (0.004)} & \makecell{1.057 \\ (0.009)} \\
 & ACP-AEC & \makecell{0.892 \\ (0.004)} & \textcolor{red}{\makecell{1.152 \\ (0.007)}} & \makecell{0.843 \\ (0.023)} & \makecell{1.823 \\ (0.025)} & \textcolor{red}{\makecell{0.843 \\ (0.012)}} & \makecell{1.711 \\ (0.020)} & \makecell{0.898 \\ (0.004)} & \makecell{1.013 \\ (0.006)} \\
 & ACP-MC & \makecell{0.897 \\ (0.004)} & \textcolor{red}{\makecell{1.087 \\ (0.009)}} & \textcolor{red}{\makecell{0.908 \\ (0.005)}} & \makecell{1.308 \\ (0.031)} & \textcolor{red}{\makecell{0.808 \\ (0.011)}} & \makecell{1.528 \\ (0.022)} & \makecell{0.901 \\ (0.005)} & \makecell{1.028 \\ (0.008)} \\
 & AdaTS & \makecell{0.904 \\ (0.004)} & \makecell{1.354 \\ (0.026)} & \makecell{0.577 \\ (0.019)} & \makecell{1.538 \\ (0.019)} & \makecell{0.747 \\ (0.013)} & \makecell{1.550 \\ (0.018)} & \textcolor{red}{\makecell{0.959 \\ (0.004)}} & \makecell{1.317 \\ (0.028)} \\
 & Retrain & \makecell{0.904 \\ (0.004)} & \makecell{1.243 \\ (0.015)} & \textcolor{red}{\makecell{0.890 \\ (0.006)}} & \makecell{1.326 \\ (0.028)} & \makecell{0.804 \\ (0.012)} & \makecell{1.366 \\ (0.022)} & \textcolor{red}{\makecell{0.912 \\ (0.004)}} & \makecell{1.226 \\ (0.014)} \\
 & Standard & \makecell{0.899 \\ (0.003)} & \makecell{1.281 \\ (0.026)} & \makecell{0.489 \\ (0.022)} & \makecell{1.442 \\ (0.022)} & \makecell{0.728 \\ (0.017)} & \makecell{1.532 \\ (0.026)} & \textcolor{red}{\makecell{0.966 \\ (0.003)}} & \makecell{1.242 \\ (0.030)} \\
\bottomrule
\end{tabular}

\end{table}


\begin{figure*}[!htb]
    \centering
    \includegraphics[width=\linewidth]{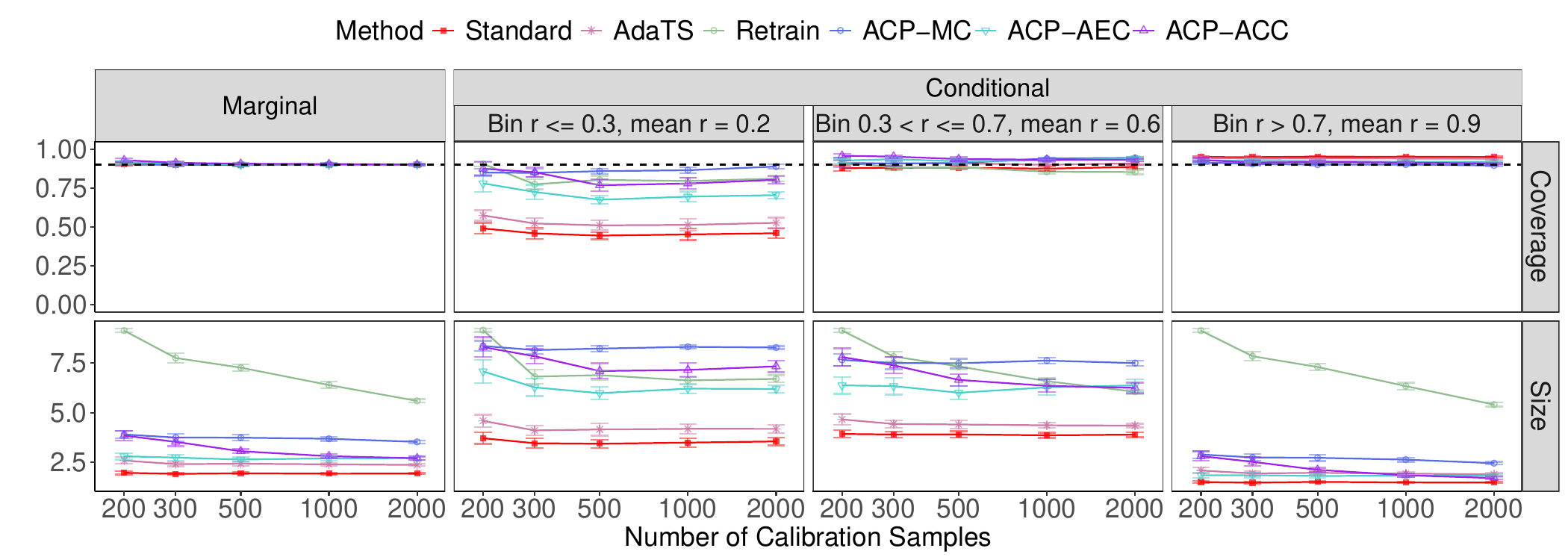}
    \caption{Performance of conformal prediction sets on the CIFAR-10 and CIFAR-10-C image classification task \citep{hendrycks2019robustness} as a function of calibration sample size $|\mathcal{D}_{\text{cal}}|$, at miscoverage level $\alpha=0.1$. The target distribution contains $90\%$ in-domain (CIFAR-10) and $10\%$ out-of-domain 
    (CIFAR-10-C) samples. Left: marginal coverage and average prediction set 
    size; all methods achieve nominal $90\%$ marginal coverage, and ACP variants 
    (ACP-MC, ACP-AEC, ACP-ACC) maintain compact set sizes. Right: coverage 
    conditional on estimated reliability bin (fraction of correct legacy-model 
    predictions among $50$ nearest neighbors in embedding space). Standard CP and 
    AdaTS undercover unreliable samples, while ACP variants achieve higher conditional 
    coverage without inflating set sizes. Error bars denote two standard errors. See 
    Table~\ref{tab:supp_exp_real_ndata_cifar10} for numerical details.}
    \label{fig:supp_exp_real_nnew_overall_cifar10}
\end{figure*}

\begin{table}[!htb]
\centering
    \caption{Performance of conformal prediction sets on the CIFAR-10 and CIFAR-10-C datasets as a function of the total calibration sample size. All methods achieve the target marginal coverage of \(90\%\), while our methods (ACP-MC, ACP-AEC, and ACP-ACC) retain practically efficient set sizes and are more uncertainty-aware, assigning larger sets to more unreliable samples and smaller sets to more reliable ones. Red numbers indicate the three smallest average set sizes and the three highest conditional coverages. See the corresponding plots in Figure~\ref{fig:supp_exp_real_nnew_overall_cifar10}.}
  \label{tab:supp_exp_real_ndata_cifar10}
\centering
\fontsize{6}{6}\selectfont
\begin{tabular}[t]{rlllllllll}
\toprule
\multicolumn{2}{c}{ } & \multicolumn{2}{c}{\makecell{Marginal}} & \multicolumn{2}{c}{\makecell{$r^* \leq 0.33$ \\ mean $r^* = 0.18$}} & \multicolumn{2}{c}{\makecell{$0.33 < r^* \leq 0.67$ \\ mean $r^* = 0.56$}} & \multicolumn{2}{c}{\makecell{$r^* > 0.67$ \\ mean $r^* = 0.92$}} \\
\cmidrule(l{3pt}r{3pt}){3-4} \cmidrule(l{3pt}r{3pt}){5-6} \cmidrule(l{3pt}r{3pt}){7-8} \cmidrule(l{3pt}r{3pt}){9-10}
\makecell{Calibration \\ Size} & Method & Coverage & Size & Coverage & Size & Coverage & Size & Coverage & Size \\
\midrule
200 & ACP-ACC & \makecell{0.929 \\ (0.005)} & \makecell{3.843 \\ (0.122)} & \textcolor{red}{\makecell{0.874 \\ (0.022)}} & \makecell{8.302 \\ (0.254)} & \textcolor{red}{\makecell{0.958 \\ (0.007)}} & \makecell{7.798 \\ (0.222)} & \textcolor{red}{\makecell{0.932 \\ (0.006)}} & \makecell{2.809 \\ (0.121)} \\
 & ACP-AEC & \makecell{0.915 \\ (0.005)} & \textcolor{red}{\makecell{2.801 \\ (0.084)}} & \makecell{0.780 \\ (0.027)} & \makecell{7.078 \\ (0.290)} & \textcolor{red}{\makecell{0.927 \\ (0.009)}} & \makecell{6.374 \\ (0.214)} & \makecell{0.928 \\ (0.006)} & \makecell{1.844 \\ (0.062)} \\
 & ACP-MC & \makecell{0.908 \\ (0.005)} & \makecell{3.900 \\ (0.098)} & \textcolor{red}{\makecell{0.849 \\ (0.012)}} & \makecell{8.360 \\ (0.130)} & \makecell{0.904 \\ (0.011)} & \makecell{7.658 \\ (0.156)} & \makecell{0.915 \\ (0.005)} & \makecell{2.890 \\ (0.095)} \\
 & AdaTS & \makecell{0.909 \\ (0.005)} & \textcolor{red}{\makecell{2.588 \\ (0.085)}} & \makecell{0.573 \\ (0.017)} & \makecell{4.587 \\ (0.153)} & \makecell{0.914 \\ (0.009)} & \makecell{4.659 \\ (0.136)} & \textcolor{red}{\makecell{0.944 \\ (0.005)}} & \makecell{2.082 \\ (0.077)} \\
 & Retrain & \makecell{0.911 \\ (0.004)} & \makecell{9.136 \\ (0.045)} & \textcolor{red}{\makecell{0.906 \\ (0.008)}} & \makecell{9.156 \\ (0.044)} & \textcolor{red}{\makecell{0.916 \\ (0.010)}} & \makecell{9.141 \\ (0.046)} & \makecell{0.910 \\ (0.004)} & \makecell{9.133 \\ (0.045)} \\
 & Standard & \makecell{0.902 \\ (0.003)} & \textcolor{red}{\makecell{1.968 \\ (0.043)}} & \makecell{0.489 \\ (0.018)} & \makecell{3.712 \\ (0.142)} & \makecell{0.877 \\ (0.008)} & \makecell{3.926 \\ (0.094)} & \textcolor{red}{\makecell{0.950 \\ (0.003)}} & \makecell{1.505 \\ (0.030)} \\
\midrule
300 & ACP-ACC & \makecell{0.912 \\ (0.005)} & \makecell{3.529 \\ (0.109)} & \textcolor{red}{\makecell{0.852 \\ (0.016)}} & \makecell{7.843 \\ (0.183)} & \textcolor{red}{\makecell{0.952 \\ (0.005)}} & \makecell{7.388 \\ (0.209)} & \makecell{0.914 \\ (0.006)} & \makecell{2.522 \\ (0.111)} \\
 & ACP-AEC & \makecell{0.909 \\ (0.004)} & \textcolor{red}{\makecell{2.736 \\ (0.078)}} & \makecell{0.724 \\ (0.023)} & \makecell{6.271 \\ (0.220)} & \textcolor{red}{\makecell{0.936 \\ (0.007)}} & \makecell{6.337 \\ (0.215)} & \textcolor{red}{\makecell{0.925 \\ (0.004)}} & \makecell{1.850 \\ (0.058)} \\
 & ACP-MC & \makecell{0.900 \\ (0.005)} & \makecell{3.746 \\ (0.091)} & \textcolor{red}{\makecell{0.848 \\ (0.013)}} & \makecell{8.152 \\ (0.101)} & \textcolor{red}{\makecell{0.909 \\ (0.009)}} & \makecell{7.502 \\ (0.157)} & \makecell{0.905 \\ (0.005)} & \makecell{2.745 \\ (0.088)} \\
 & AdaTS & \makecell{0.900 \\ (0.004)} & \textcolor{red}{\makecell{2.404 \\ (0.055)}} & \makecell{0.522 \\ (0.018)} & \makecell{4.104 \\ (0.129)} & \makecell{0.907 \\ (0.006)} & \makecell{4.431 \\ (0.095)} & \textcolor{red}{\makecell{0.939 \\ (0.004)}} & \makecell{1.936 \\ (0.051)} \\
 & Retrain & \makecell{0.900 \\ (0.005)} & \makecell{7.754 \\ (0.116)} & \textcolor{red}{\makecell{0.774 \\ (0.016)}} & \makecell{6.816 \\ (0.177)} & \makecell{0.881 \\ (0.008)} & \makecell{7.828 \\ (0.122)} & \makecell{0.917 \\ (0.005)} & \makecell{7.842 \\ (0.117)} \\
 & Standard & \makecell{0.899 \\ (0.002)} & \textcolor{red}{\makecell{1.917 \\ (0.033)}} & \makecell{0.458 \\ (0.017)} & \makecell{3.459 \\ (0.122)} & \makecell{0.880 \\ (0.006)} & \makecell{3.901 \\ (0.078)} & \textcolor{red}{\makecell{0.949 \\ (0.002)}} & \makecell{1.470 \\ (0.023)} \\
\midrule
500 & ACP-ACC & \makecell{0.906 \\ (0.003)} & \makecell{3.055 \\ (0.054)} & \textcolor{red}{\makecell{0.768 \\ (0.019)}} & \makecell{7.097 \\ (0.194)} & \textcolor{red}{\makecell{0.936 \\ (0.006)}} & \makecell{6.650 \\ (0.149)} & \makecell{0.917 \\ (0.003)} & \makecell{2.118 \\ (0.054)} \\
 & ACP-AEC & \makecell{0.902 \\ (0.003)} & \textcolor{red}{\makecell{2.636 \\ (0.055)}} & \makecell{0.674 \\ (0.013)} & \makecell{5.982 \\ (0.159)} & \textcolor{red}{\makecell{0.922 \\ (0.008)}} & \makecell{6.002 \\ (0.175)} & \textcolor{red}{\makecell{0.924 \\ (0.004)}} & \makecell{1.807 \\ (0.043)} \\
 & ACP-MC & \makecell{0.898 \\ (0.003)} & \makecell{3.738 \\ (0.077)} & \textcolor{red}{\makecell{0.859 \\ (0.009)}} & \makecell{8.227 \\ (0.072)} & \textcolor{red}{\makecell{0.911 \\ (0.007)}} & \makecell{7.488 \\ (0.113)} & \makecell{0.900 \\ (0.004)} & \makecell{2.724 \\ (0.082)} \\
 & AdaTS & \makecell{0.899 \\ (0.004)} & \textcolor{red}{\makecell{2.425 \\ (0.054)}} & \makecell{0.509 \\ (0.016)} & \makecell{4.152 \\ (0.155)} & \makecell{0.909 \\ (0.007)} & \makecell{4.403 \\ (0.106)} & \textcolor{red}{\makecell{0.940 \\ (0.004)}} & \makecell{1.968 \\ (0.051)} \\
 & Retrain & \makecell{0.900 \\ (0.003)} & \makecell{7.268 \\ (0.084)} & \textcolor{red}{\makecell{0.802 \\ (0.013)}} & \makecell{6.891 \\ (0.119)} & \makecell{0.884 \\ (0.006)} & \makecell{7.331 \\ (0.088)} & \makecell{0.913 \\ (0.003)} & \makecell{7.300 \\ (0.084)} \\
 & Standard & \makecell{0.900 \\ (0.003)} & \textcolor{red}{\makecell{1.943 \\ (0.026)}} & \makecell{0.443 \\ (0.012)} & \makecell{3.433 \\ (0.103)} & \makecell{0.881 \\ (0.006)} & \makecell{3.904 \\ (0.080)} & \textcolor{red}{\makecell{0.952 \\ (0.002)}} & \makecell{1.512 \\ (0.023)} \\
\midrule
1000 & ACP-ACC & \makecell{0.904 \\ (0.003)} & \makecell{2.807 \\ (0.050)} & \textcolor{red}{\makecell{0.779 \\ (0.018)}} & \makecell{7.152 \\ (0.179)} & \textcolor{red}{\makecell{0.932 \\ (0.005)}} & \makecell{6.351 \\ (0.158)} & \makecell{0.914 \\ (0.003)} & \makecell{1.845 \\ (0.049)} \\
 & ACP-AEC & \makecell{0.900 \\ (0.003)} & \textcolor{red}{\makecell{2.698 \\ (0.054)}} & \makecell{0.694 \\ (0.016)} & \makecell{6.208 \\ (0.119)} & \textcolor{red}{\makecell{0.928 \\ (0.005)}} & \makecell{6.273 \\ (0.197)} & \textcolor{red}{\makecell{0.918 \\ (0.003)}} & \makecell{1.823 \\ (0.039)} \\
 & ACP-MC & \makecell{0.902 \\ (0.003)} & \makecell{3.684 \\ (0.054)} & \textcolor{red}{\makecell{0.864 \\ (0.009)}} & \makecell{8.315 \\ (0.047)} & \textcolor{red}{\makecell{0.940 \\ (0.004)}} & \makecell{7.623 \\ (0.073)} & \makecell{0.901 \\ (0.004)} & \makecell{2.632 \\ (0.054)} \\
 & AdaTS & \makecell{0.899 \\ (0.003)} & \textcolor{red}{\makecell{2.396 \\ (0.036)}} & \makecell{0.512 \\ (0.020)} & \makecell{4.184 \\ (0.125)} & \makecell{0.899 \\ (0.006)} & \makecell{4.363 \\ (0.069)} & \textcolor{red}{\makecell{0.939 \\ (0.002)}} & \makecell{1.927 \\ (0.033)} \\
 & Retrain & \makecell{0.900 \\ (0.003)} & \makecell{6.392 \\ (0.081)} & \textcolor{red}{\makecell{0.795 \\ (0.010)}} & \makecell{6.624 \\ (0.098)} & \makecell{0.855 \\ (0.007)} & \makecell{6.585 \\ (0.076)} & \makecell{0.917 \\ (0.003)} & \makecell{6.338 \\ (0.088)} \\
 & Standard & \makecell{0.900 \\ (0.002)} & \textcolor{red}{\makecell{1.924 \\ (0.025)}} & \makecell{0.451 \\ (0.018)} & \makecell{3.490 \\ (0.116)} & \makecell{0.875 \\ (0.008)} & \makecell{3.859 \\ (0.064)} & \textcolor{red}{\makecell{0.951 \\ (0.002)}} & \makecell{1.484 \\ (0.019)} \\
\midrule
2000 & ACP-ACC & \makecell{0.901 \\ (0.002)} & \textcolor{red}{\makecell{2.701 \\ (0.043)}} & \textcolor{red}{\makecell{0.803 \\ (0.012)}} & \makecell{7.331 \\ (0.146)} & \textcolor{red}{\makecell{0.931 \\ (0.004)}} & \makecell{6.247 \\ (0.131)} & \makecell{0.908 \\ (0.002)} & \makecell{1.701 \\ (0.039)} \\
 & ACP-AEC & \makecell{0.902 \\ (0.002)} & \makecell{2.709 \\ (0.046)} & \makecell{0.704 \\ (0.011)} & \makecell{6.192 \\ (0.100)} & \textcolor{red}{\makecell{0.941 \\ (0.005)}} & \makecell{6.376 \\ (0.149)} & \textcolor{red}{\makecell{0.918 \\ (0.003)}} & \makecell{1.818 \\ (0.039)} \\
 & ACP-MC & \makecell{0.900 \\ (0.002)} & \makecell{3.526 \\ (0.037)} & \textcolor{red}{\makecell{0.888 \\ (0.007)}} & \makecell{8.280 \\ (0.036)} & \textcolor{red}{\makecell{0.944 \\ (0.004)}} & \makecell{7.496 \\ (0.063)} & \makecell{0.895 \\ (0.003)} & \makecell{2.453 \\ (0.037)} \\
 & AdaTS & \makecell{0.899 \\ (0.002)} & \textcolor{red}{\makecell{2.365 \\ (0.028)}} & \makecell{0.526 \\ (0.017)} & \makecell{4.182 \\ (0.104)} & \makecell{0.909 \\ (0.006)} & \makecell{4.346 \\ (0.048)} & \textcolor{red}{\makecell{0.939 \\ (0.002)}} & \makecell{1.891 \\ (0.032)} \\
 & Retrain & \makecell{0.900 \\ (0.002)} & \makecell{5.591 \\ (0.044)} & \textcolor{red}{\makecell{0.810 \\ (0.010)}} & \makecell{6.701 \\ (0.092)} & \makecell{0.853 \\ (0.007)} & \makecell{6.053 \\ (0.052)} & \makecell{0.917 \\ (0.003)} & \makecell{5.406 \\ (0.051)} \\
 & Standard & \makecell{0.900 \\ (0.002)} & \textcolor{red}{\makecell{1.929 \\ (0.021)}} & \makecell{0.458 \\ (0.015)} & \makecell{3.549 \\ (0.098)} & \makecell{0.886 \\ (0.007)} & \makecell{3.893 \\ (0.065)} & \textcolor{red}{\makecell{0.950 \\ (0.002)}} & \makecell{1.478 \\ (0.018)} \\
\bottomrule
\end{tabular}

\end{table}

\begin{figure*}[tbh]
    \centering
    \includegraphics[width=\linewidth]{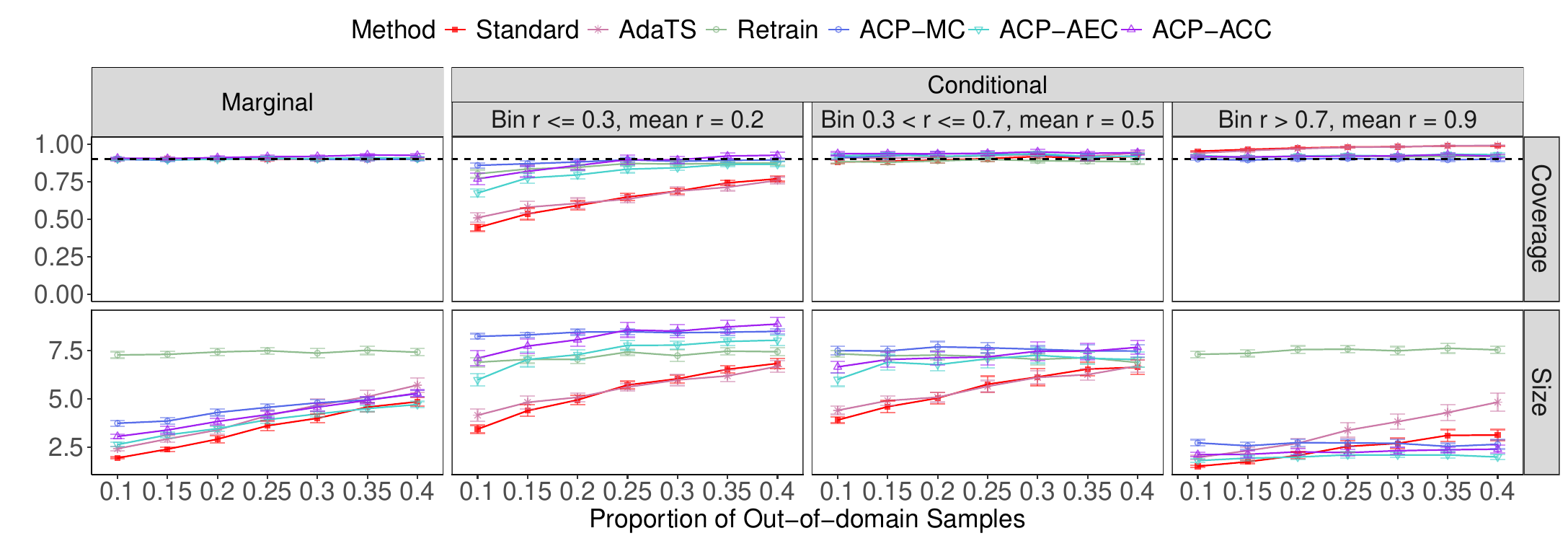}
    \caption{Performance of conformal prediction sets for CIFAR-10 and CIFAR-10-C datasets as a function of the total calibration sample size. Left Panel: marginal coverage and average prediction set size. All methods achieve the marginal coverage of $90\%$, while our methods (ACP-MC, ACP-AEC, ACP-ACC) retain practically efficient sets sizes. Right Panel: coverage conditional on estimated reliability bin. Our methods achieve higher conditional coverage.
    Error bars denote two standard errors. See Table~\ref{tab:supp_exp_real_beta_cifar10} for further details.}
    \label{fig:supp_exp_real_beta_overall_cifar10}
\end{figure*}

\begin{table}[!htb]
\centering
    \caption{Performance of conformal prediction sets on the CIFAR-10 and CIFAR-10-C datasets as a function of the proportion of out-of-domain samples. All methods achieve the target marginal coverage of \(90\%\), while our methods (ACP-MC, ACP-AEC, and ACP-ACC) retain practically efficient set sizes and are more uncertainty-aware, assigning larger sets to more unreliable samples and smaller sets to more reliable ones. Red numbers indicate the three smallest average set sizes and the three highest conditional coverages. See the corresponding plots in Figure~\ref{fig:supp_exp_real_beta_overall_cifar10}.}
  \label{tab:supp_exp_real_beta_cifar10}
\centering
\fontsize{6}{6}\selectfont
\begin{tabular}[t]{rlllllllll}
\toprule
\multicolumn{2}{c}{ } & \multicolumn{2}{c}{\makecell{Marginal}} & \multicolumn{2}{c}{\makecell{$r^* \leq 0.33$ \\ mean $r^* = 0.18$}} & \multicolumn{2}{c}{\makecell{$0.33 < r^* \leq 0.67$ \\ mean $r^* = 0.53$}} & \multicolumn{2}{c}{\makecell{$r^* > 0.67$ \\ mean $r^* = 0.92$}} \\
\cmidrule(l{3pt}r{3pt}){3-4} \cmidrule(l{3pt}r{3pt}){5-6} \cmidrule(l{3pt}r{3pt}){7-8} \cmidrule(l{3pt}r{3pt}){9-10}
\makecell{Proportion \\ of \\ Hard \\ samples} & Method & Coverage & Size & Coverage & Size & Coverage & Size & Coverage & Size \\
\midrule
0.1 & ACP-ACC & \makecell{0.906 \\ (0.003)} & \makecell{3.055 \\ (0.054)} & \textcolor{red}{\makecell{0.768 \\ (0.019)}} & \makecell{7.097 \\ (0.194)} & \textcolor{red}{\makecell{0.936 \\ (0.006)}} & \makecell{6.650 \\ (0.149)} & \makecell{0.917 \\ (0.003)} & \makecell{2.118 \\ (0.054)} \\
 & ACP-AEC & \makecell{0.902 \\ (0.003)} & \textcolor{red}{\makecell{2.636 \\ (0.055)}} & \makecell{0.674 \\ (0.013)} & \makecell{5.982 \\ (0.159)} & \textcolor{red}{\makecell{0.922 \\ (0.008)}} & \makecell{6.002 \\ (0.175)} & \textcolor{red}{\makecell{0.924 \\ (0.004)}} & \makecell{1.807 \\ (0.043)} \\
 & ACP-MC & \makecell{0.898 \\ (0.003)} & \makecell{3.738 \\ (0.077)} & \textcolor{red}{\makecell{0.859 \\ (0.009)}} & \makecell{8.227 \\ (0.072)} & \textcolor{red}{\makecell{0.911 \\ (0.007)}} & \makecell{7.488 \\ (0.113)} & \makecell{0.900 \\ (0.004)} & \makecell{2.724 \\ (0.082)} \\
 & AdaTS & \makecell{0.899 \\ (0.004)} & \textcolor{red}{\makecell{2.425 \\ (0.054)}} & \makecell{0.509 \\ (0.016)} & \makecell{4.152 \\ (0.155)} & \makecell{0.909 \\ (0.007)} & \makecell{4.403 \\ (0.106)} & \textcolor{red}{\makecell{0.940 \\ (0.004)}} & \makecell{1.968 \\ (0.051)} \\
 & Retrain & \makecell{0.900 \\ (0.003)} & \makecell{7.268 \\ (0.084)} & \textcolor{red}{\makecell{0.802 \\ (0.013)}} & \makecell{6.891 \\ (0.119)} & \makecell{0.884 \\ (0.006)} & \makecell{7.331 \\ (0.088)} & \makecell{0.913 \\ (0.003)} & \makecell{7.300 \\ (0.084)} \\
 & Standard & \makecell{0.900 \\ (0.003)} & \textcolor{red}{\makecell{1.943 \\ (0.026)}} & \makecell{0.443 \\ (0.012)} & \makecell{3.433 \\ (0.103)} & \makecell{0.881 \\ (0.006)} & \makecell{3.904 \\ (0.080)} & \textcolor{red}{\makecell{0.952 \\ (0.002)}} & \makecell{1.512 \\ (0.023)} \\
\midrule
0.15 & ACP-ACC & \makecell{0.904 \\ (0.004)} & \makecell{3.388 \\ (0.100)} & \textcolor{red}{\makecell{0.819 \\ (0.019)}} & \makecell{7.725 \\ (0.188)} & \textcolor{red}{\makecell{0.937 \\ (0.007)}} & \makecell{7.031 \\ (0.200)} & \textcolor{red}{\makecell{0.914 \\ (0.004)}} & \makecell{2.131 \\ (0.096)} \\
 & ACP-AEC & \makecell{0.896 \\ (0.004)} & \textcolor{red}{\makecell{3.136 \\ (0.078)}} & \makecell{0.774 \\ (0.017)} & \makecell{7.033 \\ (0.190)} & \textcolor{red}{\makecell{0.929 \\ (0.008)}} & \makecell{6.893 \\ (0.198)} & \makecell{0.912 \\ (0.005)} & \makecell{1.936 \\ (0.067)} \\
 & ACP-MC & \makecell{0.894 \\ (0.005)} & \makecell{3.853 \\ (0.087)} & \textcolor{red}{\makecell{0.869 \\ (0.009)}} & \makecell{8.301 \\ (0.066)} & \textcolor{red}{\makecell{0.919 \\ (0.010)}} & \makecell{7.473 \\ (0.125)} & \makecell{0.894 \\ (0.005)} & \makecell{2.573 \\ (0.091)} \\
 & AdaTS & \makecell{0.903 \\ (0.004)} & \textcolor{red}{\makecell{2.925 \\ (0.084)}} & \makecell{0.579 \\ (0.019)} & \makecell{4.819 \\ (0.169)} & \makecell{0.906 \\ (0.009)} & \makecell{4.905 \\ (0.127)} & \textcolor{red}{\makecell{0.957 \\ (0.004)}} & \makecell{2.323 \\ (0.084)} \\
 & Retrain & \makecell{0.897 \\ (0.005)} & \makecell{7.298 \\ (0.085)} & \textcolor{red}{\makecell{0.834 \\ (0.012)}} & \makecell{7.040 \\ (0.097)} & \makecell{0.879 \\ (0.008)} & \makecell{7.221 \\ (0.107)} & \makecell{0.910 \\ (0.005)} & \makecell{7.350 \\ (0.088)} \\
 & Standard & \makecell{0.901 \\ (0.004)} & \textcolor{red}{\makecell{2.399 \\ (0.056)}} & \makecell{0.535 \\ (0.019)} & \makecell{4.387 \\ (0.149)} & \makecell{0.887 \\ (0.010)} & \makecell{4.593 \\ (0.156)} & \textcolor{red}{\makecell{0.964 \\ (0.003)}} & \makecell{1.749 \\ (0.044)} \\
\midrule
0.2 & ACP-ACC & \makecell{0.910 \\ (0.005)} & \makecell{3.828 \\ (0.082)} & \textcolor{red}{\makecell{0.858 \\ (0.015)}} & \makecell{8.054 \\ (0.173)} & \textcolor{red}{\makecell{0.935 \\ (0.007)}} & \makecell{7.112 \\ (0.177)} & \makecell{0.918 \\ (0.005)} & \makecell{2.260 \\ (0.083)} \\
 & ACP-AEC & \makecell{0.899 \\ (0.005)} & \textcolor{red}{\makecell{3.459 \\ (0.064)}} & \makecell{0.795 \\ (0.013)} & \makecell{7.281 \\ (0.126)} & \textcolor{red}{\makecell{0.918 \\ (0.007)}} & \makecell{6.767 \\ (0.153)} & \textcolor{red}{\makecell{0.921 \\ (0.005)}} & \makecell{1.990 \\ (0.063)} \\
 & ACP-MC & \makecell{0.902 \\ (0.005)} & \makecell{4.296 \\ (0.087)} & \textcolor{red}{\makecell{0.881 \\ (0.010)}} & \makecell{8.446 \\ (0.073)} & \textcolor{red}{\makecell{0.937 \\ (0.007)}} & \makecell{7.685 \\ (0.139)} & \makecell{0.901 \\ (0.006)} & \makecell{2.735 \\ (0.094)} \\
 & AdaTS & \makecell{0.901 \\ (0.004)} & \textcolor{red}{\makecell{3.390 \\ (0.112)}} & \makecell{0.605 \\ (0.018)} & \makecell{5.077 \\ (0.114)} & \makecell{0.912 \\ (0.007)} & \makecell{5.081 \\ (0.125)} & \textcolor{red}{\makecell{0.969 \\ (0.003)}} & \makecell{2.702 \\ (0.127)} \\
 & Retrain & \makecell{0.904 \\ (0.004)} & \makecell{7.422 \\ (0.093)} & \textcolor{red}{\makecell{0.845 \\ (0.008)}} & \makecell{7.036 \\ (0.105)} & \makecell{0.888 \\ (0.007)} & \makecell{7.262 \\ (0.122)} & \makecell{0.920 \\ (0.004)} & \makecell{7.543 \\ (0.096)} \\
 & Standard & \makecell{0.899 \\ (0.003)} & \textcolor{red}{\makecell{2.916 \\ (0.090)}} & \makecell{0.591 \\ (0.015)} & \makecell{4.943 \\ (0.123)} & \makecell{0.893 \\ (0.009)} & \makecell{5.037 \\ (0.152)} & \textcolor{red}{\makecell{0.974 \\ (0.002)}} & \makecell{2.074 \\ (0.093)} \\
\midrule
0.25 & ACP-ACC & \makecell{0.917 \\ (0.006)} & \makecell{4.184 \\ (0.088)} & \textcolor{red}{\makecell{0.897 \\ (0.016)}} & \makecell{8.559 \\ (0.189)} & \textcolor{red}{\makecell{0.939 \\ (0.008)}} & \makecell{7.166 \\ (0.206)} & \makecell{0.919 \\ (0.005)} & \makecell{2.221 \\ (0.089)} \\
 & ACP-AEC & \makecell{0.904 \\ (0.006)} & \textcolor{red}{\makecell{3.923 \\ (0.093)}} & \makecell{0.833 \\ (0.012)} & \makecell{7.765 \\ (0.126)} & \textcolor{red}{\makecell{0.925 \\ (0.010)}} & \makecell{7.069 \\ (0.235)} & \makecell{0.924 \\ (0.006)} & \makecell{2.103 \\ (0.077)} \\
 & ACP-MC & \makecell{0.906 \\ (0.005)} & \makecell{4.558 \\ (0.085)} & \textcolor{red}{\makecell{0.895 \\ (0.008)}} & \makecell{8.468 \\ (0.077)} & \textcolor{red}{\makecell{0.935 \\ (0.007)}} & \makecell{7.638 \\ (0.139)} & \makecell{0.904 \\ (0.006)} & \makecell{2.722 \\ (0.091)} \\
 & AdaTS & \makecell{0.898 \\ (0.004)} & \textcolor{red}{\makecell{4.132 \\ (0.150)}} & \makecell{0.633 \\ (0.013)} & \makecell{5.606 \\ (0.090)} & \makecell{0.923 \\ (0.009)} & \makecell{5.648 \\ (0.149)} & \textcolor{red}{\makecell{0.979 \\ (0.003)}} & \makecell{3.378 \\ (0.189)} \\
 & Retrain & \makecell{0.910 \\ (0.004)} & \makecell{7.485 \\ (0.083)} & \textcolor{red}{\makecell{0.871 \\ (0.007)}} & \makecell{7.419 \\ (0.082)} & \makecell{0.898 \\ (0.007)} & \makecell{7.185 \\ (0.094)} & \textcolor{red}{\makecell{0.924 \\ (0.005)}} & \makecell{7.566 \\ (0.090)} \\
 & Standard & \makecell{0.899 \\ (0.004)} & \textcolor{red}{\makecell{3.611 \\ (0.125)}} & \makecell{0.648 \\ (0.012)} & \makecell{5.715 \\ (0.099)} & \makecell{0.904 \\ (0.009)} & \makecell{5.754 \\ (0.207)} & \textcolor{red}{\makecell{0.980 \\ (0.003)}} & \makecell{2.539 \\ (0.144)} \\
\midrule
0.3 & ACP-ACC & \makecell{0.918 \\ (0.005)} & \textcolor{red}{\makecell{4.570 \\ (0.102)}} & \textcolor{red}{\makecell{0.895 \\ (0.012)}} & \makecell{8.497 \\ (0.165)} & \textcolor{red}{\makecell{0.949 \\ (0.008)}} & \makecell{7.463 \\ (0.243)} & \makecell{0.921 \\ (0.006)} & \makecell{2.318 \\ (0.089)} \\
 & ACP-AEC & \makecell{0.903 \\ (0.005)} & \textcolor{red}{\makecell{4.220 \\ (0.077)}} & \makecell{0.842 \\ (0.008)} & \makecell{7.772 \\ (0.102)} & \textcolor{red}{\makecell{0.937 \\ (0.007)}} & \makecell{7.249 \\ (0.202)} & \textcolor{red}{\makecell{0.922 \\ (0.006)}} & \makecell{2.099 \\ (0.065)} \\
 & ACP-MC & \makecell{0.899 \\ (0.005)} & \makecell{4.794 \\ (0.085)} & \textcolor{red}{\makecell{0.885 \\ (0.006)}} & \makecell{8.415 \\ (0.075)} & \textcolor{red}{\makecell{0.931 \\ (0.006)}} & \makecell{7.562 \\ (0.085)} & \makecell{0.899 \\ (0.005)} & \makecell{2.699 \\ (0.105)} \\
 & AdaTS & \makecell{0.900 \\ (0.005)} & \makecell{4.668 \\ (0.160)} & \makecell{0.686 \\ (0.015)} & \makecell{5.978 \\ (0.144)} & \makecell{0.930 \\ (0.007)} & \makecell{6.104 \\ (0.167)} & \textcolor{red}{\makecell{0.982 \\ (0.002)}} & \makecell{3.819 \\ (0.196)} \\
 & Retrain & \makecell{0.899 \\ (0.005)} & \makecell{7.358 \\ (0.117)} & \textcolor{red}{\makecell{0.867 \\ (0.009)}} & \makecell{7.232 \\ (0.143)} & \makecell{0.889 \\ (0.006)} & \makecell{7.041 \\ (0.143)} & \makecell{0.915 \\ (0.006)} & \makecell{7.481 \\ (0.111)} \\
 & Standard & \makecell{0.899 \\ (0.003)} & \textcolor{red}{\makecell{3.997 \\ (0.108)}} & \makecell{0.687 \\ (0.010)} & \makecell{6.034 \\ (0.103)} & \makecell{0.917 \\ (0.007)} & \makecell{6.128 \\ (0.221)} & \textcolor{red}{\makecell{0.982 \\ (0.002)}} & \makecell{2.698 \\ (0.132)} \\
\midrule
0.35 & ACP-ACC & \makecell{0.928 \\ (0.005)} & \textcolor{red}{\makecell{4.925 \\ (0.087)}} & \textcolor{red}{\makecell{0.921 \\ (0.010)}} & \makecell{8.718 \\ (0.167)} & \textcolor{red}{\makecell{0.938 \\ (0.007)}} & \makecell{7.453 \\ (0.214)} & \makecell{0.930 \\ (0.006)} & \makecell{2.364 \\ (0.072)} \\
 & ACP-AEC & \makecell{0.910 \\ (0.004)} & \textcolor{red}{\makecell{4.498 \\ (0.077)}} & \makecell{0.865 \\ (0.006)} & \makecell{7.965 \\ (0.107)} & \textcolor{red}{\makecell{0.918 \\ (0.007)}} & \makecell{7.093 \\ (0.161)} & \textcolor{red}{\makecell{0.932 \\ (0.005)}} & \makecell{2.102 \\ (0.076)} \\
 & ACP-MC & \makecell{0.897 \\ (0.005)} & \makecell{4.946 \\ (0.072)} & \textcolor{red}{\makecell{0.892 \\ (0.006)}} & \makecell{8.446 \\ (0.091)} & \makecell{0.912 \\ (0.011)} & \makecell{7.476 \\ (0.165)} & \makecell{0.895 \\ (0.005)} & \makecell{2.541 \\ (0.081)} \\
 & AdaTS & \makecell{0.897 \\ (0.004)} & \makecell{5.108 \\ (0.165)} & \makecell{0.713 \\ (0.012)} & \makecell{6.184 \\ (0.147)} & \textcolor{red}{\makecell{0.918 \\ (0.006)}} & \makecell{6.252 \\ (0.150)} & \textcolor{red}{\makecell{0.988 \\ (0.002)}} & \makecell{4.290 \\ (0.200)} \\
 & Retrain & \makecell{0.900 \\ (0.004)} & \makecell{7.509 \\ (0.110)} & \textcolor{red}{\makecell{0.872 \\ (0.006)}} & \makecell{7.463 \\ (0.093)} & \makecell{0.886 \\ (0.008)} & \makecell{7.125 \\ (0.144)} & \makecell{0.918 \\ (0.005)} & \makecell{7.613 \\ (0.125)} \\
 & Standard & \makecell{0.904 \\ (0.003)} & \textcolor{red}{\makecell{4.572 \\ (0.120)}} & \makecell{0.742 \\ (0.008)} & \makecell{6.525 \\ (0.095)} & \makecell{0.907 \\ (0.008)} & \makecell{6.537 \\ (0.175)} & \textcolor{red}{\makecell{0.988 \\ (0.002)}} & \makecell{3.111 \\ (0.151)} \\
\midrule
0.4 & ACP-ACC & \makecell{0.925 \\ (0.006)} & \makecell{5.292 \\ (0.087)} & \textcolor{red}{\makecell{0.926 \\ (0.011)}} & \makecell{8.863 \\ (0.176)} & \textcolor{red}{\makecell{0.944 \\ (0.008)}} & \makecell{7.661 \\ (0.183)} & \makecell{0.919 \\ (0.006)} & \makecell{2.399 \\ (0.106)} \\
 & ACP-AEC & \makecell{0.906 \\ (0.005)} & \textcolor{red}{\makecell{4.704 \\ (0.080)}} & \makecell{0.865 \\ (0.008)} & \makecell{8.027 \\ (0.131)} & \makecell{0.917 \\ (0.010)} & \makecell{7.029 \\ (0.193)} & \textcolor{red}{\makecell{0.929 \\ (0.007)}} & \makecell{1.992 \\ (0.069)} \\
 & ACP-MC & \makecell{0.901 \\ (0.005)} & \textcolor{red}{\makecell{5.266 \\ (0.082)}} & \textcolor{red}{\makecell{0.895 \\ (0.006)}} & \makecell{8.486 \\ (0.066)} & \textcolor{red}{\makecell{0.922 \\ (0.008)}} & \makecell{7.470 \\ (0.157)} & \makecell{0.901 \\ (0.007)} & \makecell{2.651 \\ (0.121)} \\
 & AdaTS & \makecell{0.906 \\ (0.004)} & \makecell{5.707 \\ (0.184)} & \makecell{0.758 \\ (0.011)} & \makecell{6.675 \\ (0.148)} & \textcolor{red}{\makecell{0.941 \\ (0.005)}} & \makecell{6.698 \\ (0.156)} & \textcolor{red}{\makecell{0.992 \\ (0.002)}} & \makecell{4.826 \\ (0.238)} \\
 & Retrain & \makecell{0.899 \\ (0.003)} & \makecell{7.411 \\ (0.093)} & \textcolor{red}{\makecell{0.874 \\ (0.007)}} & \makecell{7.434 \\ (0.104)} & \makecell{0.883 \\ (0.008)} & \makecell{6.875 \\ (0.122)} & \makecell{0.919 \\ (0.004)} & \makecell{7.531 \\ (0.094)} \\
 & Standard & \makecell{0.904 \\ (0.003)} & \textcolor{red}{\makecell{4.849 \\ (0.116)}} & \makecell{0.769 \\ (0.009)} & \makecell{6.829 \\ (0.134)} & \makecell{0.918 \\ (0.008)} & \makecell{6.628 \\ (0.184)} & \textcolor{red}{\makecell{0.988 \\ (0.002)}} & \makecell{3.129 \\ (0.140)} \\
\bottomrule
\end{tabular}

\end{table}

\FloatBarrier

\section*{NeurIPS Paper Checklist}

\begin{enumerate}

\item {\bf Claims}
    \item[] Question: Do the main claims made in the abstract and introduction accurately reflect the paper's contributions and scope?
    \item[] Answer: \answerYes{} 
    \item[] Justification: The main claims made in the abstract and introduction accurately reflect the paper’s contributions and scope. 
    \item[] Guidelines:
    \begin{itemize}
        \item The answer \answerNA{} means that the abstract and introduction do not include the claims made in the paper.
        \item The abstract and/or introduction should clearly state the claims made, including the contributions made in the paper and important assumptions and limitations. A \answerNo{} or \answerNA{} answer to this question will not be perceived well by the reviewers. 
        \item The claims made should match theoretical and experimental results, and reflect how much the results can be expected to generalize to other settings. 
        \item It is fine to include aspirational goals as motivation as long as it is clear that these goals are not attained by the paper. 
    \end{itemize}

\item {\bf Limitations}
    \item[] Question: Does the paper discuss the limitations of the work performed by the authors?
    \item[] Answer: \answerYes{} 
    \item[] Justification: This paper discusses the limitations of the work in the Discussion section.
    \item[] Guidelines:
    \begin{itemize}
        \item The answer \answerNA{} means that the paper has no limitation while the answer \answerNo{} means that the paper has limitations, but those are not discussed in the paper. 
        \item The authors are encouraged to create a separate ``Limitations'' section in their paper.
        \item The paper should point out any strong assumptions and how robust the results are to violations of these assumptions (e.g., independence assumptions, noiseless settings, model well-specification, asymptotic approximations only holding locally). The authors should reflect on how these assumptions might be violated in practice and what the implications would be.
        \item The authors should reflect on the scope of the claims made, e.g., if the approach was only tested on a few datasets or with a few runs. In general, empirical results often depend on implicit assumptions, which should be articulated.
        \item The authors should reflect on the factors that influence the performance of the approach. For example, a facial recognition algorithm may perform poorly when image resolution is low or images are taken in low lighting. Or a speech-to-text system might not be used reliably to provide closed captions for online lectures because it fails to handle technical jargon.
        \item The authors should discuss the computational efficiency of the proposed algorithms and how they scale with dataset size.
        \item If applicable, the authors should discuss possible limitations of their approach to address problems of privacy and fairness.
        \item While the authors might fear that complete honesty about limitations might be used by reviewers as grounds for rejection, a worse outcome might be that reviewers discover limitations that aren't acknowledged in the paper. The authors should use their best judgment and recognize that individual actions in favor of transparency play an important role in developing norms that preserve the integrity of the community. Reviewers will be specifically instructed to not penalize honesty concerning limitations.
    \end{itemize}

\item {\bf Theory assumptions and proofs}
    \item[] Question: For each theoretical result, does the paper provide the full set of assumptions and a complete (and correct) proof?
    \item[] Answer: \answerYes{} 
    \item[] Justification: All assumptions are stated in the paper and all proofs are provided in the Appendix. 
    \item[] Guidelines:
    \begin{itemize}
        \item The answer \answerNA{} means that the paper does not include theoretical results. 
        \item All the theorems, formulas, and proofs in the paper should be numbered and cross-referenced.
        \item All assumptions should be clearly stated or referenced in the statement of any theorems.
        \item The proofs can either appear in the main paper or the supplemental material, but if they appear in the supplemental material, the authors are encouraged to provide a short proof sketch to provide intuition. 
        \item Inversely, any informal proof provided in the core of the paper should be complemented by formal proofs provided in appendix or supplemental material.
        \item Theorems and Lemmas that the proof relies upon should be properly referenced. 
    \end{itemize}

    \item {\bf Experimental result reproducibility}
    \item[] Question: Does the paper fully disclose all the information needed to reproduce the main experimental results of the paper to the extent that it affects the main claims and/or conclusions of the paper (regardless of whether the code and data are provided or not)?
    \item[] Answer: \answerYes{} 
    \item[] Justification: The paper fully discloses all information needed to reproduce the main experimental results, and the implementation code is also provided.
    \item[] Guidelines:
    \begin{itemize}
        \item The answer \answerNA{} means that the paper does not include experiments.
        \item If the paper includes experiments, a \answerNo{} answer to this question will not be perceived well by the reviewers: Making the paper reproducible is important, regardless of whether the code and data are provided or not.
        \item If the contribution is a dataset and\slash or model, the authors should describe the steps taken to make their results reproducible or verifiable. 
        \item Depending on the contribution, reproducibility can be accomplished in various ways. For example, if the contribution is a novel architecture, describing the architecture fully might suffice, or if the contribution is a specific model and empirical evaluation, it may be necessary to either make it possible for others to replicate the model with the same dataset, or provide access to the model. In general. releasing code and data is often one good way to accomplish this, but reproducibility can also be provided via detailed instructions for how to replicate the results, access to a hosted model (e.g., in the case of a large language model), releasing of a model checkpoint, or other means that are appropriate to the research performed.
        \item While NeurIPS does not require releasing code, the conference does require all submissions to provide some reasonable avenue for reproducibility, which may depend on the nature of the contribution. For example
        \begin{enumerate}
            \item If the contribution is primarily a new algorithm, the paper should make it clear how to reproduce that algorithm.
            \item If the contribution is primarily a new model architecture, the paper should describe the architecture clearly and fully.
            \item If the contribution is a new model (e.g., a large language model), then there should either be a way to access this model for reproducing the results or a way to reproduce the model (e.g., with an open-source dataset or instructions for how to construct the dataset).
            \item We recognize that reproducibility may be tricky in some cases, in which case authors are welcome to describe the particular way they provide for reproducibility. In the case of closed-source models, it may be that access to the model is limited in some way (e.g., to registered users), but it should be possible for other researchers to have some path to reproducing or verifying the results.
        \end{enumerate}
    \end{itemize}

\item {\bf Open access to data and code}
    \item[] Question: Does the paper provide open access to the data and code, with sufficient instructions to faithfully reproduce the main experimental results, as described in supplemental material?
    \item[] Answer: \answerYes{} 
    \item[] Justification: This paper provides open access to the code required to reproduce the main experimental results, along with detailed information on how to preprocess the public benchmark data. 
    \item[] Guidelines:
    \begin{itemize}
        \item The answer \answerNA{} means that paper does not include experiments requiring code.
        \item Please see the NeurIPS code and data submission guidelines (\url{https://neurips.cc/public/guides/CodeSubmissionPolicy}) for more details.
        \item While we encourage the release of code and data, we understand that this might not be possible, so \answerNo{} is an acceptable answer. Papers cannot be rejected simply for not including code, unless this is central to the contribution (e.g., for a new open-source benchmark).
        \item The instructions should contain the exact command and environment needed to run to reproduce the results. See the NeurIPS code and data submission guidelines (\url{https://neurips.cc/public/guides/CodeSubmissionPolicy}) for more details.
        \item The authors should provide instructions on data access and preparation, including how to access the raw data, preprocessed data, intermediate data, and generated data, etc.
        \item The authors should provide scripts to reproduce all experimental results for the new proposed method and baselines. If only a subset of experiments are reproducible, they should state which ones are omitted from the script and why.
        \item At submission time, to preserve anonymity, the authors should release anonymized versions (if applicable).
        \item Providing as much information as possible in supplemental material (appended to the paper) is recommended, but including URLs to data and code is permitted.
    \end{itemize}

\item {\bf Experimental setting/details}
    \item[] Question: Does the paper specify all the training and test details (e.g., data splits, hyperparameters, how they were chosen, type of optimizer) necessary to understand the results?
    \item[] Answer: \answerYes{}
    \item[] Justification: The paper specifies all the training and test details necessary to understand the
results. 
    \item[] Guidelines:
    \begin{itemize}
        \item The answer \answerNA{} means that the paper does not include experiments.
        \item The experimental setting should be presented in the core of the paper to a level of detail that is necessary to appreciate the results and make sense of them.
        \item The full details can be provided either with the code, in appendix, or as supplemental material.
    \end{itemize}

\item {\bf Experiment statistical significance}
    \item[] Question: Does the paper report error bars suitably and correctly defined or other appropriate information about the statistical significance of the experiments?
    \item[] Answer: \answerYes{}
    \item[] Justification: The paper reports error bars or standard deviations where appropriate. 
    \item[] Guidelines:
    \begin{itemize}
        \item The answer \answerNA{} means that the paper does not include experiments.
        \item The authors should answer \answerYes{} if the results are accompanied by error bars, confidence intervals, or statistical significance tests, at least for the experiments that support the main claims of the paper.
        \item The factors of variability that the error bars are capturing should be clearly stated (for example, train/test split, initialization, random drawing of some parameter, or overall run with given experimental conditions).
        \item The method for calculating the error bars should be explained (closed form formula, call to a library function, bootstrap, etc.)
        \item The assumptions made should be given (e.g., Normally distributed errors).
        \item It should be clear whether the error bar is the standard deviation or the standard error of the mean.
        \item It is OK to report 1-sigma error bars, but one should state it. The authors should preferably report a 2-sigma error bar than state that they have a 96\% CI, if the hypothesis of Normality of errors is not verified.
        \item For asymmetric distributions, the authors should be careful not to show in tables or figures symmetric error bars that would yield results that are out of range (e.g., negative error rates).
        \item If error bars are reported in tables or plots, the authors should explain in the text how they were calculated and reference the corresponding figures or tables in the text.
    \end{itemize}

\item {\bf Experiments compute resources}
    \item[] Question: For each experiment, does the paper provide sufficient information on the computer resources (type of compute workers, memory, time of execution) needed to reproduce the experiments?
    \item[] Answer: \answerYes{} 
    \item[] Justification: The paper provides information on the computer resources in the Appendix. 
    \item[] Guidelines:
    \begin{itemize}
        \item The answer \answerNA{} means that the paper does not include experiments.
        \item The paper should indicate the type of compute workers CPU or GPU, internal cluster, or cloud provider, including relevant memory and storage.
        \item The paper should provide the amount of compute required for each of the individual experimental runs as well as estimate the total compute. 
        \item The paper should disclose whether the full research project required more compute than the experiments reported in the paper (e.g., preliminary or failed experiments that didn't make it into the paper). 
    \end{itemize}
    
\item {\bf Code of ethics}
    \item[] Question: Does the research conducted in the paper conform, in every respect, with the NeurIPS Code of Ethics \url{https://neurips.cc/public/EthicsGuidelines}?
    \item[] Answer: \answerYes{} 
    \item[] Justification:  The research conducted in the paper conforms, in every respect, with the NeurIPS Code of Ethics. 
    \item[] Guidelines:
    \begin{itemize}
        \item The answer \answerNA{} means that the authors have not reviewed the NeurIPS Code of Ethics.
        \item If the authors answer \answerNo, they should explain the special circumstances that require a deviation from the Code of Ethics.
        \item The authors should make sure to preserve anonymity (e.g., if there is a special consideration due to laws or regulations in their jurisdiction).
    \end{itemize}

\item {\bf Broader impacts}
    \item[] Question: Does the paper discuss both potential positive societal impacts and negative societal impacts of the work performed?
    \item[] Answer: \answerYes{} 
    \item[] Justification: This work is motivated by the need to maintain reliable uncertainty estimates for machine learning models deployed under distribution shift, a challenge with broad potential impact across real-world applications. However, the paper focuses on developing a general methodological framework rather than targeting any specific application domain.
    \item[] Guidelines:
    \begin{itemize}
        \item The answer \answerNA{} means that there is no societal impact of the work performed.
        \item If the authors answer \answerNA{} or \answerNo, they should explain why their work has no societal impact or why the paper does not address societal impact.
        \item Examples of negative societal impacts include potential malicious or unintended uses (e.g., disinformation, generating fake profiles, surveillance), fairness considerations (e.g., deployment of technologies that could make decisions that unfairly impact specific groups), privacy considerations, and security considerations.
        \item The conference expects that many papers will be foundational research and not tied to particular applications, let alone deployments. However, if there is a direct path to any negative applications, the authors should point it out. For example, it is legitimate to point out that an improvement in the quality of generative models could be used to generate Deepfakes for disinformation. On the other hand, it is not needed to point out that a generic algorithm for optimizing neural networks could enable people to train models that generate Deepfakes faster.
        \item The authors should consider possible harms that could arise when the technology is being used as intended and functioning correctly, harms that could arise when the technology is being used as intended but gives incorrect results, and harms following from (intentional or unintentional) misuse of the technology.
        \item If there are negative societal impacts, the authors could also discuss possible mitigation strategies (e.g., gated release of models, providing defenses in addition to attacks, mechanisms for monitoring misuse, mechanisms to monitor how a system learns from feedback over time, improving the efficiency and accessibility of ML).
    \end{itemize}
    
\item {\bf Safeguards}
    \item[] Question: Does the paper describe safeguards that have been put in place for responsible release of data or models that have a high risk for misuse (e.g., pre-trained language models, image generators, or scraped datasets)?
    \item[] Answer: \answerNA{} 
    \item[] Justification: This paper does not release or introduce data or models with a high risk of misuse. 
    \item[] Guidelines:
    \begin{itemize}
        \item The answer \answerNA{} means that the paper poses no such risks.
        \item Released models that have a high risk for misuse or dual-use should be released with necessary safeguards to allow for controlled use of the model, for example by requiring that users adhere to usage guidelines or restrictions to access the model or implementing safety filters. 
        \item Datasets that have been scraped from the Internet could pose safety risks. The authors should describe how they avoided releasing unsafe images.
        \item We recognize that providing effective safeguards is challenging, and many papers do not require this, but we encourage authors to take this into account and make a best faith effort.
    \end{itemize}

\item {\bf Licenses for existing assets}
    \item[] Question: Are the creators or original owners of assets (e.g., code, data, models), used in the paper, properly credited and are the license and terms of use explicitly mentioned and properly respected?
    \item[] Answer:\answerYes{} 
    \item[] Justification: This paper uses open-domain data, properly crediting the license and creators.
    \item[] Guidelines:
    \begin{itemize}
        \item The answer \answerNA{} means that the paper does not use existing assets.
        \item The authors should cite the original paper that produced the code package or dataset.
        \item The authors should state which version of the asset is used and, if possible, include a URL.
        \item The name of the license (e.g., CC-BY 4.0) should be included for each asset.
        \item For scraped data from a particular source (e.g., website), the copyright and terms of service of that source should be provided.
        \item If assets are released, the license, copyright information, and terms of use in the package should be provided. For popular datasets, \url{paperswithcode.com/datasets} has curated licenses for some datasets. Their licensing guide can help determine the license of a dataset.
        \item For existing datasets that are re-packaged, both the original license and the license of the derived asset (if it has changed) should be provided.
        \item If this information is not available online, the authors are encouraged to reach out to the asset's creators.
    \end{itemize}

\item {\bf New assets}
    \item[] Question: Are new assets introduced in the paper well documented and is the documentation provided alongside the assets?
    \item[] Answer: \answerYes{}
    \item[] Justification: The new code accompanying this paper is documented. 
    \item[] Guidelines:
    \begin{itemize}
        \item The answer \answerNA{} means that the paper does not release new assets.
        \item Researchers should communicate the details of the dataset\slash code\slash model as part of their submissions via structured templates. This includes details about training, license, limitations, etc. 
        \item The paper should discuss whether and how consent was obtained from people whose asset is used.
        \item At submission time, remember to anonymize your assets (if applicable). You can either create an anonymized URL or include an anonymized zip file.
    \end{itemize}

\item {\bf Crowdsourcing and research with human subjects}
    \item[] Question: For crowdsourcing experiments and research with human subjects, does the paper include the full text of instructions given to participants and screenshots, if applicable, as well as details about compensation (if any)? 
    \item[] Answer:  \answerNA{}
    \item[] Justification: This paper does not involve crowdsourcing experiments or research with human subjects. 
    \item[] Guidelines:
    \begin{itemize}
        \item The answer \answerNA{} means that the paper does not involve crowdsourcing nor research with human subjects.
        \item Including this information in the supplemental material is fine, but if the main contribution of the paper involves human subjects, then as much detail as possible should be included in the main paper. 
        \item According to the NeurIPS Code of Ethics, workers involved in data collection, curation, or other labor should be paid at least the minimum wage in the country of the data collector. 
    \end{itemize}

\item {\bf Institutional review board (IRB) approvals or equivalent for research with human subjects}
    \item[] Question: Does the paper describe potential risks incurred by study participants, whether such risks were disclosed to the subjects, and whether Institutional Review Board (IRB) approvals (or an equivalent approval/review based on the requirements of your country or institution) were obtained?
    \item[] Answer: \answerNA{} 
    \item[] Justification: This paper does not involve research with human subjects. 
    \item[] Guidelines:
    \begin{itemize}
        \item The answer \answerNA{} means that the paper does not involve crowdsourcing nor research with human subjects.
        \item Depending on the country in which research is conducted, IRB approval (or equivalent) may be required for any human subjects research. If you obtained IRB approval, you should clearly state this in the paper. 
        \item We recognize that the procedures for this may vary significantly between institutions and locations, and we expect authors to adhere to the NeurIPS Code of Ethics and the guidelines for their institution. 
        \item For initial submissions, do not include any information that would break anonymity (if applicable), such as the institution conducting the review.
    \end{itemize}

\item {\bf Declaration of LLM usage}
    \item[] Question: Does the paper describe the usage of LLMs if it is an important, original, or non-standard component of the core methods in this research? Note that if the LLM is used only for writing, editing, or formatting purposes and does \emph{not} impact the core methodology, scientific rigor, or originality of the research, declaration is not required.
    \item[] Answer:\answerNA{} 
    \item[] Justification: LLM does not impact the core methodology, scientific rigor, or originality of the research. 
    \item[] Guidelines:
    \begin{itemize}
        \item The answer \answerNA{} means that the core method development in this research does not involve LLMs as any important, original, or non-standard components.
        \item Please refer to our LLM policy in the NeurIPS handbook for what should or should not be described.
    \end{itemize}

\end{enumerate}

\end{document}